\documentclass[11pt, letterpaper]{article}

\pdfoutput=1

\usepackage{geometry}
\usepackage[utf8]{inputenc} 
\usepackage[T1]{fontenc}    
\usepackage{hyperref}       
\usepackage{url}            
\usepackage{booktabs}       
\usepackage{amsfonts}       
\usepackage{nicefrac}       
\usepackage{microtype}      
\usepackage{xcolor}         

\usepackage{tcolorbox}
\usepackage{wrapfig}
\definecolor{darkred}{RGB}{150,0,0}
\definecolor{darkgreen}{RGB}{0,150,0}
\definecolor{darkblue}{RGB}{0,0,150}

\usepackage{amsfonts,amsmath}
\usepackage{amssymb}
\usepackage{mathtools}
\usepackage{amsthm}

\newtheorem{theorem}{Theorem}
\newtheorem{lemma}{Lemma}
\newtheorem{corollary}{Corollary}
\newtheorem{definition}{Definition}
\newtheorem{assumption}{Assumption}

\definecolor{darkred}{RGB}{150,0,0}
\definecolor{darkgreen}{RGB}{0,150,0}
\definecolor{darkblue}{RGB}{0,0,200}
\hypersetup{colorlinks=true, linkcolor=darkred, citecolor=darkgreen, urlcolor=darkblue}

\newcommand{\cln}[1]{\textcolor{red}{}}

\newcommand{\Equation}{Equation}
\newcommand{\Figure}{Figure}
\newcommand{\Theorem}{Theorem}

\newcommand{\Aneg}[1]{\A_{-1:#1}^{-1}}

\newcommand{\lbdb}{\boldsymbol{\lambda}}

\newcommand{\boldy}{\mathbf{y}}

\newcommand{\boldd}{\mathbf{d}}

\newcommand{\boldz}{\mathbf{z}}

\newcommand{\boldmu}{\boldsymbol{\mu}}

\newcommand{\bolda}{\mathbf{A}}

\newcommand{\boldm}{\mathbf{M}}

\newcommand{\boldq}{\mathbf{q}}

\newcommand{\boldcapq}{\mathbf{Q}}

\newcommand{\boldv}{\mathbf{v}}

\newcommand{\boldul}{\mathbf{u}}

\newcommand{\boldb}{\mathbf{b}}

\newcommand{\bolddcap}{\mathbf{D}}

\newcommand{\noonesub}{\setminus 1}

\newcommand{\noisub}{\setminus i}

\newcommand{\hatw}{\widehat{\mathbf{w}}}

\newcommand{\hatW}{\widehat{\mathbf{W}}}

\newcommand{\xtx}{\mathbf{X}^T\mathbf{X}}

\newcommand{\evente}{\mathcal{E}}

\newcommand{\dprimen}{p'(n)}

\newcommand{\lp}{\left}

\newcommand{\rp}{\right}

\newcommand{\epsilonn}{\epsilon_n}

\newcommand{\logtwon}{n}

\newcommand{\minklogn}{\min\{\sqrt{k}, \sqrt{\log(2n)}\}}

\newcommand{\minrho}{\tilde{\rho}_{n,k}}

\newcommand{\Wsvm}{\W_{\text{SVM}}}

\newcommand{\Wls}{\W_{\text{MNI}}}
\newcommand{\wls}{\w_{\text{MNI}}}

\newcommand{\cut}[1]{\textcolor{red}{}}
\newcommand{\W}{\mathbf{W}}
\newcommand{\M}{\mathbf{M}}

\newcommand{\Ub}{\mathbf{U}}

\newcommand{\X}{\mathbf{X}}

\newcommand{\Y}{\mathbf{Y}}

\newcommand{\A}{\mathbf{A}}

\newcommand{\Q}{\mathbf{Q}}

\newcommand{\z}{\mathbf{z}}

\newcommand{\x}{\mathbf{x}}
\newcommand{\ub}{\mathbf{u}}
\newcommand{\w}{\mathbf{w}}

\newcommand{\g}{\mathbf{g}}
\newcommand{\vb}{\mathbf{v}}
\newcommand{\V}{\mathbf{V}}

\newcommand{\y}{\mathbf{y}}

\newcommand{\q}{\mathbf{q}}

\newcommand{\Iden}{\mathbf{I}}

\newcommand{\betab}{\boldsymbol{\beta}}
\newcommand{\normonelambda}{\Vert \mathbf{\lambda} \Vert_1}

\newcommand{\Nn}{\mathcal{N}}
\newcommand{\Lc}{\mathcal{L}}

\newcommand{\Ec}{\mathcal{E}}

\newcommand{\Oc}{\mathcal{O}}

\newcommand{\beq}{\begin{equation}}
\newcommand{\eeq}{\end{equation}}
\newcommand{\bea}{\begin{align}}
\newcommand{\eea}{\end{align}}

\newcommand{\vp}{\vspace{4pt}}

\newcommand{\R}{\mathbb{R}}
\newcommand{\E}{\mathbb{E}}
\newcommand{\Pro}{\mathbb{P}}
\newcommand{\nn}{\notag}

\newcommand{\tn}[1]{\|{#1}\|_{2}}

  \newcommand{\Sigmab}{\boldsymbol\Sigma}
  \newcommand{\mub}{\boldsymbol\mu}
    
  \newcommand{\la}{{\lambda}}                     
  \newcommand{\eps}{\epsilon}

\newcommand{\wh}{\hat{\w}}

\newcommand{\wt}{\widetilde}

\newcommand{\ones}{\mathbf{1}}

\providecommand{\norm}[1]{\lVert#1\rVert}
\providecommand{\abs}[1]{\lvert#1\rvert}

\newcommand{\nub}{\boldsymbol{\nu}}

\newcommand{\SU}{\mathsf{SU}}
\newcommand{\CN}{\mathsf{CN}}

\newcommand{\Deltabold}{\boldsymbol{\Delta}}
\newcommand{\Deltahat}{\widehat{\Deltabold}}
\newcommand{\Ebold}{\boldsymbol{E}}
\newcommand{\Ehat}{\widehat{\Ebold}}
\newcommand{\C}{\mathbf{C}}
\newcommand{\uv}{\mathbf{u}}

\newcommand{\ehat}{\boldsymbol{{e}}}
\newcommand{\etilde}{\boldsymbol{\widetilde{e}}}

\newcommand{\xibold}{\boldsymbol{\xi}}
\newcommand{\xitilde}{\widetilde{\xibold}}
\newcommand{\Etilde}{\widetilde{\Ebold}}

\title{Benign Overfitting in Multiclass Classification:\\ All Roads Lead to Interpolation}

\author{Ke Wang\footnote{Primary correspondence to: kewang01@ucsb.edu. CT is also affiliated with the Department of Electrical and Computer Engineering, University of California, Santa Barbara.} \quad Vidya Muthukumar$^{\dagger}$ \quad Christos Thrampoulidis$^{\circ}$\\
\\
\small{$^*$Department of Statistics and Applied Probability, University of California Santa Barbara} \\
\small{$^\dagger$Electrical and Computer Engineering \& Industrial and Systems Engineering, Georgia Institute of Technology} \\
\small{$^\circ$Department of Electrical and Computer Engineering, University of British Columbia}
}

\date{}

\begin{document}

\maketitle

\begin{abstract}
The literature on ``benign overfitting'' in overparameterized models has been mostly restricted to regression or binary classification; however, modern machine learning operates in the multiclass setting.
Motivated by this discrepancy, we study benign overfitting in multiclass linear classification. Specifically, we consider the following training algorithms on separable data: (i) empirical risk minimization (ERM) with cross-entropy loss, which converges to the multiclass support vector machine (SVM) solution; (ii) ERM with least-squares loss, which converges to the min-norm interpolating (MNI) solution; and, (iii) the one-vs-all SVM classifier. 
First, we provide a simple sufficient deterministic condition under which \emph{all} three algorithms lead to classifiers that interpolate the training data and have equal accuracy.
When the data is generated from Gaussian mixtures or a multinomial logistic model, this condition holds under high enough effective overparameterization. 
We also show that this sufficient condition is satisfied under ``neural collapse'', a phenomenon that is observed in training deep neural networks.
Second, we derive novel bounds on the accuracy of the MNI classifier, thereby showing that all three training algorithms lead to benign overfitting under sufficient overparameterization.
Ultimately, our analysis shows that good generalization is possible for SVM solutions beyond the realm in which typical margin-based bounds apply.
\end{abstract}

\addtocontents{toc}{\protect\setcounter{tocdepth}{0}}

\section{Introduction}

Modern deep neural networks are overparameterized (high-dimensional) with respect to the amount of training data.
Consequently, they achieve zero training error even on noisy training data, yet generalize well on test data~\cite{zhang2016understanding}.
Recent mathematical analysis has shown that fitting of noise in regression tasks can in fact be relatively benign for linear models that are sufficiently high-dimensional~\cite{bartlett2020benign,belkin2020two,hastie2019surprises,muthukumar2020harmless,kobak2020optimal}. 
These analyses do not directly extend to classification, which requires separate treatment. 
In fact, recent progress on sharp analysis of interpolating binary classifiers~\cite{muthukumar2021classification,chatterji2020finite,wang2020benign,cao2021risk} revealed high-dimensional regimes in which binary classification generalizes well, but the corresponding regression task does \emph{not} work and/or the success \emph{cannot} be predicted by classical margin-based bounds~\cite{schapire1998boosting,bartlett2002rademacher}.

In an important separate development, these same high-dimensional regimes admit an equivalence of loss functions used for optimization at training time.
The support vector machine (SVM), which arises from minimizing the logistic loss using gradient descent~\cite{soudry2018implicit,ji2019implicit}, was recently shown to satisfy a high-probability equivalence to interpolation, which arises from minimizing the squared loss~\cite{muthukumar2021classification,hsu2020proliferation}.
This equivalence suggests that interpolation is ubiquitous in very overparameterized settings, and can arise naturally as a consequence of the optimization procedure even when this is not explicitly encoded or intended.
Moreover, this equivalence to interpolation and corresponding analysis implies that the SVM can generalize even in regimes where classical learning theory bounds are not predictive.
In the logistic model case~\cite{muthukumar2021classification} and Gaussian binary mixture model case~\cite{chatterji2020finite,wang2020benign,cao2021risk}, it is shown that good generalization of the SVM is possible beyond the realm in which classical margin-based bounds apply.
These analyses lend theoretical grounding to the surprising hypothesis that \textit{squared loss can be equivalent to, or possibly even superior}, to the cross-entropy loss for  classification tasks.
Ryan Rifkin provided empirical support for this hypothesis on kernel machines~\cite{rifkin2002everything,rifkin2004defense}; more recently, corresponding empirical evidence has been provided for state-of-the-art deep neural networks~\cite{hui2020evaluation,poggio2020explicit}. 

These perspectives have thus far been limited to regression and \textit{binary} classification settings.
In contrast, most success stories and surprising new phenomena of modern machine learning have been recorded in \emph{multiclass} classification settings, which appear naturally in a host of applications that demand the ability to automatically distinguish between large numbers of different classes.
For example, the popular ImageNet dataset~\cite{russakovsky2015imagenet} contains on the order of $1000$ classes.
Whether a) good generalization beyond effectively low-dimensional regimes where margin-based bounds are predictive is possible, and b) equivalence of squared loss and cross-entropy loss holds in multiclass settings remained open problems.

This paper makes significant progress towards a complete understanding of the optimization and generalization properties of high-dimensional linear multiclass classification, both for unconditional Gaussian covariates (where labels are generated via a multinomial logistic model), and Gaussian mixture models.
Our contributions are listed in more detail below.

\subsection{Our Contributions}
    \begin{wrapfigure}{h}{0.65\textwidth}
\centering
\vspace{-0.3in}
\includegraphics[width=0.65\textwidth]{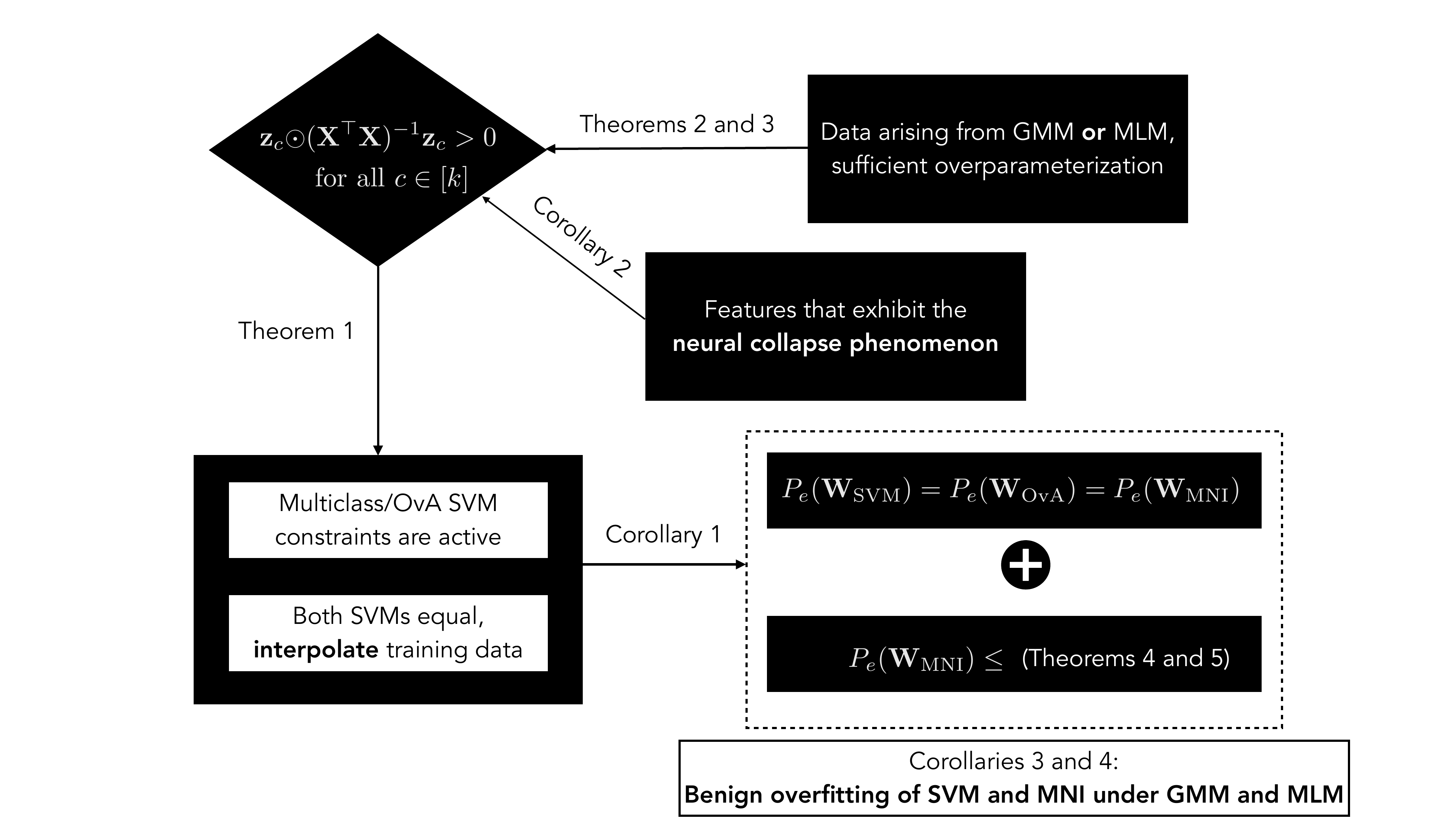}
\caption{Contributions and organization.}
\label{fig:contributions}
\end{wrapfigure}

$\bullet$~~We establish a  \emph{deterministic} sufficient condition under which the multiclass SVM solution has a very simple and symmetric structure: it is identical to the solution of a One-vs-All (OvA) SVM classifier that uses a \emph{simplex-type} encoding for the labels (unlike the classical one-hot encoding).
Moreover, the constraints at both solutions are active. 
Geometrically, this means that all data points are support vectors, and \textit{they interpolate the simplex-encoding vector representation of the labels.} See Figure \ref{fig-barplot} for a numerical illustration confirming our finding. \\
$\bullet$~~This implies a surprising equivalence between traditionally different formulations of multiclass SVM, which in turn are equivalent to the minimum-norm interpolating (MNI) classifier on the one-hot label vectors.
Thus, we show that the outcomes of training with cross-entropy (CE) loss and squared loss are identical in terms of classification error. \\
$\bullet$~~ Next, for data following a Gaussian-mixtures model (GMM) or a Multinomial logistic model (MLM), we show that the above sufficient condition is satisfied with high-probability under sufficient ``effective" overparameterization.
Our sufficient conditions are non-asymptotic and are characterized in terms of the data dimension, the number of classes, and functionals of the data covariance matrix. Our numerical results show excellent agreement with  our theoretical findings.
We also show that the sufficient condition of equivalence of CE and squared losses is satisfied when the ``neural collapse'' phenomenon occurs~\cite{papyan2020prevalence}.
\\
$\bullet$~~ Finally, we provide novel non-asymptotic bounds on the error of the MNI classifier for data generated either from the GMM or the MLM, and identify sufficient conditions under which benign overfitting occurs.
A direct outcome of our results is that benign overfitting occurs under these conditions regardless of whether the cross-entropy loss or squared loss is used during training.

\Figure~\ref{fig:contributions} describes our contributions and their implications through a flowchart.
{To the best of our knowledge, these are the first results characterizing a) equivalence of loss functions, and b) generalization of interpolating solutions in the multiclass setting.} The multiclass setting poses several challenges over and above the recently studied binary case. When presenting our results in later sections, we discuss in detail how our analysis circumvents these challenges.

\subsection{Related Work}


\paragraph{Multiclass classification and the impact of training loss functions} There is a classical body of work on algorithms for multiclass classification, e.g.  \cite{weston1998multi,bredensteiner1999multicategory,dietterich1994solving,crammer2001algorithmic,lee2004multicategory}
and several empirical studies of their comparative performance~\cite{rifkin2004defense,furnkranz2002round,allwein2000reducing} (also see~\cite{hou2016squared,gajowniczek2017generalized,kumar2018robust,bosman2020visualising,demirkaya2020exploring,hui2020evaluation,poggio2020explicit} for recent such studies in the context of deep nets). 
Many of these (e.g. \cite{rifkin2004defense,hui2020evaluation,bosman2020visualising}) have found that least-squares minimization yields competitive test classification performance to cross-entropy minimization. 
\textit{Our proof of equivalence of the SVM and MNI solutions under sufficient overparameterization provides theoretical support for this line of work.}
This is a consequence of the implicit bias of gradient descent run on the CE and squared losses leading to the multiclass SVM~\cite{soudry2018implicit,ji2019implicit} and MNI~\cite{engl1996regularization} respectively.
Numerous classical works investigated consistency \cite{zhang2004statistical,lee2004multicategory,tewari2007consistency,pires2013cost,pires2016multiclass} and finite-sample behavior, e.g.,  \cite{koltchinskii2002empirical,cortes2016structured,lei2015multi,maurer2016vector,lei2019data,maximov2016tight}  of multiclass classification algorithms in the underparameterized regime. 
{In contrast, our primary focus lies in the highly overparameterized regime, where conventional techniques of uniform convergence are inadequate. In Section IV, we elaborate on why classical training-data-dependent bounds based on margin or Rademacher complexity, are insufficient in this regime and cannot yield conclusions about benign overfitting. Recently, in \cite{abramovich2021multiclass}, the authors have studied the problem of feature selection in high-dimensional multiclass classification, identifying an intriguing phase transition as the number of classes increases with dimensions. Our work differs from theirs in that our bounds are relevant to CE minimization without explicit regularization, whereas \cite{abramovich2021multiclass} focuses on CE loss minimization with sparsity penalties to achieve feature selection.}

\paragraph{Binary classification error analyses in overparameterized regime} 
The recent wave of analyses of the minimum-$\ell_2$-norm interpolator (MNI) in high-dimensional linear regression (beginning with~\cite{bartlett2020benign,belkin2020two,hastie2019surprises,muthukumar2020harmless,kobak2020optimal}) prompted researchers to consider to what extent the phenomena of benign overfitting and double descent~\cite{belkin2019reconciling,geiger2020scaling} can be proven to occur in classification tasks.
Even the binary classification setting turns out to be significantly more challenging to study owing to the discontinuity of the 0-1 test loss function. 
Sharp asymptotic formulas for the generalization error of binary classification algorithms in the linear high-dimensional regime have been  derived in several recent works \cite{
huang2017asymptotic,
sur2019modern,mai2019large,salehi2019impact,taheri2020sharp,taheri2020fundamental,
deng2019model,montanari2019generalization,svm_abla,liang2020precise,salehi2020performance,aubin2020generalization,lolas2020regularization,dhifallah2020precise}.
These formulas are solutions to complicated nonlinear systems of equations that typically do not admit closed-form expressions.
A separate line of work provides non-asymptotic error bounds for both the MNI classifier and the SVM classifier~\cite{chatterji2020finite,muthukumar2021classification,wang2020benign,cao2021risk}; in particular,~\cite{muthukumar2021classification} analyzed the SVM in a Gaussian covariates model by explicitly connecting its solution to the MNI solution.
Subsequently,~\cite{wang2020benign} also took this route to analyze the SVM and MNI in mixture models, and even more recently, \cite{cao2021risk} provided extensions of this result to sub-Gaussian mixtures.
While these non-asymptotic analyses are only sharp in their dependences on the sample size $n$ and the data dimension $p$, they provide closed-form generalization expressions in terms of easily interpretable summary statistics.
Interestingly, these results imply good generalization of the SVM beyond the regime in which margin-based bounds are predictive.
Specifically,~\cite{muthukumar2021classification} identifies a separating regime for Gaussian covariates in which corresponding regression tasks would not generalize.
In the Gaussian mixture model, margin-based bounds~\cite{schapire1998boosting,bartlett2002rademacher} (as well as corresponding recently derived mistake bounds on interpolating classifiers~\cite{liang2021interpolating}) would require the intrinsic signal-to-noise-ratio (SNR) to scale at least as $\omega(p^{1/2})$ for good generalization; however, the analyses of~\cite{wang2020benign,cao2021risk} show that good generalization is possible for significantly lower SNR scaling as $\omega(p^{1/4})$.
The above error analyses are specialized to the binary case, where closed-form error expressions are easy to derive~\cite{muthukumar2021classification}.
The only related work applicable to the multiclass case is \cite{thrampoulidis2020theoretical}, which also highlights the numerous challenges of obtaining a sharp error analysis in multiclass settings.
Specifically,~\cite{thrampoulidis2020theoretical} derived sharp generalization formulas for multiclass least-squares in underparameterized settings; extensions to the overparameterized regime and other losses beyond least-squares remained open. Finally, \cite{kini2021phase} recently derived sharp phase-transition thresholds for the feasibility of OvA-SVM on multiclass Gaussian mixture data in the linear high-dimensional regime. However, this does not address the more challenging multiclass-SVM that we investigate here. {To summarize, our paper presents the first generalization bounds for the multiclass-SVM classifier that establish conditions for benign overfitting in the high-dimensional regime. In the process, we establish a connection between multiclass-SVM and multi-class MNI, which poses unique challenges due to the non-uniqueness of defining support vectors in multiclass settings. Our work highlights the richness of the multiclass setting compared to the binary setting, as we demonstrate the equivalence not only for GMM and MLM data but also for data following a simplex equiangular tight-frame (ETF) structure. This geometry structure is only relevant in multiclass settings and arises when training deep-net classifiers with CE loss beyond the zero-training error regime \cite{papyan2020prevalence}.}

\paragraph{Other SVM analyses} The number of support vectors in the \textit{binary} SVM has been characterized in low-dimensional separable and non-separable settings~\cite{dietrich1999statistical,buhot2001robust,malzahn2005statistical} and scenarios have been identified in which there is a vanishing fraction of support vectors, as this implies good generalization\footnote{In this context, the fact that~\cite{muthukumar2021classification,wang2020benign,cao2021risk} provide good generalization bounds in the regime where support vectors proliferate is particularly surprising. In conventional wisdom, a proliferation of support vectors was associated with overfitting but this turns out to not be the case here.}
 via PAC-Bayes sample compression bounds~\cite{vapnik2013nature,graepel2005pac,germain2011pac}.
In the highly overparameterized regime that we consider, perhaps surprisingly, the opposite behavior occurs: \textit{all training points become support vectors with high probability}~\cite{dietrich1999statistical,buhot2001robust,malzahn2005statistical,muthukumar2021classification,hsu2020proliferation}.
In particular,~\cite{hsu2020proliferation}  provided sharp non-asymptotic sufficient conditions for this phenomenon for both isotropic and anisotropic settings.
The techniques in~\cite{muthukumar2021classification,hsu2020proliferation} are highly specialized to the binary SVM and its dual, where a simple complementary slackness condition directly implies the property of interpolation.
In contrast, the complementary slackness condition for the case of multiclass SVM \textit{does not} directly imply interpolation; in fact, the operational meaning of ``all training points becoming support vectors'' is unclear in the multiclass SVM.
\textit{Our proof of deterministic equivalence goes beyond the complementary slackness condition and uncovers a surprising symmetric structure by showing equivalence of multiclass SVM to a simplex-type OvA classifier.} The simplex equiangular tight frame structure that we uncover is somewhat reminiscent of the recently observed neural collapse phenomenon in deep neural networks~\cite{papyan2020prevalence}; indeed, \textit{Section~\ref{sec:ncconnection} shows an explicit connection between our deterministic equivalence condition and the neural collapse phenomenon.}
Further,~\cite{muthukumar2021classification,hsu2020proliferation} focus on proving deterministic conditions for equivalence in the case of labels generated from covariates; the mixture model case (where covariates are generated from labels) turns out to be significantly more involved {due to the \emph{anisotropic} data covariance matrix resulting from even from isotropic noise covariance \cite{wang2020benign,cao2021risk}. As we explain further in Section \ref{sec:linksvm}, the mean vectors of the mixture model introduce  an additional rank-$k$ component that complicates the analysis and requires new ideas.}

\subsection{Organization}

The paper is organized as follows.
Section~\ref{sec:setup} describes the problem setting and sets up notation.
Section~\ref{sec:support-vectors} presents our main results on the equivalence between the multiclass SVM and MNI solutions for two data models: the Gaussian mixture model (GMM) and the multinomial logistic model (MLM). In the same section, we also show the equivalence under the Neural Collapse phenomenon.
Section~\ref{sec:error} presents our error analysis of the MNI solution (and, by our proved equivalence, the multiclass SVM) for the GMM and the MLM, and Section~\ref{sec:benignoverfitting} presents consequent conditions for benign overfitting of multiclass classification.
Finally, Section~\ref{sec:proofsmaintext} presents proofs of our main results; auxiliary proofs are deferred to the appendices.
Please refer to the table of contents (before the appendices) for a more detailed decomposition of results and proofs.

\paragraph*{Notation}~For a vector $\vb \in \mathbb{R}^p$ , let $\tn{\vb} = \sqrt{\sum_{i=1}^p v_i^2}$, $\Vert \vb \Vert_1 = {\sum_{i=1}^p |v_i|}$, $\Vert \vb \Vert_{\infty} = \max_{i}\{|v_i|\}$. $\vb>\mathbf{0}$ is interpreted elementwise. $\ones_m$ / $\mathbf{0}_m$ denote the all-ones / all-zeros vectors of dimension $m$ and $\mathbf{e}_i$ denotes the $i$-th standard basis vector. For a matrix $\mathbf{M}$, $\Vert \mathbf{M} \Vert_2$ denotes its $2 \to 2$ operator norm and $\Vert \mathbf{M} \Vert_F$ denotes the Frobenius norm. $\odot$ denotes the Hadamard product. [$n$] denotes the set $\{1,2,...,n\}$. We also use standard ``Big O'' notations $\Theta(\cdot)$, $\omega(\cdot)$, e.g. see \cite[Chapter 3]{cormen2009introduction}. Finally, we write $\Nn(\boldsymbol{\mu},\Sigmab)$ for the (multivariate) Gaussian distribution of mean $\boldsymbol{\mu}$ and covariance matrix $\Sigmab$, and, $Q(x)=\mathbb{P}(Z>x),~Z\sim\Nn(0,1)$ for the Q-function of a standard normal. Throughout, constants refer to strictly positive numbers that do not depend on the problem dimensions $n$ or $p$. 

\section{Problem setting}\label{sec:setup}
We consider the multiclass classification problem with $k$ classes. 
Let $\x \in \mathbb{R}^p$ denote the feature vector and $y \in [k]$ represent the class label associated with one of the $k$ classes. 
We assume that the training data has $n$ feature$/$label pairs $\{\x_i, y_i\}_{i=1}^n$. 
We focus on the overparameterized regime, i.e. $p > Cn$, and we will frequently consider $p \gg n$.
For convenience, we express the labels using the one-hot coding vector $\y_i \in \mathbb{R}^{k}$, where only the $y_i$-th entry of $\y_i$ is $1$ and all other entries are zero, i.e. $\y_i = \boldsymbol{e}_{y_i}$. 
With this notation, the feature and label matrices are given in compact form as follows: 
$\X = \begin{bmatrix} \x_1 & \x_2 & \cdots & \x_n
    \end{bmatrix} \in \mathbb{R}^{p \times n}$ and $\Y = \begin{bmatrix}
    \y_1 & \y_2 & \cdots & \y_n
    \end{bmatrix} = \begin{bmatrix}
    \vb_1 & \vb_2 & \cdots  \vb_k
    \end{bmatrix}^T
    \in \mathbb{R}^{k \times n},$
where we have defined $\vb_c \in \mathbb{R}^{n},c\in [k]$ to denote the $c$-th row of the matrix $\Y$. 

\subsection{Data models}
We assume that the data pairs $\{\x_i, y_i\}_{i=1}^n$ are independently and identically distributed (IID). We will consider two models for the distribution of $(\x, y)$.
For both models, we define the mean vectors $\{\boldmu_j\}_{j=1}^k \in \mathbb{R}^p$, and the mean matrix is given by
$
    \M := \begin{bmatrix}
    \boldmu_1 & \boldmu_2 & \cdots & \boldmu_k
    \end{bmatrix} \in \mathbb{R}^{p \times k}.$
    
\paragraph{Gaussian Mixture Model (GMM)} In this model, the mean vector $\boldmu_i$ represents the conditional mean vector for the $i$-th class. Specifically, each observation $(\x_i, y_i)$ belongs to to class $c \in [k]$ with probability $\pi_c$ and conditional on the label $y_i$, $\x_i$ follows  a multivariate Gaussian distribution. In summary, we have
\begin{align}
\label{def-gmm}
    \mathbb{P}(y = c) = \pi_c ~~ \text{and} ~~ \x = \boldmu_{y} + \q, ~\q \sim \Nn(\boldsymbol{0}, \Sigmab).
\end{align}
In this work, we focus on the isotropic case $\Sigmab = \Iden_p$. 
Our analysis can likely be extended to the more general anisotropic case, but we leave this to future work. 

\paragraph{Multinomial Logit Model (MLM)} In this model, the feature vector $\x \in \mathbb{R}^p$ is distributed as $\Nn(\boldsymbol{0}, \Sigmab)$, and the conditional density of the class label $y$ is given by the soft-max function. Specifically, we have
\begin{align}
\label{def:mlm}
    \x \sim \Nn(\boldsymbol{0}, \Sigmab) ~~ \text{and} ~~ \mathbb{P}(y = c|\x) = \frac{\exp(\boldmu_c^T\x)}{\sum_{j \in [k]}\exp(\boldmu_j^T\x)}.
\end{align}
For this model, we analyze both the isotropic and anisotropic cases.

\subsection{Data separability}
We consider linear classifiers parameterized by 
$
    \W = \begin{bmatrix}\w_1 & \w_2 & \cdots & \w_k
    \end{bmatrix}^T \in \mathbb{R}^{k \times p}.
$
Given input feature vector $\x$, the classifier is a function that maps $\x$ into an output of $k$ via $\x \mapsto \W\x \in \R^{k}$ (for simplicity, we ignore the bias term throughout).
We will operate in a regime where the training data are linearly separable. 
In multiclass settings, there exist multiple notions of separability. Here, we focus on (i) multiclass separability (also called $k$-class separability) (ii) one-vs-all (OvA) separability, and, recall their definitions below.

\begin{definition}[multiclass and OvA separability] The dataset $\{\x_i,y_i\}_{i\in[n]}$ is multiclass linearly separable when 
\begin{align}
    \exists\W \,:\, (\w_{y_i}-\w_c)^T\x_i\geq 1,~\forall {c\neq y_i, c\in[k]}, \text{ and } \forall i\in[n].
\end{align}
The dataset is one-vs-all (OvA) separable when 
\begin{align}
    \exists\W \,:\,  \w_{c}^T\x_i \begin{cases}
    \geq 1 \text{ if } y_i=c  \\
    \leq -1 \text{ if } y_i\neq c
    \end{cases}
    , \forall {c\in[k]}, \text{ and } \forall i\in[n].
\end{align}
\end{definition}
Under both data models of the previous section (i.e. GMM and MLM), we have $\rm{rank}(\X)=n$ almost surely in the overparameterized regime $p > n$.
This directly implies OvA separability.
It turns out that OvA separability implies multiclass separability, but not vice versa (see \cite{bennett1994multicategory} for a counterexample).

\subsection{Classification error}
Consider a linear classifier $\hatW$ and a fresh sample $(\x, y)$ generated following the same distribution as the training data. 
As is standard, we predict $\hat{y}$ by a ``winner takes it all strategy", i.e. $\hat{y} = \arg\max_{j \in [k]}\hatw_j^T\x$. 
Then, the classification error conditioned on the true label being $c$, which we refer to as the \textit{class-wise classification error}, is defined as
\begin{align}
\label{eq:classerror}
\Pro_{e|c}:=\Pro(\hat{y} \ne y|y =c) = \Pro(\hatw_c^T\x \le \max_{j \ne c}\hatw_j^T\x).
\end{align}
In turn, the \textit{total classification error} is defined as
\begin{align}
\label{eq:totalerror}
\Pro_e := \Pro(\hat{y} \ne y) = \Pro(\arg\max_{j \in [k]}\hatw_j^T\x \ne y) = \Pro(\hatw_y^T\x \le \max_{j \ne y}\hatw_j^T\x).
\end{align}

\subsection{Classification algorithms}
\label{subsec:alg}
Next, we review several different training strategies for which we characterize the total/class-wise classification error in this paper.

\paragraph{Multiclass SVM} Consider training $\W$ by minimizing the cross-entropy (CE) loss
\begin{align*}
    \Lc(\W):=-\log\left(\frac{{e^{\w_{y_i}^T\x_i}}}{{\sum_{c\in[k]}e^{\w_c^T\x_i}}}\right)
\end{align*}
with the gradient descent algorithm (with constant step size $\eta$).
In the separable regime, the CE loss $\Lc(\W)$ can be driven to zero. Moreover,~\cite[Thm.~7]{soudry2018implicit} showed that the normalized iterates $\{\W^t\}_{t\geq 1}$ converge as
\begin{align*}
\lim_{t\rightarrow\infty}\big\|\frac{\W^t}{\log{t}} - \Wsvm\big\|_F=0,
\end{align*} 
where $\Wsvm$ is the solution of the \emph{multiclass SVM} \cite{weston1998multi} given by
\begin{align}\label{eq:k-svm}
    \Wsvm:=\arg\min_\W \norm{\W}_F\quad\text{sub. to}~~ (\w_{y_i}-\w_c)^T\x_i\geq1,~\forall i \in [n], c \in [k] \text{ s.t. } c \neq y_i .
\end{align}
It is important to note that the normalizing factor $\log{t}$ here does $\emph{not}$ depend on the class label; hence, 
in the limit of GD iterations, the  solution $\W^t$ decides the same label as multiclass SVM for any test sample.

\paragraph{One-vs-all SVM} 
In contrast to \Equation~\eqref{eq:k-svm}, which optimizes the hyperplanes $\{\w_c\}_{c \in [k]}$ jointly, the one-vs-all (OvA)-SVM classifier solves $k$ separable optimization problems that maximize the margin of each class with respect to all the rest. Concretely, the OvA-SVM solves the following optimization problem for all $c \in [k]$:
\begin{align}
\label{eq:ova-svm}
    \w_{{\rm OvA},c}:=\arg\min_{\w} \tn{\w}\quad\text{sub. to}~~ \w^T\x_i \begin{cases} \ge 1,&~\text{if}~\y_i =c, \\ \le -1,&~\text{if}~\y_i\neq c,\end{cases} ~~\forall i\in[n].
\end{align}
In general, the solutions to Equations~\eqref{eq:k-svm} and \eqref{eq:ova-svm} are different. 
While the OvA-SVM does not have an obvious connection to any training loss function, its relevance will become clear in Section~\ref{sec:support-vectors}. 
Perhaps surprisingly, we will prove that in the highly overparameterized regime the multiclass SVM solution is identical to a slight variant of \eqref{eq:ova-svm}.

\paragraph{Min-norm interpolating (MNI) classifier} An alternative to the CE loss is the square loss $\Lc(\W):=\frac{1}{2n}\|\Y-\W\X\|_2^2=\frac{1}{2n}\sum_{i=1}^n\tn{\W\x_i - \y_i}^2$.
Since the square loss is tailored to regression, it might appear that the CE loss is more appropriate for classification. Perhaps surprisingly, one of the main messages of this paper is that under sufficient effective overparameterization the two losses actually have equivalent performance.
Our results lend theoretical support to empirical observations of competitive classification accuracy between the square loss and CE loss in practice~\cite{rifkin2002everything,hui2020evaluation,poggio2020explicit}.

Towards showing this, we note that when the linear model is overparameterized (i.e. $p>n$) and assuming $\rm{rank}(\X)=n$ (e.g this holds almost surely under both the GMM and MLM), the data can be linearly interpolated, i.e. the square-loss can be driven to zero. 
Then, it is well-known~\cite{engl1996regularization} that gradient descent with sufficiently small step size and appropriate initialization converges to the minimum-norm -interpolating (MNI) solution, given by:
\begin{align}
\label{eq:lsminnormsolution}
    \Wls := \arg\min_{\W}\norm{\W}_F, ~~ \text{sub. to}~~\X^T\w_c = \vb_c, \forall c\in[k].
\end{align}
Since $\X^T\X$ is invertible, the MNI solution is given in closed form as $\Wls^T=\X(\X^T\X)^{-1}\Y^T$. From here on, we refer to \eqref{eq:lsminnormsolution} as the MNI classifier. 

\section{Equivalence of solutions and geometry of support vectors}
\label{sec:support-vectors}
In this section, we show the equivalence of the solutions of the three classifiers defined above in certain high-dimensional regimes.

\subsection{A key deterministic condition}\label{sec:det_con}
We first establish a key deterministic property of SVM that holds for \emph{generic} multiclass datasets $(\X,\Y)$ (i.e. not necessarily generated by either the GMM or MLM).
Specifically, \Theorem~\ref{lem:key} below derives a sufficient condition (cf.~\eqref{eq:det-con}) under which the multiclass SVM solution has a surprisingly simple structure. 
First, the constraints are \emph{all} active at the optima (cf.~\eqref{eq:equality_ksvm}). 
Second, and perhaps more interestingly, this happens in a very specific way; the feature vectors interpolate a simplex representation of the multiclass labels, as specified below:
\begin{align}\label{eq:symmetry-constraint-scalar}
\hat\w_{c}^T\x_i=z_{ci}:=\begin{cases}
\frac{k-1}{k}  &,~ c=y_i \\
-\frac{1}{k} &,~ c\neq y_i
\end{cases}\, \text{ for all } i \in [n], c \in [k].
\end{align}
To interpret this, define an adjusted $k$-dimensional label vector $\widetilde{\y}_i:=[z_{1i},z_{2i},\ldots,z_{ki}]^T$ for each training sample $i \in [n]$.
This can be understood as a $k$-dimensional vector encoding of the original label $y_i$ that is different from the classical one-hot encoding representation $\y_i$; in particular, it has entries either $-1/k$ or $1-1/k$ (rather than $0$ or $1$). 
We call this new representation a \emph{simplex representation}, based on the following observation. 
Consider $k$ data points that each belong to a different class $1,\ldots,k$, and their corresponding vector representations $\widetilde{\y}_1,\ldots,\widetilde{\y}_k$. Then, it is easy to verify that the vectors $\{\mathbf{0},\widetilde{\y}_1,\ldots,\widetilde{\y}_k\}$ are affinely independent; hence, they form the vectices of a $k$-simplex.

\begin{theorem}\label{lem:key}
For a  multiclass separable dataset with feature matrix $\X=[\x_1,\x_2,\ldots,\x_n]\in\R^{p\times n}$
and label matrix $\Y=[\vb_1,\vb_2,\ldots,\vb_k]^T\in\R^{k\times n}$, denote by $\Wsvm=[\wh_1,\wh_2,\ldots,\wh_k]^T$ the multiclass SVM solution of~\eqref{eq:k-svm}. For each class $c\in[k]$ define vectors $\z_c\in\R^n$  such that
\begin{align}\label{eq:zc}
    \z_c = \vb_c - \frac{1}{k}\ones_n,~c\in[k].
\end{align}
Let $(\X^T\X)^{+}$ be the Moore-Penrose generalized inverse\footnote{Most of the regimes that we study are ultra-high-dimensional (i.e. $p \gg n$), and so $\X^T\X$ is invertible with high probability. Consequently, $(\X^T\X)^{+}$ can be replaced by $(\X^T\X)^{-1}$ in these cases.} of the Gram matrix $\X^T\X$
and assume that the following condition holds
\begin{align}\label{eq:det-con}
    \z_c\odot (\X^T\X)^{+}\z_c > \mathbf{0},\quad \forall c\in[k].
\end{align}
Then, the SVM solution $\Wsvm$ is such that all the constraints in~\eqref{eq:k-svm} are active. That is,
\begin{align}\label{eq:equality_ksvm}
(\wh_{y_i}-\wh_c)^T\x_i= 1,~\forall {c\neq y_i, c\in[k]}, \text{ and }~ \forall i\in[n].
\end{align}
Moreover, the features interpolate the simplex representation. That is,
\begin{align}\label{eq:symmetry-constraint}
    \X^T\wh_c = \z_c,~\forall c\in[k].
\end{align}
\end{theorem}
For $k=2$ classes, it can be easily verified that \Equation~\eqref{eq:det-con} reduces to the condition in \Equation~(22) of~\cite{muthukumar2021classification} for the binary SVM. 
Compared to the binary setting, the conclusion for the multiclass case is richer: provided that \Equation~\eqref{eq:det-con} holds, we show that not only are all data points support vectors, but also, they satisfy a set of simplex OvA-type constraints as elaborated above. 
The proof of \Equation~\eqref{eq:symmetry-constraint} is particularly subtle and involved: unlike in the binary case, it does \emph{not} follow directly from a complementary slackness condition on the dual of the multiclass SVM.
A key technical contribution that we provide to remedy this issue is a novel reparameterization of the SVM dual.
The complete proof of Theorem~\ref{lem:key} and this reparameterization is provided in Section~\ref{sec:key_thm_proof_sketch}.

We make a few additional remarks on the interpretation of \Equation~\eqref{eq:symmetry-constraint}.  

First, our proof shows a somewhat stronger conclusion: when \Equation~\eqref{eq:det-con} holds, the multiclass SVM solutions $\hat\w_c, c\in[k]$ are same as the solutions to the following \emph{simplex OvA-type classifier} (cf. \Equation~\eqref{eq:ova-svm}):
\begin{align}\label{eq:sym-cs-svm-sketch}
\min_{\w_c}~\frac{1}{2}\|\w_c\|_2^2\qquad\text{sub. to}~~~ \x_i^T\w_c\begin{cases}\geq \frac{k-1}{k} &, y_i=c,\\
\leq -\frac{1}{k} &, y_i\neq c,\end{cases}    ~~~\forall i\in[n],
\end{align}
for all $c\in[k]$.
We note that the OvA-type classifier above can also be interpreted as a binary cost-sensitive SVM classifier~\cite{masnadi2012cost} that enforces the margin corresponding to all other classes to be $(k-1)$ times smaller compared to the margin for the labeled class of the training data point.
This simplex structure is illustrated in \Figure~\ref{fig-barplot}, which evaluates the solution of the multiclass SVM on a $4$-class Gaussian mixture model with isotropic noise covariance. The mean vectors are set to be mutually orthogonal and equal in norm, with SNR $\tn{\boldmu} = 0.2\sqrt{p}$. 
We also set $n = 50$, $p = 1000$ to ensure sufficient effective overparameterization (in a sense that will be formally defined in subsequent sections).
\Figure~\ref{fig-barplot} shows the inner product $\hatw_c^T\x$ drawn from 8 samples. These inner products are consistent with the simplex OvA structure defined in \Equation~\eqref{eq:symmetry-constraint}, i.e. $\hatw_c^T\x_i = 3/4$ if $y_i = c$ and $\hatw_c^T\x_i = -1/4$ if $y_i \ne c$.
\begin{figure}[t]
    \centering
    \includegraphics[width=0.9\textwidth]{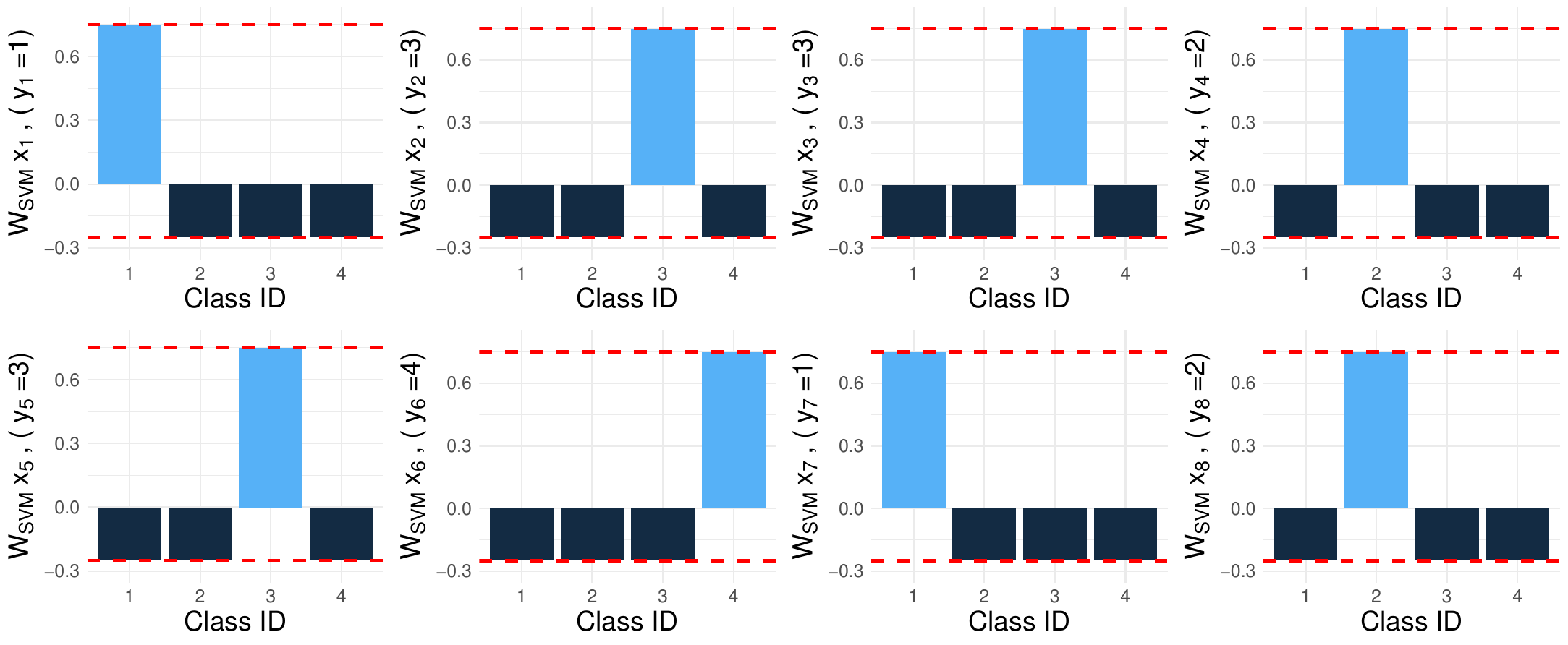}
    \caption{Inner products $\Wsvm\x_c \in \R^{4}$ for features $\x_i$ that each belongs to the  $c$-th class for $c\in[k]$ and $k=4$ total classes. The red lines correspond to the values $(k-1)/k=3/4$ and $-1/k=-1/4$ of the simplex encoding described in Theorem \ref{lem:key}. Observe that the inner products  $\Wsvm\x_c$ match with these values, that is, \Equation~\eqref{eq:symmetry-constraint-scalar} holds.}
    \label{fig-barplot}
\end{figure}

Second, \Equation~\eqref{eq:symmetry-constraint} shows that when \Equation~\eqref{eq:det-con} holds, then the multiclass SVM solution $\Wsvm$ has the same classification error as that of the minimum-norm interpolating solution.
In other words, we can show that the minimum-norm classifiers that interpolate the data with respect to \emph{either} the one-hot representations $\y_i$ or the simplex representations $\widetilde{\y}_i$ of \eqref{eq:symmetry-constraint-scalar} have identical classification performance.
This conclusion, stated as a corollary below, drives our classification error analysis in Section~\ref{sec:error}.

\begin{corollary}[SVM=MNI]\label{cor:SVM=LS} 
Under the same assumptions as in~\Theorem~\ref{lem:key}, and provided that the inequality in \Equation~\eqref{eq:det-con} holds, it holds that $\Pro_{e|c}(\Wsvm)=\Pro_{e|c}(\Wls)$ for all $c\in[k]$. Thus, the total classification errors of both solutions are equal: $\Pro_{e}(\Wsvm)=\Pro_{e}(\Wls)$.
\end{corollary}

The corollary follows directly by combining Theorem \ref{lem:key} with the following lemma applied with the choice $\alpha = 1,\beta = -1/k$. 
We include a detailed proof below for completeness.

\begin{lemma}
For constants $\alpha > 0, \beta$, consider the MNI-solution ${\w}^{\alpha,\beta}_c = \X(\X^\top\X)^{+}(\alpha\vb_c+\beta\boldsymbol{1}), c\in[k]$ corresponding to a target vector of labels $\alpha\vb_c+\beta\ones_n$. Let ${\Pro}_{e|c}^{\alpha,\beta}, c\in[k]$ be the class-conditional classification errors of the classifier $\w^{\alpha,\beta}$. Then, for any different set of constants $\alpha'>0,\beta'$, it holds that ${\Pro}_{e|c}^{\alpha,\beta}= {\Pro}_{e|c}^{\alpha',\beta'}, \forall c\in[k]$.
\end{lemma}
\begin{proof}
Note that $\w_c^{\alpha=1,\beta=0}=\mathbf{w}_{\text{MNI},c}, c\in[k]$ and for arbitrary $\alpha>0,\beta$, we have: $\w_c^{\alpha,\beta}=\alpha \mathbf{w}_{\text{MNI},c}+\beta\X(\X^\top\X)^{+}\ones$. Moreover, it is not hard to check that  $\mathbf{w}_{\text{MNI},c}^\top\x \le \max_{j \ne c}\mathbf{w}_{\text{MNI},j}^\top\x$ if and only if $(\alpha{\wls}_c+\mathbf{b})^\top\x \le \max_{j \ne c}(\alpha\mathbf{w}_{\text{MNI},j}+\mathbf{b})^\top\x$, for any $\mathbf{b}\in\R^p$. The claim then follows by choosing $\mathbf{b} = \beta\X(\X^\top\X)^{+}\mathbf{1}$ and noting that $\alpha>0,\beta$ were chosen arbitrarily.
\end{proof}

\subsection{Connection to effective overparameterization}
\label{sec:linksvm}

\Theorem~\ref{lem:key} establishes a \textit{deterministic} condition that applies to any multiclass separable dataset as long as the data matrix $\X$ is full-rank.
In this subsection, we show that the inequality in \Equation~\eqref{eq:det-con} occurs with high-probability under both the GMM and MLM data models provided that there is sufficient effective overparameterization. 

\subsubsection{Gaussian mixture model}\label{sec-linkgmm}
We assume a {nearly equal-energy, equal-prior} setting as detailed below. 
{\begin{assumption}[Nearly equal energy/prior]
\label{ass:equalmuprob}
We assume that the norms of the mean vectors are at the same order, i.e. for some large {enough} constants $\{C_i\}_{i=1}^4$, 
there exists a vector $\mub$ such that the mean vectors satisfy  $(1-\frac{1}{C_1})\tn{\mub} \le \tn{\boldmu_c} \le (1+\frac{1}{C_2})\tn{\mub}, \forall c \in [k]$ (equivalently, we have $C_1 \leq \frac{\tn{\boldmu_c}}{\tn{\boldmu_{c'}}} \leq C_2$ for all $c,c' \in [k]$ and large enough constants $C_1,C_2 >0$). Moreover, the class priors are also at the same order, i.e. they satisfy $(1-\frac{1}{C_3})\frac{1}{k} \le \pi_c \le (1+\frac{1}{C_4})\frac{1}{k}, \forall c \in [k]$ (equivalently, we have $C_3 \leq \frac{\pi_c}{\pi_{c'}} \leq C_4$ for all $c,c'\in[k]$ and large enough constants $C_3,C_4 > 0$).
\end{assumption}
}
\begin{theorem}
\label{thm:svmgmm}
Assume that the training set follows a multiclass GMM with $\Sigmab =\mathbf{I}_p$, Assumption~\ref{ass:equalmuprob} holds, and the number of training samples $n$ is large enough. There exist constants $c_1,c_2,c_3>1$ and $C_1,C_2>1$ such that \Equation~\eqref{eq:det-con} holds with probability at least $1-\frac{c_1}{n}-c_2ke^{-\frac{n}{c_3k^2}}$, provided that
\begin{align}
\label{eq:thmsvmgmm}
    p > C_1k^3n\log(kn)+n-1\quad\text{ and }\quad p>C_2k^{1.5}n\sqrt{\logtwon}\tn{\boldmu}.
\end{align}
\end{theorem}
\Theorem~\ref{thm:svmgmm} establishes a set of two conditions under which 
\Equation~\eqref{eq:det-con} and the conclusions of \Theorem~\ref{lem:key} hold, i.e. $\Wsvm = \Wls$.
The first condition requires sufficient overparameterization $p=\Omega(k^3n\log(kn)),$ 
while the second one requires that the signal strength is not too large. Intuitively, we can understand these conditions as follows. Note that \Equation~\eqref{eq:det-con} is satisfied provided that the inverse Gram matrix $(\X^T\X)^{-1}$ is ``close'' to identity, or any other positive-definite diagonal matrix. 
Recall from \Equation~\eqref{def-gmm} that $\X = \M\Y + \Q= \sum_{j=1}^k\boldmu_j\vb_j^T + \Q$ where $\Q$ is a $p\times n$ standard Gaussian matrix. The first inequality in \Equation~\eqref{eq:thmsvmgmm} (i.e. a lower bound on the data dimension $p$) is sufficient for $(\Q^T\Q)^{-1}$ to have the desired property; the major technical challenge is that $(\X^T\X)^{-1}$ involves additional terms that intricately depend on the label matrix $\Y$ itself.
Our key technical contribution is showing that these extra terms do \emph{not} drastically change the desired behavior, provided that the norms of the mean vectors (i.e. signal strength) are sufficiently small. At a high-level we accomplish this with a recursive argument as follows. Denote $\X_0=\Q$ and $\X_i=\sum_{j=1}^{i}\boldmu_j\vb_j^T + \Q$ for $i\in[k]$. Then, at each stage $i$ of the recursion, we show how to bound quadratic forms involving $\big(\X_i^T\X_i\big)^{-1}$ using bounds established previously at stage $i-1$ on quadratic forms involving $\big(\X_{i-1}^T\X_{i-1}\big)^{-1}$.
A critical property for the success of our proof strategy is the observation that the rows of $\Y$ are always orthogonal, that is, $\vb_i^T\vb_j = 0$, for $i \ne j$. 
The complete proof of the theorem is given in Section~\ref{sec:pforderoneoutline}.

We first present numerical results that support the conclusions of \Theorem~\ref{thm:svmgmm}. 
(In all our figures, we show averages over $100$ Monte-Carlo realizations, and the error bars show the standard deviation at each point.) 
\Figure~\ref{fig-svmgmm01sm}(a) plots the fraction of support vectors satisfying \Equation~\eqref{eq:symmetry-constraint} as a function of training size $n$. 
We fix dimension $p=1000$ and class priors $\pi=\frac{1}{k}$. 
To study how the outcome depends on the number of classes $k$ and signal strength $\tn{\boldmu}$, we consider $k=4,7$ and three equal-energy scenarios where $\forall c\in[k]: \tn{\boldmu_c}=\tn{\mub} = \mu\sqrt{p}$ with $\mu = 0.2, 0.3, 0.4$.
Observe that smaller $\mu$ results in larger proportion of support vectors for the same value of $n$. 
To verify our theorem's second condition (on the signal strength) in \Equation~\eqref{eq:thmsvmgmm}, \Figure~\ref{fig-svmgmm01sm}(a) also plots the same set of curves over a re-scaled axis $k^{1.5}n^{1.5}\tn{\boldmu}/p$. The six curves corresponding to different settings nearly overlap in this new scaling, showing that the condition is order-wise tight. In \Figure~\ref{fig-svmgmm01sm}(b), we repeat the experiment in  \Figure~\ref{fig-svmgmm01sm}(a) for different values of $k=3$ and $k=6$. Again, these curves nearly overlap when the x-axis is scaled according to the second condition on signal strength in \Equation~\eqref{eq:thmsvmgmm}. 
We conjecture that our second condition on the signal strength is tight up to an extra $\sqrt{n}$ factor, which we believe is an artifact of the analysis\footnote{Support for this belief comes from the fact that~\cite{wang2020benign} shows that $p>C_2\|\mub\|_2n$ is sufficient for the SVM = interpolation phenomenon to occur in the case of GMM and \emph{binary} classification.}.
 We also believe that the $k^3$ factor in the first condition can be relaxed slightly to $k^2$ (as in the MLM case depicted in \Figure~\ref{fig-svmmlm01sm}, which considers a rescaled $x$-axis and shows \textit{exact} overlap of the curves for all values of $k$).
Sharpening these dependences on both $k$ and $n$ is an interesting direction for future work.

\begin{figure}[t]
\centering
\begin{minipage}[b]{0.49\linewidth}
  \centering
  \centerline{\includegraphics[width=8cm]{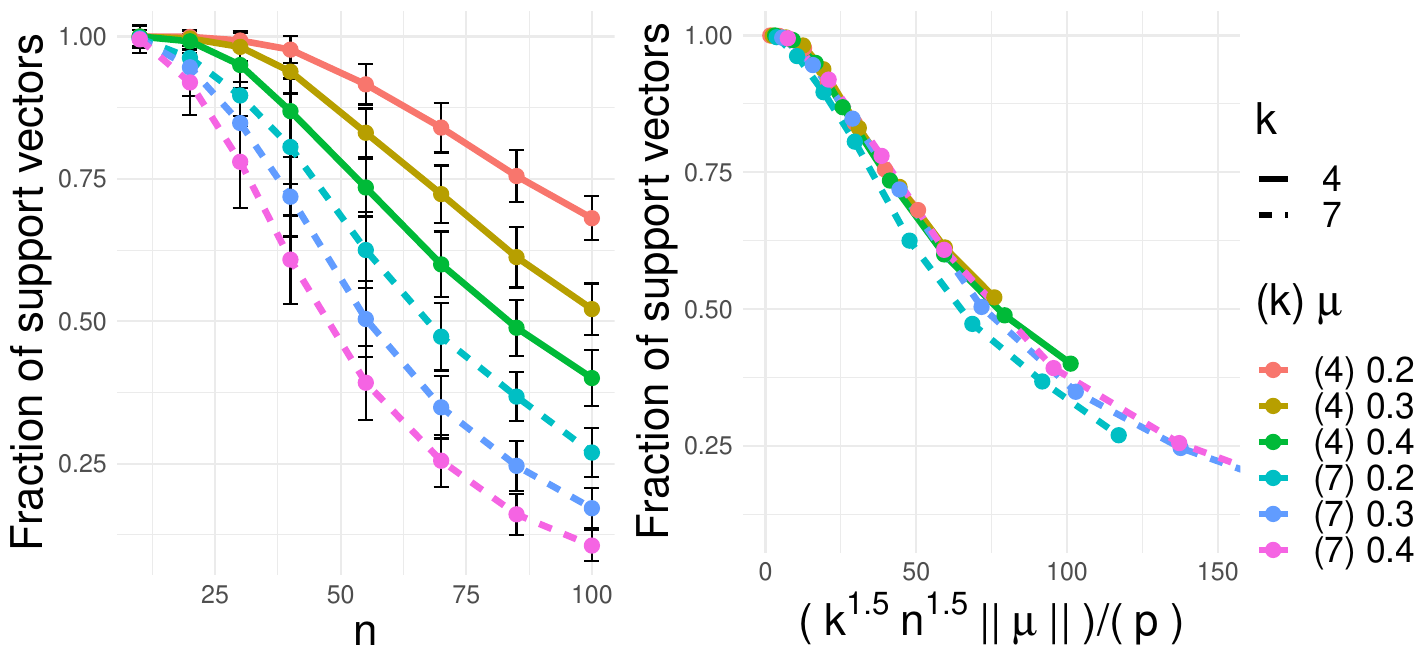}}
  \centerline{(a) $k = 4$ and $7$}\medskip
\end{minipage}
\begin{minipage}[b]{0.49\linewidth}
  \centering
  \centerline{\includegraphics[width=8cm]{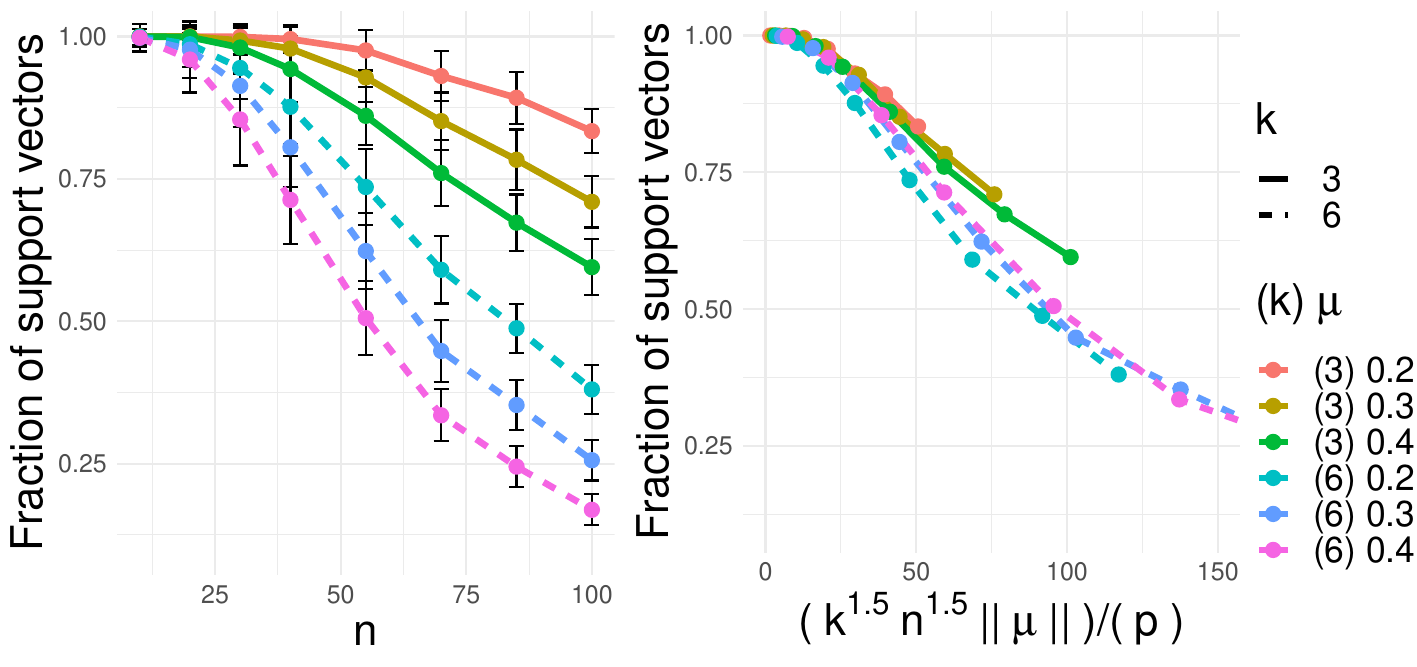}}
  \centerline{(b) $k = 3$ and $6$}\medskip
\end{minipage}
 \caption{Fraction of training examples satisfying \Equation~\eqref{eq:symmetry-constraint} (also called ``support vectors'') in the GMM case. The error bars show the standard deviation. Figure (a) considers $k = 4$ and $7$, and Figure (b) considers $k =3$ and $6$. On the legend, ``(4) 0.3'' corresponds to $k = 4$ and $\tn{ \boldmu }/\sqrt{p} =0.2$. Observe that the curves nearly overlap when plotted versus $k^{1.5}n^{1.5}\tn{\boldmu}/p$ as predicted by the second condition in \Equation~\eqref{eq:thmsvmgmm} of \Theorem~\ref{thm:svmgmm}.}
    \label{fig-svmgmm01sm}
\end{figure}

\subsubsection{Multinomial logistic model}
\label{sec-linkmlm}
We now consider the MLM data model and anisotropic data covariance.
Explicitly, the eigendecomposition of the covariance matrix is given by $\Sigmab = \sum_{i=1}^p \lambda_i\boldul_i\boldul_i^T$, where $\lbdb = [\lambda_1,  \cdots, \lambda_p]$.
We also define the effective dimensions $d_2 :={\Vert \lbdb \Vert_1^2}/{\Vert \lbdb \Vert_2^2}$ and $d_\infty:={\Vert \lbdb \Vert_1}/{\Vert \lbdb \Vert_\infty}$. 
The following result contains sufficient conditions for the SVM and MNI solutions to coincide.
\begin{theorem}
\label{thm:svmmlmaniso}
Assume $n$ training samples following the MLM defined in~\eqref{def:mlm}. There exist constants $c$ and $C_1, C_2 >1$ such that \Equation~\eqref{eq:det-con} holds
 with probability at least $(1 - \frac{c}{n})$ provided that
\begin{align}\label{eq:thmanisolink}
d_{\infty} > C_1k^2n\log(kn) \text{ and } d_2 > C_2(\log(kn) + n).
\end{align}
In fact, the only conditions we require on the generated labels is conditional independence.

For the isotropic case $\Sigmab = \mathbf{I}_p$, this implies that
\Equation~\eqref{eq:det-con} holds
 with probability at least $(1 - \frac{c}{n})$ provided that
\begin{align}\label{eq:thmisolink}
p > C_1k^2n\log ({k}n).
\end{align}
%
\end{theorem}
The sufficient conditions in \Theorem~\ref{thm:svmmlmaniso} require that the spectral structure in the covariance matrix $\Sigmab$ has sufficiently slowly decaying eigenvalues (corresponding to sufficiently large $d_2$), and that it is not too ``spiky'' (corresponding to sufficiently large $d_{\infty}$). 
When $\Sigmab = \mathbf{I}_p$, the conditions reduce to sufficient overparameterization.
For the special case of $k=2$ classes, our conditions reduce to those in \cite{hsu2020proliferation} for binary classification.
The dominant dependence on $k$, given by $k^2$, is a byproduct of the ``unequal'' margin in \Equation~\eqref{eq:symmetry-constraint-scalar}. 
\Figure~\ref{fig-svmmlm01sm} empirically verifies the sharpness of this factor. 

The proof of Theorem~\ref{thm:svmmlmaniso} is provided in Appendix~\ref{sec:svm_mlm_sm}. 
We now numerically validate our results in Theorem \ref{thm:svmmlmaniso} in \Figure~\ref{fig-svmmlm01sm}, focusing on the isotropic case. We fix $p=1000$, vary $n$ from $10$ to $100$ and the numbers of classes from $k = 3$ to $k=6$. We choose orthogonal mean vectors for each class with equal energy $\tn{\boldmu}^2=p$. The left-most plot in \Figure~\ref{fig-svmmlm01sm} shows the fraction of support vectors satisfying \Equation~\eqref{eq:symmetry-constraint} as a function of $n$. Clearly, smaller number of classes $k$ results in higher proportion of support vectors with the desired property for the same number of measurements $n$. To verify the condition in \Equation~\eqref{eq:thmisolink}, the middle plot in \Figure~\ref{fig-svmmlm01sm} plots the same curves over a re-scaled axis $k^{2}n\log(kn)/p$ (as suggested by \Equation~\eqref{eq:thmisolink}). We additionally draw the same curves over $kn\log(kn)/p$ in the right-most plot of \Figure~\ref{fig-svmgmm01sm}. Note the overlap of the curves in the middle plot.
We now numerically validate our results in Theorem \ref{thm:svmmlmaniso} in \Figure~\ref{fig-svmmlm01sm}, focusing on the isotropic case. We fix $p=1000$, vary $n$ from $10$ to $100$ and the numbers of classes from $k = 3$ to $k=6$. We choose orthogonal mean vectors for each class with equal energy $\tn{\boldmu}^2=p$. The left-most plot in \Figure~\ref{fig-svmmlm01sm} shows the fraction of support vectors satisfying \Equation~\eqref{eq:symmetry-constraint} as a function of $n$. Clearly, smaller number of classes $k$ results in higher proportion of support vectors with the desired property for the same number of measurements $n$. To verify the condition in \Equation~\eqref{eq:thmisolink}, the middle plot in \Figure~\ref{fig-svmmlm01sm} plots the same curves over a re-scaled axis $k^{2}n\log({k}n)/p$ (as suggested by \Equation~\eqref{eq:thmisolink}). We additionally draw the same curves over $kn\log({k}n)/p$ in the right-most plot of \Figure~\ref{fig-svmgmm01sm}. Note the overlap of the curves in the middle plot.
\begin{figure}[t]
    \centering
    \includegraphics[width=0.75\textwidth]{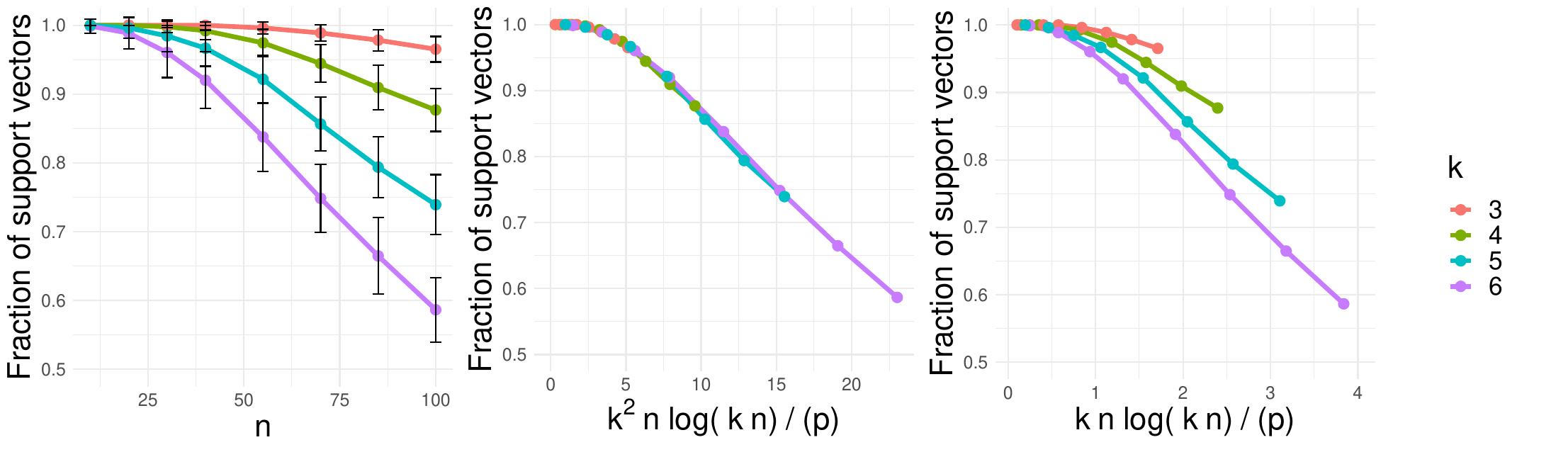}
    \caption{Fraction of training examples satisfying equality in the simplex label representation in \Equation~\eqref{eq:symmetry-constraint}  in the MLM case with $\Sigmab = \mathbf{I}_p$. The middle plot shows that the curves overlap when plotted versus $k^{2}n\log({k}n)/p$ as predicted by \Equation~\eqref{eq:thmisolink}.}
    \label{fig-svmmlm01sm}
\end{figure}


\subsection{Connection to Neural Collapse}\label{sec:ncconnection}
In this section, we provide a distinct set of sufficient conditions on the feature vectors that guarantee \Equation~\eqref{eq:det-con}, and hence the conclusions of Theorem \ref{lem:key} hold. Interestingly, these sufficient conditions
relate to the recently discovered, so called \emph{neural-collapse} phenomenon that is empirically observed in the training process of overparameterized deep nets~\cite{papyan2020prevalence} (see also e.g. \cite{NC1,NC2,NC3,NC4,NC5,NC6,NC7,NC8} for several recent follow-ups). 

\begin{corollary}\label{thm:neuralcol}
Recall the notation in Theorem \ref{lem:key}. Assume exactly balanced data, that is $|\{i:y_i=c\}|=n/k$ for all $c\in[k]$. Also, assume that the following two conditions hold:
\begin{itemize}    
    \item \textbf{Feature collapse (NC1):} For each $c\in[k]$ and all $i\in[n]:y_i=c$, it holds that $\x_i=\mub_c$, 
    where $\mub_c \triangleq \frac{k}{n}\sum_{i: y_i = c}\x_i$ is the ``mean'' vector of the corresponding class.
    \item \textbf{Simplex ETF structure (NC2)}: The matrix of mean vectors $\M:=[\mub_1,\ldots,\mub_k]_{p\times k}$ is the matrix of a simplex Equiangular Tight Frame (ETF), i.e. for some orthogonal matrix $\mathbf{U}_{p\times k}$ (with $\mathbf{U}^T\mathbf{U}=\mathbf{I}_k$) and $\alpha\in\R$, it holds that 
    \begin{align}
        \M = \alpha\sqrt{\frac{k}{n}} \mathbf{U}\left(\mathbf{I}_k - \frac{1}{k}\mathbf{1}\mathbf{1}^T\right).
    \end{align}
\end{itemize}
Then, the sufficient condition \eqref{eq:det-con} of Theorem \ref{lem:key} holds for the Gram matrix $\X^T\X$.
\end{corollary}

\begin{proof}
For simplicity, denote the sample size of each class as $m:=n/k$.
Without loss of generality under the corollary's assumptions, let the columns of the feature matrix $\X$ be ordered such that $\X=[\M,\M,\ldots,\M]=\M\otimes\mathbf{1}_m^T$. Accordingly, we have $\z_c=\left(\mathbf{e}_c\otimes\ones_m\right)-\frac{1}{k}\left(\ones_k\otimes\ones_m\right)$ where $\mathbf{e}_c$ is the $c$-th basis vector in $\R^k$.
Then, the feature Gram matrix is computed as
\begin{equation}\label{eq:gram_nc_proof}
\X^T\X=\left(\M^T\M\right)\otimes\left(\mathbf{1}_m\mathbf{1}_m^T\right) = \frac{{\alpha^2}}{m}\left(\mathbf{I}_k-\frac{1}{k}\mathbf{1}_k\mathbf{1}_k^T\right)\otimes\left(\mathbf{1}_m\mathbf{1}_m^T\right).
\end{equation}

Observe here that we can write $\left(\mathbf{I}_k-\frac{1}{k}\mathbf{1}_k\mathbf{1}_k^T\right)=\V\V^T$ for $\V \in \mathbb{R}^{k \times (k-1)}$ having orthogonal columns (i.e. $\V^T\V=\mathbf{I}_{k-1}$) and $\V^T\ones_k=\mathbf{0}_k$. Using this and the fact  that $(\V\V^T)^{+}=(\V\V^T)$, it can be checked from \eqref{eq:gram_nc_proof} that 
\begin{equation}\label{eq:gram_inv_nc_proof}
(\X^T\X)^{+}=\frac{1}{\alpha^2 m} \left(\mathbf{I}_k-\frac{1}{k}\mathbf{1}_k\mathbf{1}_k^T\right)\otimes\left(\mathbf{1}_m\mathbf{1}_m^T\right).
\end{equation}
Putting things together, we get, for any $c\in[k]$, that
\begin{align}\nn
(\X^T\X)^{+}\z_c = \frac{1}{\alpha^2 m} \left(\big(\mathbf{I}_k-\frac{1}{k}\mathbf{1}_k\mathbf{1}_k^T\big)\otimes\left(\mathbf{1}_m\mathbf{1}_m^T\right)\right) \left(\mathbf{e}_c\otimes\ones_m\right) = \frac{1}{\alpha^2}\left(\mathbf{e}_c-\frac{1}{k}\mathbf{1}_k\right)\otimes \mathbf{1}_m = \frac{1}{\alpha^2}\z_c.
\end{align}
Therefore, it follows immediately that
$$
 \z_c\odot \M^{+}\z_c = \frac{1}{\alpha^2}\z_c\odot \z_c > \mathbf{0},
$$
as desired. This completes the proof.
\end{proof}

It might initially appear that the structure of the feature vectors imposed by the properties NC1 and NC2 is too specific to be relevant in practice. To the contrary, \cite{papyan2020prevalence} showed via a principled experimental study that these properties occur at the last layer of overparameterized deep nets across several different data sets and DNN architectures. Specifically, the experiments conducted in \cite{papyan2020prevalence} suggest that training overparameterized deep nets on classification tasks with CE loss in the absence of weight decay (i.e. without explicit regularization) results in learned feature representations in the final layer that converge\footnote{Here, ``convergence'' is with respect to an increasing number of training epochs. Since the architecture is overparameterized, it can perfectly separate the data. Hence, the training 0-1 error can be driven to zero. Nevertheless, training continues despite having achieved zero 0-1 training error, since the CE loss continues to drop. \cite{papyan2020prevalence} refers to this regime as the terminal phase of training (TPT). In sum,~\cite{papyan2020prevalence} show that neural collapse is observed in TPT.} to the ETF structure described by NC1 and NC2. Furthermore, it was recently shown in \cite{NC8} that the neural collapse phenomenon continues to occur when the last-layer features of a deep net are trained with the recently proposed \emph{supervised contrastive loss} (SCL) function~\cite{khosla2020supervised} and a linear model is independently trained on these learned last-layer features.
(In fact,~\cite{NC8,khosla2020supervised} showed that this self-supervised procedure can yield superior generalization performance compared to CE loss.)

To interpret Corollary \ref{thm:neuralcol} in view of these findings, consider the following two-stage classification training process:
\begin{itemize}
\item First, train (without weight-decay and continuing training beyond the interpolation regime) the last-layer feature representations of an overparameterized deep-net with either CE or SCL losses. 
\item Second, taking as inputs those learned feature representations of the first stage, train a linear multiclass classifier (often called the ``head'' of the deep-net)  with CE loss. 
\end{itemize}
Then, from Corollary \ref{thm:neuralcol}, the resulting classifier from this two-stage process interpolates the simplex label representation, and the classification accuracy is the same as if we had used the square loss in the second stage of the above training process.
{Thus, our results lend strong theoretical justification to the empirical observation that \emph{square-loss and CE loss yield near-identical performance in large-scale classification tasks~\cite{rifkin2002everything,rifkin2004defense,hui2020evaluation,poggio2020explicit}}.}

\section{Generalization bounds}\label{sec:error}

In this section, we derive non-asymptotic bounds on the error of the MNI classifier for data generated from both GMM and MLM, as well as a natural setting in which the class means follow the simplex-ETF geometry.

\subsection{Gaussian mixture model}\label{sec:errorGMM}
We present classification error bounds under the additional assumption of {mutually incoherent} means.
{\begin{assumption}[Mutually incoherent means]
\label{ass:orthomean}
Let $M = \max_{i \ne j}\frac{|\boldmu_i^T\boldmu_j|}{\tn{\boldmu_i}\tn{\boldmu_j}}$ be the mutual coherence of mean vectors. Then, we assume that there exists a large absolute constant $C > 0$ such that $M \le 1/C$.
\end{assumption}
We remark that mutual incoherence assumptions have appeared in a completely different context, i.e.~across feature vectors, in the compressive sensing literature (e.g. for sparse signal recovery)~\cite{donoho2005stable,tropp2006just}.
There, the number of feature vectors is typically greater than the dimension of each feature vector and so the mutual incoherence suffers from fundamental lower bounds~\cite{welch1974lower}.
In our setting, the incoherence assumption  applies to the class-mean vectors. Note that the number of mean vectors ($k$) is always smaller than the dimension of each vector ($p$) and so Welch's lower bound does not apply, making our assumption reasonable.
}
\begin{theorem}
\label{thm:classerrorgmm}
Let {Assumptions~\ref{ass:equalmuprob} and~\ref{ass:orthomean},} as well as the condition in \Equation~\eqref{eq:thmsvmgmm} hold. Further assume constants $C_1,C_2
,C_3>1$ such that $\big(1-\frac{C_1}{\sqrt{n}}-\frac{C_2n}{p}\big)\tn{\boldmu} > C_3\minklogn.$
Then, there exist additional constants $c_1,c_2,c_3$ and $C_4 > 1$ such that both the MNI solution $\Wls$ and the multiclass SVM solution $\Wsvm$ satisfy
\begin{align}
\Pro_{e|c} \leq  (k-1)\exp{\left(-\tn{\boldmu}^2\frac{\left(\left(1-\frac{C_1}{\sqrt{n}}-\frac{C_2n}{p}\right)\tn{\boldmu}-C_3\minklogn\right)^2}{C_4\left(1+\frac{kp}{n\tn{\boldmu}^2}\right)}\right)}
\end{align}
with probability at least $1-\frac{c_1}{n} - c_2ke^{-\frac{n}{c_3k^2}}$, for every $c \in [k]$.
Moreover, the same bound holds for the total classification error $\Pro_{e}$.
\end{theorem}

For large enough $n$, \Theorem~\ref{thm:classerrorgmm} reduces to the results in \cite{wang2020benign} when $k = 2$ (with slightly different constants). There are two major challenges in the proof of Theorem~\ref{thm:classerrorgmm}, which is presented in Appendix~\ref{sec:classerrorgmmproof}. First, in contrast to the binary case the classification error does \emph{not} simply reduce to bounding correlations between vector means $\mub_c$ and their estimators $\wh_c$. Second, just as in the proof of \Theorem~\ref{thm:svmgmm}, technical complications arise from the multiple mean components in the training data matrix $\X$.
We use a variant of the recursion-based argument described in Section~\ref{sec:pforderoneoutline} to obtain our final bound.

{
\subsubsection{A possible extension to anisotropic noise covariances}\label{sec:non-isotropic}

Up to this point, we have concentrated on GMM data with isotropic noise,~i.e. the noise covariance matrix in {\Equation~\eqref{def-gmm}} is such that $\Sigmab=\mathbf{I}_p$. It is crucial to note that, even in the case of isotropic noise, the entire data covariance matrix $\E[\x\x^T]$ for GMM data is anisotropic, as it exhibits spikes in the direction of the mean vectors. 
Thus, it already models highly correlated features.
This already makes the analyses challenging both at the level of establishing equivalence of SVM to MNI as well as deriving generalization bounds for the MNI (analogous to the challenges faced in the initial analyses of benign overfitting for regression~\cite{bartlett2020benign,hastie2019surprises}). 
Based on this, we now make a brief comment on the possibility of extending Theorem~\ref{thm:classerrorgmm} to anisotropic GMM data. 
Although a comprehensive examination is beyond the scope of this paper, we provide evidence that our analysis can serve as a foundation for such extensions.

As a starting point, we note that the necessary and sufficient equivalence conditions of Theorem~\ref{lem:key} still hold (as they are deterministic and require no assumptions on the data). We sketch here a possible proof argument to work from Theorem~\ref{lem:key} and prove a variant of Theorem~\ref{thm:svmgmm} for anisotropic noise covariance. 
Let  $\Sigmab = \V\mathbf{\Lambda}\V^T$ be the covariance eigen-decomposition, where $\V$ is orthogonal and $\mathbf{\Lambda}$ is diagonal {with entries given by the eigenvalues $\{\lambda_j\}_{j=1}^p$}. With this, we can project the mean vectors $\mub_c$ of the GMM to the space spanned by the eigenvector basis $\vb_j, j\in[p]$ (aka columns of $\V$). Concretely, we can express $\boldmu_c$ as $\sum_{j=1}^p\beta_j\vb_j$.
Then, we may use this decomposition to prove a variant of Lemma \ref{lem:ineqforanoj}. Recall that to prove the higher-order terms in Lemma \ref{lem:ineqforanoj}, we need to start from deriving bounds for 0-order terms in Lemma \ref{lem:ineqforazero}. When $\Sigmab$ is anisotropic, the bounds in Lemma \ref{lem:ineqforazero} will have two main changes. First, the bounds will involve {\textit{signal strength in the direction of $\Sigmab$} defined as $\sum_{j=1}^p\lambda_j\beta_j^2$.} This is the result of projecting the mean vectors to the space spanned by the eigenvectors of $\Sigmab$. Second, the bounds will include effective ranks, e.g. $r_k := (\sum_{i>k}^p \lambda_i)/\lambda_{k+1}$ and $R_k := (\sum_{i>k}^p \lambda_i)^2/(\sum_{i>k}^p \lambda_i^2)$. Effective ranks play important role in benign overfitting and the equivalence between SVM and MNI \cite{bartlett2020benign, muthukumar2021classification}.
{Lemma 4 in \cite{wang2020benign} provides bounds for the 0-order terms in Lemma \ref{lem:ineqforazero} under anisotropic covariance. We show one examples here to see the adjustment. The upper bound for $t_{jj}^{(0)}$ changes from $\frac{C_1n\tn{\boldmu}^2}{p}$ to $\frac{C_2n\sum_{j=1}^p\lambda_j\beta_j^2}{\normonelambda}$, where $\mathbf{\lambda}$ is the vector with $\lambda_i$ as entries. Note that $\frac{n\sum_{j=1}^p\lambda_j\beta_j^2}{\normonelambda}$ becomes $\frac{n\tn{\boldmu}^2}{p}$ when $\Sigmab = \mathbf{I}_p$. Similar changes apply to other terms in Lemma \ref{lem:ineqforazero}. The 0-order bounds in Lemma \ref{lem:ineqforazero} can then be used to derive higher-order bounds in Lemma \ref{lem:ineqforanoj}. Similar to the binary results in \cite{muthukumar2021classification, wang2020benign}, the equivalence between MNI and SVM requires large effective ranks and benign overfitting requires large signal strength in the direction of $\Sigmab$. However, a detailed analysis of this general setting is beyond the scope of this paper.} 

}

\subsection{Multinomial logistic model}\label{sec:errorMLM}

In this section, we present our error analysis of the MNI classifier when data is generated by the MLM.
Importantly, for this case we consider more general anisotropic structure in the covariance matrix $\Sigmab:= \boldsymbol{U} \boldsymbol{\Lambda} \boldsymbol{U}^\top$.
We begin by carrying over the assumptions made from the binary-case analysis in~\cite{muthukumar2021classification}, beginning with a natural assumption of $s$-sparsity.

\begin{assumption}[$s$-sparse class means]\label{as:ssparse}
We assume that all of the class means $\boldmu_c, c \in [k]$ are $s$-sparse in the basis given by the eigenvectors of $\Sigmab$.
In other words, we have
\begin{align*}
    \boldsymbol{U}^{-1} \boldmu_{c,j} = 0 \text{ if } j > s .
\end{align*}
\end{assumption}
This $s$-sparse assumption is also made in corresponding works on regression (e.g. for the results for the anisotropic case in~\cite{hastie2019surprises}) and shown to be necessary in an approximate sense for consistency of MSE of the minimum-$\ell_2$-norm interpolation arising from bias~\cite{tsigler2020benign}.
Next, we make a special assumption of bi-level structure in the covariance matrix.
\begin{assumption}[Bi-level ensemble]\label{as:bilevel}
We assume that the eigenvalues of the covariance matrix, given by $\lbdb$, have a {bilevel} structure.
In particular, our bi-level ensemble is parameterized by $(n,m,q,r)$ where $m > 1$, $0 \leq r < 1$ and $0 < q < {(m-r)}$.
We set parameters $p = n^m$, $s = n^r$ and $a = n^{-q}$.
Then, the eigenvalues of the covariance matrix are given by
\begin{align*}
    \lambda_j = \begin{cases}
        \lambda_H := \frac{ap}{s},\; 1 \leq j \leq s \\
        \lambda_L := \frac{(1-a)p}{p-s},\; \text{ otherwise.}
    \end{cases}
\end{align*}
We will fix $(m,q,r)$ and study the classification error as a function of $n$.
While the bi-level ensemble structure is not in principle \textit{needed} for complete statements of results, it admits particularly clean characterizations of classification error rates as well as easily interpretable conditions for consistency\footnote{See~\cite{muthukumar2021classification} for additional context on the bi-level ensemble and examples of its manifestation in high-dimensional machine learning models.}.
\end{assumption}

Assumption \ref{as:bilevel} splits the covariance spectrum in a small set of large eigenvalues $\la_H$ and the remaining large set of small eigenvalues $\la_L$. 
The bi-level ensemble is friendly to consistency of the MNI solution for three reasons: a) the number of small eigenvalues is much larger than the sample size, b) the ratio between the large-valued and small-valued eigenvalues grows with the sample size $n$, and c) the number of large-valued eigenvalues is exponentially small relative to the sample size $n$.
Note that condition a) facilitates benign overfitting of noise (as first pointed out in the more general anisotropic case by~\cite{bartlett2020benign}), while conditions b) and c) facilitate signal recovery.
To verify these conditions more quantitatively, note that: a) the number of small eigenvalues is on the order of $p \gg n$, b) the ratio between the large-valued and small-valued eigenvalues can be verified to be on the order of $n^{m - q - r}$ which grows with $n$, and c) the number of large-valued eigenvalues is equal to $s=n^r$, which is exponentially smaller than $n$.

Finally, we imbue the above assumptions with an equal energy and orthogonality assumption, as in the GMM case.
These assumptions are specific to the multiclass task, and effectively subsume Assumption~\ref{as:ssparse}.
\begin{assumption}[Equal energy and orthogonality]\label{as:mlmfull}
We assume that the class means are equal energy, i.e. $\tn{\boldmu} = 1/\sqrt{\lambda_H}$ for all $c \in [k]$, and are orthogonal, i.e. $\boldmu_i^\top \boldmu_j = 0$ for all $i \neq j \in [k]$.
Together with Assumptions~\ref{as:ssparse} and~\ref{as:bilevel}, a simple coordinate transformation gives us
\begin{align*}
    \boldmu_c &= \frac{1}{\sqrt{\lambda_H}} \boldsymbol{e}_{j_c} \; \text{ for some $j_c \in [s]$, $j_c \neq j_{c'}$ for all $c \neq c' \in [k]$, and } \\
    \Sigmab &= \boldsymbol{\Lambda}
\end{align*}
without loss of generality.
The normalization by the factor $\frac{1}{\sqrt{\lambda_H}}$ is done to ensure that the signal strength is equal to $1$, i.e. $\E[(\x^\top \boldmu_c)^2] = 1$ for all $c \in [k]$.
\end{assumption}
Under these assumptions, we state our main result for the total classification error of MLM.
Our error bounds will be on the \textit{excess} risk over and above the Bayes error rate incurred by the optimal classifier 
$\{\hatw_c = \boldmu_c\}_{c \in [k]}$, which we denote by $\Pro_{e,\mathsf{Bayes}}$. 
\begin{theorem}\label{thm:classerrormlm}
Under Assumptions~\ref{as:bilevel} and~\ref{as:mlmfull}, there is a universal constant $c_k$ (that may depend on $k$, but not $n$ or $p$) such that the total excess classification error of $\Wls$ and $\Wsvm$ under the MLM model is given by
\begin{align}\label{eq:mlmerror}
\Pro_e - \Pro_{e,\mathsf{Bayes}} &\leq k^2 \left(\frac{1}{2} - \frac{1}{\pi} \mathsf{tan}^{-1}(\mathsf{SNR}(n))\right),\;\text{ where } \\
   \mathsf{SNR}(n) &\geq c_k (\log n)^{-1/2} \cdot n^{\frac{\min\{(m-1),(2q+r-1),(2q + 2r - 3/2)\}}{2} + (1 - r) - q},\; q > (1 - r) \nonumber
\end{align}
for $q > 1 - r$.
\end{theorem}

The proof of Theorem \ref{thm:classerrormlm} is presented in Section \ref{sec:proofclasserrormlm}.
We will show in the subsequent Section~\ref{sec:benignoverfitting} that, although the rate in Equation~\eqref{eq:mlmerror} is worse in its dependence on $q$ and $r$ than for the equivalent binary classification problem, the conditions for benign overfitting turn out to coincide in the regime where we keep $k$ constant with respect to $n$.
{\subsection{Means following the simplex-ETF geometry}
Next, we derive generalization bounds under an entirely different assumption on the geometry of mean vectors. Specifically, we consider the setting in which the mean vectors follow the simplex ETF geometry structure that was discussed in Section~\ref{sec:ncconnection}. Recall, this setting is particularly interesting and relevant to practice, as the ETF geometry describes the geometry of learnt class-mean embeddings of deep-nets when trained with the CE loss to completion (i.e.,~beyond achieving zero $0-1$ training error)~\cite{papyan2020prevalence}.}
{
\begin{theorem}
\label{cor:classerrorgmmnc}
Let the nearly equal energy/prior Assumption~\ref{ass:equalmuprob} and the conditions in \Equation~\eqref{eq:thmsvmgmm} hold. Additionally, assume the means form a ETF structure, i.e. $\mub_i^T\mub_i = -(k-1)\mub_i^T\mub_j$, for $i \ne j$. Further assume constants $C_1,C_2
,C_3>1$ such that $\big(1-\frac{C_1}{\sqrt{n}}-\frac{C_2n}{p}\big)\tn{\boldmu} > C_3\minklogn.$
Then, there exist additional constants $c_1,c_2,c_3$ and $C_4 > 1$ such that both the MNI solution $\Wls$ and the multiclass SVM solution $\Wsvm$ satisfy
\begin{align}
\Pro_{e|c} \leq  (k-1)\exp{\left(-\tn{\boldmu}^2\frac{\left(\left(1-\frac{C_1}{\sqrt{n}}-\frac{C_2n}{p}\right)\tn{\boldmu}-C_3\minklogn\right)^2}{C_4\left(1+\frac{kp}{n\tn{\boldmu}^2}\right)}\right)}
\end{align}
with probability at least $1-\frac{c_1}{n} - c_2ke^{-\frac{n}{c_3k^2}}$, for every $c \in [k]$.
Moreover, the same bound holds for the total classification error $\Pro_{e}$.
\end{theorem}
{The proof of this theorem is provided in Appendix \ref{sec:etfgenproof}. The non-zero inner products between $\mub_i^T$ and $\mub_j$ contribute ``negatively'' to the signal $\mub_i^T\mub_i$. This negative contribution can be negated because of the simplex ETF structure $\mub_i^T\mub_i = -(k-1)\mub_i^T\mub_j$, hence the bounds in \Theorem~\ref{thm:classerrorgmm} still hold.} 
}
\section{Conditions for benign overfitting}\label{sec:benignoverfitting}



Thus far, we have studied the classification error of the MNI classifier under the GMM data model (\Theorem~\ref{thm:classerrorgmm}), and shown equivalence of the multiclass SVM and MNI solutions (Theorems~\ref{lem:key},~\ref{thm:svmgmm} and Corollary~\ref{cor:SVM=LS}). 
Combining these results, we now provide sufficient conditions under which the classification error of the multiclass SVM solution (also of the MNI) approaches $0$ as the number of parameters $p$ increases.
First, we state our sufficient conditions for harmless interpolation under the GMM model --- these arise as a consequence of Theorem~\ref{thm:classerrorgmm}, and the proof is provided in Appendix~\ref{sec:benignproof}. 

\begin{corollary}
\label{cor:benigngmm}
Let the same assumptions as in \Theorem~\ref{thm:classerrorgmm} hold. Then, for finite number of classes $k$ and sufficiently large sample size $n$, there exist positive constants $c_i$'s and $C_i$'s $> 1$, such that the multiclass SVM classifier $\Wsvm$ in ~\eqref{eq:k-svm} satisfies the {simplex} interpolation constraint in~\eqref{eq:symmetry-constraint} and its total classification error approaches $0$ as $\lp(\frac{p}{n}\rp) \to \infty$ with probability at least $1-\frac{c_1}{n} - c_2ke^{-\frac{n}{c_3k^2}}$, provided that the following conditions hold:

\noindent (1). When $\tn{\boldmu}^2 > \frac{kp}{n}$,
\begin{align*}
     \frac{n}{C_1k}\tn{\boldmu}^2 > p > \max\{C_2k^3n\log(kn)+n-1, C_3k^{1.5}n^{1.5}\tn{\boldmu}\}.
\end{align*}


\noindent (2). When $\tn{\boldmu}^2 \le \frac{kp}{n}$,
\begin{align*}
    &p > \max\{C_2k^3n\log(kn)+n-1, C_3k^{1.5}n^{1.5}\tn{\boldmu}, \frac{n\tn{\boldmu}^2}{k}\},\\
    \text{and} \ &\tn{\boldmu}^4 \ge C_4\lp(\frac{p}{n}\rp)^{\alpha}, \ \ \text{for} \ \alpha >1.
\end{align*}
When $n$ is fixed, the conditions for benign overfitting for $\Wsvm$ become $\tn{\boldmu} = \Theta(p^\beta) \text{ for } \beta\in(1/4,1).$
\end{corollary}
Note that the upper bound on $\tn{\boldmu}$ comes from the conditions that make SVM=MNI in Theorem \ref{thm:svmgmm}; indeed, a distinct corollary of Theorem~\ref{thm:classerrorgmm} is that $\Wls$ overfits benignly with sufficient signal strength $\tn{\boldmu} = \Omega(p^{1/4})$.
We can compare our result with the binary case~\cite{wang2020benign}.
When $k$ and $n$ are both finite, the condition $\tn{\boldmu} = \Theta(p^\beta) \text{ for } \beta\in(1/4,1)$ is the same as the binary result. 

Next, we state our sufficient and necessary conditions for harmless interpolation under the MLM model.
\begin{corollary}\label{cor:benignmlm}
Let the same assumptions as in Theorem~\ref{thm:classerrormlm} hold.
Then, for finite number of classes $k$, the following parameters of the bilevel ensemble (Assumption~\ref{as:bilevel}) ensure that the total classification error of $\Wsvm$ approaches $0$ as $n \to \infty$:
\begin{align}\label{eq:benignmlm}
p > 1 \text{ and } q < (1 - r) + \frac{(m-1)}{2}.
\end{align}
Further, when $q > (1 - r)$, the same conclusion holds for $\Wls$.
\end{corollary}

\begin{proof}
   We work from Equation~\eqref{eq:mlmerror} of Theorem~\ref{thm:classerrormlm}.
   For $\Pro_e - \Pro_{e,\mathsf{Bayes}} \to 0$ as $n \to \infty$, we require the exponent \\ $\frac{\min\{(m-1),(2q+r-1),(2q + 2r - 3/2)\}}{2} + (1 - r) - q \geq 0$. If $2q + 2r - 3/2$ is the minimizer, we would have $q + r - 3/4 + 1 - r - q = 1/4$, in which case the inequality is satisfied. 
   If $2q+r-1$ is the minimizer, we would have $q+r/2 - 1/2 + 1 - r - q = \frac{1-r}{2} > 0$, in which case the inequality is again satisfied. 
   Otherwise, we have $\frac{m-1}{2} + 1 - r - q > 0$, which implies $q < (1 -r) + \frac{(m-1)}{2}$.
\end{proof}
We can again compare our result with the binary case~\cite{muthukumar2021classification}: when $k$ is finite, the conditions in Equation~\eqref{eq:benignmlm} are identical to those for the binary case.
We also note that while Theorem~\ref{thm:classerrormlm} only provides an upper bound on MLM classification error,~\cite{muthukumar2021classification} provides lower bounds for the binary case that automatically apply to the MLM for the special case $k = 2$.
While there is a gap between the non-asymptotic rates, the necessary conditions for consistency coincide with Equation~\eqref{eq:benignmlm}.
Therefore, Equation~\eqref{eq:benignmlm} encapsulates sufficient and necessary conditions for consistency when $k$ is kept constant with respect to $n$.
Moreover, as~\cite{muthukumar2021classification} show, the condition $q \leq (1 - r)$ would be requirement for a corresponding \emph{regression} task to generalize; consequently, Corollary~\ref{cor:benignmlm} shows that \emph{multiclass classification can generalize even when regression does not.}

We particularly note that, Corollaries~\ref{cor:benigngmm} and~\ref{cor:benignmlm} imply benign overfitting in regimes that cannot be explained by classical \textit{training-data-dependent} bounds based on the margin~\cite{schapire1998boosting}.
While the shortcomings of such margin-based bounds in the highly overparameterized regime are well-documented, e.g. \cite{dziugaite2017computing}, we provide a brief description here for completeness.
For the MLM,~\cite[Section 6]{muthukumar2021classification} shows (for the binary case) that margin-based bounds could only predict harmless interpolation if we had the significantly stronger condition $q \leq (1 - r)$ (also required for consistency of the corresponding regression task).
For the GMM, we verify here that the margin-based bounds could only predict benign overfitting if we had the significantly stronger condition $\beta \in (1/2,1)$ (see also \cite[Section 9.1]{wang2020benign}):
in the regime where SVM = MNI, the margin is exactly equal to $1$.
The margin-based bounds (as given in, e.g.~\cite{bartlett2002rademacher}), can be verified to scale as $\mathcal{O}\left(\sqrt{\frac{\text{trace}(\Sigmab_{\text{un}})}{n || \Sigmab_{\text{un}} ||_2}}\right)$ with high probability, where $\Sigmab_{\text{un}} := \E\left[\x\x^\top\right]$ denotes the \textit{unconditional} covariance matrix under the GMM.
In the case of the binary GMM and isotropic noise covariance, an elementary calculation shows that the spectrum of $\Sigmab_{\text{un}}$ is given by $\begin{bmatrix} \tn{\boldmu}^2 + 1 & 1 & \ldots & 1 \end{bmatrix}$; plugging this into the above bound requires $\tn{\boldmu}^2 \gg \frac{p}{n}$ for the margin-based upper bound to scale as $o(1)$.
This clearly does not explain benign overfitting when SVM = MNI, which we showed requires $\tn{\boldmu}^2 \leq \frac{p}{n}$.




\begin{figure}[t]
\centering
\begin{minipage}[b]{0.49\linewidth}
  \centering
  \centerline{\includegraphics[width=8cm]{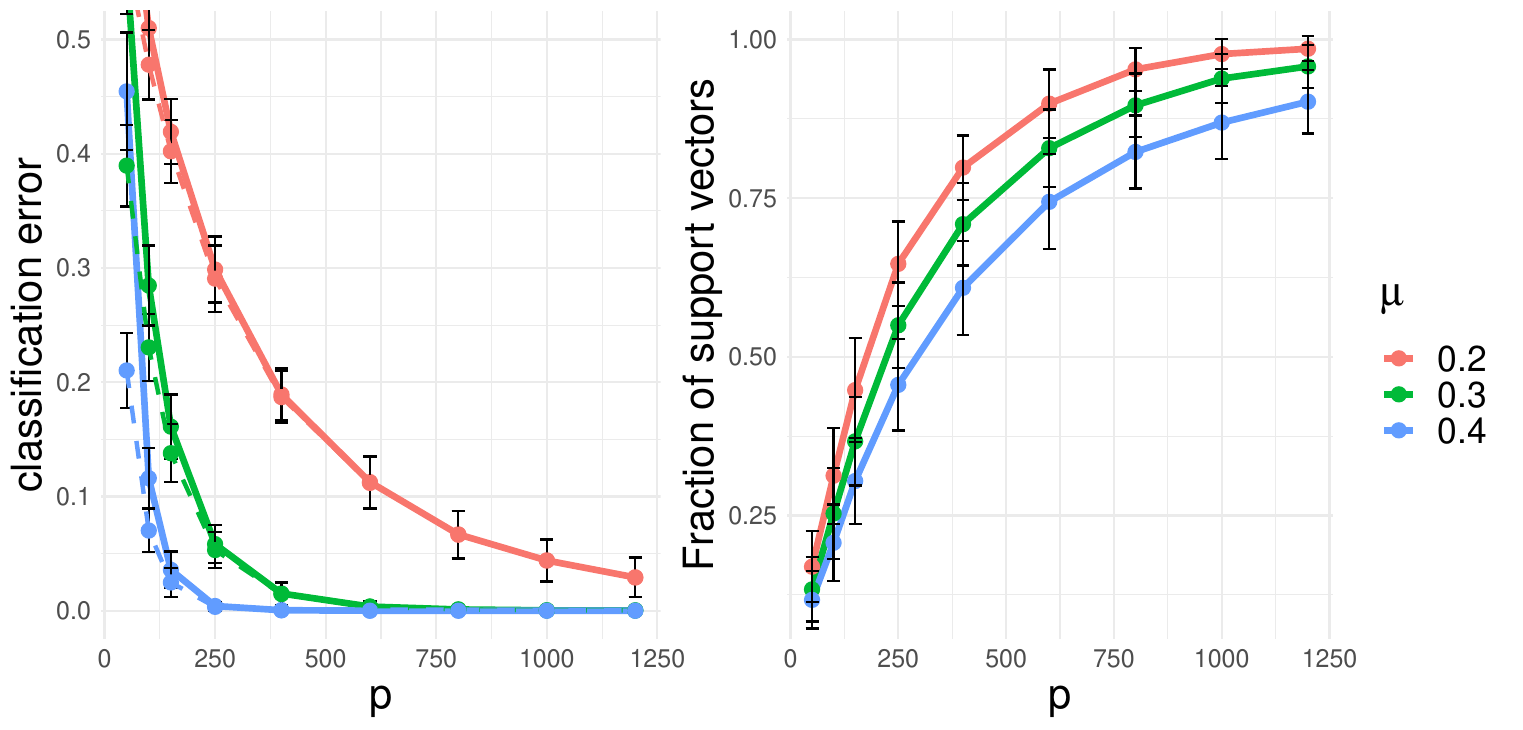}}
  \centerline{(a) $k = 4$}\medskip
\end{minipage}
\begin{minipage}[b]{0.49\linewidth}
  \centering
  \centerline{\includegraphics[width=8cm]{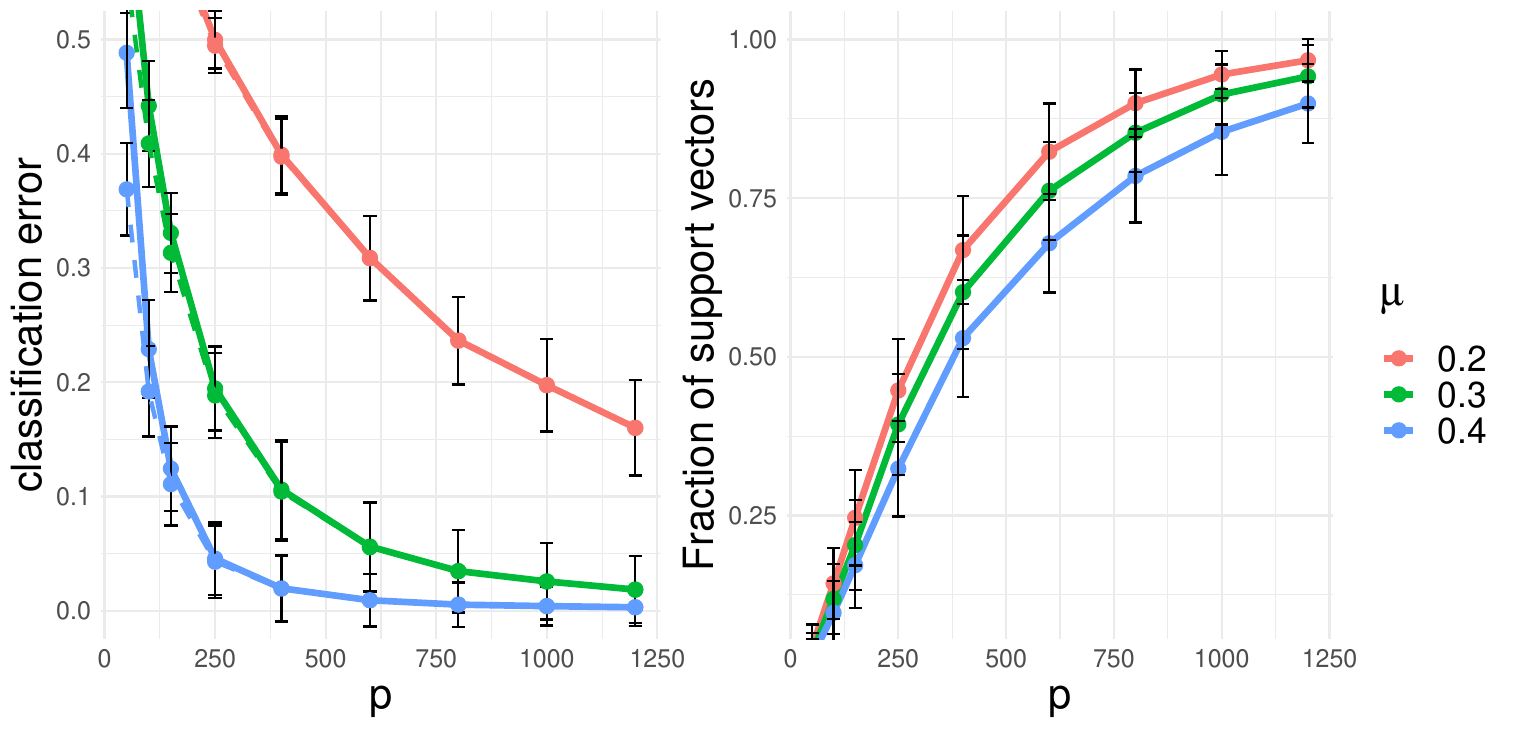}}
  \centerline{(b) $k = 6$}\medskip
\end{minipage}
    \caption{Evolution of total classification error and fraction of support vectors as a function of $p$ in the GMM case. Figure (a) considers $k = 4$ and Figure (b) considers $k = 6$. We consider the energy of all class means to be $\tn{\boldmu} = \mu\sqrt{p}$, where $\mu = 0.2, 0.3$ and $0.4$. Observe that the total classification error approaches $0$ and the fraction of support vectors approaches $1$ as $p$ gets larger.}
    \label{fig-gmmbenign01sm}
\end{figure}
Finally, we present numerical illustrations validating our benign overfitting results in Corollary~\ref{cor:benigngmm}. In \Figure~\ref{fig-gmmbenign01sm}(a), we set the number of classes $k=4$. To guarantee sufficient overparameterization, we fix $n=40$ and vary $p$ from $50$ to $1200$. 
We simulate $3$ different settings for the mean matrices: each has orthogonal and equal-norm mean vectors $\tn{\boldmu} = \mu\sqrt{p}$, with $\mu = 0.2, 0.3$ and $0.4$. 
\Figure~\ref{fig-gmmbenign01sm} plots the classification error as a function of $p$ for both MNI estimates (solid lines) and multiclass SVM solutions (dashed lines). Different colors correspond to different mean norms. The solid and dashed curves almost overlap as predicted from our results in Section~\ref{sec:support-vectors}. We verify that as $p$ increases, the  classification error decreases towards zero. Observe that the fraction of support vectors approaches $1$ as $p$ gets larger. Further, the classification error goes to zero very fast when $\mu$ is large, but then the proportion of support vectors increases at a slow rate. In contrast, when $\mu$ is small, the proportion of support vectors increases fast, but the classification error decreases slowly. \Figure~\ref{fig-gmmbenign01sm}(b) uses the same setting as in \Figure~\ref{fig-gmmbenign01sm}(a) except for setting $k=6$ and $n=30$. 
Observe that the classification error continues to go to zero and the proportion of support vectors continues to increase, but both become slower as the number of classes is now greater.


\section{Proofs of main results}\label{sec:proofsmaintext}

In this section, we provide the proofs of Theorems~\ref{lem:key},~\ref{thm:svmgmm} and~\ref{thm:classerrormlm}.
The proof techniques we developed for these results convey novel technical ideas that also form the core of the rest of the proofs, which we defer to the Appendix.

\subsection{Proof of \Theorem~\ref{lem:key}}\label{sec:key_thm_proof_sketch}

\noindent\textbf{Argument sketch.}~We split the proof of the theorem in three steps. To better convey the main ideas, we first outline the three steps in this paragraph before discussing their details in the remaining of this section.

\underline{Step 1:} The first key step to  prove \Theorem~\ref{lem:key} is constructing a new parameterization of the dual of the multiclass SVM, which we show takes the following form:
\begin{align}\label{eq:k-svm-dual-beta-sketch}
&\max_{\betab_c\in\R^n, c\in[k]} ~~~~\sum_{c\in[k]} {\betab_c^T\z_c -\frac{1}{2}\|\X\betab_c\|_2^2}\\
&~~~~~\text{sub. to}~~~~~~~~\beta_{y_i,i} = -\sum_{c\neq y_i}\beta_{c,i},~\forall i\in[n] \nn
\quad\text{and}\quad \betab_c\odot\z_c\geq \mathbf{0}, \forall c\in[k].
\end{align}
Here, for each $c\in[k]$ we let $\betab_c=[\beta_{c,1},\beta_{c,2},\ldots,\beta_{c,n}]\in\R^n.$ We also show by complementary slackness the following implication for any \emph{optimal} $\beta^*_{c,i}$ in \eqref{eq:k-svm-dual-beta-sketch}:
\begin{align}
    \label{eq:complementary_slack-sketch}
    z_{c,i}\beta_{c,i}^*>0 ~\implies~ (\wh_{y_i}-\wh_c)^T\x_i = 1.
\end{align}
Thus, to prove \Equation~\eqref{eq:equality_ksvm}, it will suffice showing that $z_{c,i}\beta_{c,i}^*>0, \forall i\in[n], c\in[k]$ provided that \Equation~\eqref{eq:det-con} holds. 
%

\underline{Step 2:}~To do this, we prove that the \emph{unconstrained} maximizer in~\eqref{eq:k-svm-dual-beta-sketch}, that is
$
\hat\betab_c = (\X^T\X)^{+}\z_c,~\forall c\in[k]
$ is feasible, and therefore optimal, in~\eqref{eq:k-svm-dual-beta-sketch}. 
Now, note that \Equation~\eqref{eq:det-con} is equivalent to $\z_c\odot\hat\betab_c>0$; thus, we have found that $\hat\betab_c, c\in[k]$ further satisfies the $n$ \textit{strict} inequality constraints in~\eqref{eq:complementary_slack-sketch}
which completes the proof of the first part of the theorem (\Equation~\eqref{eq:equality_ksvm}).

\underline{Step 3:} Next, we outline the proof of \Equation~\eqref{eq:symmetry-constraint}. 
We consider the simplex-type OvA-classifier in~\eqref{eq:sym-cs-svm-sketch}. 
The proof has two steps. First, using similar arguments to what was done above, we show that when \Equation~\eqref{eq:det-con} holds, then all the inequality constraints in~\eqref{eq:sym-cs-svm-sketch} are active at the optimal. 
That is, the minimizers $\w_{\text{OvA},c}$ of~\eqref{eq:sym-cs-svm-sketch} satisfy \Equation~\eqref{eq:symmetry-constraint}. 
Second, to prove that \Equation~\eqref{eq:symmetry-constraint} is satisfied by the minimizers $\hat\w_c$ of the multiclass SVM in~\eqref{eq:k-svm}, we need to show that $\w_{\text{OvA},c}=\hat\w_c$ for all $c\in[k]$. We do this by showing that, under \Equation~\eqref{eq:det-con}, the duals of~\eqref{eq:k-svm} and~\eqref{eq:sym-cs-svm-sketch} are equivalent. 
By strong duality, the optimal costs of the primal problems are also the same. Then, because a) the objective is the same for the two primals, b)  $\w_{\text{OvA},c}$ is feasible in \eqref{eq:sym-cs-svm-sketch} and c)~\eqref{eq:k-svm} is strongly convex, we can conclude with the desired.

\vspace{5pt}
\noindent\textbf{Step 1: Key alternative parameterization of the dual.}~We start by writing the dual of the multiclass SVM, repeated here for convenience:
\begin{align}\label{eq:k-svm-square}
    \min_\W \frac{1}{2}\norm{\W}_F^2\quad\text{sub. to}~~ (\w_{y_i}-\w_c)^\top\x_i\geq1,~\forall i \in [n], c \in [k] : c \neq y_i .
\end{align}
We have dual variables $\{\la_{c,i}\}$ for every $i \in [n], c \in [k]: c \neq y_i$ corresponding to the constraints on the primal form above.
Then, the dual of the multiclass SVM takes the form
\begin{align}
    \label{eq:k-svm-dual}
    \max_{\la_{c,i}\geq 0}~\sum_{i\in[n]}\Big(\sum_{\substack{c\in[k]\\c\neq y_i}}\la_{c,i}\Big) - \frac{1}{2}\sum_{c\in[k]}\Big\|{
    \sum_{i\in[n]:y_i=c} \Big(\sum_{\substack{c'\in[k]\\c'\neq y_i}}\la_{c',i}\Big)\x_i - \sum_{i\in[n]:y_i\neq c}\la_{c,i}\x_i
    }\Big\|_2^2.
\end{align}
Let $\hat\la_{c,i}, i\in[n], c\in[k]:c \neq y_i$ be maximizers in Equation~\eqref{eq:k-svm-dual}. 
By complementary slackness, we have 
\begin{align}
    \label{eq:complementary_slack}
    \hat\la_{c,i}>0 ~\implies~ (\wh_{y_i}-\wh_c)^\top\x_i = 1.
\end{align}
Thus, it will suffice to prove that $\hat\la_{c,i}>0, \forall i\in[n], c\in[k]: c \neq y_i$ provided that \eqref{eq:det-con} holds. 

It is challenging to work directly with Equation~\eqref{eq:k-svm-dual} because the variables $\la_{c,i}$ are coupled in the objective function.
Our main idea is to re-parameterize the dual objective in terms of new variables $\{\beta_{c,i}\}$, which we define as follows for all $c\in[k]$ and $i\in[n]$:
\begin{align}\label{eq:betas}
    \beta_{c,i} = \begin{cases}
   \sum_{c'\neq y_i}\la_{c',i} & , y_i=c,\\
   -\la_{c,i} & , y_i\neq c.
    \end{cases}
\end{align}
For each $c\in[k]$, we denote $\betab_c=[\beta_{c,1},\beta_{c,2},\ldots,\beta_{c,n}]\in\R^n.$
With these, we show that the dual objective becomes
\begin{align}
    \label{eq:dual_obj_betas}
    \sum_{c\in[k]} \betab_c^\top\z_c - \frac{1}{2}\sum_{c\in[k]}\Big\|\sum_{i\in[n]}\beta_{c,i}\x_i\Big\|_2^2 = \sum_{c\in[k]} {\betab_c^\top\z_c -\frac{1}{2}\|\X\betab_c\|_2^2} .
\end{align}
The equivalence of the quadratic term in $\betab$ is straightforward.
To show the equivalence of the linear term in $\betab$, we denote $A:=\sum_{i\in[n]}\Big(\sum_{\substack{c\in[k],c\neq y_i}}\la_{c,i}\Big)$, and simultaneously get
\begin{align*}
    A = \sum_{i \in [n]} \beta_{y_i,i} \qquad\text{and}\qquad  A = \sum_{i \in [n]} \sum_{c \neq y_i} (-\beta_{c,i}),
\end{align*}
by the definition of variables $\{\beta_{c,i}\}$ in Equation~\eqref{eq:betas}.
Then, we have
\begin{align*}
A = \frac{k-1}{k} \cdot A + \frac{1}{k} \cdot A &= \frac{k-1}{k} \sum_{i\in[n]}\beta_{y_i,i} +\frac{1}{k}\sum_{i\in[n]}\sum_{c\neq y_i} (-\beta_{c,i})\\
&\stackrel{\mathsf{(i)}}{=} \sum_{i\in[n]}\z_{y_i,i}\beta_{y_i,i} +\sum_{i\in[n]}\sum_{c\neq y_i} \z_{c,i}\beta_{c,i}\\
&= 
\sum_{i\in[n]}\sum_{c\in[k]} \z_{c,i}\beta_{c,i} = \sum_{c\in[k]}\betab_c^\top\z_c.
\end{align*}
Above, inequality $(\mathsf{i})$ follows from the definition of $\z_c$ in Equation~\eqref{eq:zc}, rewritten coordinate-wise as:
\begin{align*}
z_{c,i} = \begin{cases}
\frac{k-1}{k} ,& y_i=c,\\
-\frac{1}{k} ,& y_i\neq c.
\end{cases}
\end{align*}
Thus, we have shown that the objective of the dual can be rewritten in terms of variables $\{\beta_{c,i}\}$.
After rewriting the constraints in terms of $\{\beta_{c,i}\}$, we have shown that the dual of the SVM (Equation~\eqref{eq:k-svm}) can be equivalently written as in Equation \eqref{eq:k-svm-dual-beta-sketch}.
%
Note that the first constraint in \eqref{eq:k-svm-dual-beta-sketch} ensures consistency with the definition of $\betab_c$ in \Equation~\eqref{eq:betas}. 
The second constraint guarantees the non-negativity constraint of the original dual variables in \eqref{eq:k-svm-dual}, because we have
\begin{align*}
    \beta_{c,i} z_{c,i} = \frac{\lambda_{c,i}}{k} \text{ for all } i \in [n], c \in [k]: c \neq y_i .
\end{align*}
Consequently, we have
\begin{align}
    \beta_{c,i}z_{c,i} \geq 0 ~~\Longleftrightarrow~~ \lambda_{c,i}\geq 0 \label{eq:beta-la}
\end{align}
for all $c\in[k]$ and $i\in[n]:y_i\neq c$.
In fact, the equivalence above also holds with the inequalities replaced by strict inequalities. 
Also note that the second constraint for $c = y_i$ yields $\frac{k-1}{k} \sum_{c' \neq y_i} \lambda_{c',i} \geq 0$, which is automatically satisfied when Equation~\eqref{eq:beta-la} is satisfied.
Thus, these constraints are redundant.

\vspace{5pt}
\noindent\textbf{Step 2: Proof of Equation \eqref{eq:equality_ksvm}.}~Define
$$
\hat\betab_c := (\X^\top\X)^{+}\z_c,~\forall c\in[k].
$$
This specifies an \emph{unconstrained} maximizer in \eqref{eq:k-svm-dual-beta-sketch}. 
We will show that this unconstrained maximizer $\hat\betab_c, c \in [k]$ is feasible in the constrained program in \eqref{eq:k-svm-dual-beta-sketch}. 
Thus, it is in fact an optimal solution in \eqref{eq:k-svm-dual-beta-sketch}.

To prove this, we will first prove that $\hat\betab_c, c\in[k]$ satisfies the $n$ equality constraints in \eqref{eq:k-svm-dual-beta-sketch}.
For convenience, let $\g_i\in\R^n, i\in[n]$ denote the $i$-th row of $(\X^\top\X)^{+}$. Then, for the $i$-th element $\hat\beta_{c,i}$ of  $\hat\betab_{c}$, it holds that $\hat\beta_{c,i}=\g_i^\top\z_c$. Thus, for all $i\in[n]$, we have
\begin{align*}
    \hat\beta_{y_i,i}+\sum_{c\neq y_i}\hat\beta_{c,i} = \g_i^\top\Big( \z_{y_i} + \sum_{c\neq y_i}\z_c \Big) = \g_i^\top\Big( \sum_{c\in[k]}\z_c \Big) = 0,
\end{align*}
where in the last equality we used the definition of $\z_c$ in \eqref{eq:zc} and the fact that $\sum_{c\in[k]}\vb_c = \ones_n$, since each column of the label matrix $\Y$ has exactly one non-zero element equal to $1$.
Second, since Equation~\eqref{eq:det-con} holds, $\hat\betab_c, c\in[k]$ further satisfies the $n$ \textit{strict} inequality constraints in \eqref{eq:k-svm-dual-beta-sketch}.

We have shown that the unconstrained maximizer is feasible in the constrained program \eqref{eq:k-svm-dual-beta-sketch}. Thus, we can conclude that it is also a global solution to the latter. 
By Equation~\eqref{eq:beta-la}, we note that the \textcolor{black}{corresponding} original dual variables $\{\hat\lambda_{c,i}\}=\{k\,\hat\beta_{c,i}z_{c,i}\}$ are all strictly positive. \textcolor{black}{Now recall that under strong duality, any pair of primal-dual optimal solutions satisfies the KKT conditions. Hence the primal-dual pair $\big(\{\hat\w_c\},\{\hat\lambda_{c,i}\}\big)$ satisfies the complementary slackness condition of  Equation \eqref{eq:complementary_slack}. 
This together with the positivity of $\{\hat\lambda_{c,i}\}$} complete the proof of the first part of the theorem, i.e. the proof of  Equation~\eqref{eq:equality_ksvm}.

\vspace{5pt}
\noindent\textbf{Step 3: Proof of Equation \eqref{eq:symmetry-constraint}.}~To prove Equation \eqref{eq:symmetry-constraint}, consider the following OvA-type classifier: for all $c\in[k]$,
\begin{align}\label{eq:sym-cs-svm}
\min_{\w_c}~\frac{1}{2}\|\w_c\|_2^2\qquad\text{sub. to}~~~ \x_i^\top\w_c\begin{cases}\geq \frac{k-1}{k} ,& y_i=c,\\
\leq -\frac{1}{k} ,& y_i\neq c,\end{cases}    ~~~\forall i\in[n].
\end{align}
To see the connection with Equation  \eqref{eq:symmetry-constraint}, note the condition for the constraints in \eqref{eq:sym-cs-svm} to be active is exactly Equation  \eqref{eq:symmetry-constraint}. Thus, it suffices to prove that the constraints of \eqref{eq:sym-cs-svm} are active under the theorem's assumptions. We work again with the dual of \eqref{eq:sym-cs-svm}:
\begin{align}
    \label{eq:sym-cs-svm-dual}
    \max_{\nub_c\in\R^k}~~ -\frac{1}{2}\|\X\nub_c\|_2^2 + \z_c^\top\nub_c\qquad\text{sub. to}~~~ \z_c\odot\nub_c\geq \mathbf{0}.
\end{align}
Again by complementary slackness, the desired Equation  \eqref{eq:symmetry-constraint} holds provided that all dual constraints in \eqref{eq:sym-cs-svm-dual} are strict at the optimal.

We now observe two critical similarities between \eqref{eq:sym-cs-svm-dual} and \eqref{eq:k-svm-dual-beta-sketch}: (i) the two dual problems have the same objectives (indeed the objective in \eqref{eq:k-svm-dual-beta-sketch} is separable over $c\in[k]$); (ii) they share the constraint $\z_c\odot \nub_c\geq\mathbf{0}$~\big/~$\z_c\odot \betab_c\geq\mathbf{0}$. 
From this observation, we can use the same argument as for \eqref{eq:k-svm-dual-beta-sketch} to show that when \Equation~\eqref{eq:det-con} holds, $\hat\betab_c$ is optimal in \eqref{eq:sym-cs-svm-dual}.

Now, let $\rm{OPT}_{\eqref{eq:k-svm-square}}$ and $\rm{OPT}^c_{\eqref{eq:sym-cs-svm}}$ be the optimal costs of the multiclass SVM in \eqref{eq:k-svm-square} and of the simplex-type OvA-SVM in \eqref{eq:sym-cs-svm} parameterized by $c\in[k]$. Also, denote $\rm{OPT}_{\eqref{eq:k-svm-dual-beta-sketch}}$ and $\rm{OPT}^c_{\eqref{eq:sym-cs-svm-dual}}, c\in[k]$ the optimal costs of their respective duals in \eqref{eq:k-svm-dual-beta-sketch} and \eqref{eq:sym-cs-svm-dual}, respectively. We proved above that
\begin{align}\label{eq:duals_equality}
\rm{OPT}_{\eqref{eq:k-svm-dual-beta-sketch}} = \sum_{c\in[k]}\rm{OPT}^c_{\eqref{eq:sym-cs-svm-dual}}.
\end{align}
Further let 
$\W_{\text{OvA}}=[\w_{\text{OvA},1},\ldots,\w_{\text{OvA},k}]$ 
be the optimal solution in the simplex-type OvA-SVM in \eqref{eq:sym-cs-svm-dual}. We have proved that under Equation \eqref{eq:det-con} $\w_{\text{OvA},c}$ satisfies the constraints in \eqref{eq:sym-cs-svm} with equality, that is
$\X^\top\w_{\text{OvA},c}=\z_c,~\forall c\in[k]$.
Thus, it suffices to prove that $\W_{\text{OvA}}=\Wsvm$.
%
By strong duality (which holds trivially for \eqref{eq:sym-cs-svm} by Slater's conditions), we get
\begin{align}
\rm{OPT}^c_{\eqref{eq:sym-cs-svm}}=\rm{OPT}^c_{\eqref{eq:sym-cs-svm-dual}},~c\in[k] &\implies 
\sum_{c\in[k]}\rm{OPT}^c_{\eqref{eq:sym-cs-svm}}=\sum_{c\in[k]}\rm{OPT}^c_{\eqref{eq:sym-cs-svm-dual}}\nn
\\
&\stackrel{\eqref{eq:duals_equality}}{\implies}
\sum_{c\in[k]}\rm{OPT}^c_{\eqref{eq:sym-cs-svm}}=\rm{OPT}_{\eqref{eq:k-svm-dual-beta-sketch}} \nn\\
&\stackrel{\eqref{eq:sym-cs-svm}}{\implies}
\sum_{c\in[k]}\frac{1}{2}\|\w_{\text{OvA},c}\|_2^2=\rm{OPT}_{\eqref{eq:k-svm-dual-beta-sketch}}.
\end{align}
Again, by strong duality we get $\rm{OPT}_{\eqref{eq:k-svm-dual-beta-sketch}}=\rm{OPT}_{\eqref{eq:k-svm-square}}.$ Thus, we have
$$
\sum_{c\in[k]}\frac{1}{2}\|\w_{\text{OvA},c}\|_2^2 = \rm{OPT}_{\eqref{eq:k-svm-square}}.
$$
Note also that $\W_{\text{OvA}}$ is feasible in \eqref{eq:k-svm-square} since $$\X^\top\w_{\text{OvA},c}=\z_c,~\forall c\in[k] \implies (\w_{\text{OvA},y_i}-\w_{\text{OvA},c})^\top\x_i= 1,~\forall {c\neq y_i, c\in[k]}, \text{ and } \forall i\in[n].
$$
Therefore, $\W_{\text{OvA}}$ is optimal in \eqref{eq:k-svm-square}.
Finally, note that the optimization objective in \eqref{eq:k-svm-square} is strongly convex. Thus, it has a unique minimum and therefore $\Wsvm=\W_{\text{OvA}}$ as desired.

\subsection{Proof of \Theorem~\ref{thm:svmgmm}}
\label{sec:pforderoneoutline}


In this section, we provide the proof of Theorem~\ref{thm:svmgmm}. 
First, we remind the reader of the prescribed approach outlined in Section \ref{sec-linkgmm} and introduce some necessary notation. 
Second, we present the key Lemma~\ref{lem:ineqforanoj}, which forms the backbone of our proof. 
The proof of the lemma is rather technical and is deferred to Appendix~\ref{sec:aux_lemmas_svm_GMM} along with a series of auxiliary lemmas. 
Finally, we end this section by showing how to prove Theorem~\ref{thm:svmgmm} using  Lemma~\ref{lem:ineqforanoj}.

\vp
\noindent\textbf{Argument sketch and notation.}~We begin by presenting high-level ideas and defining notation that is specific to this proof. For  $c\in[k]$, we define $$\bolda_c := (\mathbf{Q}+\sum_{j=1}^c\boldmu_j\vb_j^T)^T(\mathbf{Q}+\sum_{j=1}^c\boldmu_j\vb_j^T).$$ 
Recall that in the above, $\boldmu_j$ denotes the $j^{th}$ class mean of dimension $p$, and $\vb_j$ denotes the $n$-dimensional indicator that each training example is labeled as class $j$.
Further, recall from \Equation~\eqref{def-gmm} that the feature matrix can be expressed as $\X = \M\Y + \Q$, where $\Q \in \mathbb{R}^{p \times n}$ is a standard Gaussian matrix. Thus, we have 
\begin{align*}
\X^T\X = \bolda_k\qquad\text{and}\qquad \mathbf{Q}^T\mathbf{Q} = \bolda_0.
\end{align*}
As discussed in Section \ref{sec-linkgmm}, our goal is to show that the inverse Gram matrix $\A_k^{-1}$ is ``close'' to a positive definite diagonal matrix.
Indeed, in our new notation, the desired inequality in \Equation~\eqref{eq:det-con} becomes
\begin{align}
    z_{ci}\mathbf{e}_i^T\bolda_k^{-1}\boldz_c >0, \ \ \text{for \ all} \ \ c\in[k] ~~ \text{and} ~ i \in[n]. \label{eq:2show_ineq}
\end{align}
The major challenge in showing inequality \eqref{eq:2show_ineq} is that $\A_k=(\mathbf{Q}+\sum_{j=1}^k\boldmu_j\vb_j^T)^T(\mathbf{Q}+\sum_{j=1}^k\boldmu_j\vb_j^T)$ involves multiple mean components through the sum $\sum_{j=1}^c\boldmu_j\vb_j^T$. This makes it challenging to bound  quadratic forms involving the Gram matrix $\bolda_k^{-1}$ directly. Instead, our idea is to work recursively starting from bounding quadratic forms involving $\bolda_0^{-1}$.
Specifically, we denote $\mathbf{P}_1 = \Q + \boldmu_1\vb_1^T$ and derive the following recursion on the $\A_0,\A_1,\ldots,\A_k$ matrices:
\begin{align}
    \bolda_1 &= \mathbf{P}_1^T\mathbf{P}_1 = \bolda_0 + \begin{bmatrix}\tn{\boldmu_1}\vb_1& \Q^T\boldmu_1& \vb_1
    \end{bmatrix}\begin{bmatrix}
    \tn{\boldmu_1} \vb_1^T\\
    \vb_1^T \\
    \boldmu_1^T\Q
    \end{bmatrix}, \notag\\
    \bolda_2 &= (\mathbf{P}_1 + \boldmu_2 \vb_2^T)^T(\mathbf{P}_1 + \boldmu_2 \vb_2^T) = \bolda_1 + \begin{bmatrix}\tn{\boldmu_2}\vb_2& \mathbf{P}_1^T\boldmu_2& \vb_2
    \end{bmatrix}\begin{bmatrix}
    \tn{\boldmu_2} \vb_2^T\\
    \vb_2^T \\
    \boldmu_2^T\mathbf{P}_1
    \end{bmatrix}, \label{eq:addnewmean}
\end{align}
and so on, until $\A_k$ (see Appendix~\ref{sec:pfinversesecc} for the complete expressions for the recursion). Using this trick, we can exploit bounds on quadratic  forms involving $\bolda_0^{-1}$ to obtain bounds for quadratic forms involving $\bolda_1^{-1}$, and so on until $\bolda_k^{-1}$. {Note that because of the nearly equal-energy Assumption~\ref{ass:equalmuprob}, the order of adding mean vectors in $\bolda$ will not change the results. In other words, including $\boldmu_1$ first, then $\boldmu_2$, in~\eqref{eq:addnewmean} will produce the same result as including $\boldmu_2$ first, then $\boldmu_1$ in the same equation.}

There are two key ideas behind this approach. 
First, we will show how to use a leave-one-out argument and the Matrix Inversion Lemma to express (recursively) the quadratic form $\mathbf{e}_i^T\bolda_k^{-1}\boldz_c$ in \eqref{eq:2show_ineq} in terms of simpler quadratic forms, which are more accessible to bound directly. 
For later reference, we define these auxiliary forms here.
Let $\mathbf{d}_c := \mathbf{Q}^T\boldmu_c$, for $c \in [k]$ and define the following quadratic forms involving $\bolda_c^{-1}$ for $c, j, m \in [k]$ and $i \in [n]$:
\begin{align}
s_{mj}^{(c)} &:=\vb_m^T \bolda_c^{-1}\vb_j, \notag\\ 
t_{mj}^{(c)} &:=\mathbf{d}_{m}^T \bolda_c^{-1}\mathbf{d}_j,\notag\\ 
h_{mj}^{(c)} &:=\mathbf{v}_m^T \bolda_c^{-1}\mathbf{d}_j,\label{eq:pfquadforc}\\ 
g_{ji}^{(c)} &:=\mathbf{v}_j^T \bolda_c^{-1}\mathbf{e}_i,\notag\\  
f_{ji}^{(c)} &:=\mathbf{d}_j^T\bolda_c^{-1}\mathbf{e}_i.\notag
\end{align}
For convenience, we refer to terms above as \emph{quadratic forms of order $c$} or \emph{the $c$-th order quadratic forms}, where $c$ indicates the corresponding superscript. 
A complementary useful observation facilitating our approach is the observation that the class label indicators are orthogonal by definition, i.e. $\vb_i^T\vb_j = 0$, for $i, j \in [k]$. 
(This is a consequence of the fact that any training data point has a unique label and we are using here one-hot encoding.)
Thus, the newly added mean component $\boldmu_{c+1}\vb_{c+1}^T$ is orthogonal to the already existing mean components included in the matrix $\bolda_c$ (see Equation \eqref{eq:addnewmean}). 
Consequently, we will see that adding new mean components will only slightly change the magnitude of these these quadratic forms as $c$ ranges from $0$ to $k$.


\vp
\noindent\textbf{Identifying and bounding quadratic forms of high orders.}~Recall the desired inequality \eqref{eq:2show_ineq}. 
We can equivalently write the definition of $\boldz_c$ in \Equation~\eqref{eq:zc} as
\begin{align}
\label{eq:defuvector}
    \boldz_c = \frac{k-1}{k}\vb_c+\sum_{j \ne c}\left(-\frac{1}{k}\right)\vb_j =\tilde{z}_{c(c)}\vb_c + \sum_{j \ne c}\tilde{z}_{j(c)}\vb_j, 
\end{align}
where we denote
\begin{align*}
    \tilde{z}_{j(c)} =\begin{cases}
    -\frac{1}{k}, ~~\text{if}~~j \ne c\\
    \frac{k-1}{k}, ~~\text{if}~~j = c
    \end{cases}.
\end{align*}
Note that by this definition, we have $\tilde{z}_{y_i(c)} := z_{ci}$.
This gives us
\begin{align}
    z_{ci}\mathbf{e}_i^T\bolda_k^{-1}\boldz_c &= z_{ci}^2\mathbf{e}_i^T\bolda_k^{-1}\vb_{y_i}+\sum_{j \ne y_i} z_{ci}\tilde{z}_{j(c)}\mathbf{e}_i^T\bolda_k^{-1}\vb_j ,\notag\\
    &=z_{ci}^2g_{y_ii}^{(k)} + \sum_{j \ne y_i}z_{ci}\tilde{z}_{j(c)}g_{ji}^{(k)}.\label{eq:zigi01}
\end{align}

Note that this expression (Equation~\eqref{eq:zigi01}) involves the $k$-th order quadratic forms $g_{ji}^{(k)}=\mathbf{e}_i^T\bolda_k^{-1}\vb_j$. For each such form, we use the matrix inversion lemma to leave the $j$-th mean component in $\bolda_k$ out and express it in terms of the \textit{leave-one-out} versions of quadratic forms that we defined in~\eqref{eq:pfquadforc}, as below (see Appendix~\ref{sec:pfinversesecc} for a detailed derivation):
\begin{align}
\label{eq:quadlast}
   g_{ji}^{(k)}=\mathbf{e}_i^T\bolda_k^{-1}\vb_j = \frac{(1+h_{jj}^{(-j)})g_{ji}^{(-j) }-s_{jj}^{(-j)}f_{ji}^{(-j)}}{s_{jj}^{(-j)}(\tn{\boldmu_j}^2 - t_{jj}^{(-j)})+(1+h_{jj}^{(-j)})^2}\,.
\end{align}
Specifically, above we defined $s_{jj}^{(-j)} := \vb_j^T\bolda^{-1}_{-j}\vb_j$, where $\bolda_{-j}$ denotes the version of the Gram matrix $\bolda_k$ with the $j$-th mean component left out.
The quadratic forms $h_{jj}^{(-j)}$, $f_{ji}^{(-j)}$, $g_{ji}^{(-j)}$ and $t_{jj}^{(-j)}$ are defined similarly in view of Equation \eqref{eq:pfquadforc}.

Specifically, to see how these ``leave-one-out'' quadratic forms relate directly to the forms in Equation \eqref{eq:pfquadforc}, note that it suffices in \eqref{eq:quadlast} to consider the case where $j=k$. Indeed, observe that when $j \ne k$ we can simply change the order of adding mean components, described in \Equation~\eqref{eq:addnewmean}, so that the $j$-th mean component is added last. On the other hand, when $j=k$ the leave-one-out quadratic terms in \eqref{eq:quadlast} involve the Gram matrix $\A_{k-1}$. Thus, they are equal to the quadratic forms of order $k-1$, given by $s_{kk}^{(k-1)}, t_{kk}^{(k-1)}$, $h_{kk}^{(k-1)}$, $g_{ki}^{(k-1)}$ and $f_{ki}^{(k-1)}$.

The following technical lemma bounds all of these quantities and its use is essential in the proof of Theorem \ref{thm:svmgmm}. Its proof, which is deferred to Appendix \ref{sec:app_proof_Thm2}, relies on the recursive argument outlined above: We start from the quadratic forms of order $0$ building up all the way to the quadratic forms of order $k-1$.

\begin{lemma}[Quadratic forms of high orders]
\label{lem:ineqforanoj}
{Let Assumption~\ref{ass:equalmuprob} hold} and further assume that $p > Ck^3n\log (kn) + n -1$ for large enough constant $C>1$ and large $n$. There exist constants $c_i$'s and $C_i$'s $>1$ such that the following bounds hold for every $i \in [n]$ and $j \in [k]$ with probability at least $1-\frac{c_1}{n}-c_2ke^{-\frac{n}{c_3 k^2}}$,
\begin{align*}
    \frac{C_1-1}{C_1}\cdot\frac{n}{kp} \le & s_{jj}^{(-j)} \le \frac{C_1+1}{C_1}\cdot\frac{n}{kp},\\
     t_{jj}^{(-j)} \le & \frac{C_2n\tn{\boldmu}^2}{p},\\
    - \minrho\frac{C_3 n\tn{\boldmu}}{\sqrt{k}p}  \le & h_{jj}^{(-j)} \le  \minrho\frac{C_3n\tn{\boldmu}}{\sqrt{k}p},\\
    |f_{ji}^{(-j)}|  \le & \frac{C_4\sqrt{\logtwon}\tn{\boldmu}}{p},\\
     g_{ji}^{(-j)} \ge & \lp(1-\frac{1}{C_5}\rp)\frac{1}{p}, ~~\text{for}~ j = y_i,\\
     |g_{ji}^{(-j)}| \le & \frac{1}{C_6k^2p},~~\text{for}~ j\ne y_i,
\end{align*}
where $\tilde{\rho}_{n,k} = \min\{1, \sqrt{{\log(2n)}/{k}}\}.$ Observe that the bounds stated in the lemma hold for any $j\in[k]$ and the bounds themselves are independent of $j$.
\end{lemma}

\vp\noindent\textbf{Completing the proof of Theorem \ref{thm:svmgmm}.}~We now show how to use Lemma~\ref{lem:ineqforanoj} to complete the proof of the theorem. 
Following the second condition in the statement of Theorem~\ref{thm:svmgmm}, we define
\begin{align}
    \epsilonn := \frac{k^{1.5}n\sqrt{\logtwon}\tn{\boldmu}}{p} \le \tau,\label{eq:eps_proof}
\end{align}
{where $\tau$ is a sufficiently small positive constant, the value of which will be specified later in the proof.}
First, we will show that the denominator of~\Equation~\eqref{eq:quadlast} is strictly positive on the event where Lemma~\ref{lem:ineqforanoj} holds.
We define $${\det}_{-j} := s_{jj}^{(-j)}(\tn{\boldmu_j}^2 - t_{jj}^{(-j)})+(1+h_{jj}^{(-j)})^2.$$ By Lemma~\ref{lem:ineqforanoj}, the quadratic forms $s_{jj}^{(-j)}$ {are of the same order $\Theta\lp(\frac{n}{kp}\rp)$} for every $j \in [k]$. Similarly, we have $t_{jj}^{(-j)}=\Oc\lp(\frac{n}{p}\tn{\boldmu}^2\rp)$ and $|h_{jj}^{(-j)}| = \minrho\Oc\lp(\frac{\epsilonn}{k^2\sqrt{\logtwon}}\rp)$ for $j \in [k]$. Thus, we have
\begin{align}
\label{eq:bounddet}
    \frac{n\tn{\boldmu}^2}{C_1kp}\lp(1-\frac{C_2n}{p}\rp)+\lp(1-\frac{C_3\epsilonn}{k^2\sqrt{\logtwon}}\rp)^2 \le \text{det}_{-j} \le \frac{C_1n\tn{\boldmu}^2}{kp}+\lp(1+\frac{C_3\epsilonn}{k^2\sqrt{\logtwon}}\rp)^2,
\end{align}
with probability at least $1-\frac{c_1}{n}-c_2ke^{-\frac{n}{c_3 k^2}}$, for every $j \in [k]$.
Here, we use the fact that $t_{jj}^{-j} \ge 0$ by the positive semidefinite property of the leave-one-out Gram matrix $\bolda_{-j}^{-1}$. 
Next, we choose $\tau$ in \Equation~\eqref{eq:eps_proof} to be sufficiently small so that $C_3\tau\leq 1/2$. 
Provided that $p$ is sufficiently large compared to $n$, there then exist constants $C_1', C_2' >0$ such that we have 
\begin{align*}
    C_1' \le \frac{\det_{-m}}{\det_{-j}} \le C_2',~~\text{for all}~j,m\in[k],
\end{align*}
with {probability at least $1-\frac{c_1}{n}-c_2ke^{-\frac{n}{c_3 k^2}}$}.
Now, assume {without loss of generality} that $y_i = k$. \Equation~\eqref{eq:bounddet} {shows that there exists constant $c>0$ such that $\det_{-j} > c$} for all $j \in [k]$ with high probability provided that $p/n$ is large enough {(guaranteed by the first condition of the theorem)}. Hence, to make the right-hand-side of \Equation~\eqref{eq:zigi01} positive, it suffices to show that the numerator will be positive.
Accordingly, we will show that 
\begin{align}
\label{eq:numeratorsum}
    z_{ci}^2\big((1+h_{kk}^{(-k)})g_{ki}^{(-k) }-s_{kk}^{(-k)}f_{ki}^{(-k)}\big) + Cz_{ci}\sum_{j \ne k}\tilde{z}_j\big((1+h_{jj}^{(-j)})g_{ji}^{(-j) }-s_{jj}^{(-j)}f_{ji}^{(-j)}\big) > 0,
\end{align}
for some $C > 1$. 

We can show by simple algebra that it suffices to consider the worst case of $z_{ci} = -1/k$. 
To see why this is true, we consider the simpler term $z_{ci}^2g_{y_ii}^{(-y_i)} - |\sum_{j \ne y_i}z_{ci}\tilde{z}_{j(c)}g_{ji}^{(-j)}|$.
Clearly, \Equation~\eqref{eq:numeratorsum} is positive only if the above quantity is also positive.
Lemma~\ref{lem:ineqforanoj} shows that {when $z_{ci} = -1/k$, then $z_{ci}^2g_{y_ii}^{(-y_i)} \ge \lp(1-\frac{1}{C_1}\rp)\frac{1}{k^2p}$ and $|z_{ci}\tilde{z}_{j(c)}g_{ji}^{(-j)}| \le \frac{1}{C_2k^3p}$, for $j \ne y_i$.}
Hence 
$$z_{ci}^2g_{y_ii}^{(-y_i)} - |\sum_{j \ne y_i}z_{ci}\tilde{z}_{j(c)}g_{ji}^{(-j)}| \ge \lp(1-\frac{1}{C_3}\rp)\frac{1}{k^2p}.$$
Here, $z_{ci} = -1/k$ minimizes the lower bound $z_{ci}^2g_{y_ii}^{(-y_i)} - |\sum_{j \ne y_i}z_{ci}\tilde{z}_{j(c)}g_{ji}^{(-j)}|$. To see this, we first drop the positive common factor $|z_{ci}|$ in the equation above and get $|z_{ci}|g_{y_ii}^{(-y_i)} - |\sum_{j \ne y_i}\tilde{z}_{j(c)}g_{ji}^{(-j)}|$. If we had $z_{ci} = -1/k$, then $|\tilde{z}_{j(c)}|$ is either $(k-1)/k$ or $1/k$. In contrast, if we consider $z_{ci} = (k-1)/k$, then we have $|\tilde{z}_{j(c)}| = 1/k$ for all $j \neq y_i$ and so the term $|z_{ci}|g_{y_ii}^{(-y_i)} - |\sum_{j \ne y_i}\tilde{z}_{j(c)}g_{ji}^{(-j)}|$ is strictly larger.

Using this worst case, i.e. $z_{ci} = -1/k$, and the trivial inequality $|\tilde{z}_{j(c)}| < 1$ for $j \ne y_i$ together with the bounds for the terms $s_{jj}^{(-j)}, t_{jj}^{(-j)}, h_{jj}^{(-j)}$ and $f_{ji}^{(-j)}$ derived in Lemma~\ref{lem:ineqforanoj} gives us
\begin{align}
   ~\eqref{eq:numeratorsum} &\ge \frac{1}{k^2}\lp(\lp(1-\frac{C_1\epsilonn}{k^2\sqrt{\logtwon}}\rp)\lp(1-\frac{1}{C_2}\rp)\frac{1}{p} - \frac{C_3\epsilonn}{k^{1.5}n}\cdot\frac{n}{kp}\rp) - k \cdot \frac{1}{C_4k}\lp(\lp(1+\frac{C_5\epsilonn}{k^2\sqrt{\logtwon}}\rp)\frac{1}{k^2p}-\frac{C_6\epsilonn}{{k}^{1.5}n}\frac{n}{kp}\rp) \notag\\
    &\ge \frac{1}{k^2}\lp(1-{\frac{1}{C_9}}-\frac{C_{10}\epsilonn}{k^2\sqrt{\logtwon}} - \frac{C_{11}\epsilonn}{k^{2}} - C_{12}\epsilonn\rp)\frac{1}{p} \nn\\
    &\geq \frac{1}{k^2p}\lp(1-\frac{1}{C_9}-C_{10}\tau\rp), \label{eq:gmmsumlower}
\end{align}
with probability at least $1-\frac{c_1}{n}-c_2ke^{-\frac{n}{c_3 k^2}}$ for some constants $C_i$'s $> 1$.
Above, we recalled the definition of $\epsilonn$ and used from Lemma~\ref{lem:ineqforanoj} that $h_{jj}^{(-j)} \le \minrho\frac{C_{11}\epsilonn}{k^2\sqrt{\logtwon}}$ and $|f_{ji}^{(-j)}|\le  \frac{C_{12}\epsilonn}{k^{1.5}n}$ with high probability. 
To complete the proof, we choose $\tau$ to be a small enough constant to guarantee $C_{10}\tau<1-1/C_9$, and substitute this in \Equation~\eqref{eq:gmmsumlower} to get the desired condition of \Equation~\eqref{eq:numeratorsum}.

\subsection{Proof of Theorem~\ref{thm:classerrormlm}}\label{sec:proofclasserrormlm}


\noindent\textbf{Challenges and notation}.~We begin by highlighting the two main non-trivialities introduced in the analysis of the multiclass setting. We compare them to the binary-error analysis in \cite{muthukumar2021classification} and we sketch our approach to each one of them:

\begin{itemize}
\item \emph{The multitude of signal vectors:} The generative model for the MLM involves $k$ distinct (high-dimensional) signal vectors $\boldmu_1,\ldots,\boldmu_k$, and the classification error is a complicated functional of all $k$ \emph{recovered} signal vectors (denoted by $\hatw_1,\ldots,\hatw_k$ respectively).
This functional has to be dealt with carefully compared to the binary case, where there is only one signal vector.
In particular, direct plug-ins of the \emph{survival signal} and \emph{contamination factor} for each recovered signal vector (here, we follow the terminology in~\cite{muthukumar2020harmless}) do not provide sufficiently sharp expressions of the multiclass classification error to predict separation between classification-consistency and regression-consistency.
We circumvent this issue by directly analyzing survival and contamination factors of the pairwise \emph{difference signal} between two classes, and showing in Lemmas~\ref{lem:sumulticlass} and~\ref{lem:cnmulticlass} that they scale very similarly to the single-signal case.
{\color{black} We note that while the survival and contamination factors of this pairwise difference signal scale identically to the single-signal case, the proofs do not follow as a corollary of the corresponding lemmas in~\cite{muthukumar2021classification}; in particular, the difference of label vectors turns out to depend not only on a single ``difference" feature but all the top $k$ features. This requires a much more complex leave-$k$-out analysis, as opposed to the simpler leave-one-out analysis carried out in~\cite{bartlett2020benign,muthukumar2021classification}.} 
\item \emph{Covariate-dependent label noise in the MLM:} The error analysis provided in~\cite{muthukumar2021classification} critically leverages that the cross-correlation between the logit and the binary label of a training example is lower bounded by a universal positive constant.
This is relatively straightforward to show when the relationship between the logit and the label is one of \emph{constant-label-noise} where the event of label error is independent of the covariate.
On the other hand, the MLM involves label errors that are highly depend on the covariate, and these cross-correlation terms need to be handled much more carefully.
We provide an elegant argument based on Stein's lemma to handle the more complex MLM-induced label noise.
\end{itemize}
%
Before proceeding we set up some notation for important quantities in the analysis.
Note that Assumption~\ref{as:mlmfull} directly implies that $\mu_{c,j_c} = 1$ for all $c \in [k]$.
For any two classes $c_1 \neq c_2$, we define the true \emph{difference signal vector} as 
\begin{align*}
\Deltabold_{c_1,c_2} := \boldmu_{c_1} - \boldmu_{c_2} = \mu_{c_1,j_{c_1}} \ehat_{j_{c_1}} - \mu_{c_2,j_{c_2}} \ehat_{j_{c_2}},
\end{align*}
where the last step follows from Assumption~\ref{as:mlmfull}.
Correspondingly, the recovered difference signal vector is defined as $\Deltahat_{c_1,c_2} := \hatw_{c_1} - \hatw_{c_2}$.

\vp
\noindent\textbf{Identifying the survival and contamination terms.}~We state and prove our main lemma that characterizes the classification error in MLM as a function of effective survival and contamination terms.

\begin{lemma}\label{lem:sucmulticlass}
The excess classification risk is bounded by
\begin{align}\label{eq:mlmupperbound}
    \Pro_e - \Pro_{e,\mathsf{Bayes}} \leq \sum_{c_1 < c_2}  \left( \frac{1}{2} - \frac{1}{\pi} \mathsf{tan}^{-1} \left(\frac{\SU(\Deltahat_{c_1,c_2},\Deltabold_{c_1,c_2})}{\CN(\Deltahat_{c_1,c_2},\Deltabold_{c_1,c_2})}\right) \right),
\end{align}
where we define for any two classes $c_1\neq c_2\in[k]$:
\begin{align*}
\SU(\Deltahat_{c_1,c_2},\Deltabold_{c_1,c_2}) &:= \frac{\Deltahat^\top{_{c_1,c_2}} \Sigmab \Deltabold{_{c_1,c_2}}}{\tn{\Sigmab^{1/2}\Deltabold{_{c_1,c_2}}}} \text{ and}  \nn\\
    \CN(\Deltahat_{c_1,c_2},\Deltabold_{c_1,c_2}) &:= \sqrt{\left(\Deltahat{_{c_1,c_2}} - \frac{\Deltahat_{c_1,c_2}^\top \Sigmab \Deltabold_{c_1,c_2}}{\tn{\Sigmab^{1/2} \Deltabold_{c_1,c_2}}^2} \Deltabold_{c_1,c_2} \right)^\top \Sigmab \left(\Deltahat_{c_1,c_2} - \frac{\Deltahat_{c_1,c_2}^\top \Sigmab \Deltabold_{c_1,c_2}}{\tn{\Sigmab^{1/2}\Deltabold_{c_1,c_2}}^2} \Deltabold_{c_1,c_2}\right)} \nn.
\end{align*}
\end{lemma}

\begin{proof}
We consider a fixed $\x$, and (following the notation in~\cite{thrampoulidis2020theoretical}) the $k$-dimensional vectors
\begin{align*}
    \mathbf{g} &:= \begin{bmatrix} \x^\top \hatw_1 & \x^\top \hatw_2 & \ldots & \x^\top \hatw_k \end{bmatrix} \\
    \mathbf{h} &:= \begin{bmatrix} \x^\top \boldmu_1 & \x^\top \boldmu_2 & \ldots & \x^\top \boldmu_k \end{bmatrix}
\end{align*}
Further, we define the \textit{multinomial logit} variable $Y(\mathbf{h})$ such that
\begin{align*}
    \Pro\left[Y(\mathbf{h}) = j\right] &= \frac{\exp\{h_j\}}{\sum_{m=1}^k \exp\{h_m\}} .
\end{align*}
Recall that $\Pro_e = \Pro\left({\arg \max} (\mathbf{g}) \neq Y(\mathbf{h})\right)$, where the probability is taken both over the fresh test sample $\x$ and the randomness in the multinomial logit variable.
We note that for there to be a classification error conditioned on $\x$, \textit{at least} one of the following two events needs to hold:
a) ${\arg \max} (\mathbf{g}) \neq {\arg \max} (\mathbf{h})$, or b) $Y(\mathbf{h}) \neq {\arg \max} (\mathbf{h})$.
To see this, note that if neither a) nor b) held, we would have ${\arg \max} (\mathbf{g}) = Y(\mathbf{h})$ and we would not have a classification error conditional on the covariate being $\x$.
Thus, applying a union bound gives us
\begin{align*}
    \Pro_e &\leq \Pro_{e,0} + \Pro_{e,\mathsf{Bayes}} \text{ where } \\
    \Pro_{e,0} &:= \Pro\left({\arg \max} (\mathbf{g}) \neq {\arg \max} (\mathbf{h}) \right) \text{ and } \\
    \Pro_{e,\mathsf{Bayes}} &:= \Pro\left({\arg \max} (\mathbf{h}) \neq Y(\mathbf{h}) \right) .
\end{align*}
Thus, it suffices to provide an upper bound on $\Pro_{e,0}$ as defined.
We note that for there to be an error of the form ${\arg \max} (\mathbf{g}) \neq {\arg \max} (\mathbf{h})$, there \textit{needs} to exist indices $c_1,c_2 \in [k]$ (whose choice can depend on $\x$) such that $\x^\top \boldmu_{c_1} \geq \x^\top \boldmu_{c_2}$ but $\x^\top \hatw_{c_1} < \x^\top \hatw_{c_2}$.
In other words, we have
\begin{align*}
    \Pro_{e,0} &\leq \Pro \left(\x^\top \boldmu_{c_1} \geq \x^\top \boldmu_{c_2} \text{ and } \x^\top \hatw_{c_1} < \x^\top \hatw_{c_2} \text{ for some } c_1 \neq c_2 \right) \\
    &\leq \sum_{c_1 \neq c_2} \Pro \left(\x^\top \boldmu_{c_1} \geq \x^\top \boldmu_{c_2} \text{ and } \x^\top \hatw_{c_1} < \x^\top \hatw_{c_2} \right) \\
    &= \sum_{c_1 < c_2} \Pro\left(\x^\top \Deltabold_{c_1,c_2} \cdot \x^\top \Deltahat_{c_1,c_2} < 0 \right).
\end{align*}
Now, we consider whitened versions of the difference signal vectors: $\Ebold_{c_1,c_2}:=\Sigmab^{1/2}\Deltabold_{c_1,c_2}$, $\Ehat_{c_1,c_2}:=\Sigmab^{1/2}\Deltahat_{c_1,c_2}$.
We also define the generalized survival and contamination terms of the difference signal vector as
\begin{align*}
\SU(\Deltahat_{c_1,c_2},\Deltabold_{c_1,c_2}) &:= \frac{\Ehat_{c_1,c_2}^T\Ebold_{c_1,c_2}}{\|\Ebold_{c_1,c_2}\|_2}
\nn\\\nn
    \CN(\Deltahat_{c_1,c_2},\Deltabold_{c_1,c_2}) &:= \sqrt{\|\Ehat_{c_1,c_2}\|_2^2-\frac{\left(\Ehat_{c_1,c_2}^T\Ebold_{c_1,c_2}\right)^2}{\|\Ebold_{c_1,c_2}\|_2^2}} \nn
\end{align*}
Recall that $\x \sim \mathcal{N}(\boldsymbol{0},\Sigmab)$.
Then, the rotational invariance property of the Gaussian distribution and Gaussian decomposition yields:
\begin{align}
    \Pro\left(\x^\top \Deltabold_{c_1,c_2} \cdot \x^\top \Deltahat_{c_1,c_2} < 0 \right) 
    &=\Pro_{\mathbf{G}\sim\mathcal{N}(\boldsymbol{0},\mathbf{I})}\left(\mathbf{G}^\top \Ebold_{c_1,c_2} \cdot \mathbf{G}^\top \Ehat_{c_1,c_2} < 0 \right)\nn\\
    &=\Pro_{\substack{G\sim\mathcal{N}(0,1) \\
    \nn H\sim\mathcal{N}(0,1)}}\left(\|\Ebold_{c_1,c_2}\|_2G  \cdot \left(  \SU(\Deltahat_{c_1,c_2},\Deltabold_{c_1,c_2})G+ \CN(\Deltahat_{c_1,c_2},\Deltabold_{c_1,c_2}) H\right) < 0 \right)\\
    \nn
    &=\Pro_{\substack{G\sim\mathcal{N}(0,1) \\ H\sim\mathcal{N}(0,1)}}\left( \left(  \SU(\Deltahat_{c_1,c_2},\Deltabold_{c_1,c_2})\,G^2+ \CN(\Deltahat_{c_1,c_2},\Deltabold_{c_1,c_2}) \,H\,G\right) < 0 \right)
    \\ 
    &= \frac{1}{2} - \frac{1}{\pi} \mathsf{tan}^{-1} \left(\frac{\SU(\Deltahat_{c_1,c_2},\Deltabold_{c_1,c_2})}{\CN(\Deltahat_{c_1,c_2},\Deltabold_{c_1,c_2})}\right) . \label{eq:Vidya}
\end{align}
%
For the last equality in Equation~\eqref{eq:Vidya}, we used the fact that the ratio $H/G$ of two independent standard normals follows the standard Cauchy distribution. 
This completes the proof.
\end{proof}

\noindent\textbf{Bounding the survival and contamination terms.}~Next, we provide characterizations of $\SU(\Deltahat_{c_1,c_2},\Deltabold_{c_1,c_2})$ and $\CN(\Deltahat_{c_1,c_2},\Deltabold_{c_1,c_2})$.
We abbreviate these by $\SU_{c_1,c_2}$ and $\CN_{c_1,c_2}$ respectively for brevity.
These characterizations address two new aspects of the MLM: the multiclass setting, and label noise generated by the logistic model. 
We start with the characterization of survival.
\begin{lemma}[Survival terms]
\label{lem:sumulticlass}
There exist positive universal constants $L_1,L_2,U_1,U_2, C$ such that
\begin{align*}
    \SU^L(n) &\leq \SU_{c_1,c_2}(n) \leq \SU^U(n),\;\text{ where } \\
    \SU^L(n) &:= \begin{cases}
    c_k (1 + L_1 n^{q - (1 - r)})^{-1},\;0 < q < 1 - r\\
    c_k L_2 n^{(1 - r) - q},\;q > 1 - r.
    \end{cases} \\
    \SU^U(n) &:= \begin{cases}
    c_k (1 + U_1 n^{q - (1 - r)})^{-1},\;0 < q < 1 - r\\
    c_k U_2 n^{(1 - r) - q},\;q > 1 - r.
    \end{cases}
\end{align*}
with probability at least $1 - Ck^3 e^{-C\sqrt{n}}$.
Above, $c_k > 0$ is a fixed strictly positive constant that depends on $k$ but not on $n$.
\end{lemma}
Lemma~\ref{lem:sumulticlass} constitutes a nontrivial extension of Lemma 11 of~\cite{muthukumar2021classification} to deal with intricacies in the new pairwise-difference signal vector and the covariate-dependent label noise induced by the MLM.
Its proof is provided in Appendix~\ref{sec:lemmasumulticlassproof}.

Next, we provide an upper-bound characterization of contamination.
\begin{lemma}[Contamination terms]
\label{lem:cnmulticlass}
There exists a universal constant $C_k$ that depends only on $k$ such that
\begin{align*}
    \CN_{c_1,c_2}(n) \leq C_k \sqrt{\log n} \cdot n^{-\frac{\min\{m-1,2q+r-1,2q + 2r - 3/2\}}{2}}, q > 1-r
\end{align*}
with probability at least $1-\frac{C_k}{n^c}$ for some constant $0 <c\leq1$.
\end{lemma}
Lemma~\ref{lem:cnmulticlass} extends Lemma 13 of~\cite{muthukumar2021classification} for binary classification, and its proof is provided in Appendix~\ref{sec:lemmacnmulticlassproof}.
As with the analysis of survival, the dependency of the label difference vector on the top $k$ features requires an intricate leave-$k$-out analysis
Accordingly, several technical lemmas established in the proof of Lemma~\ref{lem:sumulticlass} are also used in this proof.

Plugging Lemmas~\ref{lem:sumulticlass} and~\ref{lem:cnmulticlass} into Lemma~\ref{lem:sucmulticlass} directly gives us the desired statement of Theorem~\ref{thm:classerrormlm}.
\qed

\section{Conclusion and future work}
Our work provides, to the best of our knowledge, the first results characterizing a) equivalence of loss functions, and b) generalization of interpolating solutions in multiclass settings.
We outline here some immediate as well as longer-term future directions.
\textcolor{black}{First, in Section \ref{sec:non-isotropic}, we discussed in detail the potential for extending our techniques to anisotropic scenarios for GMM data. However, the formal details of such extensions require further work that is beyond the scope of this paper. Another important area for future research is the extension of our results to situations where the number of classes ($k$) scales with the problem dimensions $(n,p)$. This is particularly intriguing as past research (e.g. \cite{abramovich2021multiclass}) has shown, albeit under differing assumptions and with distinct training algorithms, that there is a different generalization error behavior between small and large numbers of classes. Despite our research's focus on the condition where $k$ is constant, our results provide a mathematical basis for such extensions. A key contribution of our work is the establishment of deterministic equivalence conditions between multiclass SVM and MNI, which not only remain valid but also serve as a basis for analyzing any probabilistic data model and any scaling regime of $k$. In fact, after the initial release of this paper, the authors of \cite{subramaniangeneralization,wu2023precise} leveraged our equivalence result and expanded our generalization bounds for the case of MLM data to the case where $k$ can grow with $n$ and $p$, which requires new technical insights.}

More generally, our fine-grained techniques are tailored to high-dimensional linear models with Gaussian features. 
Furthermore, we believe the results derived here can extend to kernel machines and other nonlinear settings; formally showing these extensions is of substantial interest. 
It is also interesting to investigate corresponding lower bounds for our results --- for example, studying the sharpness of our conditions for equivalence of SVM to MNI in Section \ref{sec:linksvm}, analogous to \cite{ardeshir2021support} for the binary case. 
Also, we have limited attention to balanced datasets throughout, i.e. we assumed that each class contains equal number of training samples.
We would like to investigate the effect of data imbalances on our results extending our analysis to CE modifications tailored to imbalanced data recently proposed in \cite{TengyuMa,Menon,VSloss}. Finally, we have established a tight connection of our findings regarding the geometry of support vectors under overparameterization with the neural collapse phenomenon. Nevertheless, many questions remain open towards better explaining what leads the learnt feature representations of overparameterized to have the observed ETF structure. It is a fascinating research direction further exploring the geometry of learnt features and of support vectors in nonlinear settings.


\section*{Acknowledgments}
{\color{black} We are grateful to Vignesh Subramanian, Rahul Arya and Anant Sahai for pointing out a subtle issue in the proofs of Lemmas~\ref{lem:sumulticlass} and~\ref{lem:cnmulticlass}, which has since been fixed.} We are also grateful to the anonymous reviewers for their valuable feedback, which has contributed to enhancing the presentation of the initial submission.
This work is partially supported by the NSF under Grant Number CCF-2009030, by an NSERC Discovery Grant and by a grant from KAUST. Part of this work was done when VM was visiting the Simons Institute for the Theory of Computing.

\bibliographystyle{alpha}
\bibliography{compbib}

\newcommand{\etalchar}[1]{$^{#1}$}
\begin{thebibliography}{CWG{\etalchar{+}}19}

\bibitem[AGL21]{abramovich2021multiclass}
Felix Abramovich, Vadim Grinshtein, and Tomer Levy.
\newblock Multiclass classification by sparse multinomial logistic regression.
\newblock {\em IEEE Transactions on Information Theory}, 67(7):4637--4646,
  2021.

\bibitem[AKLZ20]{aubin2020generalization}
Benjamin Aubin, Florent Krzakala, Yue Lu, and Lenka Zdeborov\'{a}.
\newblock Generalization error in high-dimensional perceptrons: {A}pproaching
  {B}ayes error with convex optimization.
\newblock In H.~Larochelle, M.~Ranzato, R.~Hadsell, M.~F. Balcan, and H.~Lin,
  editors, {\em Advances in Neural Information Processing Systems}, volume~33,
  pages 12199--12210. Curran Associates, Inc., 2020.

\bibitem[ASH21]{ardeshir2021support}
Navid Ardeshir, Clayton Sanford, and Daniel~J Hsu.
\newblock Support vector machines and linear regression coincide with very
  high-dimensional features.
\newblock {\em Advances in Neural Information Processing Systems}, 34, 2021.

\bibitem[ASS01]{allwein2000reducing}
Erin~L. Allwein, Robert~E. Schapire, and Yoram Singer.
\newblock Reducing multiclass to binary: A unifying approach for margin
  classifiers.
\newblock {\em Journal of Machine Learning Research}, 1:113–141, September
  2001.

\bibitem[BB99]{bredensteiner1999multicategory}
Erin~J Bredensteiner and Kristin~P Bennett.
\newblock Multicategory classification by support vector machines.
\newblock In {\em Computational Optimization}, pages 53--79. Springer, 1999.

\bibitem[BEH20]{bosman2020visualising}
Anna~Sergeevna Bosman, Andries Engelbrecht, and Mardé Helbig.
\newblock Visualising basins of attraction for the cross-entropy and the
  squared error neural network loss functions.
\newblock {\em Neurocomputing}, 400:113--136, 2020.

\bibitem[Ber09]{bernstein2009matrix}
Dennis~S Bernstein.
\newblock {\em Matrix mathematics: theory, facts, and formulas}.
\newblock Princeton university press, 2009.

\bibitem[BG01]{buhot2001robust}
Arnaud Buhot and Mirta~B Gordon.
\newblock Robust learning and generalization with support vector machines.
\newblock {\em Journal of Physics A: Mathematical and General},
  34(21):4377--4388, May 2001.

\bibitem[BHMM19]{belkin2019reconciling}
Mikhail Belkin, Daniel Hsu, Siyuan Ma, and Soumik Mandal.
\newblock Reconciling modern machine-learning practice and the classical
  bias{\textendash}variance trade-off.
\newblock {\em Proceedings of the National Academy of Sciences},
  116(32):15849--15854, 2019.

\bibitem[BHX20]{belkin2020two}
Mikhail Belkin, Daniel Hsu, and Ji~Xu.
\newblock Two models of double descent for weak features.
\newblock {\em SIAM Journal on Mathematics of Data Science}, 2(4):1167--1180,
  2020.

\bibitem[BLLT20]{bartlett2020benign}
Peter~L. Bartlett, Philip~M. Long, G{\'a}bor Lugosi, and Alexander Tsigler.
\newblock Benign overfitting in linear regression.
\newblock {\em Proceedings of the National Academy of Sciences},
  117(48):30063--30070, 2020.

\bibitem[BM94]{bennett1994multicategory}
Kristin~P. Bennett and O.L. Mangasarian.
\newblock Multicategory discrimination via linear programming.
\newblock {\em Optimization Methods and Software}, 3(1-3):27--39, 1994.

\bibitem[BM03]{bartlett2002rademacher}
Peter~L. Bartlett and Shahar Mendelson.
\newblock Rademacher and {G}aussian complexities: Risk bounds and structural
  results.
\newblock {\em Journal of Machine Learning Research}, 3:463–482, March 2003.

\bibitem[CGB21]{cao2021risk}
Yuan Cao, Quanquan Gu, and Mikhail Belkin.
\newblock Risk bounds for over-parameterized maximum margin classification on
  sub-{G}aussian mixtures.
\newblock {\em arXiv preprint arXiv:2104.13628}, 2021.

\bibitem[CKMY16]{cortes2016structured}
Corinna Cortes, Vitaly Kuznetsov, Mehryar Mohri, and Scott Yang.
\newblock Structured prediction theory based on factor graph complexity.
\newblock In D.~Lee, M.~Sugiyama, U.~Luxburg, I.~Guyon, and R.~Garnett,
  editors, {\em Advances in Neural Information Processing Systems}, volume~29.
  Curran Associates, Inc., 2016.

\bibitem[CL21]{chatterji2020finite}
Niladri~S Chatterji and Philip~M Long.
\newblock Finite-sample analysis of interpolating linear classifiers in the
  overparameterized regime.
\newblock {\em Journal of Machine Learning Research}, 22(129):1--30, 2021.

\bibitem[CLRS09]{cormen2009introduction}
Thomas~H. Cormen, Charles~E. Leiserson, Ronald~L. Rivest, and Clifford Stein.
\newblock {\em Introduction to Algorithms, Third Edition}.
\newblock The MIT Press, 3rd edition, 2009.

\bibitem[CS02]{crammer2001algorithmic}
Koby Crammer and Yoram Singer.
\newblock On the algorithmic implementation of multiclass kernel-based vector
  machines.
\newblock {\em Journal of Machine Learning Research}, 2:265–292, March 2002.

\bibitem[CWG{\etalchar{+}}19]{TengyuMa}
Kaidi Cao, Colin Wei, Adrien Gaidon, Nikos Arechiga, and Tengyu Ma.
\newblock Learning imbalanced datasets with label-distribution-aware margin
  loss.
\newblock In {\em Advances in Neural Information Processing Systems}, pages
  1567--1578, 2019.

\bibitem[DB95]{dietterich1994solving}
Thomas~G. Dietterich and Ghulum Bakiri.
\newblock Solving multiclass learning problems via error-correcting output
  codes.
\newblock {\em Journal of Artificial Intelligence Research}, 2(1):263–286,
  January 1995.

\bibitem[DCO20]{demirkaya2020exploring}
Ahmet Demirkaya, Jiasi Chen, and Samet Oymak.
\newblock Exploring the role of loss functions in multiclass classification.
\newblock In {\em 2020 54th Annual Conference on Information Sciences and
  Systems (CISS)}, pages 1--5, 2020.

\bibitem[DET05]{donoho2005stable}
David~L Donoho, Michael Elad, and Vladimir~N Temlyakov.
\newblock Stable recovery of sparse overcomplete representations in the
  presence of noise.
\newblock {\em IEEE Transactions on information theory}, 52(1):6--18, 2005.

\bibitem[DKT21]{deng2019model}
Zeyu Deng, Abla Kammoun, and Christos Thrampoulidis.
\newblock {A model of double descent for high-dimensional binary linear
  classification}.
\newblock {\em Information and Inference: A Journal of the IMA}, April 2021.

\bibitem[DL20]{dhifallah2020precise}
Oussama Dhifallah and Yue~M Lu.
\newblock A precise performance analysis of learning with random features.
\newblock {\em arXiv preprint arXiv:2008.11904}, 2020.

\bibitem[DOS99]{dietrich1999statistical}
Rainer Dietrich, Manfred Opper, and Haim Sompolinsky.
\newblock Statistical mechanics of support vector networks.
\newblock {\em Physical Review Letters}, 82:2975--2978, Apr 1999.

\bibitem[DR17]{dziugaite2017computing}
Gintare~Karolina Dziugaite and Daniel~M Roy.
\newblock Computing nonvacuous generalization bounds for deep (stochastic)
  neural networks with many more parameters than training data.
\newblock {\em arXiv preprint arXiv:1703.11008}, 2017.

\bibitem[EHN96]{engl1996regularization}
Heinz~Werner Engl, Martin Hanke, and Andreas Neubauer.
\newblock {\em Regularization of inverse problems}, volume 375.
\newblock Springer Science \& Business Media, 1996.

\bibitem[F\"02]{furnkranz2002round}
Johannes F\"{u}rnkranz.
\newblock Round robin classification.
\newblock {\em Journal of Machine Learning Research}, 2:721–747, March 2002.

\bibitem[FHLS21a]{NC5}
Cong Fang, Hangfeng He, Qi~Long, and Weijie~J Su.
\newblock Exploring deep neural networks via layer-peeled model: Minority
  collapse in imbalanced training.
\newblock {\em Proceedings of the National Academy of Sciences}, 118(43), 2021.

\bibitem[FHLS21b]{NC6}
Cong Fang, Hangfeng He, Qi~Long, and Weijie~J Su.
\newblock Exploring deep neural networks via layer-peeled model: Minority
  collapse in imbalanced training.
\newblock {\em Proceedings of the National Academy of Sciences}, 118(43), 2021.

\bibitem[GCOZ17]{gajowniczek2017generalized}
Krzysztof Gajowniczek, Leszek~J. Chmielewski, Arkadiusz Or{\l}owski, and Tomasz
  Z{\k{a}}bkowski.
\newblock Generalized entropy cost function in neural networks.
\newblock In Alessandra Lintas, Stefano Rovetta, Paul~F.M.J. Verschure, and
  Alessandro~E.P. Villa, editors, {\em Artificial Neural Networks and Machine
  Learning -- ICANN 2017}, pages 128--136, Cham, 2017. Springer International
  Publishing.

\bibitem[GHNK21]{NC8}
Florian Graf, Christoph Hofer, Marc Niethammer, and Roland Kwitt.
\newblock Dissecting supervised constrastive learning.
\newblock In {\em International Conference on Machine Learning}, pages
  3821--3830. PMLR, 2021.

\bibitem[GHST05]{graepel2005pac}
Thore Graepel, Ralf Herbrich, and John Shawe-Taylor.
\newblock Pac-bayesian compression bounds on the prediction error of learning
  algorithms for classification.
\newblock {\em Machine Learning}, 59(1-2):55--76, 2005.

\bibitem[GJS{\etalchar{+}}20]{geiger2020scaling}
Mario Geiger, Arthur Jacot, Stefano Spigler, Franck Gabriel, Levent Sagun,
  St{\'{e}}phane d'Ascoli, Giulio Biroli, Cl{\'{e}}ment Hongler, and Matthieu
  Wyart.
\newblock Scaling description of generalization with number of parameters in
  deep learning.
\newblock {\em Journal of Statistical Mechanics: Theory and Experiment},
  2020(2):023401, February 2020.

\bibitem[GLL{\etalchar{+}}11]{germain2011pac}
Pascal Germain, Alexandre Lacoste, Fran{\c{c}}ois Laviolette, Mario Marchand,
  and Sara Shanian.
\newblock A pac-bayes sample-compression approach to kernel methods.
\newblock In {\em ICML}, 2011.

\bibitem[HB20]{hui2020evaluation}
Like Hui and Mikhail Belkin.
\newblock Evaluation of neural architectures trained with square loss vs
  cross-entropy in classification tasks.
\newblock {\em arXiv preprint arXiv:2006.07322}, 2020.

\bibitem[HJ12]{horn2012matrix}
Roger~A. Horn and Charles~R. Johnson.
\newblock {\em Matrix Analysis}.
\newblock Cambridge University Press, USA, 2nd edition, 2012.

\bibitem[HMRT19]{hastie2019surprises}
Trevor Hastie, Andrea Montanari, Saharon Rosset, and Ryan~J Tibshirani.
\newblock Surprises in high-dimensional ridgeless least squares interpolation.
\newblock {\em arXiv preprint arXiv:1903.08560}, 2019.

\bibitem[HMX21]{hsu2020proliferation}
Daniel Hsu, Vidya Muthukumar, and Ji~Xu.
\newblock On the proliferation of support vectors in high dimensions.
\newblock In Arindam Banerjee and Kenji Fukumizu, editors, {\em Proceedings of
  The 24th International Conference on Artificial Intelligence and Statistics},
  volume 130 of {\em Proceedings of Machine Learning Research}, pages 91--99.
  PMLR, 13--15 Apr 2021.

\bibitem[HPD21]{NC3}
XY~Han, Vardan Papyan, and David~L Donoho.
\newblock Neural collapse under mse loss: Proximity to and dynamics on the
  central path.
\newblock {\em arXiv preprint arXiv:2106.02073}, 2021.

\bibitem[Hua17]{huang2017asymptotic}
Hanwen Huang.
\newblock Asymptotic behavior of support vector machine for spiked population
  model.
\newblock {\em Journal of Machine Learning Research}, 18(45):1--21, 2017.

\bibitem[HYS16]{hou2016squared}
Le~Hou, Chen-Ping Yu, and Dimitris Samaras.
\newblock Squared earth mover's distance-based loss for training deep neural
  networks.
\newblock {\em arXiv preprint arXiv:1611.05916}, 2016.

\bibitem[IMSV19]{masnadi2012cost}
Arya Iranmehr, Hamed Masnadi-Shirazi, and Nuno Vasconcelos.
\newblock Cost-sensitive support vector machines.
\newblock {\em Neurocomputing}, 343:50--64, 2019.

\bibitem[JT19]{ji2019implicit}
Ziwei Ji and Matus Telgarsky.
\newblock The implicit bias of gradient descent on nonseparable data.
\newblock In Alina Beygelzimer and Daniel Hsu, editors, {\em Proceedings of the
  Thirty-Second Conference on Learning Theory}, volume~99 of {\em Proceedings
  of Machine Learning Research}, pages 1772--1798, Phoenix, USA, 25--28 Jun
  2019. PMLR.

\bibitem[KA21]{svm_abla}
Abla Kammoun and Mohamed-Slim Alouini.
\newblock On the precise error analysis of support vector machines.
\newblock {\em IEEE Open Journal of Signal Processing}, 2:99--118, 2021.

\bibitem[KLS20]{kobak2020optimal}
Dmitry Kobak, Jonathan Lomond, and Benoit Sanchez.
\newblock The optimal ridge penalty for real-world high-dimensional data can be
  zero or negative due to the implicit ridge regularization.
\newblock {\em Journal of Machine Learning Research}, 21(169):1--16, 2020.

\bibitem[KP02]{koltchinskii2002empirical}
V.~Koltchinskii and D.~Panchenko.
\newblock {Empirical Margin Distributions and Bounding the Generalization Error
  of Combined Classifiers}.
\newblock {\em The Annals of Statistics}, 30(1):1 -- 50, 2002.

\bibitem[KPOT21]{VSloss}
Ganesh~Ramachandra Kini, Orestis Paraskevas, Samet Oymak, and Christos
  Thrampoulidis.
\newblock Label-imbalanced and group-sensitive classification under
  overparameterization.
\newblock {\em Advances in Neural Information Processing Systems}, 34, 2021.

\bibitem[KS18]{kumar2018robust}
Himanshu Kumar and P.~S. Sastry.
\newblock Robust loss functions for learning multi-class classifiers.
\newblock In {\em 2018 IEEE International Conference on Systems, Man, and
  Cybernetics (SMC)}, pages 687--692, 2018.

\bibitem[KT21]{kini2021phase}
Ganesh~Ramachandra Kini and Christos Thrampoulidis.
\newblock Phase transitions for one-vs-one and one-vs-all linear separability
  in multiclass gaussian mixtures.
\newblock In {\em ICASSP 2021 - 2021 IEEE International Conference on
  Acoustics, Speech and Signal Processing (ICASSP)}, pages 4020--4024, 2021.

\bibitem[KTW{\etalchar{+}}20]{khosla2020supervised}
Prannay Khosla, Piotr Teterwak, Chen Wang, Aaron Sarna, Yonglong Tian, Phillip
  Isola, Aaron Maschinot, Ce~Liu, and Dilip Krishnan.
\newblock Supervised contrastive learning.
\newblock {\em Advances in Neural Information Processing Systems},
  33:18661--18673, 2020.

\bibitem[LDBK15]{lei2015multi}
Yunwen Lei, Urun Dogan, Alexander Binder, and Marius Kloft.
\newblock Multi-class svms: From tighter data-dependent generalization bounds
  to novel algorithms.
\newblock In C.~Cortes, N.~Lawrence, D.~Lee, M.~Sugiyama, and R.~Garnett,
  editors, {\em Advances in Neural Information Processing Systems}, volume~28.
  Curran Associates, Inc., 2015.

\bibitem[LDZK19]{lei2019data}
Yunwen Lei, Ürün Dogan, Ding-Xuan Zhou, and Marius Kloft.
\newblock Data-dependent generalization bounds for multi-class classification.
\newblock {\em IEEE Transactions on Information Theory}, 65(5):2995--3021,
  2019.

\bibitem[LLW04]{lee2004multicategory}
Yoonkyung Lee, Yi~Lin, and Grace Wahba.
\newblock Multicategory support vector machines.
\newblock {\em Journal of the American Statistical Association},
  99(465):67--81, 2004.

\bibitem[Lol20]{lolas2020regularization}
Panagiotis Lolas.
\newblock Regularization in high-dimensional regression and classification via
  random matrix theory.
\newblock {\em arXiv preprint arXiv:2003.13723}, 2020.

\bibitem[LR21]{liang2021interpolating}
Tengyuan Liang and Benjamin Recht.
\newblock Interpolating classifiers make few mistakes.
\newblock {\em arXiv preprint arXiv:2101.11815}, 2021.

\bibitem[LS20]{liang2020precise}
Tengyuan Liang and Pragya Sur.
\newblock A precise high-dimensional asymptotic theory for boosting and
  min-l1-norm interpolated classifiers.
\newblock {\em arXiv preprint arXiv:2002.01586}, 2020.

\bibitem[LS22]{NC4}
Jianfeng Lu and Stefan Steinerberger.
\newblock Neural collapse under cross-entropy loss.
\newblock {\em Applied and Computational Harmonic Analysis}, 2022.

\bibitem[Mau16]{maurer2016vector}
Andreas Maurer.
\newblock A vector-contraction inequality for rademacher complexities.
\newblock In Ronald Ortner, Hans~Ulrich Simon, and Sandra Zilles, editors, {\em
  Algorithmic Learning Theory}, pages 3--17, Cham, 2016. Springer International
  Publishing.

\bibitem[MJR{\etalchar{+}}20]{Menon}
Aditya~Krishna Menon, Sadeep Jayasumana, Ankit~Singh Rawat, Himanshu Jain,
  Andreas Veit, and Sanjiv Kumar.
\newblock Long-tail learning via logit adjustment.
\newblock In {\em International Conference on Learning Representations}, 2020.

\bibitem[MLC19]{mai2019large}
Xiaoyi Mai, Zhenyu Liao, and Romain Couillet.
\newblock A large scale analysis of logistic regression: Asymptotic performance
  and new insights.
\newblock In {\em ICASSP 2019 - 2019 IEEE International Conference on
  Acoustics, Speech and Signal Processing (ICASSP)}, pages 3357--3361, 2019.

\bibitem[MNS{\etalchar{+}}21]{muthukumar2021classification}
Vidya Muthukumar, Adhyyan Narang, Vignesh Subramanian, Mikhail Belkin, Daniel
  Hsu, and Anant Sahai.
\newblock Classification vs regression in overparameterized regimes: Does the
  loss function matter?
\newblock {\em Journal of Machine Learning Research}, 22(222):1--69, 2021.

\bibitem[MO05]{malzahn2005statistical}
Dörthe Malzahn and Manfred Opper.
\newblock A statistical physics approach for the analysis of machine learning
  algorithms on real data.
\newblock {\em Journal of Statistical Mechanics: Theory and Experiment},
  2005(11):P11001--P11001, nov 2005.

\bibitem[MPP20]{NC2}
Dustin~G Mixon, Hans Parshall, and Jianzong Pi.
\newblock Neural collapse with unconstrained features.
\newblock {\em arXiv preprint arXiv:2011.11619}, 2020.

\bibitem[MR16]{maximov2016tight}
Yu~Maximov and Daria Reshetova.
\newblock Tight risk bounds for multi-class margin classifiers.
\newblock {\em Pattern Recognition and Image Analysis}, 26:673--680, 2016.

\bibitem[MRSY19]{montanari2019generalization}
Andrea Montanari, Feng Ruan, Youngtak Sohn, and Jun Yan.
\newblock The generalization error of max-margin linear classifiers:
  High-dimensional asymptotics in the overparametrized regime.
\newblock {\em arXiv preprint arXiv:1911.01544}, 2019.

\bibitem[MVSS20]{muthukumar2020harmless}
Vidya Muthukumar, Kailas Vodrahalli, Vignesh Subramanian, and Anant Sahai.
\newblock Harmless interpolation of noisy data in regression.
\newblock {\em IEEE Journal on Selected Areas in Information Theory}, 2020.

\bibitem[PGS13]{pires2013cost}
Bernardo~\'{A}vila Pires, Mohammad Ghavamzadeh, and Csaba Szepesv\'{a}ri.
\newblock Cost-sensitive multiclass classification risk bounds.
\newblock In {\em Proceedings of the 30th International Conference on
  International Conference on Machine Learning - Volume 28}, ICML'13, page
  III–1391–III–1399. JMLR.org, 2013.

\bibitem[PHD20]{papyan2020prevalence}
Vardan Papyan, X.~Y. Han, and David~L. Donoho.
\newblock Prevalence of neural collapse during the terminal phase of deep
  learning training.
\newblock {\em Proceedings of the National Academy of Sciences},
  117(40):24652--24663, 2020.

\bibitem[PL20a]{poggio2020explicit}
Tomaso Poggio and Qianli Liao.
\newblock Explicit regularization and implicit bias in deep network classifiers
  trained with the square loss.
\newblock {\em arXiv preprint arXiv:2101.00072}, 2020.

\bibitem[PL20b]{NC7}
Tomaso Poggio and Qianli Liao.
\newblock Explicit regularization and implicit bias in deep network classifiers
  trained with the square loss.
\newblock {\em arXiv preprint arXiv:2101.00072}, 2020.

\bibitem[PS16]{pires2016multiclass}
Bernardo~{\'A}vila Pires and Csaba Szepesv{\'a}ri.
\newblock Multiclass classification calibration functions.
\newblock {\em arXiv preprint arXiv:1609.06385}, 2016.

\bibitem[RDS{\etalchar{+}}15]{russakovsky2015imagenet}
Olga Russakovsky, Jia Deng, Hao Su, Jonathan Krause, Sanjeev Satheesh, Sean Ma,
  Zhiheng Huang, Andrej Karpathy, Aditya Khosla, Michael Bernstein, et~al.
\newblock Imagenet large scale visual recognition challenge.
\newblock {\em International journal of computer vision}, 115(3):211--252,
  2015.

\bibitem[Rif02]{rifkin2002everything}
Ryan~Michael Rifkin.
\newblock {\em Everything old is new again: a fresh look at historical
  approaches in machine learning}.
\newblock PhD thesis, MaSSachuSettS InStitute of Technology, 2002.

\bibitem[RK04]{rifkin2004defense}
Ryan Rifkin and Aldebaro Klautau.
\newblock In defense of one-vs-all classification.
\newblock {\em Journal of Machine Learning Research}, 5:101--141, 2004.

\bibitem[RV{\etalchar{+}}13]{rudelson2013hanson}
Mark Rudelson, Roman Vershynin, et~al.
\newblock Hanson-wright inequality and sub-gaussian concentration.
\newblock {\em Electronic Communications in Probability}, 18, 2013.

\bibitem[SAH19]{salehi2019impact}
Fariborz Salehi, Ehsan Abbasi, and Babak Hassibi.
\newblock The impact of regularization on high-dimensional logistic regression.
\newblock In {\em Advances in Neural Information Processing Systems},
  volume~32. Curran Associates, Inc., 2019.

\bibitem[SAH20]{salehi2020performance}
Fariborz Salehi, Ehsan Abbasi, and Babak Hassibi.
\newblock The performance analysis of generalized margin maximizers on
  separable data.
\newblock In {\em International Conference on Machine Learning}, pages
  8417--8426. PMLR, 2020.

\bibitem[SAS22]{subramaniangeneralization}
Vignesh Subramanian, Rahul Arya, and Anant Sahai.
\newblock Generalization for multiclass classification with overparameterized
  linear models.
\newblock In {\em Advances in Neural Information Processing Systems}, 2022.

\bibitem[SC19]{sur2019modern}
Pragya Sur and Emmanuel~J Cand{\`e}s.
\newblock A modern maximum-likelihood theory for high-dimensional logistic
  regression.
\newblock {\em Proceedings of the National Academy of Sciences},
  116(29):14516--14525, 2019.

\bibitem[SFBL98]{schapire1998boosting}
Robert~E Schapire, Yoav Freund, Peter Bartlett, and Wee~Sun Lee.
\newblock Boosting the margin: A new explanation for the effectiveness of
  voting methods.
\newblock {\em The Annals of Statistics}, 26(5):1651--1686, 1998.

\bibitem[SHN{\etalchar{+}}18]{soudry2018implicit}
Daniel Soudry, Elad Hoffer, Mor~Shpigel Nacson, Suriya Gunasekar, and Nathan
  Srebro.
\newblock The implicit bias of gradient descent on separable data.
\newblock {\em Journal of Machine Learning Research}, 19(1):2822--2878, 2018.

\bibitem[TB07]{tewari2007consistency}
Ambuj Tewari and Peter~L Bartlett.
\newblock On the consistency of multiclass classification methods.
\newblock {\em Journal of Machine Learning Research}, 8(36):1007--1025, 2007.

\bibitem[TB20]{tsigler2020benign}
Alexander Tsigler and Peter~L Bartlett.
\newblock Benign overfitting in ridge regression.
\newblock {\em arXiv preprint arXiv:2009.14286}, 2020.

\bibitem[TOS20]{thrampoulidis2020theoretical}
Christos Thrampoulidis, Samet Oymak, and Mahdi Soltanolkotabi.
\newblock Theoretical insights into multiclass classification: A
  high-dimensional asymptotic view.
\newblock In H.~Larochelle, M.~Ranzato, R.~Hadsell, M.~F. Balcan, and H.~Lin,
  editors, {\em Advances in Neural Information Processing Systems}, volume~33,
  pages 8907--8920. Curran Associates, Inc., 2020.

\bibitem[TPT20]{taheri2020sharp}
Hossein Taheri, Ramtin Pedarsani, and Christos Thrampoulidis.
\newblock Sharp asymptotics and optimal performance for inference in binary
  models.
\newblock In Silvia Chiappa and Roberto Calandra, editors, {\em Proceedings of
  the Twenty Third International Conference on Artificial Intelligence and
  Statistics}, volume 108 of {\em Proceedings of Machine Learning Research},
  pages 3739--3749. PMLR, 26--28 Aug 2020.

\bibitem[TPT21]{taheri2020fundamental}
Hossein Taheri, Ramtin Pedarsani, and Christos Thrampoulidis.
\newblock Fundamental limits of ridge-regularized empirical risk minimization
  in high dimensions.
\newblock In Arindam Banerjee and Kenji Fukumizu, editors, {\em Proceedings of
  The 24th International Conference on Artificial Intelligence and Statistics},
  volume 130 of {\em Proceedings of Machine Learning Research}, pages
  2773--2781. PMLR, 13--15 Apr 2021.

\bibitem[Tro06]{tropp2006just}
Joel~A Tropp.
\newblock Just relax: Convex programming methods for identifying sparse signals
  in noise.
\newblock {\em IEEE transactions on information theory}, 52(3):1030--1051,
  2006.

\bibitem[Vap13]{vapnik2013nature}
Vladimir Vapnik.
\newblock {\em The nature of statistical learning theory}.
\newblock Springer science \& business media, 2013.

\bibitem[Wai19]{wainwright2019high}
Martin~J Wainwright.
\newblock {\em High-dimensional statistics: A non-asymptotic viewpoint},
  volume~48.
\newblock Cambridge University Press, 2019.

\bibitem[Wel74]{welch1974lower}
Lloyd Welch.
\newblock Lower bounds on the maximum cross correlation of signals (corresp.).
\newblock {\em IEEE Transactions on Information theory}, 20(3):397--399, 1974.

\bibitem[WS23]{wu2023precise}
David~X Wu and Anant Sahai.
\newblock Precise asymptotic generalization for multiclass classification with
  overparameterized linear models.
\newblock {\em arXiv preprint arXiv:2306.13255}, 2023.

\bibitem[WT21]{wang2020benign}
Ke~Wang and Christos Thrampoulidis.
\newblock Binary classification of gaussian mixtures: Abundance of support
  vectors, benign overfitting and regularization.
\newblock {\em arXiv preprint arXiv:2011.09148}, 2021.

\bibitem[WW98]{weston1998multi}
Jason Weston and Chris Watkins.
\newblock Multi-class support vector machines.
\newblock Technical report, Citeseer, 1998.

\bibitem[ZBH{\etalchar{+}}17]{zhang2016understanding}
Chiyuan Zhang, Samy Bengio, Moritz Hardt, Benjamin Recht, and Oriol Vinyals.
\newblock Understanding deep learning requires rethinking generalization.
\newblock In {\em 5th International Conference on Learning Representations,
  {ICLR} 2017, Toulon, France, April 24-26, 2017, Conference Track
  Proceedings}. OpenReview.net, 2017.

\bibitem[ZDZ{\etalchar{+}}21]{NC1}
Zhihui Zhu, Tianyu Ding, Jinxin Zhou, Xiao Li, Chong You, Jeremias Sulam, and
  Qing Qu.
\newblock A geometric analysis of neural collapse with unconstrained features.
\newblock {\em Advances in Neural Information Processing Systems}, 34, 2021.

\bibitem[Zha04]{zhang2004statistical}
Tong Zhang.
\newblock Statistical behavior and consistency of classification methods based
  on convex risk minimization.
\newblock {\em The Annals of Statistics}, 32(1):56--85, 2004.

\end{thebibliography}

\clearpage

\appendix

\addtocontents{toc}{\protect\setcounter{tocdepth}{3}}
\tableofcontents


\section{Lemmas used in the proof of Theorem \ref{thm:svmgmm}}\label{sec:app_proof_Thm2}

\subsection{Auxiliary  Lemmas}\label{sec:aux_lemmas_svm_GMM}
In this section, we state a series of auxiliary lemmas that we use to prove Lemma~\ref{lem:ineqforanoj}.
The following result shows concentration of the norms of the label indicators $\vb_c, c \in [k]$ under the {nearly equal-priors assumption} (Assumption~\ref{ass:equalmuprob}). 
Intuitively, in this {nearly balanced} setting there are $\Theta(n/k)$ samples for each class; hence,  $\Theta(n/k)$ non-zeros (in fact, 1's) in each label indicator vector $\vb_c$.
\begin{lemma}
\label{lem:vnorm}
Under the setting of Assumption~\ref{ass:equalmuprob}, there exist large constants $C_1, C_2 > 0$ such that the event
\begin{align}
\label{eq:eventvnorm}
    \Ec_v:=\Big\{ \lp(1-\frac{1}{C_1}\rp)\frac{n}{k}\le \tn{\vb_c}^2 \le \lp(1+\frac{1}{C_1}\rp)\frac{n}{k}~,~\forall c\in[k] \Big\},
\end{align}
holds  with probability at least $1-2ke^{-\frac{n}{C_2k^2}}$.
\end{lemma}
Next, we provide bounds on the ``base case'' $0$-th order quadratic forms that involve the Gram matrix $\bolda_0^{-1}$. We do this in three lemmas presented below. {The first Lemma~\ref{lem:ineqforazero} follows by a  direct application of~\cite[Lemma 4 and 5]{wang2020benign}}. The only difference is that we keep track of throughout the proof is the scaling of $\mathcal{O}(1/k)$ arising from the multiclass case in the $\vb_j$'s. 
For instance, the bound of the term $h_{mj}^{(0)} :=\mathbf{v}_m^T \bolda_0^{-1}\mathbf{d}_j$ involves a term $\tilde{\rho}_{n,k}=\min\{1,\sqrt{\log(2n)/k}\}$ compared to the binary case.
The other two Lemmas~\ref{lem:ineqforszero} and~\ref{lem:ineqforgzero} are proved in Section~\ref{sec:pfauxlemma}.
\begin{lemma}[$0$-th order Quadratic forms, Part I]
\label{lem:ineqforazero}
Under the event $\mathcal{E}_v $, there exist constants $c_i$'s and $C_i$'s $>1$ such that the following bounds hold with probability at least $1-c_1ke^{-\frac{n}{c_2}}$.
\begin{align*}
    t_{jj}^{(0)} &\le  \frac{C_1n\tn{\boldmu}^2}{p}~~\text{for all}~ j \in [k],\\
    |h_{mj}^{(0)}| &\le  \minrho\frac{C_2n\tn{\boldmu}}{\sqrt{k}p}~~\text{for all}~ m, j \in [k],\\
    |t_{mj}^{(0)}| &\le  \frac{C_3n\tn{\boldmu}^2}{p}~~\text{for all}~ m \ne j \in [k],\\
    \tn{\mathbf{d}_j}^2 &\le C_4n\tn{\boldmu}^2~~\text{for all}~ j \in [k],\\
    \max_{i \in [n]}|f_{ji}^{(0)}|  &\le \frac{C_5\sqrt{\log(2n)}\tn{ \boldmu}}{p}~~\text{for all}~j \in [k].
\end{align*}
\end{lemma}

To sharply characterize the forms $s_{ij}^{(0)}$ we need additional work, particularly for the cross-terms where $i \neq j$. We will make use of fundamental concentration inequalities on quadratic forms of inverse Wishart matrices. {Note that the term $t_{jj}^{(0)}$ originally depends on the norm $\tn{\boldmu_j}^2$. Due to the nearly equal energy Assumption~\ref{ass:equalmuprob}, we can write $\tn{\boldmu_j}^2$ in term of the ``reference vector" norm $\tn{\boldmu}^2$ (which is defined in Assumption~\ref{ass:equalmuprob}). Consequently, we will see this ``reference norm" $\tn{\boldmu}^2$ in all our higher order terms.}
The following lemma controls these quadratic forms, and shows in particular that the $s_{ij}^{(0)}$ terms for $i \neq j$ are much smaller than the corresponding terms $s_{jj}^{(0)}$.
This sharp control of the cross-terms is essential for several subsequent proof steps.

\begin{lemma}[$0$-th order Quadratic forms, Part II]
\label{lem:ineqforszero}
Working on the event $\mathcal{E}_v$ defined in \Equation~\eqref{eq:eventvnorm}, assume that $p > Cn\log (kn) + n -1$ for large enough constant $C >1$ and large $n$. There exist constants $C_i$'s $> 1$ such that with probability at least $1-\frac{C_0}{n}$, the following bound holds:
\begin{align*}
     \frac{C_1-1}{C_1}\cdot\frac{n}{kp} \le  &s_{jj}^{(0)} \le \frac{C_1+1}{C_1}\cdot\frac{n}{kp}, ~~\text{for}~ j \in [k],\\
    -\frac{C_2+1}{C_2}\cdot\frac{\sqrt{n}}{kp}  \le  &s_{ij}^{(0)} \le \frac{C_2+1}{C_2}\cdot\frac{\sqrt{n}}{kp},~~\text{for}~ i\ne j \in [k].
\end{align*}
\end{lemma}

The proof of Lemma~\ref{lem:ineqforszero} for the cross terms with $i \neq j$ critically uses the in-built orthogonality of the label indicator vectors $\{\vb_c\}_{c \in [k]}$.
Finally, the following lemma controls the quadratic forms $g_{ji}^{(0)}$.
\begin{lemma}[$0$-th order Quadratic forms, Part III]
\label{lem:ineqforgzero}
Working on the event $\mathcal{E}_v$ defined in \Equation~\eqref{eq:eventvnorm}, given $p > Ck^3n\log(kn)+n-1$ for a large constant $C$, there exist large enough constants $C_1 , C_2$, such that with probability at least $1 - \frac{2}{kn}$, we have for every $i \in [n]$:
\begin{align}
    \lp(1-\frac{1}{C_1}\rp)\frac{1}{p} \le & g_{(y_i)i}^{(0)} \le \lp(1+\frac{1}{C_1}\rp)\frac{1}{p}, \notag\\
    -\frac{1}{C_2}\cdot\frac{1}{k^2p} \le & g_{ji}^{(0)} \le \frac{1}{C_2}\cdot\frac{1}{k^2p},~~\text{for}~ j\ne y_i. \notag
\end{align}
\end{lemma}

\subsection{Proof of Lemma~\ref{lem:ineqforanoj}} \label{sec:pfineqforanoj}

In this section, we provide the full proof of Lemma~\ref{lem:ineqforanoj}.
We begin with a proof outline.

\subsubsection{Proof outline}

As explained in Section \ref{sec:pforderoneoutline}, it suffices to consider the case where $j=k$, since
 when $j \ne k$ we can simply change the order of adding mean components, described in \Equation~\eqref{eq:addnewmean}, so that the $j$-th mean component is added last.
For concreteness, we will also fix $i \in [n]$, $y_i = k$ and define as shorthand $m := k-1$. 
These fixes are without loss of generality. {The reason why we fix $j = k$ and $m = k-1$ is that when we do the proof, we want to add the $k-1$-th and $k$-th components last. This is for ease of reading and understanding. } 

For the case $j = k$, the leave-one-out quadratic forms in Lemma~\ref{lem:ineqforanoj} are equal to the quadratic forms of order $k-1$, given by $s_{kk}^{(k-1)}, t_{kk}^{(k-1)}$, $h_{kk}^{(k-1)}$, $g_{ki}^{(k-1)}$ and $f_{ki}^{(k-1)}$.
We will proceed recursively starting from the quadratic forms of order $1$ building up all the way to the quadratic forms of order $k-1$. 
Specifically, starting from order $1$, we will work on the event
\begin{align}\label{eq:event_Eq}
    \mathcal{E}_q := \{\text{all the inequalities in Lemmas~\ref{lem:ineqforazero},~\ref{lem:ineqforszero}~and~\ref{lem:ineqforgzero} hold}\},
\end{align}
Further, we note that Lemma~\ref{lem:ineqforgzero} shows that the bound for $g_{y_ii}^{(0)}$ is different from the bound for $g_{ji}^{(0)}$ when $j \ne y_i$. 
We will show the following set of upper and lower bounds:
\begin{align}
\begin{split}
\label{eq:pfquadforone}
    \lp(\frac{C_{11}-1}{C_{11}}\rp)\frac{n}{kp} &\le  s_{kk}^{(1)} \le \lp(\frac{C_{11}+1}{C_{11}}\rp)\frac{n}{kp}, \\
    -\lp(\frac{C_{12}+1}{C_{12}}\rp)\frac{\sqrt{n}}{kp}  &\le  s_{mk}^{(1)} \le \lp(\frac{C_{12}+1}{C_{12}}\rp)\frac{\sqrt{n}}{kp},\\
     t_{kk}^{(1)} &\le  \frac{C_{13}n\tn{\boldmu}^2}{p},\\
    |h_{mk}^{(1)}| &\le  \minrho\frac{C_{14}n\tn{\boldmu}}{\sqrt{k}p},\\
    |t_{mk}^{(1)}| &\le  \frac{C_{15}n\tn{\boldmu}^2}{p},\\
    \tn{\mathbf{d}_k}^2 &\le C_{16}n\tn{\boldmu}^2,\\
    |f_{ki}^{(1)}|  &\le \frac{C_{17}\sqrt{\logtwon}\tn{ \boldmu}}{p},\\
    \lp(1-\frac{1}{C_{18}}\rp)\frac{1}{p} &\le  g_{(y_i)i}^{(1)} \le \lp(1+\frac{1}{C_{18}}\rp)\frac{1}{p},\text{ and }  \\
    -\frac{1}{C_{19}k^2p} &\le  g_{mi}^{(1)} \le \frac{1}{C_{19}k^2p}
    \end{split}
\end{align}
with probability at least $1-\frac{c}{kn^2}$.
Comparing the bounds on the terms of order $1$ in \Equation~\eqref{eq:pfquadforone} with the terms in Lemmas~\ref{lem:ineqforazero},~\ref{lem:ineqforszero} and~\ref{lem:ineqforgzero} of order $0$, the key observation is that they are all at the same order.
This allows us to repeat the same argument to now bound corresponding terms of order $2$, and so on until order $k-1$.
Note that for each $j \in [k]$, we have $n$ terms of the form $g_{ji}^{(1)}$, corresponding to each value of $i \in [n]$.
Thus, we will adjust the final probabilities by applying a union bound over the $n$ training examples.

\subsubsection{Proofs for 1-st order quadratic forms in \Equation~(\ref{eq:pfquadforone})}
%
The proof makes repeated use of Lemmas~\ref{lem:ineqforazero}, \ref{lem:ineqforszero} and \ref{lem:ineqforgzero}. In fact, we will throughout condition on the event $\mathcal{E}_q$, defined in Equation \eqref{eq:event_Eq}, {which holds with probability at least $1-\frac{c_1}{n} - c_2e^{-\frac{n}{c_0k^2}}$.} Specifically, by Lemma~\ref{lem:ineqforazero} we have
\begin{align}
\label{eq:pfboundeta01}
h_{mj}^{(0)} \le \minrho\frac{C_1\epsilonn}{k^2\sqrt{\logtwon}},\qquad \max_{i \in [n]}|f_{mi}^{(0)}|\le \frac{C_2\epsilonn}{k^{1.5}n},\quad \text{and}\quad \frac{s_{mj}^{(0)}}{s_{kk}^{(0)}} \le \frac{C}{\sqrt{n}} \text{ for } m,j\ne k,
\end{align}
where we recall from \Equation~\eqref{eq:eps_proof} the notation $\epsilonn := \frac{k^{1.5}n\sqrt{\logtwon}\tn{\boldmu}}{p}.$ {Also, recall that we choose $\epsilonn\leq\tau$ for a sufficiently small constant $\tau$.} 

In order to make use of Lemmas~\ref{lem:ineqforazero},~\ref{lem:ineqforszero} and~\ref{lem:ineqforgzero}, we need to relate the quantities of interest to corresponding quadratic forms involving $\bolda_0$.
We do this recursively and make repeated use of the Woodbury identity.
The recursions are proved in Appendix~\ref{sec:pfinversesecc}.
We now provide the proofs for the bounds on the terms in \Equation~\eqref{eq:pfquadforone} one-by-one.

\noindent\textbf{Bounds on $s_{mk}^{(1)}$.}~By \Equation~\eqref{eq:pfinverseskm} in Appendix~\ref{sec:pfinversesecc}, we have
\begin{align}
\label{eq:pfsijzero01}
s_{mk}^{(1)} 
&= s_{mk}^{(0)} - \frac{1}{\det_0}(\star)_{s}^{(0)}, 
\end{align}
where we define
\begin{align}
\begin{split}
\label{eq:pfsijzero02}
    (\star)_{s}^{(0)} &:= (\tn{\boldmu_1}^2 - t_{11}^{(0)})s_{1k}^{(0)}s_{1m}^{(0)} + s_{1m}^{(0)}h_{k1}^{(0)}h_{11}^{(0)}+s_{1k}^{(0)}h_{m1}^{(0)}h_{11}^{(0)}-s_{11}^{(0)}h_{k1}^{(0)}h_{m1}^{(0)} + s_{1m}^{(0)}h_{k1}^{(0)}+s_{1k}^{(0)}h_{m1}^{(0)} \text{ and } \\
    {\det}_0 &:= s_{11}^{(0)}(\tn{ \boldmu_1}^2 - t_{11}^{(0)})+(1+h_{11}^{(0)})^2.
\end{split}
\end{align}
The essential idea is to show that $|\frac{(\star)_{s}^{(0)}}{\det_0}|$ is sufficiently small compared to $|s_{mk}^{(0)}|$. We first look at the first term given by $\Big((\tn{\boldmu_1}^2 - t_{11}^{(0)})s_{1k}^{(0)}s_{1m}^{(0)}\Big)/\det_0$. By Lemmas~\ref{lem:ineqforazero},~\ref{lem:ineqforszero} and the definition of $\det_0$, we have
\begin{align*}
    \Big|\frac{1}{\det_0}\Big((\tn{\boldmu_1}^2 - t_{11}^{(0)})s_{1k}^{(0)}s_{1m}^{(0)}\Big)\Big| \le \frac{(\tn{\boldmu_1}^2 - t_{11}^{(0)})|s_{1k}^{(0)}s_{1m}^{(0)}|}{s_{11}^{(0)}(\tn{\boldmu_1}^2 - t_{11}^{(0)})} = \Big|\frac{s_{1k}^{(0)}s_{1m}^{(0)}}{s_{11}^{(0)}}\Big| \le \frac{C_1}{\sqrt{n}}\cdot\frac{C_2+1}{C_2}\cdot\frac{\sqrt{n}}{kp},
\end{align*}
where we use $\det_0 \ge s_{11}^{(0)}(\tn{ \boldmu_1}^2 - t_{11}^{(0)})$ and $s_{mj}^{(0)}/s_{kk}^{(0)} \le C/\sqrt{n} \text{ for all } m,j\ne k$. 
Now, we upper bound the other two dominant terms $|s_{1m}^{(0)}h_{k1}^{(0)}/\det_0|$ and $|{s_{1k}^{(0)}h_{m1}^{(0)}}/\det_0|$. 
Note that the same bound will apply to the remaining terms in \Equation~\eqref{eq:pfsijzero02} because we trivially have $|h_{ij}^{(0)}| = \Oc(1)$ for all $(i, j) \in [k]$. 
Again, Lemmas~\ref{lem:ineqforazero} and~\ref{lem:ineqforszero} give us
\begin{align*}
    \Big|\frac{s_{1m}^{(0)}h_{k1}^{(0)}}{\det_{0}}\Big| \le \frac{|s_{1m}^{(0)}h_{k1}^{(0)}|}{(1+h_{11}^{(0)})^2} \le \frac{\minrho C_3\epsilonn}{\lp(1-\frac{C_5\minrho\epsilonn}{k^2\sqrt{\logtwon}}\rp)^2 k^2\sqrt{\logtwon}}\cdot\frac{C_2+1}{C_2}\cdot\frac{\sqrt{n}}{kp}.
\end{align*}
The identical bound holds for $|{s_{1k}^{(0)}h_{m1}^{(0)}}|$. Noting that $|s_{mk}^{(0)}| \le \frac{C_2+1}{C_2} \cdot \frac{\sqrt{n}}{kp}$, we then have
\begin{align}
|s_{mk}^{(1)}| &\le |s_{mk}^{(0)}| + \Big|\frac{(\star)_s^{(0)}}{\det_0}\Big| \notag\\ 
&\le \lp(1+ \frac{C_6}{\sqrt{n}}+ \frac{C_7\minrho\epsilonn}{\lp(1-\frac{C_5\minrho\epsilonn}{k^2\sqrt{\logtwon}}\rp)^2 k^2\sqrt{\logtwon}}\rp)\frac{C_2+1}{C_2}\cdot\frac{\sqrt{n}}{kp} \notag\\ 
&\le (1+\alpha) \cdot \frac{C_2+1}{C_2}\cdot\frac{\sqrt{n}}{kp}, \label{eq:lemma7alpha}
\end{align}
where in the last inequality, we use that $\eps\leq \tau$ for sufficiently small constant $\tau>0$, and defined 
$$\alpha:=\frac{C_6}{\sqrt{n}}+\frac{C_7\tau}{\lp(1-\frac{C_5\tau}{k^2\sqrt{\logtwon}}\rp)^2k^2\sqrt{\logtwon}}.$$
Now, we pick $\tau$ to be sufficiently small and $n$ to be sufficiently large such that $(1+\alpha)\frac{C_2+1}{C_2} \le \frac{C_8+1}{C_8}$ for some constant $C_8 > 0$. Then, we conclude with the following upper bound:
\begin{align*}
    |s_{mk}^{(1)}| \le \frac{C_8+1}{C_8}\cdot\frac{\sqrt{n}}{kp}.
\end{align*}
 
\noindent\textbf{Bounds on $s_{kk}^{(1)}$.}~\Equation~\eqref{eq:pfinverseskk} in Appendix~\ref{sec:pfinversesecc} gives us
\begin{align}
s_{kk}^{(1)} = s_{kk}^{(0)} - \frac{1}{\det_0}\Big((\tn{\boldmu_1}^2 - t_{11}^{(0)}){s_{1k}^{(0)}}^2 + 2s_{1k}^{(0)}{h_{k1}^{(0)}}h_{11}^{(0)}-s_{11}^{(0)}{h_{k1}^{(0)}}^2 + 2s_{1k}^{(0)}h_{k1}^{(0)}\Big). \notag
\end{align}
First, we lower bound $s_{kk}^{(1)}$ by upper bounding $\frac{1}{\det_0}\Big((\tn{\boldmu_1}^2 - t_{11}^{(0)}){s_{1k}^{(0)}}^2\Big)$. Lemmas~\ref{lem:ineqforazero} and~\ref{lem:ineqforszero} yield 
\begin{align*}
    \frac{1}{\det_0}\Big((\tn{\boldmu_1}^2 - t_{11}^{(0)}){s_{1k}^{(0)}}^2\Big) \le \frac{(\tn{\boldmu_1}^2 - t_{11}^{(0)}){s_{1k}^{(0)}}^2}{s_{11}^{(0)}(\tn{ \boldmu_1}^2 - t_{11}^{(0)})+(1+h_{11}^{(0)})^2} \le \frac{(\tn{\boldmu_1}^2 - t_{11}^{(0)}){s_{1k}^{(0)}}^2}{s_{11}^{(0)}(\tn{ \boldmu_1}^2 - t_{11}^{(0)})} \le \frac{C_1}{n}\cdot\frac{n}{kp}.
\end{align*}
It suffices to upper bound the other dominant term $|s_{1k}^{(0)}h_{k1}^{(0)}|/\det_0$.
For this term, we have
\begin{align*}
    \Big|\frac{s_{1k}^{(0)}h_{k1}^{(0)}}{\det_{0}}\Big| \le \frac{|s_{1k}^{(0)}h_{k1}^{(0)}|}{(1+h_{11}^{(0)})^2} \le \frac{C_3\minrho\epsilonn}{\lp(1-\frac{C_4\minrho\epsilonn}{k^2\sqrt{\logtwon}}\rp)^2k^2\sqrt{\logtwon}}\cdot\frac{C_2+1}{C_2}\cdot\frac{\sqrt{n}}{kp}.
\end{align*}
Thus, we get
\begin{align*}
    s_{kk}^{(1)} \ge \lp(1- \frac{C_1}{{n}}- \frac{C_5\minrho\logtwon\epsilonn}{\lp(1-\frac{C_4\minrho\epsilonn}{k^2\sqrt{\logtwon}}\rp)^2k^2\sqrt{\logtwon}}\rp)\frac{C_6-1}{C_6}\cdot\frac{{n}}{kp} \ge (1-\alpha) \cdot \frac{C_6-1}{C_6}\cdot\frac{{n}}{kp}.
\end{align*}
Next, we upper bound $s_{kk}^{(1)}$ by a similar argument, and get
\begin{align*}
    s_{kk}^{(1)} &\le |s_{kk}^{(0)}|+ \frac{1}{\det_0}\Big|2s_{1k}^{(0)}{h_{k1}^{(0)}}h_{11}^{(0)}+s_{11}^{(0)}{h_{k1}^{(0)}}^2 + 2s_{1k}^{(0)}h_{k1}^{(0)}\Big|\\
    &\le \lp(1+ \frac{C_7\minrho\epsilonn}{\lp(1-\frac{C_4\minrho\epsilonn}{k^2\sqrt{\logtwon}}\rp)^2k^2\sqrt{\logtwon}}\rp)\frac{C_8+1}{C_8}\cdot\frac{{n}}{kp} \le (1+\alpha')\frac{C_8+1}{C_8}\cdot\frac{{n}}{kp},
\end{align*}
where we used $\frac{1}{\det_0}\Big((\tn{\boldmu_1}^2 - t_{11}^{(0)}){s_{1k}^{(0)}}^2\Big) > 0$ in the first step.
As above, we can tune $\epsilon$ and $n$ such that $(1+\alpha')\frac{C_8+1}{C_8} \le \frac{C_9+1}{C_9}$ and $(1-\alpha)\frac{C_6-1}{C_6} \ge \frac{C_9-1}{C_9}$ for sufficiently large constant $C_9 >0$.

\noindent\textbf{Bounds on $h_{mk}^{(1)}$.}~\Equation~\eqref{eq:pfinverseh} in Appendix~\ref{sec:pfinversesecc} gives us
\begin{align*}
    h_{mk}^{(1)} &= h_{mk}^{(0)} - \frac{1}{\det_0}(\star)_h^{(0)},
\end{align*}
where we define $$(\star)_h^{(0)}=(\tn{\boldmu_1}^2 - t_{11}^{(0)})s_{1m}^{(0)}h_{1k}^{(0)} + h_{m1}^{(0)}h_{1k}^{(0)}h_{11}^{(0)}+h_{m1}^{(0)}h_{1k}^{(0)} + s_{1m}^{(0)}t_{k1}^{(0)}+s_{1m}^{(0)}t_{k1}^{(0)}h_{11}^{(0)}-s_{11}^{(0)}t_{k1}^{(0)}h_{m1}^{(0)}.$$ 
We focus on the two dominant terms $((\tn{\boldmu_1}^2 - t_{11}^{(0)})s_{1m}^{(0)}h_{1k}^{(0)})/\det_0$ and $s_{1m}^{(0)}t_{k1}^{(0)}/\det_0$.
For the first dominant term $((\tn{\boldmu_1}^2 - t_{11}^{(0)})s_{1m}^{(0)}h_{1k}^{(0)})/\det_0$, Lemmas~\ref{lem:ineqforazero} and~\ref{lem:ineqforszero} yield
\begin{align*}
    \Big|\frac{1}{\det_0}\Big((\tn{\boldmu_1}^2 - t_{11}^{(0)})s_{1m}^{(0)}h_{1k}^{(0)}\Big)\Big| \le \frac{(\tn{\boldmu_1}^2 - t_{11}^{(0)})|s_{1m}^{(0)}h_{1k}^{(0)}|}{s_{11}^{(0)}(\tn{\boldmu_1}^2 - t_{11}^{(0)})} \le \Big|\frac{s_{1m}^{(0)}h_{1k}^{(0)}}{s_{11}^{(0)}}\Big| \le \frac{C_1}{\sqrt{n}}|h_{1k}^{(0)}| \le\frac{C_2\minrho\epsilonn}{k^2\sqrt{\logtwon}}.
\end{align*}
For the second dominant term $s_{1m}^{(0)}t_{k1}^{(0)}/\det_0$, we have
\begin{align*}
    \frac{1}{\det_0}s_{1m}^{(0)}t_{k1}^{(0)} \le \frac{|s_{1m}^{(0)}t_{k1}^{(0)}|}{(1+h_{11}^{(0)})^2} \le \frac{C_3n\sqrt{n}\tn{\boldmu}^2}{\lp(1-\frac{C_4\minrho\epsilonn}{k^2\sqrt{\logtwon}}\rp)^2kp^2} \le \frac{C_5\epsilonn}{\lp(1-\frac{C_4\minrho\epsilonn}{k^2\sqrt{n}}\rp)^2k^{1.5}\sqrt{n}}\cdot\frac{\minrho\epsilonn}{k^2\sqrt{\logtwon}},
\end{align*}
where we use the fact $1/\sqrt{k} < \minrho$ for $k > 1$.
Thus, we get
\begin{align*}
    |h_{mk}^{(1)}| &\le |h_{mk}^{(0)}| +\Big| \frac{1}{\det_0}(\star)_h^{(0)}\Big| \le \lp(1+ \frac{C_1}{\sqrt{n}}+ \frac{C_5\minrho\epsilonn}{\lp(1-\frac{C_4\minrho\epsilonn}{k^2\sqrt{\logtwon}}\rp)^2k^{1.5}\sqrt{n}}\rp)\frac{C_6\minrho\epsilonn}{k^2\sqrt{\logtwon}} \\
    &\le (1+\alpha)\frac{C_7\minrho\epsilonn}{k^2\sqrt{\logtwon}},
\end{align*}
and there exists constant $C_8$ such that $(1+\alpha)C_7 \le C_8$, which shows the desired upper bound.

\noindent\textbf{Bounds on $t_{kk}^{(1)}$.}~\Equation~\eqref{eq:pfinversetkk} in Appendix~\ref{sec:pfinversesecc} gives us
\begin{align}
t_{kk}^{(1)} = t_{kk}^{(0)} - \frac{1}{\det_0}\lp(\left(\tn{\boldmu_1}^2 - t_{11}^{(0)}\right){h_{1k}^{(0)}}^2 + 2t_{1k}^{(0)}{h_{1k}^{(0)}}h_{11}^{(0)}-s_{11}^{(0)}{t_{1k}^{(0)}}^2 + 2t_{1k}^{(0)}h_{1k}^{(0)}\rp). \notag
\end{align}
We only need an upper bound on $t_{kk}^{(1)}$.
The first dominant term $s_{11}^{(0)}{t_{1k}^{(0)}}^2/\det_0$ is upper bounded as follows:
\begin{align*}
  \frac{s_{11}^{(0)}{t_{1k}^{(0)}}^2}{\det_0} \le \frac{s_{11}^{(0)}{t_{1k}^{(0)}}^2}{(1+h_{11}^{(0)})^2} \le \frac{C_6n^3\tn{\boldmu}^4}{
\lp(1-\frac{C_3\minrho\epsilonn}{k^2\sqrt{\logtwon}}\rp)^2kp^3} \le \frac{C_7\epsilonn^2}{\lp(1-\frac{C_3\minrho\epsilonn}{k^2\sqrt{\logtwon}}\rp)^2pk^4{n}}\cdot\frac{n\tn{\boldmu}^2}{p}.
\end{align*}
Next, the second dominant term, $t_{1k}^{(0)}h_{1k}^{(0)}/\det_0$, is upper bounded as
\begin{align*}
    \frac{t_{1k}^{(0)}h_{1k}^{(0)}}{\det_0} \le \frac{|t_{1k}^{(0)}h_{1k}^{(0)}|}{(1+h_{11}^{(0)})^2} \le \frac{C_8\minrho\epsilonn}{k^2\sqrt{\logtwon}\lp(1-\frac{C_3\minrho\epsilonn}{k^2\sqrt{\logtwon}}\rp)^2}\cdot\frac{n\tn{\boldmu}^2}{p}.
\end{align*}
Combining the results above gives us
\begin{align*}
    t_{kk}^{(1)} &\le t_{kk}^{(0)} + \frac{1}{\det_0}\Big|  2t_{1k}^{(0)}{h_{1k}^{(0)}}h_{11}^{(0)} + s_{11}^{(0)}{t_{1k}^{(0)}}^2 +  2t_{1k}^{(0)}h_{1k}^{(0)}\Big|\\ 
    &\le \lp(1+ \frac{C_9\minrho\epsilonn}{\lp(1-\frac{C_3\minrho\epsilonn}{k^2\sqrt{\logtwon}}\rp)^2k^2\sqrt{\logtwon}}\rp)\frac{{n\tn{\boldmu}^2}}{p} \le \frac{{C_{5}n\tn{\boldmu}^2}}{p}.
\end{align*}
This shows the desired upper bound.

\noindent\textbf{Bounds on $t_{mk}^{(1)}$.}~\Equation~\eqref{eq:pfinversetkm} in Appendix~\ref{sec:pfinversesecc} gives us
\begin{align*}
   t_{mk}^{(1)} &= t_{mk}^{(0)} - \frac{1}{\det_0}(\star)_t^{(0)},
\end{align*}
where we define $$(\star)_t^{(0)} = (\tn{\boldmu_1}^2 - t_{11}^{(0)})h_{1m}^{(0)}h_{1k}^{(0)} + t_{m1}^{(0)}h_{1k}^{(0)}h_{11}^{(0)}+t_{k1}^{(0)}h_{1m}^{(0)}h_{11}^{(0)} + t_{1m}^{(0)}h_{1k}^{(0)}+t_{1k}^{(0)}h_{1m}^{(0)}-s_{11}^{(0)}t_{1m}^{(0)}t_{1k}^{(0)}.$$
Again, we only need an upper bound on $t_{mk}^{(1)}$.
As in the previously derived bounds, we have
\begin{align*}
    \frac{1}{\det_0}(\tn{\boldmu_1}^2 - t_{11}^{(0)}){h_{1m}^{(0)}}{h_{1k}^{(0)}} \le \frac{(\tn{\boldmu}^2 - t_{11}^{(0)})|{h_{1m}^{(0)}}{h_{1k}^{(0)}}|}{s_{11}^{(0)}(\tn{ \boldmu}^2 - t_{11}^{(0)})} \le \frac{C_1\minrho^2n^2\tn{\boldmu}^2}{kp^2}\cdot\frac{kp}{n} \le \frac{C_1n\tn{\boldmu}^2}{p}.
\end{align*}
The other dominant term $t_{1m}^{(0)}h_{1m}^{(0)}/\det_0$ is upper bounded as:
\begin{align*}
    \frac{t_{1m}^{(0)}h_{1m}^{(0)}}{\det_0} \le \frac{|t_{1m}^{(0)}h_{1m}^{(0)}|}{(1+h_{11}^{(0)})^2} \le \frac{C_2\minrho\epsilonn}{k^2\sqrt{\logtwon}(\lp(1-\frac{C_3\minrho\epsilonn}{k^2\sqrt{\logtwon}}\rp)^2}\cdot\frac{n\tn{\boldmu}^2}{p}.
\end{align*}
Combining the results above yields
\begin{align*}
    |t_{mk}^{(1)}| &\le |t_{mk}^{(0)}| + \frac{1}{\det_0}\Big|(\star)_t^{(0)}\Big|\\
    &\le \lp(C_1+ \frac{C_2\minrho\epsilonn}{\lp(1-\frac{C_3\minrho\epsilonn}{k^2\sqrt{\logtwon}}\rp)^2k^2\sqrt{\logtwon}}\rp)\frac{{n\tn{\boldmu}^2}}{p} \le \frac{{C_{4}n\tn{\boldmu}^2}}{p}.
\end{align*}
Note that both $t_{kk}^{(0)}$ and $t_{mk}^{(0)}$ are much smaller than $\tn{\boldmu}^2$.
The above upper bound shows that this continues to hold for $t_{kk}^{(1)}$ and $t_{mk}^{(1)}$ since $p \gg n$.

\noindent\textbf{Bounds on $f_{ki}^{(1)}$.}~
Consider $i \in [n]$ and fix $y_i = k$ without loss of generality.
\Equation~\eqref{eq:pfinversef} in Appendix~\ref{sec:pfinversesecc} gives us
\begin{align}
\label{eq:pfineqfone01}
    f_{ki}^{(1)} &= f_{ki}^{(0)} - \frac{1}{\det_0}(\star)_f^{(0)},
\end{align}
where we define
\begin{align}
\label{eq:pfineqfone02}
    (\star)_f^{(0)} = (\tn{\boldmu_1}^2 - t_{11}^{(0)})h_{1k}^{(0)}g_{1i}^{(0)} + t_{1k}^{(0)}g_{1i}^{(0)}+t_{1k}^{(0)}h_{11}^{(0)}g_{1i}^{(0)} + h_{1k}^{(0)}f_{1i}^{(0)}+h_{1k}^{(0)}h_{11}^{(0)}f_{1i}^{(0)}-s_{11}^{(0)}t_{1k}^{(0)}f_{1i}^{(0)}.
\end{align}
We only need an upper bound on $f_{ki}^{(1)}$.
We consider the dominant terms $(\tn{\boldmu_1}^2 - t_{11}^{(0)})h_{1k}^{(0)}g_{1i}^{(0)}/\det_0$, $t_{1k}^{(0)}g_{1i}^{(0)}/\det_0$, $h_{1k}^{(0)}f_{1i}^{(0)}/\det_0$ and $s_{11}^{(0)}t_{1k}^{(0)}f_{1i}^{(0)}/\det_0$. Lemmas~\ref{lem:ineqforazero},~\ref{lem:ineqforszero} and~\ref{lem:ineqforgzero} give us
\begin{align*}
    &\frac{(\tn{\boldmu_1}^2 - t_{11}^{(0)})h_{1k}^{(0)}g_{1i}^{(0)}}{\det_0} \le \frac{(\tn{\boldmu_1}^2 - t_{11}^{(0)})|h_{1k}^{(0)}g_{1i}^{(0)}|}{(\tn{\boldmu_1}^2 - t_{11}^{(0)})s_{11}^{(0)}}\\ 
    &\le \frac{C_1\minrho n\tn{\boldmu}}{\sqrt{k}p}\cdot\frac{1}{C_2k^2p}\cdot\frac{kp}{n}\le \frac{C_3}{k^{1.5}\sqrt{\logtwon}}\cdot\frac{\sqrt{\logtwon}\tn{\boldmu}}{p}, \\
    &\frac{t_{1k}^{(0)}g_{1i}^{(0)}}{\det_0} \le \frac{|t_{1k}^{(0)}g_{1i}^{(0)}|}{(1+h_{11}^{(0)})^2} \le \frac{C_4n\tn{\boldmu}^2}{
\lp(1-\frac{C_5\minrho \epsilon}{k^2\sqrt{\logtwon}}\rp)^2k^2p^2} \le \frac{C_7\epsilonn}{\lp(1-\frac{C_5\minrho\epsilonn}{k^2\sqrt{\logtwon}}\rp)^2k^{3.5}{{\logtwon}}}\cdot\frac{\sqrt{\logtwon}\tn{\boldmu}}{p},\\
   & \frac{h_{1k}^{(0)}f_{1i}^{(0)}}{\det_0} \le \frac{|h_{1k}^{(0)}f_{1i}^{(0)}|}{(1+h_{11}^{(0)})^2}  \le \frac{C_6\minrho\epsilonn}{\lp(1-\frac{C_5\minrho\epsilonn}{k^2\sqrt{\logtwon}}\rp)^2k^2\sqrt{\logtwon}}\cdot\frac{\sqrt{\logtwon}\tn{\boldmu}}{p}, \text{ and } \\
   & \frac{s_{11}^{(0)}t_{1k}^{(0)}f_{1i}^{(0)}}{\det_0} \le \frac{|s_{11}^{(0)}t_{1k}^{(0)}f_{1i}^{(0)}|}{(1+h_{11}^{(0)})^2}  \le \frac{C_7\epsilonn^2}{k^4{\logtwon}\lp(1-\frac{C_5\minrho\epsilonn}{k^2\sqrt{\logtwon}}\rp)^2}\cdot\frac{\sqrt{\logtwon}\tn{\boldmu}}{p},
\end{align*}
where, in the last two steps, {we used the  upper bound ${C\sqrt{\logtwon}\tn{ \boldmu}}/{p}$ for $|f_{ji}^{(0)}|$ and previously derived bounds on $|h_{1k}^{(0)}|$ and $|s_{11}^{(0)}t_{1k}^{(0)}|$}. Thus, we have
\begin{align*}
    |f_{ki}^{(1)}| &\le |f_{ki}^{(0)}| + \Big| \frac{1}{\det_0}(\star)_f^{(0)}\Big| \\
    &\le \lp(1+ \frac{C_3}{k^{1.5}\sqrt{\logtwon}}+ \frac{C_8\epsilonn}{\lp(1-\frac{C_5\minrho\epsilonn}{k^2\sqrt{\logtwon}}\rp)^2k^2\sqrt{\logtwon}}\rp)\frac{C_9\sqrt{\logtwon}\tn{\boldmu}}{p} \\
    &\le (1+\alpha)\frac{C_{10}\epsilonn}{k^{1.5}{n}},
\end{align*}
and we have $(1+\alpha)C_{10} \le C_{11}$ for a large enough positive constant $C_{11}$.
This shows the desired upper bound.

\noindent\textbf{Bounds on $g_{ki}^{(1)}$ and $g_{mi}^{(1)}$.}~ 
\Equation~\eqref{eq:pfinverseg} in Appendix~\ref{sec:pfinversesecc} gives
\begin{align}
\label{eq-gisteptwo}
 z_{ci}\mathbf{e}_i^T\bolda_1^{-1}\boldul_k = |z_{ci}|^2\Big(\mathbf{e}_i^T\bolda_0^{-1}\vb_k - \frac{1}{\det_0}(\star)_{gk}^{(0)}\Big) = |z_{ci}|^2\Big(g_{ki}^{(0)} - \frac{1}{\det_0}(\star)_{gk}^{(0)}\Big) ,
\end{align}
where we define
$$(\star)_{gk}^{(0)} = (\Vert \boldmu_1 \Vert_2^2 - t_{11}^{(0)})s_{1k}^{(0)}g_{1i}^{(0)} + g_{1i}^{(0)}h_{11}^{(0)}h_{k1}^{(0)}+g_{1i}^{(0)}h_{k1}^{(0)}+s_{1k}^{(0)}f_{1i}^{(0)}+s_{1k}^{(0)}h_{11}^{(0)}f_{1i}^{(0)} - s_{11}^{(0)}h_{k1}^{(0)}f_{1i}^{(0)}.$$
Lemmas~\ref{lem:ineqforazero},~\ref{lem:ineqforszero} and~\ref{lem:ineqforgzero} give us
\begin{align*}
   &\frac{(\tn{\boldmu_1}^2 - t_{11}^{(0)})|s_{1k}^{(0)}g_{1i}^{(0)}|}{\det_0} \le \frac{(\tn{\boldmu_1}^2 - t_{11}^{(0)})|s_{1k}^{(0)}g_{1i}^{(0)}|}{(\tn{\boldmu_1}^2 - t_{11}^{(0)})s_{11}^{(0)}} \le \frac{C_1}{\sqrt{n}}\cdot\frac{1}{C_2k^2p}, \\
   & \frac{|h_{k1}^{(0)}g_{1i}^{(0)}|}{\det_0} \le \frac{|h_{k1}^{(0)}g_{1i}^{(0)}|}{(1+h_{11}^{(0)})^2}  \le \frac{C_3\minrho\epsilonn}{\lp(1-\frac{C_4\minrho\epsilonn}{k^2\sqrt{\logtwon}}\rp)^2k^2\sqrt{\logtwon}}\cdot\frac{1}{C_2k^2p},\text{ and } \\
   & \frac{|s_{1k}^{(0)}f_{1i}^{(0)}|}{\det_0} \le \frac{|s_{1k}^{(0)}f_{1i}^{(0)}|}{(1+h_{11}^{(0)})^2}  \le \frac{C_5\epsilonn}{\lp(1-\frac{C_4\minrho\epsilonn}{k^2\sqrt{\logtwon}}\rp)^2\sqrt{k}\sqrt{n}}\cdot\frac{1}{C_2k^2p}.
\end{align*}
We then have
\begin{align*}
    g_{ki}^{(1)} &\ge g_{ki}^{(0)} - \frac{1}{\det_0}|(\star)_{gk}^{(0)}|
    \ge \lp(1-\frac{1}{C}\rp)\lp(1-\frac{C_1}{k^2\sqrt{n}}- \frac{C_6\epsilonn}{\lp(1-\frac{C_4\minrho\epsilonn}{k^2\sqrt{\logtwon}}\rp)^2k^{2.5}\sqrt{n}}\rp)\frac{1}{p}\ge \lp(1-\frac{1}{C}\rp)\frac{1-\alpha}{p}\\
    g_{ki}^{(1)} &\le  g_{ki}^{(0)} + \frac{1}{\det_0}|(\star)_{gk}^{(0)}| \le \lp(1+\frac{1}{C}\rp)\lp(1+\frac{C_1}{k^2\sqrt{n}}+ \frac{C_7\epsilonn}{\lp(1-\frac{C_4\minrho\epsilonn}{k^2\sqrt{\logtwon}}\rp)^2k^{2.5}\sqrt{n}}\rp)\frac{1}{p} \le \lp(1+\frac{1}{C}\rp)\frac{1+\alpha}{p},
\end{align*}
where for large enough $n$ and positive constant $C_9$, we have $(1+\alpha)\frac{C+1}{C} \le \frac{C_9+1}{C_9}$ and $(1-\alpha)\frac{C-1}{C} \ge \frac{C_9-1}{C_9}$.
Similarly, for the case $m \ne k$, we have
\begin{align}
\label{eq-gisteptwo02}
 z_{ci}\mathbf{e}_i^T\bolda_1^{-1}\boldul_{m} = |z_{ci}|^2\lp(\mathbf{e}_i^T\bolda_0^{-1}\vb_{m} - \frac{1}{\det_0}(\star)_{gm}^{(0)}\rp) = |z_{ci}|^2\Big(g_{mi}^{(0)} - \frac{1}{\det_0}(\star)_{gm}^{(0)}\Big),
\end{align}
where we define
$$(\star)_{gm}^{(0)} = (\Vert \boldmu_1 \Vert_2^2 - t_{11}^{(0)})s_{1m}^{(0)}g_{1i}^{(0)} + g_{1i}^{(0)}h_{11}^{(0)}h_{m1}^{(0)}+g_{1i}^{(0)}h_{m1}^{(0)}+s_{1m}^{(0)}f_{1i}^{(0)}+s_{1m}^{(0)}h_{11}^{(0)}f_{1i}^{(0)} - s_{11}^{(0)}h_{m1}^{(0)}f_{1i}^{(0)}.$$
As a consequence of our nearly equal energy and priors assumption (Assumption~\ref{ass:equalmuprob}), we can directly use the bounds of the terms in $(\star)_{gk}^{(0)}$ to bound terms in $(\star)_{gm}^{(0)}$.
We get
\begin{align*}
   |g_{mi}^{(1)}| \le \frac{1}{C}\lp(1+\frac{C_1}{\sqrt{n}}+ \frac{C_8\epsilonn}{(1-(\frac{C_4\minrho\epsilonn}{k^2\sqrt{\logtwon}}))^2\sqrt{k}\sqrt{n}}\rp)\frac{1}{k^2p} \le \frac{1}{C}\cdot\frac{1+\alpha}{k^2p}.
\end{align*}
Finally, there exists a sufficiently large constant $C_{10}$ such that $(1+\alpha)/C \le 1/C_{10}$.
This shows the desired bounds.

\subsubsection{Completing the proof for k-th order quadratic forms}
{Notice from the above analysis that the $1$-st order quadratic forms exhibit the same order-wise dependence on $n,k$ and $p$ as the $0$-th order quadratic forms, e.g. both $s_{mk}^{(0)}$ and $s_{mk}^{(1)}$ are of order $\Theta(\frac{\sqrt{n}}{kp})$.
Thus, the higher-order quadratic forms that arise by including more mean components will not change too much\footnote{There are several low-level reasons for this. 
One critical reason is the aforementioned orthogonality of the label indicator vectors $\{\vb_c\}_{c \in [k]}$, which ensures by Lemma~\ref{lem:ineqforszero} that the cross-terms $|s_{mk}^{(j)}|$ are always dominated by the larger terms $|s_{kk}^{(j)}|$. Another reason is that $h_{mk}^{(0)}$, which can be seen as the ``noise'' term in our analysis, is small and thus does not affect other terms.
}.
By \Equation~\eqref{eq:addnewmean}, we can see that we can bound the $2$-nd order quadratic forms by bounding quadratic forms with order $1$. We consider $s_{mk}^{(2)}$ as an example:
\begin{align*}
s_{mk}^{(2)} 
&= s_{mk}^{(1)} - \frac{1}{\det_1}(\star)_{s}^{(1)}, 
\end{align*}
where 
\begin{align*}
    (\star)_{s}^{(1)} &:= (\tn{\boldmu_2}^2 - t_{22}^{(1)})s_{2k}^{(1)}s_{2m}^{(1)} + s_{2m}^{(1)}h_{k2}^{(1)}h_{22}^{(1)}+s_{2k}^{(1)}h_{m2}^{(1)}h_{22}^{(1)}-s_{22}^{(1)}h_{k2}^{(1)}h_{m2}^{(1)} + s_{2m}^{(1)}h_{k2}^{(1)}+s_{2k}^{(1)}h_{m2}^{(1)},\\
    {\det}_1 &:= s_{22}^{(1)}(\tn{\boldmu_2}^2 - t_{22}^{(1)})+(1+h_{22}^{(1)})^2.
\end{align*}
We additionally show how $f_{ki}^{(2)}$ relates to the $1$-st order quadratic forms: 
\begin{align*}
    f_{ki}^{(2)} &= f_{ki}^{(1)} - \frac{1}{\det_1}(\star)_f^{(1)},
\end{align*}
where we define
$$(\star)_f^{(1)} = (\tn{\boldmu_2}^2 - t_{22}^{(1)})h_{2k}^{(1)}g_{2i}^{(1)} + t_{2k}^{(1)}g_{2i}^{(1)}+t_{2k}^{(1)}h_{22}^{(1)}g_{2i}^{(1)} + h_{2k}^{(1)}f_{2i}^{(1)}+h_{2k}^{(1)}h_{22}^{(1)}f_{2i}^{(1)}-s_{22}^{(1)}t_{2k}^{(1)}f_{2i}^{(1)}.$$ 
Observe that the equations above are very similar to Equations~\eqref{eq:pfsijzero01} and~\eqref{eq:pfsijzero02} (for $s$), and Equations~\eqref{eq:pfineqfone01} and~\eqref{eq:pfineqfone02} (for $f$), except that the quadratic forms are in terms of Gram matrix $\bolda_1$. We have shown that the quadratic forms with order $1$ will not be drastically different different from the quadratic forms with order $0$. Hence, we repeat the above procedures of bounding these quadratic forms $k-1$ times to obtain the desired bounds in Lemma~\ref{lem:ineqforanoj}. 
The only quantity that will change in each iteration is $\alpha$, which nevertheless remains negligible\footnote{To see this, recall that in the first iteration we had $\alpha_1 := \alpha =\frac{C_1}{\sqrt{n}}+\frac{C_2\tau}{(1-(C_5\tau/(k^2\sqrt{n})))^2k^2\sqrt{n}}$ for the first-order terms. 
Thus, even if we repeat the procedure $k-1$ times, then we have $\alpha_k \leq Ck\alpha_1$, which remains small since we consider $n \gg k$.}.} 

Our analysis so far is conditioned on event $\evente_q$. 
We define the $\textit{unconditional}$ event $\evente_u:=$ $\{$all the inequalities in Lemma~\ref{lem:ineqforanoj} hold$\}$.
Then, we have
\begin{align*}
    \Pro(\evente_u^c) &\le \Pro(\evente_u^c|\evente_q)+\Pro(\evente_q^c)\le \Pro(\evente_u^c|\evente_q)+\Pro(\evente_q^c|\evente_v)+\Pro(\evente_v^c)\\
    & \le \frac{c_1}{kn} + \frac{c_2}{n} + c_3k(e^{-\frac{n}{c_4}}+e^{-\frac{n}{c_5k^2}})\\
    & \le \frac{c_6}{n}+c_7ke^{-\frac{n}{c_5k^2}},
\end{align*}
for constants $c_i$'s $>1$. This completes the proof.

\subsection{Proofs of Auxiliary lemmas 
} \label{sec:pfauxlemma}

We complete this section by proving the auxiliary Lemmas~\ref{lem:vnorm}, \ref{lem:ineqforszero} and~\ref{lem:ineqforgzero}, which were used in the proof of Lemma~\ref{lem:ineqforanoj}.

\subsubsection{Proof of Lemma~\ref{lem:vnorm}}
Our goal is to upper and lower bound $\tn{\vb_c}^2$, for $c \in [k]$. Note that every entry of $\vb_{c}$ is either $1$ or $0$, hence these entries are independent sub-Gaussian random variables with sub-Gaussian parameter $1$ \cite[Chapter 2]{wainwright2019high}. Recall that under the {nearly equal-prior Assumption~\ref{ass:equalmuprob}, we have $(1-(1/C_1))(n/k) \le \E[\tn{\vb_c}^2] \le (1+(1/C_2))(n/k)$ for large enough constants $C_1,C_2 > 0$.
Thus, a straightforward application of Hoeffding's concentration inequality on bounded random variables~\cite[Chapter 2]{wainwright2019high} gives us
\begin{align*}
    \Pro\lp(\lp|\tn{\vb_c}^2 - \E[\tn{\vb_c}^2]\rp| \ge t\rp) \le 2\exp\lp(-\frac{t^2}{2n}\rp).
\end{align*}
We complete the proof by setting $t = \frac{n}{C_3k}$ for a large enough constant $C_3$ and applying the union bound over all $c \in [k]$.}

\subsubsection{Proof of Lemma~\ref{lem:ineqforszero}}
We use the following lemma adapted from \cite[Lemma 2]{muthukumar2021classification} to bound quadratic forms of inverse Wishart matrices.
\begin{lemma}\label{lem:wishartineq}
Define $\dprimen := (p-n+1)$, and consider matrix $\boldm \sim \text{Wishart}(p, \mathbf{I}_n)$. For any unit Euclidean norm vector $\vb$ and any $t >0$, we have
\begin{align*}
    \Pro\Big(\frac{1}{\vb^T\boldm^{-1}\vb} > \dprimen + \sqrt{2t\dprimen}+2t\Big) \le e^{-t}\quad\text{and}\quad
    \Pro\Big(\frac{1}{\vb^T\boldm^{-1}\vb} < \dprimen - \sqrt{2t\dprimen}\Big) \le e^{-t},
\end{align*}
provided that $\dprimen > 2\max\{t, 1\}$.
\end{lemma}

We first upper and lower bound $s_{cc}^{(0)}$ for a fixed $c \in [k]$. Recall that we assume $p > Cn\log(kn)+n-1$ for sufficiently large constant $C >1$ and this can be obtained by assuming $\dprimen > Cn\log(kn)$. Let $t = 2\log(kn)$. Working on the event $\mathcal{E}_v$ defined in~\eqref{eq:eventvnorm}, Lemma~\ref{lem:wishartineq} gives us
\begin{align*}
    s_{cc}^{(0)} \le \frac{\tn{\vb_c}^2}{\dprimen-\sqrt{4\log(kn)\dprimen}} \le \frac{C_1+1}{C_1}\cdot\frac{n/k}{\dprimen\lp(1-\frac{2}{\sqrt{Cn}}\rp)} \le \frac{C_2+1}{C_2}\cdot\frac{n}{kp}
\end{align*}
with probability at least $1 - \frac{2}{k^2n^2}$.
Here, the last inequality comes from the fact that $p$ is sufficiently large compared to $n$ and $C$ is large enough. 
Similarly, for the lower bound, we have 
\begin{align*}
    s_{cc}^{(0)} \ge \frac{\tn{\vb_c}^2}{\dprimen+\sqrt{4\log(kn)\dprimen}+2\log(kn)} \ge \frac{C_1-1}{C_1}\cdot\frac{n/k}{\dprimen\lp(1+\frac{4}{\sqrt{Cn}}\rp)} \ge \frac{C_2-1}{C_2}\cdot\frac{n}{kp}
\end{align*}
with probability $1 - \frac{2}{k^2n^2}$.

Now we upper and lower bound $s_{cj}^{(0)}$ for a fixed choice $j \ne c \in [k]$. 
We use the parallelogram law to get
\begin{align*}
    \vb_c^T\bolda_0^{-1}\vb_j = \frac{1}{4}\Big((\vb_c+\vb_j)^T\bolda_0^{-1}(\vb_c+\vb_j) - (\vb_c-\vb_j)^T\bolda_0^{-1}(\vb_c-\vb_j)\Big).
\end{align*}
Because of the orthogonality of the label indicator vectors ($\vb_c^T\vb_j=0$ for any $j \ne c$), we have $\tn{\vb_c+\vb_j}^2 = \tn{\vb_c-\vb_j}^2$, which we denote by $\tilde{n}$ as shorthand.
Then, we have 
\begin{align*}
   \vb_c^T\bolda_0^{-1}\vb_j &\le \frac{1}{4} \lp(\frac{\tilde{n}}{\dprimen-\sqrt{4\log({k}n)\dprimen}} - \frac{\tilde{n}}{\dprimen+\sqrt{4\log({k}n)\dprimen} + 4\log({k}n)}\rp) \notag\\
    &\le \frac{1}{4} \cdot\frac{ 2\tilde{n}\sqrt{4\log({k}n)\dprimen} + 4\tilde{n}\log(kn)}{(\dprimen-\sqrt{4\log({k}n)\dprimen})(\dprimen+\sqrt{4\log({k}n)\dprimen})} \notag\\
    &\le \frac{C_1+1}{2C_1k} \cdot\frac{ 2{n}\sqrt{4\log({k}n)\dprimen} + 4{n}\log(kn)}{(\dprimen-\sqrt{4\log({k}n)\dprimen})(\dprimen+\sqrt{4\log({k}n)\dprimen})} \notag
\end{align*}
with probability at least $1 - \frac{2}{k^2n^2}$
Here, the last inequality follows because we have  $\tilde{n} \le \frac{2(C_1 + 1)}{C_1}\cdot\frac{n}{k}$ on $\evente_v$. 
Because $\dprimen > Cn\log(kn)$, we have
\begin{align*}
    \vb_c^T\bolda_0^{-1}\vb_j &\le \frac{C_1+1}{2C_1k}\cdot\frac{2\sqrt{n}\dprimen \cdot \sqrt{4/C} + 4/C \cdot \dprimen}{\lp(1-\sqrt{4/(Cn)}\rp)\dprimen^2} \notag\\ 
    &\le \frac{C_1+1}{2C_1}\cdot\frac{\sqrt{n}}{{k}}\cdot\frac{2\sqrt{4/C} + \sqrt{4/(Cn)}}{\dprimen(1-\sqrt{4/(Cn)})} \\
     &\le \frac{C_2+1}{C_2}\cdot\frac{\sqrt{n}}{kp},
\end{align*}
where in the last step we use the fact that $C >1$ is large enough.
To lower bound $s_{cj}^{(0)}$, we get 
\begin{align*}
   \vb_c^T\bolda_0^{-1}\vb_j &\ge \frac{1}{4} \lp(\frac{\tilde{n}}{(\dprimen+\sqrt{4\log({k}n)\dprimen} + 4\log({k}n))} - \frac{\tilde{n}}{(\dprimen-\sqrt{4\log({k}n)\dprimen})}\rp) \notag\\
    &\ge \frac{1}{4} \cdot\frac{ -2\tilde{n}\sqrt{4\log({k}n)\dprimen} - 4\tilde{n}\log(kn)}{(\dprimen-\sqrt{4\log({k}n)\dprimen})(\dprimen+\sqrt{4\log({k}n)\dprimen})} \notag\\
    &\ge -\frac{C_1+1}{2C_1k} \cdot\frac{ 2{n}\sqrt{4\log({k}n)\dprimen} + 4{n}\log(kn)}{(\dprimen-\sqrt{4\log({k}n)\dprimen})(\dprimen+\sqrt{4\log({k}n)\dprimen})} \notag
\end{align*}
with probability at least $1 - \frac{2}{k^2n^2}$.
Then following similar steps to the upper bound of $\vb_c^T\bolda_0^{-1}\vb_j$ gives us
\begin{align*}
    \vb_c^T\bolda_0^{-1}\vb_j &\ge -\frac{C_1+1}{2C_1k}\cdot\frac{2\sqrt{n}\dprimen\sqrt{4/C} + (4/C)\dprimen}{(\dprimen-\sqrt{4/(Cn)}\dprimen)\dprimen} \notag\\ 
    &\ge -\frac{C_1+1}{2C_1}\cdot\frac{\sqrt{n}}{{k}}\cdot\frac{2\sqrt{4/C} + (4/C\sqrt{n})}{\dprimen(1-\sqrt{4/(Cn)})} \\
     &\ge -\frac{C_2+1}{C_2}\cdot\frac{\sqrt{n}}{kp}.
\end{align*}
We finally apply the union bound on all pairs of $c, j \in [k]$ and complete the proof.

\subsubsection{Proof of Lemma~\ref{lem:ineqforgzero}}
We first lower and upper bound $g_{(y_i)i}^{(0)}$. Recall that we assumed $y_i = k$ without loss of generality. With a little abuse of notation, we define $\tn{\vb_k}^2 = \tilde{n}$ and $\boldul := \sqrt{\tilde{n}}\mathbf{e}_i$. 
We use the parallelogram law to get
\begin{align*}
    \mathbf{e}_i^T\bolda_0^{-1}\vb_k = \frac{1}{4\sqrt{\tilde{n}}}\lp((\boldul+\vb_k)^T\bolda_0^{-1}(\boldul+\vb_k) - (\boldul-\vb_k)^T\bolda_0^{-1}(\boldul-\vb_k)\rp).
\end{align*}
Note that $\tn{\boldul+\vb_k}^2  = 2(\tilde{n}+\sqrt{\tilde{n}})$ and $\tn{\boldul-\vb_k}^2  = 2(\tilde{n}-\sqrt{\tilde{n}})$. 
As before, we apply Lemma~\ref{lem:wishartineq} with $t = 2\log(kn)$ to get with probability at least $1 - \frac{2}{k^2n^2}$,
\begin{align*}
   \mathbf{e}_i^T\bolda_0^{-1}\vb_k &\ge \frac{1}{4\sqrt{\tilde{n}}} \lp(\frac{2(\tilde{n}+\sqrt{\tilde{n}})}{(\dprimen+\sqrt{4\log({k}n)\dprimen} + 4\log({k}n))} - \frac{2(\tilde{n}-\sqrt{\tilde{n}})}{(\dprimen-\sqrt{4\log({k}n)\dprimen})}\rp) \notag\\
    &\ge \frac{1}{4\sqrt{\tilde{n}}}\cdot \frac{4\sqrt{\tilde{n}}\dprimen- 4\tilde{n}\sqrt{4\log({k}n)\dprimen} - 8\tilde{n}\log(kn)}{(\dprimen+\sqrt{4\log({k}n)\dprimen}+4\log(kn))\dprimen} \notag\\
    &\ge \frac{\dprimen- \sqrt{\tilde{n}}\sqrt{4\log({k}n)\dprimen} - 2\sqrt{\tilde{n}}\log(kn)}{(\dprimen+\sqrt{4\log({k}n)\dprimen}+4\log(kn))\dprimen}, \notag\\
    & \ge \frac{\dprimen- \sqrt{(1+1/C_1)n/k}\sqrt{4\log({k}n)\dprimen} - 2\sqrt{(1+1/C_1)n/k}\log(kn)}{(\dprimen+\sqrt{4\log({k}n)\dprimen}+4\log(kn))\dprimen}. \notag 
\end{align*}
The last inequality works on event $\evente_v$, by which we have $\tilde{n} \le \frac{2(C_1 + 1)n}{C_1k}$.
Then, $\dprimen > Ck^3n\log(kn)$ gives us
\begin{align*}
    \mathbf{e}_i^T\bolda_0^{-1}\vb_k &\ge \frac{\dprimen - \sqrt{(1+1/C_1)n/k}\sqrt{4/(Ck^3n)}\dprimen -\sqrt{(1+1/C_1)n/k} (2/Ck^3n)\dprimen}{(\dprimen+\sqrt{4\log({k}n)\dprimen}+4\log(kn))\dprimen} \notag\\ 
    &\ge \frac{1- (1/(C_2\sqrt{k^4})) - (1/(C_3k^{3.5}\sqrt{n}))}{\dprimen(1+2\sqrt{4/(Ck^{3}n)})} \\
     &\ge \frac{C_4-1}{C_4}\cdot\frac{1}{p},
\end{align*}
where in the last step we use the fact that $C, C_2, C_3 >1$ are large enough.
To upper bound $g_{(y_i)i}^{(0)}$, we have with probability at least $1 - \frac{2}{k^2n^2}$,
\begin{align*}
   \mathbf{e}_i^T\bolda_0^{-1}\vb_k &\le \frac{1}{4\sqrt{\tilde{n}}} \lp(\frac{2(\tilde{n}+\sqrt{\tilde{n}})}{(\dprimen-\sqrt{4\log({k}n)\dprimen})} - \frac{2(\tilde{n}-\sqrt{\tilde{n}})}{(\dprimen+\sqrt{4\log({k}n)\dprimen} + 4\log({k}n))}\rp) \notag\\
    &\le \frac{1}{4\sqrt{\tilde{n}}}\cdot \frac{4\sqrt{\tilde{n}}\dprimen+ 4\tilde{n}\sqrt{4\log({k}n)\dprimen} + 8\tilde{n}\log(kn)}{(\dprimen-\sqrt{4\log({k}n)\dprimen})\dprimen} \notag\\
    &\le \frac{\dprimen+ \sqrt{\tilde{n}}\sqrt{4\log({k}n)\dprimen} + 2\sqrt{\tilde{n}}\log(kn)}{(\dprimen-\sqrt{4\log({k}n)\dprimen})\dprimen}, \notag\\
    & \le \frac{\dprimen+ \sqrt{(1+1/C_1)n/k}\sqrt{4\log({k}n)\dprimen} + 2\sqrt{(1+1/C_1)n/k}\log(kn)}{(\dprimen-\sqrt{4\log({k}n)\dprimen})\dprimen}. \notag 
\end{align*}
Then $\dprimen > Ck^3n\log(kn)$ gives us
\begin{align*}
    \mathbf{e}_i^T\bolda_0^{-1}\vb_k &\le \frac{\dprimen + \sqrt{(1+1/C_1)n/k}\sqrt{4/(Ck^3n)}\dprimen +2\sqrt{(1+1/C_1)n/k} (4/Ck^3n)\dprimen}{(\dprimen-\sqrt{4\log({k}n)\dprimen})\dprimen} \notag\\ 
    &\le \frac{1+ (1/(C_2\sqrt{k^4})) + (1/(C_3k^{3.5}\sqrt{n}))}{\dprimen(1-2\sqrt{4/(Ck^{3}n)})} \\
     &\le \frac{C_4+1}{C_4}\cdot\frac{1}{p}.
\end{align*}
We now upper and lower bound $g_{ji}^{(0)}$ for a fixed $j \ne y_i$. 
As before, we have
\begin{align*}
    \mathbf{e}_i^T\bolda_0^{-1}\vb_j = \frac{1}{4\sqrt{\tilde{n}}}\lp((\boldul+\vb_j)^T\bolda_0^{-1}(\boldul+\vb_j) - (\boldul-\vb_j)^T\bolda_0^{-1}(\boldul-\vb_j)\rp).
\end{align*}
Since $\mathbf{e}_i^T\vb_j =0$, we now have $\tn{\boldul+\vb_j}^2  = \tn{\boldul-\vb_j}^2= 2\tilde{n}$. We apply Lemma~\ref{lem:wishartineq} with $t = 2\log(kn)$ to get, with probability at least $1 - \frac{2}{k^2n^2}$,
\begin{align*}
   \mathbf{e}_i^T\bolda_0^{-1}\vb_j &\le \frac{1}{4\sqrt{\tilde{n}}} \lp(\frac{2\tilde{n}}{(\dprimen-\sqrt{4\log({k}n)\dprimen})} - \frac{2\tilde{n}}{(\dprimen+\sqrt{4\log({k}n)\dprimen} + 4\log({k}n))}\rp) \notag\\
    &\le \frac{1}{4\sqrt{\tilde{n}}}\cdot \frac{4\tilde{n}\sqrt{4\log({k}n)\dprimen} + 8\tilde{n}\log(kn)}{(\dprimen-\sqrt{4\log({k}n)\dprimen})\dprimen} \notag\\
    &\le \frac{ \sqrt{\tilde{n}}\sqrt{4\log({k}n)\dprimen} + 2\sqrt{\tilde{n}}\log(kn)}{(\dprimen-\sqrt{4\log({k}n)\dprimen})\dprimen}, \notag\\
    & \le \frac{\sqrt{(1+1/C_1)n/k}\sqrt{4\log({k}n)\dprimen} + 2\sqrt{(1+1/C_1)n/k}\log(kn)}{(\dprimen-\sqrt{4\log({k}n)\dprimen})\dprimen}. \notag 
\end{align*}
The last inequality works on event $\evente_v$, by which we have $\tilde{n} \le \frac{2(C_1 + 1)n}{C_1k}$. Then, $\dprimen > Ck^3n\log(kn)$ gives us
\begin{align*}
    \mathbf{e}_i^T\bolda_0^{-1}\vb_j &\le \frac{\sqrt{(1+1/C_1)n/k}\sqrt{4/(Ck^3n)}\dprimen +\sqrt{(1+1/C_1)n/k} (2/Ck^3n)\dprimen}{(\dprimen-\sqrt{4\log({k}n)\dprimen})\dprimen} \notag\\ 
    &\le \frac{ (1/(C_2\sqrt{k^4})) + (1/(C_3k^{3.5}\sqrt{n}))}{\dprimen(1-\sqrt{4/(Ck^{3}n)})} \\
     &\le \frac{C_4+1}{C_4}\cdot\frac{1}{k^2p},
\end{align*}
where in the last step we use the fact that $C, C_2, C_3 >1$ are large enough.
To lower bound $g_{ij}^{(0)}$, we have with probability at least $1 - \frac{2}{k^2n^2}$,
\begin{align*}
   \mathbf{e}_i^T\bolda_0^{-1}\vb_j &\ge \frac{1}{4\sqrt{\tilde{n}}} \lp(\frac{2\tilde{n}}{(\dprimen+\sqrt{4\log({k}n)\dprimen} + 4\log(kn))} - \frac{2\tilde{n}}{(\dprimen-\sqrt{4\log({k}n)\dprimen})}\rp) \notag\\
    &\ge \frac{1}{4\sqrt{\tilde{n}}}\cdot \frac{-4\tilde{n}\sqrt{4\log({k}n)\dprimen} - 8\tilde{n}\log(kn)}{(\dprimen-\sqrt{4\log({k}n)\dprimen})\dprimen} \notag\\
    &\ge -\frac{ \sqrt{\tilde{n}}\sqrt{4\log({k}n)\dprimen} + 2\sqrt{\tilde{n}}\log(kn)}{(\dprimen-\sqrt{4\log({k}n)\dprimen})\dprimen}, \notag\\
    & \ge -\frac{\sqrt{(1+1/C_1)n/k}\sqrt{4\log({k}n)\dprimen} + 2\sqrt{(1+1/C_1)n/k}\log(kn)}{(\dprimen-\sqrt{4\log({k}n)\dprimen})\dprimen}. \notag 
\end{align*}
Because $\dprimen > Ck^3n\log(kn)$, we get
\begin{align*}
    \mathbf{e}_i^T\bolda_0^{-1}\vb_j &\ge -\frac{\sqrt{(1+1/C_1)n/k}\sqrt{4/(Ck^3n)}\dprimen +\sqrt{(1+1/C_1)n/k} (2/Ck^3n)\dprimen}{(\dprimen-\sqrt{4\log({k}n)\dprimen})\dprimen} \notag\\ 
    &\ge -\frac{ (1/(C_2\sqrt{k^4})) + (1/(C_3k^{3.5}\sqrt{n}))}{\dprimen(1-\sqrt{4/(Ck^{3}n)})} \\
     &\ge -\frac{C_4+1}{C_4}\cdot\frac{1}{k^2p},
\end{align*}
where in the last step we use the fact that $C, C_2, C_3 >1$ are large enough. We complete the proof by applying a union bounds over all $k$ classes and $n$ training examples.

\section{Proof of Theorem~\ref{thm:svmmlmaniso}}\label{sec:svm_mlm_sm}
In this section, we provide the proof of Theorem~\ref{thm:svmmlmaniso}, which was discussed in Section~\ref{sec-linkmlm}. 
After having derived the interpolation condition in Equation \eqref{eq:det-con} for multiclass SVM, 
the proofs is in fact a rather simple extension of the arguments provided in~\cite{muthukumar2021classification,hsu2020proliferation} to the multiclass case. This is unlike the GMM case that we considered in Section~\ref{sec:pforderoneoutline}, which required substantial additional effort over and above the binary case~\cite{wang2020benign}.

For this section, we define $\A = \X^T\X$ as shorthand (we denoted the same quantity as $\A_k$ in Section~\ref{sec:pforderoneoutline}). 
Recall that the eigendecomposition of the covariance matrix is given by $\Sigmab = \sum_{i=1}^p \lambda_i\vb_i\vb_i^T =\boldsymbol{V} \boldsymbol{\Lambda} \boldsymbol{V}^T$. 
By rotation invariance of the standard normal variable, we can write $\A = \Q^T \boldsymbol{\Lambda}\Q$, where the entries of $\Q \in \R^{p \times n}$ are IID $\Nn(0,1)$ random variables. 
Finally, recall that we denoted $\lbdb = \begin{bmatrix}\lambda_1 & \cdots & \lambda_p\end{bmatrix}$ and defined the effective dimensions $d_2 =\frac{\Vert \lbdb \Vert_1^2}{\Vert \lbdb \Vert_2^2}$ and $d_\infty=\frac{\Vert \lbdb \Vert_1}{\Vert \lbdb \Vert_\infty}$.
Observe that \Equation~\eqref{eq:det-con} in \Theorem~\ref{lem:key} is equivalent to the condition
\begin{align}
\label{eq:dualsep}
    z_{ci}\mathbf{e}_i^T\A^{-1}\boldz_c >0, ~~\text{for \ all} \ \ c\in[k]~~\text{and}~~i \in [n]. 
\end{align}
We fix $c \in [k]$ and drop the subscript $c$, using $\overline{\boldz}$ to denote the vector $\boldz_c$. 
We first provide a deterministic equivalence to \Equation~\eqref{eq:det-con} that resembles the condition provided in~\cite[Lemma 1]{hsu2020proliferation}. Our proof is slightly modified compared to~\cite[Lemma 1]{hsu2020proliferation} and relies on elementary use of block matrix inversion identity.
\begin{lemma}
Let $\Q \in \mathbb{R}^{p \times n} = [\boldq_1, \cdots, \boldq_n]$. In our notation, \Equation~\eqref{eq:det-con} holds for a fixed $c$ if and only if:
\begin{align}
    \frac{1}{z_i}\overline{\boldz}_{\setminus i}^T\Big(\boldcapq_{\noisub}^T\boldsymbol{\Lambda}\boldcapq_{\noisub}\Big)^{-1}\boldcapq_{\noisub}^T\boldsymbol{\Lambda}\boldq_i < 1, ~~ \text{for all} ~~ i =1, \cdots, n.
\end{align}
Above, $\overline{\boldz}_{\setminus i} \in \mathbb{R}^{(n-1)\times 1}$ is obtained by removing the $i$-th entry from vector $\overline{\boldz}$ and $\boldcapq_{\noisub}\in \mathbb{R}^{d \times (n-1)}$ is obtained by removing the $i$-th column from $\boldcapq$.
\end{lemma}
\begin{proof}
By symmetry, it suffices to consider the case $i=1$. We first write
\begin{align}
    \bolda = \begin{bmatrix}\boldq_1^T\mathbf{\Lambda}\boldq_1 & \boldq_1^T\mathbf{\Lambda}\boldcapq_{\noonesub}\\
    \boldcapq_{\noonesub}^T\mathbf{\Lambda}\boldq_1 & \boldcapq_{\noonesub}^T\mathbf{\Lambda}\boldcapq_{\noonesub}
    \end{bmatrix} \triangleq \begin{bmatrix} \alpha & \mathbf{b}^T\\
    \mathbf{b} & \mathbf{D}
    \end{bmatrix}. \notag
\end{align}
By Schur complement \cite{bernstein2009matrix}, we have $$\bolda \succ \mathbf{0}~ \textit{iff} ~\text{either}~ \lp\{\alpha > 0~\text{and}~\mathbf{D} - \frac{\mathbf{b}\mathbf{b}^T}{\alpha} \succ \mathbf{0}\rp\}~\text{or}~\lp\{\mathbf{D} \succ \mathbf{0}~\text{and}~\alpha - \mathbf{b}^T\mathbf{D}^{-1}\mathbf{b} > 0\rp\}.$$
Since the entries of $\Q$ are drawn from a continuous distribution (IID standard Gaussian), both $\bolda$ and $\bolddcap=\boldcapq_{\noonesub}^T\mathbf{\Lambda}\boldcapq_{\noonesub}$ are positive definite almost surely. Therefore, $\alpha - \mathbf{b}^T\mathbf{D}^{-1}\mathbf{b} > 0$ almost surely.

Thus, by block matrix inversion identity \cite{bernstein2009matrix}, we have
\begin{align}
    \bolda^{-1} = \begin{bmatrix} (\alpha - \boldb^T\bolddcap^{-1}\boldb)^{-1} & -(\alpha - \boldb^T\bolddcap^{-1}\boldb)^{-1}\boldb^T\bolddcap^{-1}\\
    -\bolddcap^{-1}\boldb(\alpha - \boldb^T\bolddcap^{-1}\boldb)^{-1} & \bolddcap^{-1} + \bolddcap^{-1}\boldb(\alpha - \boldb^T\bolddcap^{-1}\boldb)^{-1}\boldb^T\bolddcap^{-1}
    \end{bmatrix}. \notag
\end{align}
Therefore,$
    \mathbf{e}_1^T\bolda^{-1} = (\alpha - \boldb^T\bolddcap^{-1}\boldb)^{-1}\begin{bmatrix} 1 & -\boldb^T\bolddcap^{-1}
    \end{bmatrix}. \notag
$
Hence we have
\begin{align}
    z_1\mathbf{e}_1^T\bolda^{-1}\overline{\boldz} = (\alpha - \boldb^T\bolddcap^{-1}\boldb)^{-1}(z_1^2 - \boldb^T\bolddcap^{-1}(z_1\overline{\boldz}_{\setminus 1})), \notag
\end{align}
where we use the fact that $\overline{\boldz}_1 = z_1$. Since $\alpha - \mathbf{b}^T\mathbf{D}^{-1}\mathbf{b} > 0$ almost surely, we have 
\begin{align}
    z_1\mathbf{e}_1^T\bolda^{-1}\overline{\boldz} > 0 &\iff (\alpha - \boldb^T\bolddcap^{-1}\boldb)^{-1}(z_1^2 - \boldb^T\bolddcap^{-1}(z_1\overline{\boldz}_{\setminus 1})) >0 \notag\\
    &\iff \frac{1}{z_1}\boldb^T\bolddcap^{-1}\overline{\boldz}_{\setminus 1} < 1. \notag
\end{align}
Recall that $\boldb^T =  \boldq_1^T\mathbf{\Lambda}\boldcapq_{\noonesub}$ and $\bolddcap = \boldcapq_{\noonesub}^T\mathbf{\Lambda}\boldcapq_{\noonesub}$. This completes the proof.
\end{proof}

Next, we define the following events:
\begin{enumerate}
\item For $i \in [n]$, $\mathcal{B}_i := \Big\{    \frac{1}{z_i}\overline{\boldz}_{\setminus i}^T\bolda_{\noisub}^{-1}\boldcapq_{\noisub}^T\boldsymbol{\Lambda}\boldq_i \ge 1 \Big\}$.

\item For $i \in [n]$, given $t >0$, $\mathcal{E}_i(t) := \Big\{\Vert (\overline{\boldz}_{\setminus i}^T\bolda_{\noisub}^{-1}\boldcapq_{\noisub}^T\boldsymbol{\Lambda})^T \Vert_2^2 \ge \frac{1}{t} \Big\}$.

\item $\mathcal{B} := \cup_{i=1}^n \mathcal{B}_i$.
\end{enumerate}
We know all the data points are support vectors i.e. \Equation~\eqref{eq:dualsep} holds, if none of the events $\mathcal{B}_i$ happens; hence, $\mathcal{B}$ is the undesired event. We want to upper bound the probability of event $\mathcal{B}$.
As in the argument provided in~\cite{hsu2020proliferation}, we have
\begin{align}
\label{eq:svmsumbevent}
    \mathbb{P}(\mathcal{B}) \le \sum_{i=1}^n \Big(\mathbb{P}(\mathcal{B}_i|\mathcal{E}_i(t)^c )+ \mathbb{P}(\mathcal{E}_i(t))\Big).
\end{align}
The lemma below gives an upper bound on $\Pro(\mathcal{B}_i|\mathcal{E}_i(t)^c )$.
\begin{lemma}
\label{lem-eventb}
For any $t >0$, $\mathbb{P}(\mathcal{B}_i|\mathcal{E}_i(t)^c) \le 2\exp\lp(-\frac{t}{2ck^2}\rp)$.
\end{lemma}
\begin{proof}
On the event $\mathcal{E}_i(t)^c$, we have $\Vert (\overline{\boldz}_{\setminus i}^T\bolda_{\noisub}^{-1}\boldcapq_{\noisub}^T\boldsymbol{\Lambda})^T \Vert_2^2 \le \frac{1}{t}$. Since, by its definition, $|\frac{1}{z_i}| \le k$,  we have $\frac{1}{z_i}\overline{\boldz}_{\setminus i}^T\bolda_{\noisub}^{-1}\boldcapq_{\noisub}^T\boldsymbol{\Lambda}\boldq_i$ is conditionally sub-Gaussian \cite[Chapter 2]{wainwright2019high} with parameter at most $ck^2\Vert (\overline{\boldz}_{\setminus i}^T\bolda_{\noisub}^{-1}\boldcapq_{\noisub}^T\boldsymbol{\Lambda})^T \Vert_2^2 \le ck^2/t$. Then the sub-Gaussian tail bound gives
\begin{align}
    \mathbb{P}(\mathcal{B}_i|\mathcal{E}_i(t)^c) \le 2\exp\lp(-\frac{t}{2ck^2}\rp),
\end{align}
which completes the prof.
\end{proof}
Next we upper bound $\mathbb{P}(\mathcal{E}_i(t))$ with $t = d_\infty/(2n)$. Since $\Vert \boldz_{\noisub}\Vert_2 \le \Vert \boldy_{\noisub}\Vert_2$, we can directly use \cite[Lemma 4]{hsu2020proliferation}.
\begin{lemma}[Lemma 4,~\cite{hsu2020proliferation}]
\label{lem-evente}
$\mathbb{P}\lp(\mathcal{E}_i\lp(\frac{d_{\infty}}{2n}\rp)\rp) \le 2\cdot 9^{n-1}\cdot \exp \lp(-c_1\min \lp\{\frac{d_2}{4c^2},\frac{d_{\infty}}{c}\rp\}\rp)$.
\end{lemma}
The results above are proved for fixed choices of $i \in [n]$ and $c \in [k]$.
We combine Lemmas~\ref{lem-eventb} and~\ref{lem-evente} with a union bound over all $n$ training examples and $k$ classes to upper bound the probability of the undesirable event $\mathcal{B}$ over all $k$ classes by:
\begin{align}
    kn9^{n-1}\cdot \exp \lp(-c_1\min \lp\{\frac{d_2}{4c^2},\frac{d_{\infty}}{c}\rp\}\rp) &\le \exp\lp(-c_1\min \lp\{\frac{d_2}{4c^2},\frac{d_{\infty}}{c}\rp\} + C_1\log(kn) + C_2n \rp) \notag \\
    \text{and} ~2kn\cdot\exp\lp(-\frac{d_\infty}{2ck^2n}\rp) &\le \exp\lp(-\frac{c_2d_\infty}{ck^2n}+C_3\log(kn)\rp). \notag
\end{align}
Thus, the probability that every data point is a support vector is at least
\begin{align}
    1 - \exp\lp(-c_1\min \lp\{\frac{d_2}{4c^2},\frac{d_{\infty}}{c}\rp\} + C_1\log(kn) + C_2n \rp) - \exp\lp(-\frac{c_2d_\infty}{ck^2n}+C_3\log(kn)\rp). \notag
\end{align}
To ensure that $\exp\lp(-c_1\min \lp\{\frac{d_2}{4c^2},\frac{d_{\infty}}{c}\rp\} + C_1\log(kn) + C_2n \rp) + \exp\lp(-\frac{c_2d_\infty}{ck^2n}+C_3\log(kn)\rp) \le \frac{c_4}{n}$, we consider the conditions $c_1\min \lp\{\frac{d_2}{4c^2},\frac{d_{\infty}}{c}\rp\} - C_1\log(kn) - C_2n \ge \log(n)$ and $\frac{c_2d_\infty}{ck^2n}-C_3\log(kn) \ge \log(n)$ to be satisfied. 
These are equivalent to the conditions provided in \Equation~\eqref{eq:thmanisolink}.
This completes the proof.
Note that throughout the proof, we did not use any generative model assumptions on the labels given the covariates, so in fact our proof applies to scenarios beyond the MLM.
\qed

\section{Classification error proofs for GMM}\label{sec:mni_sm}

In this section, we provide the proofs of classification error under the GMM (\Theorem~\ref{thm:classerrorgmm} and \Theorem~\ref{cor:classerrorgmmnc}).

\subsection{Proof of Theorem~\ref{thm:classerrorgmm}}\label{sec:classerrorgmmproof}

\subsubsection{Proof strategy and notations}
The notation and main arguments of this proof follow closely the content of Section \ref{sec:pforderoneoutline}. 

Our starting point here is the lemma below (adapted from \cite[D.10]{thrampoulidis2020theoretical}) that provides a simpler upper bound on the class-wise error $\Pro_{e|c}$.
\begin{lemma}
\label{lem:testerror01}
Under GMM, $\Pro_{e|c}\leq\sum_{j \ne c}Q\lp(\frac{(\hatw_c -\hatw_j)^T\boldmu_c}{\tn{\hatw_c -\hatw_j}}\rp)$. In particular, if $(\hatw_c -\hatw_j)^T\boldmu_c > 0$, then $\Pro_{e|c} \le  \sum_{j \ne c}\exp\lp(-\frac{((\hatw_c -\hatw_j)^T\boldmu_c)^2}{4(\hatw_c^T\hatw_c + \hatw_j^T\hatw_j)}\rp)$.
\end{lemma}
\begin{proof}
\cite[D.10]{thrampoulidis2020theoretical} shows $\Pro_{e|c}$ is upper bounded by $\sum_{j \ne c}Q\lp(\frac{(\hatw_c -\hatw_j)^T\boldmu_c}{\tn{\hatw_c -\hatw_j}}\rp)$. Then if $(\hatw_c -\hatw_j)^T\boldmu_c > 0$, the Chernoff bound \cite[Ch. 2]{wainwright2019high} gives
\begin{align*}
     \Pro_{e|c} \le \sum_{j \ne c}\exp\lp(-\frac{((\hatw_c -\hatw_j)^T\boldmu_c)^2}{2\tn{\hatw_c -\hatw_j}^2}\rp) \le \sum_{j \ne c}\exp\lp(-\frac{((\hatw_c -\hatw_j)^T\boldmu_c)^2}{4(\hatw_c^T\hatw_c + \hatw_j^T\hatw_j)}\rp), 
\end{align*}
where the last inequality uses the identity $\mathbf{a}^T\mathbf{b} \le 2(\mathbf{a}^T\mathbf{a}+\mathbf{b}^T\mathbf{b})$.
\end{proof}
Thanks to Lemma \ref{lem:testerror01}, we can upper bound $P_{e|c}$ by lower bounding the terms  
\begin{align}
\label{eq:pfinnerratio}
    \frac{((\hatw_c -\hatw_j)^T\boldmu_c)^2}{(\hatw_c^T\hatw_c + \hatw_j^T\hatw_j)},~~\text{for~all}~~c\ne j \in [k].
\end{align}
Our key observation is that this can be accomplished without the need to control the more intricate cross-correlation terms $\hatw_c^T\hatw_j$ for $c\neq j\in[k].$ 

Without loss of generality, we assume onwards that $c=k$ and $j=k-1$ (as in Section~\ref{sec:pforderoneoutline}).
Similar to Section~\ref{sec:pforderoneoutline}, the quadratic forms introduced in \Equation~\eqref{eq:pfquadforc} play key role here, as well. For convenience, we recall the definitions of the \emph{$c$-th order quadratic forms} for $c, j, m \in [k]$ and $i \in [n]$:
\begin{align}
s_{mj}^{(c)} &:=\vb_m^T \bolda_c^{-1}\vb_j, \notag\\ 
t_{mj}^{(c)} &:=\mathbf{d}_{m}^T \bolda_c^{-1}\mathbf{d}_j,\notag\\ 
h_{mj}^{(c)} &:=\mathbf{v}_m^T \bolda_c^{-1}\mathbf{d}_j,\notag\\ 
g_{ji}^{(c)} &:=\mathbf{v}_j^T \bolda_c^{-1}\mathbf{e}_i,\notag\\  
f_{ji}^{(c)} &:=\mathbf{d}_j^T\bolda_c^{-1}\mathbf{e}_i.\notag
\end{align}
Further, recall that $\hatw_c = \X(\X^T\X)^{-1}\vb_c$ and $\X=\sum_{j=1}^k\boldmu_j \vb_j^T + \boldcapq$. 
Thus,
\begin{align}
\begin{split}
\label{eq:pfgmmriskall01}
    \hatw_c^T\boldmu_c &= \tn{\boldmu_c}^2\vb_c^T(\xtx)^{-1}\vb_c + \sum_{m \ne c}\boldmu_m^T\boldmu_c\vb_m^T(\xtx)^{-1}\vb_c + \vb_c^T(\xtx)^{-1}\boldd_c \text{ and } \\
    \hatw_j^T\boldmu_c &= \tn{\boldmu_c}^2\vb_j^T(\xtx)^{-1}\vb_c + \boldmu_j^T\boldmu_c\vb_j^T(\xtx)^{-1}\vb_j + \sum_{m \ne c,j}\boldmu_m^T\boldmu_c\vb_m^T(\xtx)^{-1}\vb_j+\vb_j^T(\xtx)^{-1}\boldd_c.
\end{split}
\end{align}
Additionally,
\begin{align}
    \hatw_c^T\hatw_c = \vb_c^T(\xtx)^{-1}\vb_c,~~\text{and}~~\hatw_j^T\hatw_j = \vb_j^T(\xtx)^{-1}\vb_j.\notag
\end{align}
{To lower bound $\hatw_c^T\boldmu_c - \hatw_j^T\boldmu_c$, we first focus on the dominant terms of \Equation~\eqref{eq:pfgmmriskall01},} 
\begin{align}
\label{eq:pfgmmrisk01}
\tn{\boldmu_c}^2\vb_c^T(\xtx)^{-1}\vb_c+\vb_c^T(\xtx)^{-1}\boldd_c - \tn{\boldmu_c}^2\vb_c^T(\xtx)^{-1}\vb_j-\vb_j^T(\xtx)^{-1}\boldd_c.
\end{align}
{The above terms dominate the bound because, according to Assumption \ref{ass:orthomean}, the inner products between different mean vectors are small compared to the norms of mean vectors.}

{We now lower bound Equation~\eqref{eq:pfgmmrisk01} divided by $(\hatw_c^T\hatw_c + \hatw_j^T\hatw_j)$.}
Using the leave-one-out trick in Section \ref{sec:pforderoneoutline} and the matrix-inversion lemma, {we show in Appendix \ref{sec:pfinverssece} that}
\begin{align}
\label{eq:lowerquadratio}
     &\frac{\eqref{eq:pfgmmrisk01}}{(\hatw_c^T\hatw_c + \hatw_j^T\hatw_j)} = \frac{D_1}{D_2},\\
     &D_1 = \lp(\frac{\tn{\boldmu_c}^2s_{cc}^{(j)}-s_{cc}^{(j)}t_{cc}^{(j)}+{h_{cc}^{(j)}}^2 + h_{cc}^{(j)}-\tn{\boldmu_c}^2s_{jc}^{(j)}-h_{jc}^{(j)}-h_{jc}^{(j)}h_{cc}^{(j)}+s_{jc}^{(j)}t_{cc}^{(j)}}{\det_{j}}\rp)^2,\notag\\
     &D_2 = \lp(\frac{s_{cc}^{(j)}}{\det_j}+\frac{s_{jj}^{(-j)}}{\det_{-j}}\rp),\notag
\end{align}
where $\det_{j} = (\tn{\boldmu_c}^2 - t_{cc}^{(j)})s_{cc}^{(j)}+({h_{cc}^{(j)}}+1)^2$. Note that $\det_j = \det_{-c}$ when $c=k$ and $j = k-1$.

Next, we will prove that
\begin{align}\label{eq:class_error_gmm_app_main_eq}
    \eqref{eq:lowerquadratio} \ge \tn{\boldmu}^2\frac{\lp(\lp(1-\frac{C_1}{\sqrt{n}}-\frac{C_2n}{p}\rp)\tn{\boldmu}-C_3\minklogn\rp)^2}{C_6\lp(\tn{\boldmu}^2+\frac{kp}{n}\rp)}\,.
\end{align}
\subsubsection{Proof of Equation (\ref{eq:class_error_gmm_app_main_eq})}
We will lower bound the numerator and upper bound the denominator of \Equation~\eqref{eq:lowerquadratio}.
We will work on the high-probability event $\mathcal{E}_v$ defined in \Equation~\eqref{eq:eventvnorm} in Appendix~\ref{sec:aux_lemmas_svm_GMM}.
For quadratic forms such as $s_{cc}^{(j)}, t_{cc}^{(j)}$ and $h_{cc}^{(j)}$, the Gram matrix $\bolda_{j}^{-1}$ does not ``include'' the $c$-th mean component because we have fixed $c = k,j = k-1$.
Thus, we can directly apply Lemma~\ref{lem:ineqforanoj} to get 
\begin{align*}
    \frac{C_1-1}{C_1}\cdot\frac{n}{kp} \le & s_{cc}^{(j)} \le \frac{C_1+1}{C_1}\cdot\frac{n}{kp},\\
    t_{cc}^{(j)} \le & \frac{C_2n\tn{\boldmu}^2}{p},\\
    -\minrho\frac{C_3n\tn{\boldmu}}{\sqrt{k}p}  \le & h_{cc}^{(j)} \le  \minrho\frac{C_3n\tn{\boldmu}}{\sqrt{k}p},
\end{align*}
on the event $\mathcal{E}_v$.
We need some additional work to bound $s_{jc}^{(j)} = \vb_j\bolda_j^{-1}\vb_c$ and $h_{jc}^{(j)} = \vb_j\bolda_j^{-1}\boldd_c$, since the Gram matrix $\bolda_j^{-1}$ ``includes'' $\vb_j$. 
The proof here follows the machinery introduced in Appendix \ref{sec:pfineqforanoj} for proving Lemma~\ref{lem:ineqforanoj}. We provide the core argument and refer the reader therein for additional justifications. By \Equation~\eqref{eq:pfinverseskm} in Appendix~\ref{sec:pfinversesecc} (with the index $j-1$ replacing the index $0$), we first have 
\begin{align*}
s_{jc}^{(j)} 
&= s_{jc}^{(j-1)} - \frac{1}{\det_{j-1}}(\star)_{s}^{(j-1)}, 
\end{align*}
where we define $$(\star)_{s}^{(j-1)} = (\tn{\boldmu_j}^2 - t_{jj}^{(j-1)})s_{jj}^{(j-1)}s_{jc}^{(j-1)} + s_{jc}^{(j-1)}{h_{jj}^{(j-1)}}^2 + s_{jc}^{(j-1)}h_{jj}^{(j-1)}+s_{jj}^{(j-1)}h_{jc}^{(j-1)},$$ 
and $\det_{j-1} = (\tn{\boldmu_j}^2 - t_{jj}^{(j-1)})s_{jj}^{(j-1)}+({h_{jj}^{(j-1)}}+1)^2$. Further, we have
\begin{align*}
    |s_{jc}^{(j)}| &= \lp|\lp(1-\frac{(\tn{\boldmu_j}^2 - t_{jj}^{(j-1)})s_{jj}^{(j-1)}+ {h_{jj}^{(j-1)}}^2}{\det_{j-1}}  \rp)s_{jc}^{(j-1)} -\frac{1}{\det_{j-1}}( s_{jc}^{(j-1)}h_{jj}^{(j-1)}+s_{jj}^{(j-1)}h_{jc}^{(j-1)})\rp|\\
    &\le \frac{1}{C}|s_{jc}^{(j-1)}| + \frac{1}{\det_{j-1}}|(s_{jc}^{(j-1)}h_{jj}^{(j-1)}+s_{jj}^{(j-1)}h_{jc}^{(j-1)})|.
\end{align*}
We focus on the dominant term $|s_{jj}^{(j-1)}h_{jc}^{(j-1)}|$.
Using a similar argument to that provided in Appendix~\ref{sec:pfineqforanoj}, we get 
\begin{align*}
    \frac{|s_{jj}^{(j-1)}h_{jc}^{(j-1)}|}{\det_{j-1}} &\le \frac{|s_{jj}^{(j-1)}h_{jc}^{(j-1)}|}{(1+{h_{jj}^{(j-1)}})^2} \le \frac{C_1}{\lp(1-\frac{C_2\minrho\epsilonn}{k^2\sqrt{\logtwon}}\rp)^2}\cdot\frac{n}{kp}\cdot\frac{\minrho\epsilonn}{k^{2}\sqrt{\logtwon}} \\
    &\le \frac{C_3\minrho\epsilonn}{\lp(1-\frac{C_2\minrho\epsilonn}{k^2\sqrt{n}}\rp)^2k^2}\cdot\frac{\sqrt{n}}{kp}.
\end{align*}
Thus, we have
\begin{align}
&|s_{jc}^{(j-1)}| \le  \frac{C_4+1}{C_4}\cdot\frac{\sqrt{n}}{kp}. \notag
\end{align}
Similarly, we bound the remaining term $h_{jc}^{(j)}$.
Specifically, by \Equation~\eqref{eq:pfinverseh} in Section~\ref{sec:pfinversesecc}, we have
\begin{align*}
    h_{jc}^{(j)} &= h_{jc}^{(j-1)} - \frac{1}{\det_{j-1}}(\star)_h^{(j-1)},
\end{align*}
where we define $$(\star)_h^{(j-1)}=(\tn{\boldmu_j}^2 - t_{jj}^{(j-1)})s_{jj}^{(j-1)}h_{jc}^{(j-1)} + h_{jc}^{(j-1)}{h_{jj}^{(j-1)}}^2 + h_{jc}^{(j-1)}h_{jj}^{(j-1)}+s_{jj}^{(j-1)}t_{jc}^{(j-1)}.$$
Furthermore,
\begin{align*}
    |h_{jc}^{(j)}| &= \lp|\lp(1-\frac{(\tn{\boldmu_j}^2 - t_{jj}^{(j-1)})s_{jj}^{(j-1)}+ {h_{jj}^{(j-1)}}^2}{\det_{j-1}}  \rp)h_{jc}^{(j-1)} -\frac{1}{\det_{j-1}}( h_{jc}^{(j-1)}h_{jj}^{(j-1)}+s_{jj}^{(j-1)}t_{jc}^{(j-1)})\rp|\\
    &\le \frac{1}{C}|h_{jc}^{(j-1)}| + \frac{1}{\det_{j-1}}|(h_{jc}^{(j-1)}h_{jj}^{(j-1)}+s_{jj}^{(j-1)}t_{jc}^{(j-1)})|.
\end{align*}
We again consider the dominant term $|s_{jj}^{(j-1)}t_{jc}^{(j-1)}|/\det_{j-1}$ and get
\begin{align*}
    \frac{|s_{jj}^{(j-1)}t_{jc}^{(j-1)}|}{\det_{j-1}} &\le \frac{|s_{jj}^{(j-1)}t_{jc}^{(j-1)}|}{(1+{h_{jj}^{(j-1)}})^2} \le \frac{C_1}{\lp(1-\frac{C_2\minrho\epsilonn}{k^2\sqrt{n}}\rp)^2}\cdot\frac{n}{kp}\cdot\frac{n\tn{\boldmu}^2}{p} \\
    &\le \frac{C_3\epsilonn}{\lp(1-\frac{C_2\minrho\epsilonn}{k^{1.5}\sqrt{n}}\rp)^2k^2\sqrt{n}}\cdot\frac{\minrho n\tn{\boldmu}}{\sqrt{k}p}.
\end{align*}
Thus, we find that
\begin{align*}
    |h_{jc}^{(j-1)}| \le  \minrho\frac{C_4 n\tn{\boldmu}}{\sqrt{k}p}.
\end{align*}
We are now ready to lower bound the RHS in \Equation~\eqref{eq:lowerquadratio} by lower bounding its numerator and upper bounding its denominator. 

First, for the numerator we have the following sequence of inequalities:
\begin{align}
    &\tn{\boldmu_c}^2s_{cc}^{(j)}-s_{cc}^{(j)}t_{cc}^{(j)}+{h_{cc}^{(j)}}^2 + h_{cc}^{(j)}-\tn{\boldmu_c}^2s_{jc}^{(j)}-h_{jc}^{(j)}-h_{jc}^{(j)}h_{cc}^{(j)}+s_{jc}^{(j)}t_{cc}^{(j)} \notag \\
    \ge & \tn{\boldmu_c}^2s_{cc}^{(j)}-\tn{\boldmu_c}^2s_{jc}^{(j)}-s_{cc}^{(j)}t_{cc}^{(j)} + s_{jc}^{(j)}t_{cc}^{(j)} + h_{cc}^{(j)}-h_{jc}^{(j)}-h_{jc}^{(j)}h_{cc}^{(j)}\notag \\
    \ge &\frac{C_1-1}{C_1}\cdot\frac{\tn{\boldmu}^2n}{kp}-\frac{C_2 +1}{C_2}\cdot\frac{\tn{\boldmu}^2\sqrt{n}}{kp}-\frac{C_3n}{p}\cdot\frac{\tn{\boldmu}^2n}{kp} - \frac{C_4n}{p}\cdot\frac{\tn{\boldmu}^2\sqrt{n}}{kp} - \frac{C_5\minrho n\tn{\boldmu}}{\sqrt{k}p}. \notag
\end{align}
Above, we use the fact that the terms $|h_{cc}^{(j)}|, |h_{jc}^{(j)}| \le C\epsilon/(k^2\sqrt{n})$ are sufficiently small compared to $1$. 
Consequently, the numerator is lower bounded by
\begin{align}
\label{eq:pfgmmerrornum01}
    \lp(\frac{C_1-1}{C_1}\cdot\frac{\tn{\boldmu}^2n}{kp}-\frac{C_2 +1}{C_2}\cdot\frac{\tn{\boldmu}^2\sqrt{n}}{kp}-\frac{C_3n}{p}\cdot\frac{\tn{\boldmu}^2n}{kp} - \frac{C_4n}{p}\cdot\frac{\tn{\boldmu}^2\sqrt{n}}{kp} - \frac{C_5\minrho n\tn{\boldmu}}{\sqrt{k}p}\rp)^2\Big/\text{det}_{j}^2.
\end{align}

Second, we  upper bound the denominator. For this, note that under the assumption of nearly equal energy and equal priors on class means (Assumption~\ref{ass:equalmuprob}), there exist constants $C_1, C_2 >0$ such that $C_1 \le \det_{j}/\det_{-j} \le C_2$.
(In fact, a very similar statement was proved in \Equation~\eqref{eq:bounddet} and used in the proof of Theorem~\ref{thm:svmgmm}). 
Moreover, Lemma~\ref{lem:ineqforanoj} shows that the terms $s_{cc}^{(j)}$ and $s_{jj}^{(-j)}$ are of the same order, so it suffices to upper bound $\frac{s_{cc}^{(j)}}{\det_j}$.
Again applying Lemma~\ref{lem:ineqforanoj}, we have
\begin{align}
\label{eq:pfgmmerrordeno01}
    \frac{s_{cc}^{(j)}}{\det_j} \le \frac{C_6}{\det_j}\cdot\frac{n}{kp}\,
\end{align}
on the event $\mathcal{E}_v$.
Then, combining Equations~\eqref{eq:pfgmmerrornum01} and~\eqref{eq:pfgmmerrordeno01} gives us
\begin{align}
   ~\eqref{eq:lowerquadratio} &\ge
    \frac{n}{C_0kp}\cdot\frac{1}{\det_{j}}\lp((1-\frac{C_1}{\sqrt{n}}-\frac{C_2n}{p})\tn{\boldmu}^2-C_3\minklogn\tn{\boldmu}\rp)^2 \notag\\
    & \ge \frac{n}{C_0kp}\cdot\frac{1}{\frac{C_4\tn{\boldmu}^2n}{kp} + 2 + \frac{C_5n^2\tn{\boldmu}^2}{kp^2}}\lp(\lp(1-\frac{C_1}{\sqrt{n}}-\frac{C_2n}{p}\rp)\tn{\boldmu}^2-C_3\minklogn\tn{\boldmu}\rp)^2 \notag\\
    &\ge \tn{\boldmu}^2\frac{\lp(\lp(1-\frac{C_1}{\sqrt{n}}-\frac{C_2n}{p}\rp)\tn{\boldmu}-C_3\minklogn\rp)^2}{C_6\lp(\tn{\boldmu}^2+\frac{kp}{n}\rp)}, \label{eq:lowerquadratio02}
\end{align}
where the second inequality follows from the following upper bound on $\det_j$ on the event $\mathcal{E}_v$:
\begin{align*}
    \text{det}_{j} = (\tn{\boldmu_c}^2 - t_{cc}^{(j)})s_{cc}^{(j)}+({h_{cc}^{(j)}}+1)^2 \le \tn{\boldmu_c}^2 s_{cc}^{(j)} + 2({h_{cc}^{(j)}}^2 + 1) \le \frac{C_4\tn{\boldmu}^2n}{kp} + 2 + \frac{C_5n^2\tn{\boldmu}^2}{kp^2}.
\end{align*}
\subsubsection{Bounding the remaining terms in (\ref{eq:pfgmmriskall01})}
{The previous sections of the proof bounded the dominant terms in \Equation~\eqref{eq:pfgmmriskall01}. Now, we turn to bounding the remaining terms  $\boldmu_m^T\boldmu_c\vb_m^T(\xtx)^{-1}\vb_c$ and $\boldmu_j^T\boldmu_c\vb_j^T(\xtx)^{-1}\vb_j$. Under the nearly equal energy and priors assumption, the $\vb_j^T(\xtx)^{-1}\vb_j$ terms have the same bound for every $j \in [k]$ except for some constants. Similarly, the $\vb_j^T(\xtx)^{-1}\vb_m$ terms also have the same bound for all $j \ne m \in [k]$ except for some constants. 
An upper bound on classification error can then be derived in terms of the inner products between the mean vectors. Specifically, we need to include the bounds of  $\sum_{m \ne c}\boldmu_m^T\boldmu_c\vb_m^T(\xtx)^{-1}\vb_c$, $\boldmu_j^T\boldmu_c\vb_j^T(\xtx)^{-1}\vb_j$ and $\sum_{m \ne c,j}\boldmu_m^T\boldmu_c\vb_m^T(\xtx)^{-1}\vb_j$. Recall that in Appendix \ref{sec:appcomproof} we show that $\vb_c(\xtx)^{-1}\vb_j = {(s_{cj}^{(j)}+s_{cj}^{(j)}h_{cc}^{(j)}-s_{cc}^{(j)}h_{jc}^{(j)})}/{\det_j}$ and $\vb_j(\xtx)^{-1}\vb_j= s_{jj}^{(-j)}/{\det_{-j}}$. We also show that the bound for ${s_{cc}^{(j)}}$ (also ${s_{jj}^{(-j)}}$) is at the order of $\Oc(n/(kp))$ and the bound for $s_{cj}^{(j)}$ is at the order of $\Oc(\sqrt{n}/(kp))$ when $c \ne j$, which is significantly smaller than $\Oc(n/(kp))$ when $n$ is large. Additionally, the bound for $|h_{jc}^{(j)}|$ is sufficiently small. Combining these results and the assumption of mutually incoherent means, we can see that the bounds for these additional terms included are still much smaller than the bound of $\tn{\boldmu}^2\vb_c^T(\xtx)^{-1}\vb_c$, which is the dominant term. Therefore, they will not change the generalization bound of \eqref{eq:lowerquadratio} except up to constant factors. }
\subsubsection{Completing the proof}
\label{sec:appcomproof}
Because of our assumption of nearly equal energy on class means and equal priors, the analysis above can be applied to bound $\frac{((\hatw_c -\hatw_j)^T\boldmu_c)^2}{(\hatw_c^T\hatw_c + \hatw_j^T\hatw_j)}$, for every $j \ne c$ and $c \in [k]$. 
We define the $\textit{unconditional}$ event $$\evente_{u2}:=\lp\{ \frac{((\hatw_c -\hatw_j)^T\boldmu_c)^2}{(\hatw_c^T\hatw_c + \hatw_j^T\hatw_j)} \text{ is lower bounded by} ~\eqref{eq:lowerquadratio02} ~\text{for every}~ j \ne c\rp\}.$$
We have
\begin{align*}
    \Pro(\evente_{u2}^c) &\le \Pro(\evente_{u2}^c|\evente_v)+\Pro(\evente_v^c) \\
    &\le \frac{c_4}{n}+c_5k(e^{-\frac{n}{c_6}}+e^{-\frac{n}{c_7k^2}}) \le \frac{c_4}{n}+c_8ke^{-\frac{n}{c_7k^2}} 
\end{align*}
for constants $c_i$'s $>1$. Thus, the class-wise error $\Pro_{e|c}$ is upper bounded by
\begin{align*}
    (k-1)\exp{\lp(-\tn{\boldmu}^2\frac{\lp(\lp(1-\frac{C_1}{\sqrt{n}}-\frac{C_2n}{p}\rp)\tn{\boldmu}-C_3\minklogn\rp)^2}{C_4\lp(\tn{\boldmu}^2+\frac{kp}{n}\rp)}\rp)}
\end{align*}
with probability at least $1-\frac{c_4}{n}-c_8ke^{-\frac{n}{c_7k^2}}$.
This completes the proof.
\qed

\subsubsection{Proof of Equation (\ref{eq:lowerquadratio})}
\label{sec:pfinverssece}
Here, using the results of Section \ref{sec:pfinversesecc}, we show how to obtain \Equation~\eqref{eq:lowerquadratio} from \Equation~\eqref{eq:pfinnerratio}. First, by \cite[Appendix C.2]{wang2020benign} (with $\boldy$ replaced by $\vb_m$), we have
\begin{align*}
    \vb_m(\xtx)^{-1}\vb_m= \frac{s_{mm}^{(-m)}}{\det_{-m}},~~\text{for~all}~~m \in [k],
\end{align*}
where $\det_{-m} = (\tn{\boldmu_m}^2 - t_{mm}^{(-m)})s_{mm}^{(-m)}+({h_{mm}^{(-m)}}+1)^2$. Then \cite[Equation (44)]{wang2020benign} gives
\begin{align*}
    \tn{\boldmu_c}^2 \cdot \vb_c(\xtx)^{-1}\vb_c+\vb_c(\xtx)^{-1}\boldd_c = \frac{\tn{\boldmu_c}^2s_{cc}^{(j)}-s_{cc}^{(j)}t_{cc}^{(j)}+{h_{cc}^{(j)}}^2 + h_{cc}^{(j)}}{\det_{j}},
\end{align*}
where $\det_{j} = (\tn{\boldmu}^2 - t_{cc}^{(j)})s_{cc}^{(j)}+({h_{cc}^{(j)}}+1)^2$. Note that $\det_j = \det_{-c}$ when $c=k$ and $j = k-1$. 

For $\vb_c(\xtx)^{-1}\vb_j$ and $\vb_j(\xtx)^{-1}\boldd_c$, we can again express the $k$-th order quadratic forms in terms of $j$-th order quadratic forms as follows:
\begin{align*}
    \vb_c(\xtx)^{-1}\vb_j &= \frac{s_{cj}^{(j)}+s_{cj}^{(j)}h_{cc}^{(j)}-s_{cc}^{(j)}h_{jc}^{(j)}}{\det_j}, \\
    \vb_j(\xtx)^{-1}\boldd_c &= \frac{\tn{\boldmu}^2s_{cc}^{(j)}h_{jc}^{(j)} - \tn{\boldmu}^2s_{cj}^{(j)}h_{cc}^{(j)} + h_{cc}^{(j)}h_{jc}^{(j)}+h_{jc}^{(j)}-s_{cj}^{(j)}t_{cc}^{(j)}}{\det_j}.
\end{align*}
Thus, we have
\begin{align*}
\tn{\boldmu_c}^2\vb_c(\xtx)^{-1}\vb_j+\vb_j(\xtx)^{-1}\boldd_c = \frac{\tn{\boldmu_c}^2s_{jc}^{(j)}+h_{jc}^{(j)}+h_{jc}^{(j)}h_{cc}^{(j)}-s_{jc}^{(j)}t_{cc}^{(j)}}{\det_{j}}.
\end{align*}
This completes the proof.
\qed
{\subsection{Proof of \Theorem~\ref{cor:classerrorgmmnc}}
\label{sec:etfgenproof}
In this section we prove Theorem \ref{cor:classerrorgmmnc}.
The simplex ETF setting for the class means gives us $\tn{\mub}^2 = -(k-1)\mub_m^T\mub_c$ for $m \ne c$.
Therefore, following the analysis above, the additional term $\sum_{m \ne c}\boldmu_m^T\boldmu_c\vb_m^T(\xtx)^{-1}\vb_c$ is upper bounded by $\tn{\mub}^2\max_{m,c}|\vb_m^T(\xtx)^{-1}\vb_c|$. Since the dominating term in $\vb_m^T(\xtx)^{-1}\vb_c$ is ${s_{cj}^{(j)}}$, which has a much smaller upper bound than the dominating term in $\vb_c^T(\xtx)^{-1}\vb_c$, and the term $\boldmu_j^T\boldmu_c\vb_j^T(\xtx)^{-1}\vb_j$ in $\hatw_j^T\boldmu_c$ has a positive contribution to $\hatw_c^T\boldmu_c - \hatw_j^T\boldmu_c$ under the simplex ETF setting, the final generalization bound does not change except up to constant factors.}

\subsection{Proof of Corollary~\ref{cor:benigngmm}}\label{sec:benignproof}
We now prove the condition for benign overfitting provided in Corollary~\ref{cor:benigngmm}. Note that following Theorem \ref{thm:svmgmm}, we assume that 
\begin{align}
\label{eq:pfsvmlink}
    p > C_1k^3n\log(kn)+n-1\quad\text{ and }\quad p>C_2k^{1.5}n^{1.5}\tn{\boldmu}.
\end{align}

We begin with the setting where $\tn{\boldmu}^2 > C\frac{kp}{n}, \ \ \text{for some} \ \ C > 1.$
In this case, we get that \Equation~\eqref{eq:lowerquadratio02} is lower bounded by
$
    \frac{1}{c}{\lp(\lp(1-\frac{C_3}{\sqrt{n}}-\frac{C_4n}{p}\rp)\tn{\boldmu}-C_5\sqrt{k}\rp)^2},
$
and we have
\begin{align}
    {\lp(\lp(1-\frac{C_3}{\sqrt{n}}-\frac{C_4n}{p}\rp)\tn{\boldmu}-C_5\sqrt{k}\rp)^2} &> \tn{\boldmu}^2 - 2\tn{\boldmu}^2\frac{C_3}{\sqrt{n}} -2\tn{\boldmu}^2\frac{C_4n}{p}- 2C_5\sqrt{k}\tn{\boldmu} \notag\\
    \label{eq:pfhighsnr02}
    &> \lp(1-\frac{2C_3}{\sqrt{n}}\rp)\frac{kp}{n} -2\tn{\boldmu}^2\frac{C_4n}{p}- 2C_5\sqrt{k}\tn{\boldmu}.
\end{align}
Then \Equation~\eqref{eq:pfsvmlink} gives
\begin{align}
    \eqref{eq:pfhighsnr02} & > \lp(1-\frac{2C_3}{\sqrt{n}}\rp)\frac{kp}{n}- \lp(\frac{p}{k^{1.5}n^{1.5}}\rp)^2\frac{C_6n}{p}- \frac{C_7\sqrt{k}p}{k^{1.5}n^{1.5}} \notag\\
    \label{eq:pfhighsnr03}
    & = \frac{kp}{n}\lp(1 - \frac{2C_3}{\sqrt{n}}- \frac{C_6}{k^4n} - \frac{C_7}{k^2\sqrt{n}}\rp),
\end{align}
which goes to $+\infty$ as $\lp(\frac{p}{n}\rp) \to \infty$.

Next, we consider the case $\tn{\boldmu}^2 \le \frac{kp}{n}$. 
Moreover, we assume that 
$\tn{\boldmu}^4 = C_2\lp(\frac{p}{n}\rp)^{\alpha},$ for $\alpha > 1$.
Then, \Equation~\eqref{eq:lowerquadratio02} is lower bounded by $\frac{n}{ckp}\tn{\boldmu}^4\lp(\lp(1-\frac{C_3}{\sqrt{n}}-\frac{C_4n}{p}\rp)-\frac{C_5\sqrt{k}}{\tn{\boldmu}}\rp)^2$, and we get
\begin{align}
    \frac{n}{kp}\tn{\boldmu}^4\lp(\lp(1-\frac{C_3}{\sqrt{n}}-\frac{C_4n}{p}\rp)-\frac{C_5\sqrt{k}}{\tn{\boldmu}}\rp)^2 \notag 
    &> \lp(1-\frac{2C_3}{\sqrt{n}}\rp)\frac{n}{kp}\tn{\boldmu}^4 - \frac{C_6n^2}{kp^2}\tn{\boldmu}^4 - \frac{C_7n}{\sqrt{k}p}\tn{\boldmu}^3 \notag\\
    \label{pf-lowsnr03}
    &\ge \lp(1-\frac{2C_3}{\sqrt{n}}\rp)\frac{1}{k}\lp(\frac{p}{n}\rp)^{\alpha -1}- \frac{C_6}{k}\lp(\frac{p}{n}\rp)^{\alpha -2} -\frac{C_7}{\sqrt{k}}\lp(\frac{p}{n}\rp)^{0.75\alpha -1},
\end{align}
where the last inequality uses Equations~\eqref{eq:pfsvmlink} and condition $\tn{\boldmu}^2 \le \frac{kp}{n}$.
Consequently, the RHS of \Equation~\eqref{pf-lowsnr03} will go to $+\infty$ as $\lp(\frac{p}{n}\rp) \to \infty$, provided that $\alpha  > 1$. Overall, it suffices to have
\begin{align*}
    &p > \max\lp\{C_1k^3n\log(kn)+n-1, C_2k^{1.5}n^{1.5}\tn{\boldmu}, \frac{n\tn{\boldmu}^2}{k}\rp\},\\
    \text{and} \ &\tn{\boldmu}^4 \ge C_8\left(\frac{p}{n}\right)^{\alpha}, \ \ \text{for} \ \alpha \in (1, 2].
\end{align*}
All of these inequalities hold provided that $\tn{\boldmu}= \Theta(p^\beta)$ for $\beta\in(1/4,1/2]$ for finite $k$ and $n$.
This completes the proof.
\qed


\section{Main lemmas used in error analysis of MLM}\label{sec:classerrormlmprooflemmas}

{In this section, we collect the proofs of the main lemmas that are used in the error analysis of MLM (proof of Theorem~\ref{thm:classerrormlm}, provided in Section~\ref{sec:proofclasserrormlm}).
We first introduce notation that is specific to these proofs.}

{For two indices $\ell,j \in [k]$, we use the Kronecker delta notation $\delta_{\ell,j} = \mathbb{I}[\ell \neq j]$.
For a diagonal covariance matrix $\Sigmab$ and $\ell \geq 1$, we define the leave-$\ell$-out covariance matrix $\Sigmab_{-1:\ell}$ as $\Sigmab$ with the first $\ell$ rows and columns removed.
For an arbitrary PSD matrix $\boldsymbol{M} \in \R^{d \times d}$ with eigenvalues $\la_1,\ldots,\la_d$
and any index $k \in \{0,\ldots, d-1\}$, the effective rank of the first kind (in the sense of~\cite{bartlett2020benign}) is defined as 
\begin{align}\label{eq:effrank} 
r_k(\boldsymbol{M}) := \frac{1}{\lambda_{k+1}} \cdot \sum_{\ell = k+1}^{d} \lambda_{\ell}.
\end{align} }

{Additionally, we state our convention for constants for this proof.
Hereafter, we let $c, C \ldots>0$ denote positive absolute constants in lower and upper bounds respectively. We also use $c_k,C_k>0$ in a similar manner to denote constants that may only depend on the number of classes $k$. To simplify exposition in the proof, the values of these constants may be changing from line to line without explicit reference.
Finally, by ``large enough" $n$ we mean that $n \geq C_k$ for some universal constant $C_k$ that depends only on $k$.}

\subsection{Proof of Lemma~\ref{lem:sumulticlass}}\label{sec:lemmasumulticlassproof}

{The proof of Lemma~\ref{lem:sumulticlass} follows similarly to the proof of Theorem 4 in Appendix D.3 and Lemma 11 in Appendix E of~\cite{muthukumar2021classification}, with the two important and nontrivial extensions mentioned above: one, to the multiclass case involving $k$ signal vectors, and two, considering the logistic model for label noise.
Without loss of generality\footnote{The reason this is without loss of generality is because we can carry out the same analysis otherwise with the appropriate permutation of the index labels.}, we assume for simplicity that $j_c = c$ for all $c \in [k]$, and consider classes $c_1 = 1,c_2 = 2$ for the argument.
First, we consider the following adjusted orthonormal basis
\begin{align*}
\etilde_1 &= \frac{(\boldmu_1 - \boldmu_2)}{\|\boldmu_1 - \boldmu_2 \|_2}, \qquad
\etilde_2 = \frac{(\boldmu_1 + \boldmu_2)}{\|\boldmu_1 - \boldmu_2 \|_2}, \qquad
\etilde_j = \ehat_j \text{ for all } j \geq 3.
\end{align*}
This orthonormal basis together with the bilevel ensemble structure in Definition~\ref{as:bilevel} then gives us
\begin{align*}
\SU_{1,2} &= \frac{\sqrt{\lambda_H}}{\|\boldmu_1 - \boldmu_2\|_2} \cdot (\boldmu_1 - \boldmu_2)^\top (\hatw_1 - \hatw_2) = \sqrt{\lambda_H} \cdot \etilde_1^\top \X (\X^\top \X)^{-1} (\boldv_1 - \boldv_2)
\\
&= \sqrt{\lambda_H} \cdot \etilde_1^\top \X (\X^\top \X)^{-1} \y_1.
\end{align*}
where in the second line we introduce the shorthand $\y_1 := \boldv_1 - \boldv_2\,$.
Next, we define 
\[\A := \X^\top \X= \sum_{j=1}^p \lambda_j \z_j \z_j^\top\,,
\]
with  $\z_j := \frac{1}{\sqrt{\lambda_j}} \X^\top \etilde_j, j = 1,\ldots,p$ and note that $\z_j  \stackrel
{\text{iid}}{\sim} \mathcal{N}(0,I_n)$.
(This uses again the rotational invariance of Gaussianity 
and the bilevel ensemble structure.) Finally, for $\ell=1,\ldots,p-1$, we denote the ``leave-$\ell$-out'' matrices \emph{corresponding to the changed basis} by $\A_{-1:\ell} := \sum_{j = \ell+1}^p \lambda_j \z_j \z_j^\top$.
Note that, by definition, $\A_{-1:0}
:=\A.$}


{Using the above notation, we can then write the survival terms as follows \[\SU_{1,2} = \lambda_H \cdot \z_1^\top \A^{-1} \y_1.\]}

{The main challenge in characterizing the term above is that $\A^{-1}$ is dependent on both $\z_1$ and $\y_1$. In particular, $\y_1$ depends on $\z_1$ itself but also it depends on $\z_2,\ldots,\z_k$. 
In the binary case, a simple \emph{leave-one-out} analysis suffices to circumvent this difficulty as shown in~\cite{muthukumar2021classification}.
In the multiclass setting, we need to do a much more challenging \emph{leave-$k$-out} analysis which we outline below.
In particular, we outline a recursive argument over $k$ steps that iteratively removes the dependencies on $\z_1,\ldots,\z_k$ from $\A^{-1}$. This process is described in the following subsections.}

\subsubsection{Key recursion: Removing dependencies}

{Start by defining the following ``quadratic-like" terms: 
\begin{subequations}\label{eq:survival quadratics defns}
\begin{align}
    Q_\ell&:=\y_1^T\Aneg{\ell}\z_1,\qquad \ell=0,\ldots, k\,,
    \\
    \wt Q_\ell&:=\z_1^T\Aneg{\ell}\z_1,\qquad \ell=1,\ldots, k\,,
    \\
    R_{\ell,j}&:=\y_1^T\Aneg{\ell}\z_j,\qquad \ell\geq j, j=1,\ldots, k-1\,.
\end{align}
\end{subequations}}

{Recall the term we wish to control is  $Q_0=\z_1^T\Aneg{0}\y_1=\z_1^T\A^{-1}\y_1$.
A single application of the matrix inversion lemma (which was also done in~\cite{muthukumar2021classification} for the binary case and is described in a self-contained manner in Appendix~\ref{sec:proof of Lemma 2}) yields that $\SU_{1,2} = \frac{\lambda_H Q_1 }{1 + \lambda_H \wt Q_1}$.
However, unlike in the binary case $Q_1$ can no longer be easily controlled, as $\y_1$ still depends on $\Aneg{1}$ as it is a functional of not only $\z_1$, but also $\{\z_2,\ldots,\z_k\}$.}

{On the other hand, the term involving the leave-$k$-out Gram matrix, i.e.~$Q_k=\z_1^T\Aneg{k}\y_1$ avoids this issue. 
This is because $\y_1$ is only a functional of $\{\z_1,\ldots,\z_k\}$, which ensures that $\y_1$ is independent of $\Aneg{k}$.
This allows us to sharply characterize $Q_k$ via the Hanson-Wright inequality, as shown in the lemma below.}

{\begin{lemma}\label{lem:Qk} For large enough $n$, we have 
\begin{align}\label{eq:Qkexpression}
    c_k \left(\frac{cn}{\lambda_L r_s(\Sigmab)} + \frac{c' n^{3/4}}{\lambda_L r_s(\Sigmab)} \right) \geq Q_k &\geq c_k \left( \frac{(n-s)}{c\lambda_L r_s(\Sigmab)} - \frac{c' n^{3/4}}{\lambda_L r_s(\Sigmab)} \right) .
\end{align}
with probability at least $1-2e^{-\sqrt{n}}$.
\end{lemma}
See Appendix~\ref{sec:proof of Lemma 4} for the proof of this lemma.}

{Thus, it suffices to characterize $Q_1$ in terms of $Q_k$, so that we can translate upper/lower bounds on $Q_k$ to upper/lower bounds on $Q_1$ and thereby characterize the survival $\SU_{1,2}$.
The main result of this section, shown below, does precisely this, guaranteeing that $|Q_1-Q_k|=o(Q_k)$ with high probability.}

{\begin{lemma}\label{lem:Q0 close to Qk}
We have  
\[
\left(1 - \frac{C_k}{n^{1/4}}\right) Q_k \leq Q_1 \leq \left(1 + \frac{C_k}{ n^{1/4}}\right) Q_k.
\]
with probability at least $1-C'k^3e^{-C\sqrt{n}}$.
\end{lemma}}

{In the remainder of this section we prove Lemma \ref{lem:Q0 close to Qk}. 
We introduce the following recursion for any $\ell=k,\ldots,1$ by directly applying the matrix inversion lemma:
\begin{align}\label{eq:Q recursion eq}
    \nn Q_{\ell-1} &= \z_1^T\Aneg{\ell-1}\y_1 = \z_1^T(\A_{-1:\ell}+\z_\ell\z_\ell^T)^{-1}\y_1 = Q_{\ell} - \frac{\la_H(\z_1^T\Aneg{\ell}\z_\ell)(\y_1^T\Aneg{\ell}\z_\ell)}{1+\la_H \z_\ell^T\Aneg{\ell}\z_\ell} \nonumber \\
    &= Q_\ell - \wt Q_\ell\, \left(\frac{\z_1^T\Aneg{\ell}\z_\ell}{\wt Q_\ell}\, \right) \left(\frac{\la_H R_{\ell,\ell}}{1+\la_H \z_\ell^T\Aneg{\ell}\z_\ell}\right) \,,
\end{align}
where in the last line we recalled the definitions of $\wt Q_\ell$ and of $R_{\ell,\ell}$ in  Eqs. \eqref{eq:survival quadratics defns}.}

{In order to prove Lemma \ref{lem:Q0 close to Qk} using the above recursion, we establish the following bounds on each of the three terms $\frac{\z_1^T\Aneg{\ell}\z_\ell}{\wt Q_\ell}$,  $R_{\ell,\ell}$ and $\wt Q_\ell$ that appear in Eq. \eqref{eq:Q recursion eq}. We provide the proofs of each of these technical lemmas in Appendix~\ref{sec:proofs of technical survival}.}

{\begin{lemma}\label{lem:ratio bound}
For large enough $n$ and for all $\ell \in [k]$, we have
\begin{align}
    \abs{\z_1^T\Aneg{\ell}\z_\ell} &\leq  \frac{C}{n^{1/4}} \z_1^T\Aneg{\ell}\z_1 =\frac{C}{n^{1/4}} \wt Q_\ell \,,\label{eq:delta_smallterm}
\end{align}
with probability at least $1-C k^3e^{-\sqrt{n}}$.
\end{lemma}}

{\begin{lemma}\label{lem:Rellell bound}
For all $\ell \in [k]$, we have
\begin{align}
    \abs{R_{\ell,\ell}} &\leq C_k\cdot  \z_\ell^T\Aneg{\ell}\z_\ell\,,\quad \ell=k,k-1,\ldots,1\,.\label{eq:delta_Rterm}
\end{align}
with probability at least $1-Ck^3e^{-\sqrt{n}}$.
\end{lemma}}

{\begin{lemma}\label{lem:Qtilde main bound}
For large enough $n$ and for all $\ell \in [k]$, we have
\begin{align}
       0 \leq \wt Q_\ell &\leq \frac{2}{c_k} Q_{k}
    \label{eq:Qtilde_rec}\,
\end{align}
with probability at least $1 - Cke^{-\sqrt{n}}$.
\end{lemma}}

{\paragraph{Proof of Lemma \ref{lem:Q0 close to Qk}}Combining the bounds in Lemmas \ref{lem:ratio bound}, \ref{lem:Rellell bound},  \ref{lem:Qtilde main bound} and the fact that $\z_\ell^T\Aneg{\ell}\z_\ell\geq 0$ within  Equation \eqref{eq:Q recursion eq} immediately yields for all $\ell=1,\ldots,k$:
\begin{align*}
\abs{Q_{\ell} - Q_{\ell-1}} &\leq \frac{2}{c_k} Q_k \,\left(\frac{C}{n^{1/4}}\right)\, \left(\frac{C_k \cdot \la_H \z_\ell^T\Aneg{\ell}\z_\ell}{1+ \la_H\z_\ell^T\Aneg{\ell}\z_\ell} \right)
\\
&\leq Q_{k} \frac{C_k}{c_k n^{1/4}}.
\end{align*}
The desired then follows by the bound $|Q_k - Q_1| \leq \sum_{\ell=2}^k |Q_{\ell} - Q_{\ell-1}|$.}




{\subsubsection{Completing the proof of Lemma \ref{lem:sumulticlass}}\label{sec:proof of Lemma 2}
Armed with Lemma~\ref{lem:Q0 close to Qk}, we now complete the proof of Lemma~\ref{lem:sumulticlass}.
Recall that 
\[\SU_{1,2}= \lambda_H Q_0 = \lambda_H \cdot \z_1^\top \A^{-1} \y_1  \,.\]
Applying Equation~\eqref{eq:Q recursion eq} for $\ell=1$, we can write
\[
\SU_{1,2} = \la_H Q_0 = \la_H\left(Q_1-\frac{\la_H\wt Q_1\,Q_1}{1+\la_H \la_H\wt Q_1}\right) = \frac{\la_H Q_1}{1+\la_H\wt Q_1}\,.
\]
Thus, combining Lemmas \ref{lem:Qtilde main bound} and \ref{lem:Q0 close to Qk}, we can obtain the following lower/upper bounds on $\SU_{1,2}$:
\begin{align}\label{eq:suqk_intermediate}
\frac{\la_H \left(1 - \frac{C_k}{n^{1/4}}\right) Q_k}{1+\la_H \left(\frac{2}{c_k}\right) Q_k}   \leq  \SU_{1,2} \leq \la_H \left(1 + \frac{C_k}{n^{1/4}}\right) Q_k.  
\end{align}}

{It remains to substitute the upper/lower bounds on $Q_k$ we obtained in Lemma \ref{lem:Qk}.
Plugging in the definition of the bilevel ensemble gives $\lambda_L r_s(\Sigmab) = n^m - n^r$.
Noting that $m > 1$ and $r < 1$ gives
\begin{align*}
    \frac{cn}{\lambda_L r_s(\Sigmab)} + \frac{c' n^{3/4}}{\lambda_L r_s(\Sigmab)} &\leq C n^{1 - m} \text{ and } \\
    \frac{(n-s)}{c \lambda_L r_s(\Sigmab)} - \frac{c' n^{3/4}}{\lambda_L r_s(\Sigmab)} &\geq c n^{1 - m}.
\end{align*}
Therefore, we have
\begin{align*}
    c n^{1-m} \leq Q_k \leq C n^{1-m}.
\end{align*}
Noting that $\lambda_H = n^{m - q - r}$, we then have $ c n^{1-q - r} \leq \lambda_H Q_k \leq C n^{1-q - r}$.
Plugging this back into Equation~\eqref{eq:suqk_intermediate} gives
\begin{align*}
    c_k n^{1 - q - r} \leq \SU_{1,2} \leq C_k n^{1- q - r}
\end{align*}
for large enough $n$, which is the desired statement.
A union bound over Lemmas~\ref{lem:Qk} and~\ref{lem:Q0 close to Qk} implies that this statement holds with probability at least $1 - Ck^3 e^{-C\sqrt{n}}$.
This completes the proof of Lemma~\ref{lem:sumulticlass}.
\qed}

\subsection{Proof of Lemma~\ref{lem:cnmulticlass}}\label{sec:lemmacnmulticlassproof}

{This proof extends the argument in~\cite[Proof of Theorem 24]{muthukumar2021classification} using the same change-of-basis argument that we used to characterize the survival.
As with the proof of Lemma~\ref{lem:sumulticlass}, we assume without loss of generality that $c_1 = 1,c_2 = 2$.
First, we recall that
\begin{align}\label{eq:deltahat_def}
   \Deltahat_{1,2} :=\X (\xtx)^{-1} (\boldv_1 - \boldv_2) = \X \A^{-1} \y_1,
\end{align}
and that we defined $\A := \xtx$ and $\y_1 := \boldv_1 - \boldv_2$ as shorthand.
We first state and prove the following lemma, which is analogous to~\cite[Lemma 28, Eq. (53a)]{muthukumar2021classification}.
\begin{lemma}\label{lem:contaminationformula}
The contamination term $\CN_{1,2}$ can be expressed as,
\begin{align}\label{eq:contaminationformula}
    \CN_{1,2} &= \sqrt{\y_1^\top \C \y_1}, \text{ where } \\
    \C &:= \A^{-1} \left(\sum_{j=1,j \neq 1}^d \lambda_j^2 \z_j \z_j^\top \right) \A^{-1} \nonumber.
\end{align}
This is a consequence of the relation $\CN^2_{1,2} := \sum_{j =1, j \neq 1}^d \lambda_j \hat{\alpha}_j^2$, where we define $\hat{\alpha}_j := \sqrt{\lambda_j} \cdot \z_j^\top \A^{-1} \y_j$.
\end{lemma}
See Appendix~\ref{sec:technicallemmas_cn} for the proof of this lemma.}

{Note that the expression in Lemma~\ref{lem:contaminationformula} is still challenging to characterize, as the difference of label vectors $\y_1$ is dependent on the matrix $\C$.
To make progress, we will write a $k$-step recursive equation to express $\hat{\alpha}_j$ (and, thereby, $\CN_{1,2}$) in terms of $\Aneg{k}$ instead of $\A^{-1}$, leading to a possible characterization in terms of quadratic forms for which we can apply the Hanson-Wright inequality.
We begin by reproducing the first recursion from the proof of~\cite[Lemma 28]{muthukumar2021classification}, which directly yields
\begin{align*}
\hat{\alpha}_j = \sqrt{\lambda_j} \cdot \z_j^\top \Aneg{1} (\y_1 - \SU_{1,2} \z_1).
\end{align*}
We now recurse this argument to get an expression in terms of $\Aneg{2}$.
Applying the Sherman-Morrison formula yields
\begin{align*}
\Aneg{1} = \Aneg{2} - \frac{\lambda_H \cdot \Aneg{2} \z_2 \z_2 \Aneg{2}}{1 + \lambda_H \cdot \z_2^\top \Aneg{2} \z_2}
\end{align*}
and, consequently,
\begin{align}\label{eq:recursion2stepfull}
\hat{\alpha}_j = \sqrt{\lambda_j} \cdot \z_j^\top \Aneg{2} \left(\tilde{\y}_1 - \z_2 \cdot \frac{\lambda_H \cdot \z_2^\top \Aneg{2} \tilde{\y}_1}{1 + \lambda_H \cdot \z_2^\top \Aneg{2} \z_2}\right).
\end{align}
To write the entire $k$-step recursion, we define some shorthand notation.
For $\ell = 2,\ldots,k$ we define
\begin{align*}
\SU^{(\ell)}_{1,2} &:= \frac{\lambda_H \cdot \z_\ell^\top \Aneg{\ell} \tilde{\y}_{\ell-1}}{1 + \lambda_H \cdot \z_\ell^\top \Aneg{\ell} \z_\ell} \text{ and } \\
\tilde{\y}_{\ell} &:= \tilde{\y}_{\ell - 1} - \SU^{(\ell)}_{1,2} \z_\ell \\
\implies \tilde{\y}_k &= \y_1 - \sum_{\ell = 1}^k \SU^{(\ell)}_{1,2} \z_\ell.
\end{align*}
Consequently, rewriting Equation~\eqref{eq:recursion2stepfull} in terms of this shorthand notation gives
\begin{align*}
\hat{\alpha}_j = \sqrt{\lambda_j} \cdot \z_j^\top \Aneg{2} \tilde{\y}_2,
\end{align*}
and repeating this argument for $\ell = 3,\ldots,k$ ultimately yields
\begin{align*}
\hat{\alpha}_j &= \sqrt{\lambda_j} \cdot \z_j^\top \Aneg{k} \tilde{\y}_k.
\end{align*}
Then, we use an identical set of manipulations to the proof of~\cite[Lemma 28]{muthukumar2021classification} (reproduced for completeness) to get
\begin{align*}
\CN^2_{1,2} = \sum_{j = 1,j \neq 1}^d \lambda_j \hat{\alpha}_j^2 &= \sum_{j=1,j \neq 1}^d \lambda_j^2 \tilde{\y}_k^\top \Aneg{k} \z_j \z_j^\top \Aneg{k} \tilde{\y}_k \\
&= \tilde{\y}_k^\top \Aneg{k} \left(\sum_{j=1,j\neq 1}^d \lambda_j^2 \z_j \z_j^\top \right) \Aneg{k} \tilde{\y}_k \\
&= \tilde{\y}_k^\top \widetilde{\C}_k \tilde{\y}_k, \text{ where } \\
\widetilde{\C}_k &:= \Aneg{k} \left(\sum_{j=1,j\neq 1}^d \lambda_j^2 \z_j \z_j^\top \right) \Aneg{k}.
\end{align*}
We now complete the proof of Lemma~\ref{lem:cnmulticlass} by working with the expression $\CN^2_{1,2} = \tilde{\y}_k^\top \widetilde{\C}_k \tilde{\y}_k$.
First, we note that we can write
\begin{align*}
\CN^2_{1,2} &= \tilde{\y}_k^\top \widetilde{\C}_{k,1} \tilde{\y}_k + \tilde{\y}_k^\top \widetilde{\C}_{k,2} \tilde{\y}_k \text{ where } \\
\widetilde{\C}_{k,1} &:= \Aneg{k} \left(\sum_{j=1,j\neq 1}^k \lambda_j^2 \z_j \z_j^\top \right) \Aneg{k} \text{ and } \\
\widetilde{\C}_{k,2} &:= \Aneg{k} \left(\sum_{j=k+1}^d \lambda_j^2 \z_j \z_j^\top \right) \Aneg{k}.
\end{align*}
Then, we can sharply upper-bound the terms $T_1 := \tilde{\y}_k^\top \widetilde{\C}_{k,1} \tilde{\y}_k$ and $T_2 := \tilde{\y}_k^\top \widetilde{\C}_{k,2} \tilde{\y}_k$, which we do below beginning with the second term $T_2$.}

{\paragraph{Controlling the term $T_2 := \tilde{\y}_k^\top \widetilde{\C}_{k,2} \tilde{\y}_k$}}

{We apply the algebraic identity $(\x - \y)^\top \M (\x - \y) \leq 2 (\x^\top \M \x + \y^\top \M \y)$ $k-1$ times to get
\begin{align*}
\tilde{\y}_k^\top \widetilde{\C}_{k,2} \tilde{\y}_k \leq 2^{k-1} \left(\y_1^\top \widetilde{\C}_{k,2} \y_1 + \sum_{\ell = 1}^k (\SU^{(\ell)})_{1,2}^2 \cdot \z_\ell^\top \widetilde{\C}_{k,2} \z_\ell \right).
\end{align*}
We use the following technical lemma, which is proved in Appendix~\ref{sec:technicallemmas_cn}.
\begin{lemma}\label{lem:su_ck}
For $\ell = 2,\ldots,k$ we define
\begin{align*}
\SU^{(\ell)}_{1,2} := \frac{\lambda_H \cdot \z_\ell^\top \Aneg{\ell} \tilde{\y}_{\ell-1}}{1 + \lambda_H \cdot \z_\ell^\top \Aneg{\ell} \z_\ell}\,, 
\end{align*}
where $\tilde{\y}_{\ell} := \tilde{\y}_{\ell - 1} - \SU^{(\ell)}_{1,2} \z_\ell$ and 
$\SU^{(1)}_{1,2}:=\SU_{1,2}=\frac{\lambda_H \cdot \z_1^\top \A_{-1}^{-1} \y_1 }{1 + \lambda_H \cdot \z_1^\top \A_{-1}^{-1} \z_1}.$ Then, for all $\ell =2,\ldots,k$ we have 
\begin{align}
\abs{\SU^{(\ell)}_{1,2}} \leq \frac{C_k}{n^{1/4}} < C_k \,.
\end{align}
with probability at least $1-Ck^2e^{-\sqrt{n}}$.
\end{lemma}
Applying Lemma~\ref{lem:su_ck} thus gives
\begin{align*}
\tilde{\y}_k^\top \widetilde{\C}_{k,2} \tilde{\y}_k \leq C_k \left(\y_1^\top \widetilde{\C}_{k,2} \y_1 + \sum_{\ell = 1}^k \z_\ell^\top \widetilde{\C}_{k,2} \z_\ell \right).
\end{align*}
Now, we note that the matrix $\widetilde{\C}_{k,2}$ only depends on $\{\z_j\}_{j=k+1}^p$ and is therefore independent of $\y_1$ as well as $\{\z_\ell\}_{\ell=1}^k$.
Recall that each of $\{\z_\ell\}_{\ell = 1}^k$ is isotropic Gaussian and that $\y_1$ is sub-Gaussian with uncorrelated components, i.e. $y_{1,i}^2 \leq 1$ and $\E[y_{1,i} y_{c,i'}] = 0$ for $i \neq i' \in [n]$.
Therefore, we can apply the Hanson-Wright inequality~\cite{rudelson2013hanson} with the parameters stated in~\cite[Eq (44)]{muthukumar2021classification} to get
\begin{align*}
\tilde{\y}_k^\top \widetilde{\C}_{k,2} \tilde{\y}_k \leq C_k \cdot \text{Tr}(\widetilde{\C}_{k,2}) \cdot \log n 
\end{align*}
with probability at least $\left(1 - \frac{1}{n}\right)$.
We denote by $\{\tilde{\lambda}_j\}_{j=1}^{p-k}$ the diagonal entries of the leave-$k$-out covariance matrix $\Sigmab_{-1:k}$.
A direct application of~\cite[Lemma 30]{muthukumar2021classification} (which, in turn, is taken from~\cite[Lemma 11]{bartlett2020benign}) gives
\begin{align*}
\text{Tr}(\widetilde{\C}_{k,2}) &\leq C \left(\frac{s - k}{n} + n \cdot \frac{\sum_{j = s - k + 1}^{p - k} \tilde{\lambda}_j^2 }{(\sum_{j > s - k + 1}^{p-k} \tilde{\lambda_j})^2}\right).
\end{align*}
Then, substituting the bilevel ensemble parameterization in a manner identical to the proof of~\cite[Lemma 35]{muthukumar2021classification} gives
\begin{align}\label{eq:T2upperbound}
T_2 \leq C_k \cdot n^{-\min(m-1,2q+r-1)} \cdot \log n.
\end{align}
for $q > 1 -r$.}

{\paragraph{Controlling the term $T_1 := \tilde{\y}_k^\top \widetilde{\C}_{k,1} \tilde{\y}_k$}
Unfortunately, this term is more delicate than $T_2$, because the matrix $\widetilde{\C}_{k,1}$ intricately depends on $\z_2,\ldots,\z_k$.
However, we can unravel the expression back to get
\begin{align*}
\tilde{\y}_k^\top \widetilde{\C}_{k,1} \tilde{\y}_k &= \sum_{j=2}^k \lambda^2_j (\z_j^\top \Aneg{k} \tilde{\y}_k)^2 \\
&\leq \sum_{j=2}^k \lambda^2_j \left(|\z_j^\top \Aneg{k} \y_1| + \sum_{\ell = 1}^k |\SU^{(\ell)}_{1,2}| |\z_j^\top \Aneg{k} \z_\ell| \right)^2 \\
&\leq C_k \sum_{j=2}^k \lambda^2_j \left(|\z_j^\top \Aneg{k} \y_1| + |\z_j^\top \Aneg{k} \z_1| + \frac{1}{n^{1/4}} \sum_{\ell = 2}^k |\z_j^\top \Aneg{k} \z_\ell| \right)^2
\end{align*}
where the last inequality uses Lemma~\ref{lem:su_ck} and Lemma~\ref{lem:sumulticlass}.}

{The key observation is that there are only $O(k^2)$ such terms that we need to control.
Noting that $\Aneg{k}$ is independent of each of $\y_1$ and $\{\z_j\}_{j = 1}^k$, we now use the Hanson-Wright inequality to control each of the terms $\{\z_j^\top \Aneg{k} \y_1\}_{j=2}^k$ and $\{\z_j^\top \Aneg{k} \z_\ell\}_{j \neq \ell}$.
Note that for $j = 2,\ldots,k$, we have $\E[\y_1 \z_j^\top] = \mathbf{0}$ from Lemma~\ref{lem:expectationmlm} (a base technical lemma, proved in Appendix~\ref{sec:mlmlemmasbasic}) and $\E[\z_\ell \z_j^\top] = \delta_{\ell,j} \mathbf{I}_p$. 
We apply this inequality (as stated in~\cite[Lemma 26]{muthukumar2021classification}) for the choice $t = \|\Aneg{k}\|_2 \cdot \sqrt{n \log n}$ to get
\begin{align*}
\z_j^\top \Aneg{k} \y_1 &\leq \|\Aneg{k}\|_2 \cdot \sqrt{n \log n} \text{ and } \\
\z_j^\top \Aneg{k} \z_\ell &\leq \delta_{\ell,j} \cdot \text{tr}(\Aneg{k}) + \|\Aneg{k}\|_2 \cdot \sqrt{n \log n},
\end{align*}
each with probability at least $1 - \frac{1}{n^c}$ for some $c > 0$.
Next, applying Lemma~\ref{lem:trace and operator norm} (a base technical lemma, proved in Appendix~\ref{sec:mlmlemmasbasic}) gives us $\|\Aneg{k}\|_2 \leq \frac{C}{\lambda_L r_s(\Sigmab)}$ with probability at least $1 - 2e^{-\frac{n}{c}}$ over the random matrix $\Aneg{k}$.
We further recall that $\lambda_L r_s(\Sigmab) = n^m - n^r \geq c n^m$ for large enough $n$, and that $\lambda_j = \lambda_H=n^{m-q-r}$ for $j=2,\ldots,k$ (because under our assumptions $s>k$).
Excluding the terms $\{\z_j^\top \Aneg{k} \z_j\}_{j=2}^k$ for now, each of the above contributes the following to $T_1$:
\begin{align*}
\frac{C_k \cdot \lambda_H^2 \cdot n \log n}{\lambda_L^2 r_s^2(\Sigmab_{-1:k})} &\leq C_k \cdot n^{1 - 2q - 2r} \cdot \log n =: C_k \cdot n^{- (2q + 2r - 1)} \cdot \log n < n^{-(2q+r-1)},
\end{align*}
which is identical to the scaling for $T_2$.
We finally return to controlling the terms $\{\z_j^\top \Aneg{k} \z_j\}_{j=2}^k$.
Note that each of these terms is pre-multiplied by the factor $\frac{1}{n^{1/4}}$
Applying Lemma~\ref{lem:trace and operator norm} again gives $\text{tr}(\Aneg{k}) \leq \frac{Cn}{\lambda_L r_s(\Sigmab)}$ with probability at least $1 - 2e^{-\frac{n}{c}}$.
The contribution from each of these terms, thus, becomes
\begin{align*}
\frac{1}{n^{1/2}} \frac{C_k\lambda_H^2 n^2}{\lambda_L^2 r_s^2(\Sigmab_{-1:k})} + \frac{1}{n^{1/2}} \frac{C_k \cdot \lambda_H^2 \cdot n \log n}{\lambda_L^2 r_s^2(\Sigmab_{-1:k})} &\leq \frac{1}{n^{1/2}} \frac{C_k\lambda_H^2 n^2}{\lambda_L^2 r_s^2(\Sigmab_{-1:k})} \\
&\leq C_k \cdot n^{2 - 2q - 2r - 1/2} = C_k \cdot n^{-(2q + 2r - 3/2)}.
\end{align*}
Thus, we get
\begin{align}\label{eq:T1upperbound}
T_1 &\leq C_k \cdot n^{-\min(2q +r-1,2q+2r-3/2)} \cdot \log n.
\end{align}}
{\paragraph{Putting it all together}
Recall that $\CN_{1,2}^2 := T_1 + T_2$.
Therefore, putting together the upper bounds from Equations~\eqref{eq:T1upperbound} and~\eqref{eq:T2upperbound} gives us the following statement:
\begin{align*}
    \CN_{1,2}(n) \leq C_k \sqrt{\log n} \cdot n^{-\frac{\min\{m-1,2q+r-1,2q + 2r - 3/2\}}{2}}
\end{align*}
for $q > 1-r$ and a universal constant $C_k$ that depends only on $k$.
This is the desired statement.
Further, a union bound over each of the probabilistic inequalities implies that the statement holds with probability at least $1 - \frac{C_k}{n^c}$ for some $0 < c \leq 1$.
This completes the proof of Lemma~\ref{lem:cnmulticlass}.
\qed}

\section{Supporting technical lemmas for MLM error analysis}

{In this section, we prove the supporting technical lemmas for the MLM error analysis.}

\subsection{Basic lemmas about the MLM}\label{sec:mlmlemmasbasic}

{We begin by collecting basic lemmas about the MLM that form building blocks to prove the rest of the technical lemmas.
The first such basic lemma controls the expectation of certain product forms involving the difference label vector $\y_1$ and individual feature vectors $\{\z_\ell\}_{\ell = 1}^p$.}

{\begin{lemma}\label{lem:expectationmlm}
Let $\y_1 = \boldv_1 - \boldv_2$ be the difference label vector for $c_1 = 1,c_2 = 2$ and $\{\z_\ell\}_{\ell = 1}^p$ be defined as in the proof of Lemma~\ref{lem:sumulticlass}.
Then, we have for every $i \in [n]$,
\begin{align*}
    c_{k,\ell} := \E[y_{1,i} z_{\ell,i}] = c_k \delta_{1,\ell},
\end{align*}
where $c_k > 0$ is a universal positive constant that depends only on $k$.
\end{lemma}}

{\color{black}
\begin{proof}
To prove this lemma we  utilize the orthogonality and equal-weight Assumption~\ref{as:mlmfull} as well as the details of the MLM.
We denote $\uv_j := \X^\top \ehat_j$.
It is easy to see from the definition of the changed basis $\{\z_j\}_{j=1}^p$ that $\z_j = \uv_j$ for all $j \geq 3$, and $\z_1 = \frac{1}{\sqrt{2}}(\uv_1 - \uv_2)$ and $\z_2 = \frac{1}{\sqrt{2}}(\uv_1 + \uv_2)$.
We now use the simplex-ETF-type structure of $\boldv_1,\boldv_2$ together with the structure in the MLM model to get 
\begin{align*}
\Pro\left(y_{1,i} = 1 \Big{|} \{u_{1,i},u_{2,i},\ldots,u_{k,i}\}\right) = \frac{\exp(u_{1,i})}{\sum_{c' \in [k]} \exp(u_{c',i})} \text{ and } \\
\Pro\left(y_{1,i} = -1 \Big{|} \{u_{1,i},u_{2,i},\ldots,u_{k,i}\right) = \frac{\exp(u_{2,i})}{\sum_{c' \in [k]} \exp(u_{j_{c',i}})},
\end{align*}
and $y_{1,i} = 0$ otherwise.
Note here that $\{u_{c,i}\}_{c \in [k]}$ are i.i.d. standard Gaussian.
We start with the case $\ell = 1$.
Here, we get
\begin{align*}
    \E[z_{1,i} y_{1,i}] &= \frac{1}{\sqrt{2}} \cdot \E\left[(u_{1,i} - u_{2,i}) \cdot \frac{\exp(u_{1,i})}{\sum_{c' \in [k]} \exp(u_{c',i})} - (u_{1,i} - u_{2,i}) \cdot  \frac{\exp(u_{2,i})}{\sum_{c' \in [k]} \exp(u_{j_{c',i}})} \right] \\
    &= \frac{1}{\sqrt{2}} \cdot \E\left[(U_1 - U_2) \cdot \frac{(e^{U_1} - e^{U_2})}{\sum_{c = 1}^k e^{U_c}}\right] \\
    &= \sqrt{2} \cdot \E\left[U_1 \cdot \frac{(e^{U_1} - e^{U_2})}{\sum_{c = 1}^k e^{U_c}}\right],
\end{align*}
where the last step follows by symmetry.
Note that we have overloaded notation and written $U_c := u_{c,i}$ for each  $c \in [k]$.
We also write $\mathbf{U} := \begin{bmatrix} U_1 & \ldots & U_k \end{bmatrix}$ as shorthand.
Because $U_c \text{ i.i.d. } \sim \mathcal{N}(0,1)$, we have
\begin{align*}
    c_k &= \E\left[U_1 \cdot g(\mathbf{U})\right]
\end{align*}
where $g(\mathbf{U}) := \frac{e^{U_1} - e^{U_2}}{\sum_{c=1}^k e^{U_c}}$.
Then, applying Stein's lemma, we get
\begin{align*}
    \E\left[U_1 \cdot g(\mathbf{U})\right] &= \sum_{i=1}^n \E[U_1 U_i] \cdot \E\left[\frac{\partial g}{\partial U_i} \right] \\
    &= \E\left[\frac{\partial g}{\partial U_1} \right] \\
    &= \E\left[\frac{\sum_{i \geq 3} e^{U_1 + U_i} + 2e^{U_1 + U_2}}{(\sum_{i=1}^k e^{U_i})^2}\right] =: c_k > 0 .
    \end{align*}
The last step follows because the argument inside the expectation can never take value $0$ and is always non-negative. Thus, we have proved that $\E[y_{1,i} z_{1,i}] = c_k > 0$.

We now prove that $\E[y_{1,i} z_{\ell,i}] = 0$ for $\ell \neq 1$.
First, for $\ell \geq 3$, we have
\begin{align*}
\E[y_{1,i} z_{\ell,i}] = \frac{1}{\sqrt{2}} \cdot \E\left[u_{\ell,i} \cdot \frac{\exp(u_{1,i})}{\sum_{c' \in [k]} \exp(u_{c',i})} - u_{\ell,i} \cdot  \frac{\exp(u_{2,i})}{\sum_{c' \in [k]} \exp(u_{{c',i}})} \right] = 0
\end{align*}
by symmetry.
Next, for $\ell = 2$, we have
\begin{align*}
 \E[z_{2,i} y_{1,i}] &= \frac{1}{\sqrt{2}} \cdot \E\left[(u_{1,i} + u_{2,i}) \cdot \frac{\exp(u_{1,i})}{\sum_{c' \in [k]} \exp(u_{c',i})} - (u_{1,i} + u_{2,i}) \cdot  \frac{\exp(u_{2,i})}{\sum_{c' \in [k]} \exp(u_{j_{c',i}})} \right] \\
    &= \frac{1}{\sqrt{2}} \cdot \E\left[(U_1 + U_2) \cdot \frac{(e^{U_1} - e^{U_2})}{\sum_{c = 1}^k e^{U_c}}\right] \\
    &= \E\left[U_1 \cdot \frac{(e^{U_1} - e^{U_2})}{\sum_{c = 1}^k e^{U_c}}\right] - \E\left[U_2 \cdot \frac{(e^{U_2} - e^{U_1})}{\sum_{c = 1}^k e^{U_c}}\right] = 0,
\end{align*}
where the last equality follows by symmetry.
This completes the proof.
\end{proof}
}

{The next basic lemma controls the trace and operator norm of leave-$\ell$-out Gram matrices and leverages ideas first appearing in~\cite{bartlett2020benign}.
\begin{lemma}\label{lem:trace and operator norm}
For all $\ell\in[k]$ and sufficiently large $n$, the following inequalities are true for universal constants $c,C>0$, each with probability at least $1 - 2e^{-\frac{n}{c}}$:
\[
\|\Aneg{\ell}\|_2  \leq \frac{c}{\lambda_L r_s(\Sigmab)}\,
\]
and
\[
\frac{cn}{{\lambda}_{L} r_s(\Sigmab)}\geq \mathsf{tr}(\A_{-1:\ell}^{-1}) \geq \frac{(n-s)}{c{\lambda}_{L} r_s(\Sigmab)}\,.
\]
In particular, these imply
\begin{align}\label{eq:operator over trace}
\frac{\|\A_{-1:\ell}^{-1} \|_2 \cdot n^{3/4}}{\mathsf{tr}(\A_{1:\ell}^{-1})} \leq \frac{C_2}{n^{1/4}}\,.
\end{align}
\end{lemma}}

{\color{black}
\begin{proof}
First, we upper bound the operator norm term. Observe that
\begin{align*}
\|\A_{-1:\ell}^{-1} \|_{2} = \mu_1(\A_{-1:\ell}^{-1}) = \frac{1}{\mu_n(\A_{-1:\ell})} \leq \frac{1}{\mu_n(\A_{-1:s})} \leq \frac{c}{\lambda_{s+1} r_s(\Sigmab)},
\end{align*}
where the last inequality uses \cite[Lemma 5]{bartlett2020benign}.
The second-to-last inequality holds for any choice of $s > k \geq \ell$.

Next, we prove the bounds for the trace term. We lower bound the trace term as
\begin{align*}
\mathsf{tr}(\A_{-1:\ell}^{-1}) &= \sum_{j=1}^n \frac{1}{\mu_j(\A_{-1:\ell})} \geq \sum_{j=s}^n \frac{1}{\mu_j(\A_{-1:\ell})} \geq \frac{(n-s)}{\mu_{s+1}(\A_{-1:\ell})}.
\end{align*}
Thus, it remains to upper bound $\mu_{s+1}(\A_{-1:\ell})$.
Let $\{\widetilde{\lambda}_{j}\}_{j=1}^{p - \ell}$ denote the re-indexed eigenvalues of $\Sigmab_{-1:\ell}$.
Then, Equation (38) from Lemma 25 in~\cite{muthukumar2021classification} directly yields
\begin{align*}
\mu_{s+1}(\A_{-1:\ell}) \leq C \widetilde{\lambda}_{s+1} r_s(\Sigmab_{-1:\ell}) 
\end{align*}
provided that $r_s(\Sigmab_{-1:\ell}) \geq bn$.
(Note that, under the bilevel ensemble, we have $r_s(\Sigmab_{-1:\ell}) = \frac{n^m - \ell - s}{\lambda_L} \geq c n^m > b n$ for large enough $n$.)
Similarly, we upper bound the trace term as
\begin{align*}
\mathsf{tr}(\A_{-1:\ell}^{-1}) \leq \frac{n}{\mu_n(\A_{-1:\ell})} \leq \frac{cn}{\widetilde{\lambda}_{s+1} r_s(\Sigmab_{-1:\ell})}
\end{align*}
where we now used Equation (37) from Lemma 25 in~\cite{muthukumar2021classification}.
To complete the proof for the trace term, we show that $\wt \lambda_{s+1} r_s(\Sigmab_{-1:\ell}) \asymp \lambda_L r_s(\Sigmab)$.
First, we note that $\widetilde{\lambda}_{s+1} = \lambda_{s+1} = \lambda_L$ under the bilevel ensemble. 
Also recall that $\ell\leq k < s$; hence we have $r_s(\Sigmab_{-1:\ell}) = p - s - \ell $ and $r_s(\Sigmab) = p - s$, which implies that $r_s(\Sigmab_{-1:\ell}) \asymp r_s(\Sigmab)$  for large enough $n$. Putting all of this together yields the desired inequalities about the trace.

Finally, we prove Equation~\eqref{eq:operator over trace}. 
This follows because, as already shown, we have 
\begin{align*}
\|\A_{-1:\ell}^{-1} \|_2 \cdot n^{3/4} \leq \frac{c\,n^{3/4}}{\lambda_L r_s(\Sigmab)} 
\mathsf{tr}(\A_{-1:\ell}^{-1}) &\geq \frac{n - s}{c\,\lambda_L r_s(\Sigmab)}.
\end{align*}
thereby giving us
\begin{align}\label{eq:ratioupperbound}
\frac{\|\A_{-1:\ell}^{-1} \|_2 \cdot n^{3/4}}{\mathsf{tr}(\A_{-1:\ell}^{-1})} \leq \frac{n^{3/4}}{n - s} \leq \frac{2}{n^{1/4}}\,,
\end{align}
where the last inequality follows for large enough $n$ because $s = n^r$ and we have assumed $r < 1$.
This completes the proof of the lemma.
\end{proof}
}

{The following basic lemma relates the ratios of quadratic forms that are ``similar" in their probability distribution.
\begin{lemma}\label{lem:easy}
We have
\[
 \frac{\z_\ell^T\Aneg{k}\z_\ell}{\z_{\ell'}^T\Aneg{k}\z_{\ell'}}\leq C \qquad \text{and} \qquad
 \frac{\z_\ell^T\Aneg{\ell}\z_\ell}{\z_{\ell'}^T\Aneg{\ell'}\z_{\ell'}}\leq C
\]
for all $\ell,\ell'\in[k]$ with probability at least $1 - c\,k\,e^{-\sqrt{n}}$.
\end{lemma}} 

{\color{black}
\begin{proof}
Recall that for any $\ell,\ell'$, we have that $\z_\ell,\z_{\ell'}$ are both independent of $\Aneg{k}$.
Therefore, we have 
\begin{align*}
\z_{\ell'}^T\Aneg{k}\z_{\ell'} &\geq \mathsf{tr}(\A_{-1:k}^{-1})-  c_1 \|\A_{-1:k}^{-1}\|_{2} \cdot n^{3/4}  \text{ and } \\
\z_\ell^T\Aneg{k}\z_\ell &\leq \mathsf{tr}(\A_{-1:k}^{-1})+  c_1 \|\A_{-1:k}^{-1}\|_{2} \cdot n^{3/4} \nonumber
\end{align*}
with probability at least $1 - ke^{-\sqrt{n}}$.
Putting these together gives
\begin{align*}
\frac{\z_\ell^T\Aneg{k}\z_\ell}{\z_{\ell'}^T\Aneg{k}\z_{\ell'}} &\leq \frac{\mathsf{tr}(\A_{-1:k}^{-1})+  c_1 \|\A_{-1:k}^{-1}\|_{2} \cdot n^{3/4}}{\mathsf{tr}(\A_{-1:k}^{-1})-  c_1 \|\A_{-1:k}^{-1}\|_{2} \cdot n^{3/4}} \leq \frac{1 + \frac{c_1}{n^{1/4}}}{1 - \frac{c_1}{n^{1/4}}} \leq 2\,,
\end{align*}
for large enough $n$, where in the above we used Eq. \eqref{eq:operator over trace}.

To prove the second inequality, recall that for any $\ell,\ell' \in [k]$, $\z_\ell$ is independent of $\Aneg{\ell}$ and $\z_{\ell'}$ is independent of $\Aneg{\ell'}$.
Consequently, we have 
\begin{align*}
\z_{\ell'}^T\Aneg{\ell'}\z_{\ell'} &\geq \mathsf{tr}(\Aneg{\ell'})-  c_1 \|\Aneg{\ell'}\|_{2} \cdot n^{3/4}  \text{ and }  \\
\z_\ell^T\Aneg{\ell}\z_\ell &\leq \mathsf{tr}(\Aneg{\ell})+  c_1 \|\Aneg{\ell}\|_{2} \cdot n^{3/4} 
\end{align*}
with probability at least $1 - ke^{-\sqrt{n}}$.
Putting these together gives
\begin{align*}
\frac{\z_\ell^T\Aneg{\ell}\z_\ell}{\z_{\ell'}^T\Aneg{\ell'}\z_{\ell'}} &\leq \frac{\mathsf{tr}(\Aneg{\ell})+  c_1 \|\Aneg{\ell}\|_{2} \cdot n^{3/4}}{\mathsf{tr}(\Aneg{\ell'})-  c_1 \|\Aneg{\ell'}\|_{2} \cdot n^{3/4}} \leq \frac{1 + \frac{c_1}{n^{1/4}}}{1 - \frac{c_1}{n^{1/4}}} \leq 2\,,
\end{align*}
for large enough $n$, where we again used Eq. \eqref{eq:operator over trace}. 
This completes the proof of the lemma.
\end{proof}}

{Finally, the following basic lemma controls the ratio of traces of the leave-$\ell$-out Gram matrix and the leave-$k$-out Gram matrix for any $\ell \in [k]$.
\begin{lemma}\label{lem:trace ratio}
For all $\ell\in[k]$ and sufficiently large $n$, it holds for universal constant $C$ that
\[
\frac{\mathsf{tr}(\A_{-1:\ell}^{-1})}{\mathsf{tr}(\A_{1:k}^{-1})} \geq \left(1-\frac{C}{n}\right)^{k-\ell} \geq \left(1-\frac{C}{n}\right)^k
\]
\end{lemma}
\begin{proof}
Fix any $\ell\in[k]$. 
We first lower-bound the ratio $\frac{\mathsf{tr}({\Aneg{\ell}})}
    {\mathsf{tr}({\Aneg{\ell+1}})}$, and then apply the argument recursively.
Since $\A_{-1:\ell)}=\A_{-1:\ell+1}+\la_H\z_{\ell+1}\z_{\ell+1}^T$, we can apply the matrix inversion lemma to get
\[
\Aneg{\ell} = \Aneg{\ell+1}-\frac{\la_H\Aneg{\ell+1}\z_{\ell+1}\z_{\ell+1}^T\Aneg{\ell+1}}{1+\la_H\z_{\ell+1}^T\Aneg{\ell+1}\z_{\ell+1}}\,.
\]
Hence, we have
\begin{align*}
\mathsf{tr}({\Aneg{\ell}}) &= \mathsf{tr}(\Aneg{\ell+1})-\frac{\la_H\,\mathsf{tr}(\Aneg{\ell+1}\z_{\ell+1}\z_{\ell+1}^T\Aneg{\ell+1})}{1+\la_H\z_{\ell+1}^T\Aneg{\ell+1}\z_{\ell+1}} = \mathsf{tr}(\Aneg{\ell+1})-\frac{\la_H\,\z_{\ell+1}^T\A_{-1:\ell+1}^{-2}\z_{\ell+1}}{1+\la_H\z_{\ell+1}^T\Aneg{\ell+1}\z_{\ell+1}}\,
\\
&\geq \mathsf{tr}(\Aneg{\ell+1})-\|\Aneg{\ell+1}\|_2\,\cdot\frac{\la_H\,\z_{\ell+1}^T\A_{-1:\ell+1}^{-1}\z_{\ell+1}}{1+\la_H\z_{\ell+1}^T\Aneg{\ell+1}\z_{\ell+1}}
\\
&\geq \mathsf{tr}(\Aneg{\ell+1})-\|\Aneg{\ell+1}\|_2
\end{align*}
(The second inequality follows because for any positive semidefinite matrix $\M$ with eigendecomposition $\M=\Ub\boldsymbol{\Lambda}\Ub^T=\sum_{i}\lambda_i\ub_i\ub_i^T$ we have
\[
\x^T\M^2\x = (\Ub\x)^T\boldsymbol{\Lambda}^2(\Ub\x)=\sum_{i}\lambda_i^2(\ub_i^T\x)^2 \leq \left(\max_{i}\lambda_i\right)\sum_{i}\lambda_i(\ub_i^T\x)^2 = \|\M\|_2\cdot\x^T\M\x\,
\]
for any vector $\x$.)
Continuing from the penultimate display, we obtain
\begin{align*}
    \frac{\mathsf{tr}({\Aneg{\ell}})}
    {\mathsf{tr}({\Aneg{\ell+1}})}\geq 1- \frac{\|\Aneg{\ell+1}\|_2}{{\mathsf{tr}({\Aneg{\ell+1}})}}
    \geq 1 - \frac{C}{n}
\end{align*}
where the last inequality applies Eq. \eqref{eq:operator over trace}. 
Recursively applying the above for $\ell+1,\ldots,k$ completes the proof of the lemma.
\end{proof}}

\subsection{Survival Term}\label{sec:proofs of technical survival}

{In this section we provide the proofs of Lemmas~\ref{lem:Qk},~\ref{lem:ratio bound},~\ref{lem:Rellell bound} and~\ref{lem:Qtilde main bound}.}

\subsubsection{Proof of Lemma \ref{lem:Qk}}\label{sec:proof of Lemma 4}

{First, we note that $\y_1$ remains independent of $\A_{-1:k}$ as $\y_1$ only depends on $\z_1,\ldots,\z_k$ (which are in turn mutually independent of $\z_{k+1},\ldots,\z_p$ which comprise of $\A_{-1:k}$).
Therefore, we can directly apply the Hanson-Wright inequality to get 
\begin{subequations}\label{eq:qkbounds}
\begin{align}
Q_k &\geq c_k \cdot \sqrt{\frac{2}{\pi}} \mathsf{tr}(\A_{-1:k}^{-1})- 2 c_1 \|\A_{-1:k}^{-1}\|_{2} \cdot n^{3/4}  \text{ and } \label{eq:qklowerbound} \\
Q_k &\leq c_k \cdot \sqrt{\frac{2}{\pi}} \mathsf{tr}(\A_{-1:k}^{-1})+ 2 c_1 \|\A_{-1:k}^{-1}\|_{2} \cdot n^{3/4}.\label{eq:qkupperbound}
\end{align}
\end{subequations}
with probability at least $1-2e^{-\sqrt{n}}$.
Combining the above with Lemma \ref{lem:trace and operator norm} applied for $\ell=k$ directly gives the desired statement of Equation~\eqref{eq:Qkexpression}, completing the proof of the lemma.
\qed}

\subsubsection{Proof of Lemma \ref{lem:ratio bound}}

{Note that the quadratic-like terms in both the LHS and RHS of \eqref{eq:delta_smallterm} are well-suited for an application of the Hanson-Wright inequality, since $\z_1,\z_\ell$ are independent of $\Aneg{\ell}$ for all $\ell=2,\ldots,k$. This is formalized in the lemma below. Specifically, the desired statement to prove Lemma~\ref{lem:ratio bound}, i.e.~Eq. \eqref{eq:delta_smallterm} for $\ell=1,\ldots,k$, follows directly by applying Lemma \ref{lem:zz terms} below for the special case $\ell'=j=1$. 
(The slightly more general statement of the lemma below will prove useful for proving subsequent lemmas.)}

{\begin{lemma}\label{lem:zz terms} 
For large enough $n$, for all $\ell \in [k]$ and $\ell' < \ell,j \leq \ell$ we have
\[
\abs{\z_\ell^T\Aneg{\ell}\z_{\ell'}} \leq \frac{C}{n^{1/4}} \,\z_j^T\Aneg{\ell}\z_j\,.
\]
with probability at least $1-Ck^3e^{-\sqrt{n}}$.
\end{lemma}
\begin{proof}
The key observation is that for all $\ell'<\ell,j\leq\ell$, we have that $\z_\ell,\z_{\ell'},$ and $\z_{j}$ are all mutually independent of $\Aneg{\ell}$.
Therefore, applying the Hanson-Wright inequality in the form stated by~\cite{muthukumar2021classification} gives us the following: for all $\ell\in[k],\ell'<\ell,j\leq\ell$, we have
\begin{align*}
  \abs{\z_{\ell}^T\Aneg{\ell}\z_{\ell'}} &\leq  2 c_1 \|\A_{-1:\ell}^{-1}\|_{2} \cdot n^{3/4} \qquad \text{ and } \label{eq:qklowerbound2} \\
 \z_{j}^T\Aneg{\ell}\z_j&\geq \mathsf{tr}(\A_{-1:\ell}^{-1})-  c_1 \|\A_{-1:\ell}^{-1}\|_{2} \cdot n^{3/4}\,,
\end{align*}
with probability at least $1 - Ck^3e^{-\sqrt{n}}$.
Above, we used the fact that $\z_\ell,\z_{\ell'}$ are independent. Therefore, to prove the desired it suffices to show that 
\begin{align}
{\mathsf{tr}(\A_{1:\ell}^{-1})} \geq \frac{n^{1/4}}{C_2}\|\A_{-1:\ell}^{-1} \|_2 \cdot n^{3/4}\,.
\end{align}
This follows immediately from Eq. \eqref{eq:operator over trace} in Lemma \ref{lem:trace and operator norm}.
This completes the proof.
\end{proof}}

\subsubsection{Proof of Lemma \ref{lem:Rellell bound}}


{Recall that $R_{\ell,\ell}:=\y_1^T\Aneg{\ell}\z_\ell.$ Bounding this term is  difficult because $\y_1$ depends on $\Aneg{\ell}$ for any $\ell<k$. The only ``easy'' case is for $\ell=k$ for which $\y_1$ is independent of $\Aneg{k}$. 
As a starting point, we exploit this independence to control the terms $R_{k,\ell}=\y_1^T\Aneg{k}\z_{\ell}$ for all $\ell\in [k]$, in the lemma below.}

{\begin{lemma}\label{lem:Rkl}
We have, for large enough $n$,
\[
 \abs{R_{k,\ell}}=\abs{\y_1^T\Aneg{k}\z_\ell} \leq \frac{C_k}{n^{1/4}}\, \z_k^T\Aneg{k}\z_k\ \text{ for any } \ell = 2,\ldots,k\,,
\]
and
\[
 \abs{R_{k,1}}=\abs{\y_1^T\Aneg{k}\z_1} \leq C_k \, \z_k^T\Aneg{k}\z_k\,
\]
with probability at least $1 - Cke^{-\sqrt{n}}$.
\end{lemma}}

{\begin{proof}
Recall that all of $\y_1,\z_\ell,\z_k$ are independent of $\Aneg{k}$.
Therefore, we can apply the Hanson-Wright inequality to the RHS of the above, as well as $R_{k,\ell}$ (using the parallelogram law in the latter case) to get
\begin{align*}
 \mathsf{tr}(\A_{-1:k}^{-1})- c_1 \|\A_{-1:k}^{-1}\|_{2} \cdot n^{3/4}\leq \z_k^\top \Aneg{k} \z_k &\leq \mathsf{tr}(\A_{-1:k}^{-1})+ c_1 \|\A_{-1:k}^{-1}\|_{2} \cdot n^{3/4} \\
c_{k,\ell} \cdot \mathsf{tr}(\A_{-1:k}^{-1}) - 2 c_1 \|\A_{-1:k}^{-1}\|_{2} \cdot n^{3/4}\leq R_{k,\ell} &\leq c_{k,\ell} \cdot \mathsf{tr}(\A_{-1:k}^{-1}) + 2 c_1 \|\A_{-1:k}^{-1}\|_{2} \cdot n^{3/4},
\end{align*}
with probability at least $1 - Cke^{-\sqrt{n}}$.
Above, we define $c_{k,\ell} := \E[y_{1,i} z_{\ell,i}]$ (identically for any $i \in [n]$).
There are then two cases:
\begin{enumerate}
\item $\ell = 1$: In this case we get $c_{k,\ell} =: c_k > 0$ from Lemma~\ref{lem:expectationmlm}.
Plugging this above gives
\begin{align*}
\frac{R_{k,1}}{\z_k^T\Aneg{k}\z_k} &\leq \frac{c_k \cdot \mathsf{tr}(\A_{-1:k}^{-1}) + 2 c_1 \|\A_{-1:k}^{-1}\|_{2} \cdot n^{3/4}}{\mathsf{tr}(\A_{-1:k}^{-1})- c_1 \|\A_{-1:k}^{-1}\|_{2} \cdot n^{3/4}} \\
&\leq c_k \frac{1 + \frac{c_1}{n^{1/4}}}{1 - \frac{c_1}{n^{1/4}}} \leq 2 c_k =: C_k,
\end{align*}
where {the second inequality follows from Eq. \eqref{eq:operator over trace} in Lemma \ref{lem:trace and operator norm}} and the last inequality follows for large enough $n$.
Similarly, we have
\begin{align*}
\frac{R_{k,1}}{\z_k^T\Aneg{k}\z_k} &\geq -\frac{  2 c_1 \|\A_{-1:k}^{-1}\|_{2} \cdot n^{3/4}}{\mathsf{tr}(\A_{-1:k}^{-1}) + c_1 \|\A_{-1:k}^{-1}\|_{2} \cdot n^{3/4}} \\
&= -\frac{  2 c_1 }{\frac{\mathsf{tr}(\A_{-1:k}^{-1})}{\|\A_{-1:k}^{-1}\|_{2} \cdot n^{3/4}} + c_1}
\\
&\geq -\frac{  2 c_1 }{\frac{n^{1/4}}{2} + c_1} \geq - \frac{C}{n^{1/4}},
\end{align*}
where the second-to-last inequality in the above again used Equation~\eqref{eq:operator over trace}.
\item $\ell \neq 1$: In this case we have $c_{k,\ell} = 0$, again from Lemma~\ref{lem:expectationmlm}.
Plugging this above gives
\begin{align*}
\frac{R_{k,\ell}}{\z_k^T\Aneg{k}\z_k} &\leq \frac{2 c_1 \|\A_{-1:k}^{-1}\|_{2} \cdot n^{3/4}}{\mathsf{tr}(\A_{-1:k}^{-1})- c_1 \|\A_{-1:k}^{-1}\|_{2} \cdot n^{3/4}} \\
&\leq \frac{2 c_1}{\frac{\mathsf{tr}(\A_{-1:k}^{-1})}{\|\A_{-1:k}^{-1}\|_{2} \cdot n^{3/4}} + c_1} \leq \frac{2c_1}{\frac{n^{1/4}}{2} + c_1} \leq \frac{C}{n^{1/4}},
\end{align*}
where the last inequality follows for large enough $n$.
Similarly, we have
\begin{align*}
\frac{R_{k,\ell}}{\z_k^T\Aneg{k}\z_k} &\geq -\frac{  2 c_1 \|\A_{-1:k}^{-1}\|_{2} \cdot n^{3/4}}{\mathsf{tr}(\A_{-1:k}^{-1}) + c_1 \|\A_{-1:k}^{-1}\|_{2} \cdot n^{3/4}} \\
&= -\frac{  2 c_1 }{\frac{\mathsf{tr}(\A_{-1:k}^{-1})}{\|\A_{-1:k}^{-1}\|_{2} \cdot n^{3/4}} + c_1}
\\
&\geq -\frac{  2 c_1 }{\frac{n^{1/4}}{2} + c_1} \geq - \frac{C}{n^{1/4}}
\end{align*}
where in the penultimate line we again used Eq. \eqref{eq:operator over trace}.
\end{enumerate}
\end{proof}}

{We now build on the ``base case" Lemma~\ref{lem:Rkl} to control the terms $R_{\ell,\ell}$ in a similar manner to $R_{k,\ell}$.
In particular, we note that the desired Eq.~\eqref{eq:delta_Rterm} to prove Lemma~\ref{lem:Rellell bound} follows by applying the slightly more general lemma below for the case $\ell'=\ell$.
\begin{lemma}\label{lem:Rell general}
For all $\ell \in [k]$ and all $\ell' \leq \ell$, we have
\begin{align}\label{eq:R_rec_2prove}
 \abs{R_{\ell,\ell'}}\leq \begin{cases}
     C_k \cdot\z_\ell^T\Aneg{\ell}\z_\ell\ \text{ if } \ell' = 1 \\
     \frac{C_k}{n^{1/4}}  \cdot\z_\ell^T\Aneg{\ell}\z_\ell\ \text{ if } \ell' \neq 1
 \end{cases}
\end{align}
with probability at least $1-ck^3e^{-\sqrt{n}}$.
\end{lemma}}

{We complete the proof of Lemma~\ref{lem:Rellell bound} by proving Lemma~\ref{lem:Rell general}, which we do in the next section.}
\subsubsection{Proof of Lemma \ref{lem:Rell general}}

{We will use recursion starting from $\ell=k$ to prove the desired statement for all $\ell=k-1,k-2,\ldots,1$. Throughout, we condition on the events of Lemmas \ref{lem:easy}, \ref{lem:zz terms}, and \ref{lem:Rkl}. The key to allow proving the statement recursively is the following relation that follows by the matrix-inversion-lemma and holds for all $\ell'\leq\ell$:
\begin{align}
    R_{\ell,\ell'} &=\y_1^T\A_{-1:\ell}^{-1}\z_{\ell'} =  \y_1^T\left(\A_{-1:\ell+1}+\z_{\ell+1}\z_{\ell+1}^T\right)^{-1}\z_{\ell'}
    \nn
    \\
    &= \y_1^T\Aneg{\ell+1}\z_{\ell'} - \frac{\la_H\left(\y_1^T\Aneg{\ell+1}\z_{\ell'}\right)\left(\z_{\ell+1}^T\Aneg{\ell+1}\z_{\ell'}\right)}{1+\la_H \z_{\ell+1}^T\Aneg{\ell+1}\z_{\ell+1}}
    \nn
    \\
    &= R_{\ell+1,\ell'} - R_{\ell+1,\ell+1}\, \frac{\la_H\left(\z_{\ell+1}^T\Aneg{\ell+1}\z_{\ell'}\right)}{1+\la_H \z_{\ell+1}^T\Aneg{\ell+1}\z_{\ell+1}}\,.\label{eq:R_rec_key}
\end{align}}

{First we prove the statement for the base case $\ell=k-1$. For any $\ell'\leq k-1$, Equation~\eqref{eq:R_rec_key} gives us
\begin{align*}
    R_{k-1,\ell'}&= R_{k,\ell'} - R_{k,k}\, \frac{\la_H\left(\z_{k}^T\Aneg{k}\z_{\ell'}\right)}{1+\la_H \z_{k}^T\Aneg{k}\z_{k}}\,.
\end{align*}
Note that because $\ell' \leq k-1$, we have $\ell' < k$.
Thus, we can apply Lemma~\ref{lem:zz terms} to get
\begin{align*}
    \abs{\eps_{k,\ell'}}:=\frac{\la_H\abs{\left(\z_{k}^T\Aneg{k}\z_{\ell'}\right)}}{1+\la_H \z_{k}^T\Aneg{k}\z_{k}} \leq \frac{C}{n^{1/4}}\, \frac{\la_H{\left(\z_{k}^T\Aneg{k}\z_{k}\right)}}{1+\la_H \z_{k}^T\Aneg{k}\z_{k}} \leq \frac{C}{n^{1/4}}\,.
\end{align*}
Also, by Lemma \ref{lem:Rkl}, we have
\begin{align*}
    \abs{R_{k,1}}&\leq C_k \z_k^T\Aneg{k}\z_k \text{ and } \\
    \abs{R_{k,j}}&\leq \frac{C_k}{n^{1/4}} \cdot \z_k^T\Aneg{k}\z_k \text{ for all } j = 2,\ldots,k.
\end{align*}
Combining the three displays above with the recursion in Equation~\eqref{eq:R_rec_key} yields the following for large enough $n$:
\begin{align*}
    \abs{R_{k-1,1}}&\leq  C_k\left(1+C n^{-1/4}\right)\cdot \z_k^T\Aneg{k}\z_k \leq C_k\cdot \z_k^T\Aneg{k}\z_k \text{ and } \\
    \abs{R_{k-1,\ell'}}&\leq  C_k n^{-1/4}\left(1+C\right)\cdot \z_k^T\Aneg{k}\z_k \leq \frac{C_k}{n^{1/4}}\cdot \z_k^T\Aneg{k}\z_k \text{ for all } \ell' \in \{2,\ldots,k-1\}.
\end{align*}
Lemma \ref{lem:easy} (applied for the pair ($k,k-1$)) then gives us the desired Equation~\eqref{eq:R_rec_2prove} for $\ell=k-1$, i.e. $\abs{R_{k-1,1}}\leq C_k\cdot \z_{k-1}^T\Aneg{k-1}\z_{k-1}$ and $\abs{R_{k-1,\ell'}}\leq C_k n^{-1/4} \cdot \z_{k-1}^T\Aneg{k-1}\z_{k-1}$ for $\ell' = 2,\ldots,k-1$.
The base case is therefore proved.}

{Next, we prove the inductive step. In particular, we assume that Equation~\eqref{eq:R_rec_2prove} is true for $\ell+1$ and use it to prove the claim for $\ell$. Our starting point is, again, the recursive relation in Equation~\eqref{eq:R_rec_key}.
Noting that $\ell' < \ell +1$, we can again apply Lemma \ref{lem:zz terms} to get
\begin{align*} \abs{\eps_{\ell+1,\ell'}}:=\frac{\la_H\abs{\left(\z_{\ell+1}^T\Aneg{\ell+1}\z_{\ell'}\right)}}{1+\la_H \z_{\ell+1}^T\Aneg{\ell+1}\z_{\ell+1}} \leq \frac{C}{n^{1/4}}\, \frac{\la_H{\left(\z_{\ell+1}^T\Aneg{\ell+1}\z_{\ell+1}\right)}}{1+\la_H \z_{\ell+1}^T\Aneg{\ell+1}\z_{\ell+1}} \leq \frac{C}{n^{1/4}}\,.
\end{align*}
Also, by the induction hypothesis, we have
\begin{align*}
    \abs{R_{\ell+1,1}}&\leq C_k\, \z_{\ell+1}^T\Aneg{\ell+1}\z_{\ell+1} \text{ and } \\
     \abs{R_{\ell+1,j}}&\leq \frac{C_k}{n^{1/4}}\, \z_{\ell+1}^T\Aneg{\ell+1}\z_{\ell+1} \text{ for all } j = 2,\ldots,k.
\end{align*}
Note that the sharper second inequality above applies to the term $R_{\ell+1,\ell+1}$ because we always have $\ell+1 \geq 2$.
Combining the two displays above with the recursion in Equation~\eqref{eq:R_rec_key} yields the following for large enough $n$:
\begin{align*}
    \abs{R_{\ell,1}}&\leq  C_k\left(1+C n^{-1/4}\right)\cdot \z_{\ell+1}^T\Aneg{\ell+1}\z_{\ell+1} \leq C_k\cdot \z_{\ell+1}^T\Aneg{\ell+1}\z_{\ell+1} \leq C_k\cdot \z_{\ell}^T\Aneg{\ell}\z_{\ell}
\end{align*}
Similarly, we have for all $\ell' = 2,\ldots,\ell$,
\begin{align*}
    \abs{R_{\ell,\ell'}}&\leq  C_k n^{-1/4}\left(1+C\right)\cdot \z_{\ell+1}^T\Aneg{\ell+1}\z_{\ell+1} \leq C_k n^{-1/4} \cdot \z_{\ell+1}^T\Aneg{\ell+1}\z_{\ell+1} \leq C_k n^{-1/4} \cdot \z_{\ell}^T\Aneg{\ell}\z_{\ell}
\end{align*}
In both cases above, the last inequality follows from Lemma~\ref{lem:easy}.
This completes the proof of Lemma~\ref{lem:Rell general}, and therefore the proof of Lemma~\ref{lem:Rellell bound}.
\qed}

\subsubsection{Proof of Lemma \ref{lem:Qtilde main bound}}

{Recall the definitions 
\[
Q_k :=\z_1^T\Aneg{k}\y_1\quad\text{and}\quad\wt Q_\ell:=\z_1^T\Aneg{\ell}\z_1.
\]
Since $\wt Q_\ell$ is a quadratic form, we have $\wt Q_\ell \geq 0$ and so it suffices to upper bound $\wt Q_\ell$.
Because $\z_1$ is independent of $\Aneg{\ell}$ for any $\ell = 1,\ldots,k$, we can directly apply the Hanson-Wright inequality to get 
\begin{align*}
\wt Q_\ell = \z_1^\top \A_{-1:\ell}^{-1} \z_1 \leq \mathsf{tr}(\A_{-1:\ell}^{-1}) + c_1 \|\A_{-1:\ell}^{-1} \|_2 \cdot n^{3/4}
\end{align*}
with probability $1-Ce^{-\sqrt{n}}$.
Similarly, applying the Hanson-Wright inequality to the term $Q_k$ (see Eq.\eqref{eq:qkbounds})
we also have 
\begin{align*}
Q_k \geq c_k \cdot \mathsf{tr}(\A_{-1:k}^{-1})- 2 c_1 \|\A_{-1:k}^{-1}\|_{2} \cdot n^{3/4}.
\end{align*}
with the same probability.
Putting these together, we get
\begin{align*}
\frac{\wt Q_\ell}{Q_k} &\leq \frac{\mathsf{tr}(\A_{1:\ell}^{-1}) + c_1 \|\A_{-1:\ell}^{-1} \|_2 \cdot n^{3/4}}{c_k \cdot  \mathsf{tr}(\A_{-1:k}^{-1}) - 2 c_1 \|\A_{-1:k}^{-1}\|_{2} \cdot n^{3/4}} \\
&\leq \frac{\mathsf{tr}(\A_{-1:k}^{-1})}{\mathsf{tr}(\A_{1:\ell}^{-1})}\, \frac{1 + \frac{c_1}{n^{1/4}}}{c_k - \frac{2c_1}{n^{1/4}}} \\
&\leq \left(1-\frac{C}{n}\right)^{-k}\, \left(\frac{1 + \frac{c_1}{n^{1/4}}}{c_k - \frac{2c_1}{n^{1/4}}}\right) \\
&\leq \frac{2}{c_k},
\end{align*}
where the second inequality follows from Lemma~\ref{lem:trace and operator norm} for $\ell$ and $k$ (for large enough $n$) and the second-to-last inequality uses Lemma~\ref{lem:trace ratio}. The last inequality follows again assuming large enough $n$.
This completes the proof of the lemma.
\qed}
\subsection{Contamination Term}\label{sec:technicallemmas_cn}

{In this section we prove Lemmas~\ref{lem:contaminationformula} and~\ref{lem:su_ck}.}
\subsubsection{Proof of Lemma~\ref{lem:contaminationformula}}
{First, we note that the desired Equation~\eqref{eq:contaminationformula} is a direct consequence of the expression $\CN^2_{1,2} := \sum_{j =1, j \neq 1}^d \lambda_j \hat{\alpha}_j^2$, where
\begin{align*}
\hat{\alpha}_j := \sqrt{\lambda_j} \cdot \z_j^\top \A^{-1} \y_j.
\end{align*}
Therefore, it suffices to show that $\CN^2_{c_1,c_2} = \sum_{j=1,j \neq 1}^d \lambda_j \hat{\alpha}_j^2$.}

{We denote the error vector $\xibold := \Deltahat_{1,2} - \frac{\Deltahat_{1,2}^\top \Sigmab \Deltabold_{1,2}}{\tn{\Sigmab^{1/2}\Deltabold_{1,2}}^2} \Deltabold_{1,2}$ as shorthand, and recall from Lemma~\ref{lem:sucmulticlass} that we have $\CN^2_{1,2} = \xibold^\top \Sigmab \xibold$.
Further, we define $\Etilde := \begin{bmatrix} \etilde_1 \\ \etilde_2 \\ \vdots \\ \etilde_d \end{bmatrix} \in \R^{d \times d}$ as the changed basis in matrix form, and we define $\xitilde := \Etilde \xibold$.
Then, the desired follows from these two statements:
\begin{enumerate}
\item We have $\CN^2_{1,2} = \xitilde^\top \Sigmab \xitilde$.
\item We have $\tilde{\xi}_1 = 0$ and $\tilde{\xi}_j = \hat{\alpha}_j$ for all $j = 2,\ldots,d$.
\end{enumerate}
We complete the proof by proving statements $1$ and $2$ for the specific form of $\Sigmab$ admitted by the bilevel ensemble.}

{\paragraph{Proof of statement 1} We prove this statement for a generic vector $\y \in \R^d$.
Consider the vector $\tilde{\y} := \Etilde \y$. We will show that $\tilde{\y}^\top \Sigmab \tilde{\y} = \y^\top \Sigmab \y$.
Because $\Sigmab$ is a diagonal matrix, we have $\y^\top \Sigmab \y = \sum_{j=1}^d \lambda_j y_j^2$.
Further, it is straightforward to show from the specific form of the changed basis $\Etilde$ that $\tilde{\y}_1 = \frac{y_1 - y_2}{\sqrt{2}}$, $\tilde{\y}_2 = \frac{y_1 + y_2}{\sqrt{2}}$, and $\tilde{\y}_j = y_j$ for $j = 3,\ldots,d$.
Therefore, we have
\begin{align*}
\lambda_1 \tilde{y}_1^2 + \lambda_2 \tilde{y}_2^2 &= \lambda_H (\tilde{y}_1^2 + \tilde{y}_2^2) \\
&= \lambda_H \left(\frac{y_1^2 + y_2^2 - 2y_1y_2 + y_1^2 + y_2^2 +2y_1y_2}{2}\right) \\
&= \lambda_H(y_1^2 + y_2^2) = \lambda_1 y_1^2 + \lambda_2 y_2^2.
\end{align*}
Consequently, we have
\begin{align*}
\tilde{\y}^\top \Sigmab \tilde{\y} &= \sum_{j=1}^d \lambda_j \tilde{y}_j^2 = \lambda_1 \tilde{y}_1^2 + \lambda_2 \tilde{y}_2^2 + \sum_{j=3}^d \lambda_j \tilde{y}_j^2 \\
&= \lambda_1 y_1^2 + \lambda_2 y_2^2 + \sum_{j=3}^d \lambda_j y_j^2 = \sum_{j=1}^d \lambda_j y_j^2 = \y^\top \Sigmab \y.
\end{align*}
Hence, we have proved statement $1$.}

{\paragraph{Proof of statement 2} 
First, note that $\xitilde = \Etilde \Deltahat_{1,2} - \frac{\Deltahat_{1,2}^\top \Sigmab \Deltabold_{1,2}}{\tn{\Sigmab^{1/2}\Deltabold_{1,2}}^2} \cdot \Etilde \Deltabold_{1,2}$.
Recall that $\Deltabold_{1,2} \propto \etilde_1$, and so, $\Etilde \Deltabold_{1,2} \propto \ehat_1$.
Next, simple algebra shows that
\begin{align*}
(\Etilde \Deltahat_{1,2})_j &= \ehat_j^\top \Etilde \Deltahat_{1,2} \\
&= \etilde_j^\top \Deltahat_{1,2} = \etilde_j^\top \X \A^{-1} \y_1 \\
&= \sqrt{\lambda_j} \z_j^\top \A^{-1} \y_1 =: \hat{\alpha}_j.
\end{align*}
where the third equality recalls the definition of $\Deltahat_{1,2}$ from Equation~\eqref{eq:deltahat_def} and the second-to-last equality recalls the definition $\z_j := \frac{1}{\sqrt{\lambda_j}} \X^\top \etilde_j$.
Noting that, by definition, $(\Etilde \Deltabold_{1,2})_j = 0$ for all $j \neq 1$, we have thus shown that $\tilde{\xi}_j = \hat{\alpha}_j$ for all $j = 2,\ldots,d$.
To complete the proof of statement $2$, we need to show that $\tilde{\xi}_1 = 0$.
Denote $\Deltabold_{1,2} = \alpha \etilde_1$ for some $\alpha > 0$ (as a consequence, we also have $\Etilde \Deltabold_{1,2} = \alpha \ehat_1$).
Then, it is equivalent to show that
\begin{align*}
\frac{\Deltahat_{1,2}^\top \Sigmab \Deltabold_{1,2}}{\tn{\Sigmab^{1/2}\Deltabold_{1,2}}^2} \cdot \alpha = \etilde_1^\top \Deltahat_{1,2}.
\end{align*}
(Recall that $\Etilde \Deltabold_{1,2} \propto \ehat_1$, so this equality suffices to show the desired.)
Starting with the LHS of the above, we get
\begin{align*}
\Sigmab \Deltabold_{1,2} &= \lambda_1 \alpha \etilde_1, \text{ and } \\
\Sigmab^{1/2} \Deltabold_{1,2} &= \lambda_1^{1/2} \alpha \etilde_1.
\end{align*}
Therefore, we have
\begin{align*}
\frac{\Sigmab \Deltabold_{1,2}}{\tn{\Sigmab^{1/2} \Deltabold_{1,2}}^2} &= \frac{1}{\alpha} \etilde_1, \text{ and } \\
\frac{\Deltahat_{1,2}^\top \Sigmab \Deltabold_{1,2}}{\tn{\Sigmab^{1/2}\Deltabold_{1,2}}^2} \cdot \alpha &= \Deltahat_{1,2}^\top \etilde_1.
\end{align*}
This completes the proof of statement $2$.}

{With statements $1$ and $2$ proved, the proof of this lemma is complete.
\qed}

\subsubsection{Proof of Lemma~\ref{lem:su_ck}}
{We prove the lemma using induction on $\ell=1,\ldots,k$. For the base case $\ell=1$, we have shown in Lemma~\ref{lem:sumulticlass} that
\begin{align}
\abs{\SU^{(1)}_{1,2}} :=  \frac{\lambda_H \cdot \abs{\z_1^\top \Aneg{1} {\y}_{1}}}{1 + \lambda_H \cdot \z_1^\top \Aneg{1} \z_1} = \abs{\SU_{1,2}} \leq C_k\,. \label{eq:SU ell num}
\end{align}}

{Now, we prove the inductive step. We fix $\ell>1$ and assume, along with the base case (Equation~\eqref{eq:SU ell num}), that the statement is also true for $2,\ldots,\ell-1$, i.e.
\begin{align}\label{eq:hypo SU ell}
\forall j=2,\ldots,\ell-1,\qquad \abs{\SU^{(j)}_{1,2}} &:= \frac{\lambda_H \cdot \abs{\z_j^\top \Aneg{j} \tilde{\y}_{j-1}}}{1 + \lambda_H \cdot \z_j^\top \Aneg{j} \z_j} \leq \frac{C_k}{n^{1/4}} < C_k.
\end{align}
(In fact, as we will see, we will only need to apply the weaker inequality $\abs{\SU^{(j)}_{1,2}} \leq C_k$.)
We use Equation~\eqref{eq:hypo SU ell} to prove the desired statement for $\ell$.
Consider first the numerator in the definition of $\SU^{(\ell)}_{1,2}$, i.e. the term 
\begin{align*}
\z_\ell^\top \Aneg{\ell} \tilde{\y}_{\ell-1} = \z_\ell^\top \Aneg{\ell} \left(\y_1 - \sum_{j=1}^{\ell-1} \SU^{(j)}_{c_1,c_2} \z_{j}\right) =
 \z_\ell^\top \Aneg{\ell} {\y}_{1} - \sum_{j=1}^{\ell-1} \SU^{(j)}_{c_1,c_2}\cdot \z_\ell^\top \Aneg{\ell}\z_{j} \end{align*}}

{Recall that $\z_\ell^\top \Aneg{\ell} {\y}_{1}=R_{\ell,\ell}$.
Note that Lemma~\ref{lem:Rellell bound} shows for $\ell \geq 2$ that
\[
\abs{R_{\ell,\ell}}  \leq \frac{C_k}{n^{1/4}} \z_\ell\Aneg{\ell}\z_\ell\,,
\]
Also recall from Lemma \ref{lem:zz terms} that for all $j < \ell$, we have
\[
\abs{\z_\ell^T\Aneg{\ell}\z_j} \leq Cn^{-1/4} \cdot\z_j^T\Aneg{\ell}\z_j \leq Cn^{-1/4} \cdot\z_\ell^T\Aneg{\ell}\z_\ell \,,
\]
where again, the second inequality uses Lemma \ref{lem:easy}.}

{Putting the above together, applying the triangle inequality and using the induction hypothesis (i.e.~$\abs{\SU^{(j)}_{1,2}} \leq C_k$) we conclude that
\begin{align}
\abs{\z_\ell^\top \Aneg{\ell} \tilde{\y}_{\ell-1}} &\leq  C \cdot \z_\ell\Aneg{\ell}\z_\ell \left(C_k \cdot n^{-1/4} +  \ell \cdot C_k \cdot n^{-1/4} \right) \nn
\\
&\leq  \frac{C_k}{n^{1/4}} \cdot \z_\ell\Aneg{\ell}\z_\ell \,.
\end{align} }

{This gives us
\begin{align}
\abs{\SU^{(\ell)}_{1,2}} &:= \frac{\lambda_H \cdot \abs{\z_\ell^\top \Aneg{\ell} \tilde{\y}_{\ell-1}}}{1 + \lambda_H \cdot \z_\ell^\top \Aneg{\ell} \z_\ell}  
\leq 
\frac{1}{n^{1/4}} \cdot \frac{\lambda_H \cdot  C_k \cdot \z_\ell\Aneg{\ell}\z_\ell  }{1 + \lambda_H \cdot \z_\ell^\top \Aneg{\ell} \z_\ell} \leq \frac{C_k}{n^{1/4}}\,
\end{align}
for all $\ell \geq 2$.
This completes the proof of the lemma.}



\section{Recursive formulas for higher-order quadratic forms}\label{sec:recursive_sm}
We first show how quadratic forms involving the $j$-th order Gram matrix $\bolda_j^{-1}$ can be expressed using quadratic forms involving the $(j-1)$-th order Gram matrix $\bolda_{j-1}^{-1}$. For concreteness, we consider $j=1$; identical expressions hold for any $j > 1$ with the only change being in the superscripts. 
Recall from Section~\ref{sec:pforderoneoutline} that we can write
\begin{align*}
        \bolda_1 & = \bolda_0 + \begin{bmatrix}\tn{\boldmu_1}\vb_1& \Q^T\boldmu_1& \vb_1
    \end{bmatrix}\begin{bmatrix}
    \tn{\boldmu_1} \vb_1^T\\
    \vb_1^T \\
    \boldmu_1^T\Q
    \end{bmatrix} = \Q^T\Q + \begin{bmatrix}\tn{\boldmu_1}\vb_1& \boldd_1& \vb_1
    \end{bmatrix}\begin{bmatrix}
    \tn{\boldmu_1} \vb_1^T\\
    \vb_1^T \\
    \boldd_1^T
    \end{bmatrix}.\notag
\end{align*}
The first step is to derive an expression for $\bolda_1^{-1}$. By the Woodbury identity~\cite{horn2012matrix}, we get
\begin{equation}
\label{eq-pfxinverse02}
   \bolda_1^{-1} = \bolda_0^{-1} - \bolda_0^{-1}\begin{bmatrix}\tn{\boldmu_1}\vb_1& \boldd_1& \vb_1
    \end{bmatrix} \begin{bmatrix} \mathbf{I} + \begin{bmatrix}
    \tn{\boldmu_1} \vb_1^T\\
    \vb_1^T \\
    \boldd_1^T
    \end{bmatrix} \bolda_0^{-1} \begin{bmatrix}\tn{\boldmu_1}\vb_1& \boldd_1& \vb_1
    \end{bmatrix}\end{bmatrix}^{-1}\begin{bmatrix}
    \tn{\boldmu_1} \vb_1^T\\
    \vb_1^T \\
    \boldd_1^T
    \end{bmatrix}\bolda_0^{-1}.
\end{equation}
We first compute the inverse of the $3 \times 3$ matrix $\mathbf{B} := \begin{bmatrix} \mathbf{I} + \begin{bmatrix}
    \tn{\boldmu_1} \vb_1^T\\
    \vb_1^T \\
    \boldd_1^T
    \end{bmatrix} \bolda_0^{-1} \begin{bmatrix}\tn{\boldmu_1}\vb_1& \boldd_1& \vb_1
    \end{bmatrix}\end{bmatrix}$. 
Recalling our definitions of the terms $s_{mj}^{(c)},h_{mj}^{(c)}$ and $t_{mj}^{(c)}$ in \Equation~\eqref{eq:pfquadforc} in Section~\ref{sec:pforderoneoutline}, we have:
\begin{align*}
    \mathbf{B} = \begin{bmatrix}
    1+\tn{\boldmu_1}^2 s_{11}^{(0)} & \tn{\boldmu_1}h_{11}^{(0)} & \tn{\boldmu_1}s_{11}^{(0)} \\
    \tn{\boldmu_1}s_{11}^{(0)} & 1+h_{11}^{(0)} & s_{11}^{(0)} \\
    \tn{\boldmu_1}h_{11}^{(0)} & t_{11}^{(0)} & 1+h_{11}^{(0)}
    \end{bmatrix}.
\end{align*}    
Recalling $\mathbf{B}^{-1} = \frac{1}{\det_0}\text{adj}(\mathbf{B})$, where $\det_0$ is the determinant of $\mathbf{B}$ and $\text{adj}(\mathbf{B})$ is the adjoint of $\mathbf{B}$, simple algebra gives us
\begin{align*}
    \text{det}_0 = s_{11}^{(0)}(\tn{\boldmu_1}^2 - t_{11}^{(0)}) + (h_{11}^{(0)}+1)^2,
\end{align*}
and
\begin{align*}
     \text{adj}(\mathbf{B})=\begin{bmatrix}
    (h_{11}^{(0)}+1)^2-s_{11}^{(0)}t_{11}^{(0)} & \tn{\boldmu_1}(s_{11}^{(0)}t_{11}^{(0)}-h_{11}^{(0)}-{h_{11}^{(0)}}^2) & -\tn{\boldmu_1} s_{11}^{(0)} \\
    -\tn{\boldmu_1} s_{11}^{(0)} & h_{11}^{(0)}+1+\tn{\boldmu_1}^2 s_{11}^{(0)} & -s_{11}^{(0)} \\
    \tn{\boldmu_1}(s_{11}^{(0)}t_{11}^{(0)}-h_{11}^{(0)}-{h_{11}^{(0)}}^2) & \tn{\boldmu_1}^2{h_{11}^{(0)}}^2-t_{11}^{(0)}(1+\tn{\boldmu_1}^2s_{11}^{(0)}) & h_{11}^{(0)}+1+\tn{\boldmu_1}^2 s_{11}^{(0)}
    \end{bmatrix}.
\end{align*}
We will now use these expressions to derive expressions for the $1$-order quadratic forms that are used in Appendix~\ref{sec:pfineqforanoj}.

\subsection{Expressions for 1-st order quadratic forms}
\label{sec:pfinversesecc}
We now show how quadratic forms of order $1$ can be expressed as a function of quadratic forms of order $0$. 
All of the expressions are derived as a consequence of plugging in the expression for $\mathbf{B}^{-1}$ together with elementary matrix algebra.

First, we have
\begin{align}
    s_{mk}^{(1)} = \vb_m^T\bolda_1^{-1}\vb_k & = \vb_m^T\bolda_0^{-1}\vb_k - \begin{bmatrix}\tn{\boldmu_1} s_{m1}^{(0)}& h_{m1}^{(0)}& s_{m1}^{(0)}\end{bmatrix} \frac{\text{adj}(\mathbf{B})}{\det_0}\begin{bmatrix}
    \tn{\boldmu_1}s_{k1}^{(0)}\\ 
    s_{k1}^{(0)}\\
    h_{k1}^{(0)}
    \end{bmatrix} \notag\\
    & = s_{mk}^{(0)} - \frac{1}{\det_0}(\star)_{s}^{(0)}, \label{eq:pfinverseskm}
\end{align}
where we define
\begin{align*}
    (\star)_{s}^{(0)} &:= (\tn{\boldmu_1}^2 - t_{11}^{(0)})s_{1k}^{(0)}s_{1m}^{(0)} + s_{1m}^{(0)}h_{k1}^{(0)}h_{11}^{(0)}+s_{1k}^{(0)}h_{m1}^{(0)}h_{11}^{(0)}-s_{11}^{(0)}h_{k1}^{(0)}h_{m1}^{(0)} + s_{1m}^{(0)}h_{k1}^{(0)}+s_{1k}^{(0)}h_{m1}^{(0)}.
\end{align*}
Thus, for the case $m = k$ we have
\begin{align}
    s_{kk}^{(1)} = \vb_k^T\bolda_1^{-1}\vb_k & = \vb_k^T\bolda_0^{-1}\vb_k - \begin{bmatrix}\tn{\boldmu_1} s_{k1}^{(0)}& h_{k1}^{(0)}& s_{k1}^{(0)}\end{bmatrix} \frac{\text{adj}(\mathbf{B})}{\det_0}\begin{bmatrix}
    \tn{\boldmu_1}s_{k1}^{(0)}\\ 
    s_{k1}^{(0)}\\
    h_{k1}^{(0)}
    \end{bmatrix} \notag\\
    & = s_{kk}^{(0)} - \frac{1}{\det_0}\Big((\tn{\boldmu_1}^2 - t_{11}^{(0)}){s_{1k}^{(0)}}^2 + 2s_{1k}^{(0)}{h_{k1}^{(0)}}h_{11}^{(0)}-s_{11}^{(0)}{h_{k1}^{(0)}}^2 + 2s_{1k}^{(0)}h_{k1}^{(0)}\Big).\label{eq:pfinverseskk}
\end{align}
Next, we have
\begin{align}
    h_{mk}^{(1)} = \vb_m^T\bolda_1^{-1}\boldd_k & = \vb_m^T\bolda_0^{-1}\boldd_k - \begin{bmatrix}\tn{\boldmu_1} s_{m1}^{(0)}& h_{m1}^{(0)}& s_{m1}^{(0)}\end{bmatrix} \frac{\text{adj}(\mathbf{B})}{\det_0}\begin{bmatrix}
    \tn{\boldmu_1}h_{1k}^{(0)}\\ 
    h_{1k}^{(0)}\\
    t_{1k}^{(0)}
    \end{bmatrix} \notag\\
    & =  h_{mk}^{(0)} - \frac{1}{\det_0}(\star)_h^{(0)}, \label{eq:pfinverseh}
\end{align}
where we define $$(\star)_h^{(0)}=(\tn{\boldmu_1}^2 - t_{11}^{(0)})s_{1m}^{(0)}h_{1k}^{(0)} + h_{m1}^{(0)}h_{1k}^{(0)}h_{11}^{(0)}+h_{m1}^{(0)}h_{1k}^{(0)} + s_{1m}^{(0)}t_{k1}^{(0)}+s_{1m}^{(0)}t_{k1}^{(0)}h_{11}^{(0)}-s_{11}^{(0)}t_{k1}^{(0)}h_{m1}^{(0)}.$$ 
Next, we have
\begin{align}
    t_{km}^{(1)} = \boldd_k^T\bolda_1^{-1}\boldd_m & = \boldd_k^T\bolda_0^{-1}\boldd_m - \begin{bmatrix}\tn{\boldmu_1} h_{1k}^{(0)}& t_{1k}^{(0)}& h_{1k}^{(0)}\end{bmatrix} \frac{\text{adj}(\mathbf{B})}{\det_0}\begin{bmatrix}
    \tn{\boldmu_1}h_{1m}^{(0)}\\ 
    h_{1m}^{(0)}\\
    t_{1m}^{(0)}
    \end{bmatrix} \notag\\
    & =  t_{km}^{(0)} - \frac{1}{\det_0}(\star)_t^{(0)},\label{eq:pfinversetkm}
\end{align}
where we define $$(\star)_t^{(0)} = (\tn{\boldmu_1}^2 - t_{11}^{(0)})h_{1m}^{(0)}h_{1k}^{(0)} + t_{m1}^{(0)}h_{1k}^{(0)}h_{11}^{(0)}+t_{k1}^{(0)}h_{1m}^{(0)}h_{11}^{(0)} + t_{1m}^{(0)}h_{1k}^{(0)}+t_{1k}^{(0)}h_{1m}^{(0)}-s_{11}^{(0)}t_{1m}^{(0)}t_{1k}^{(0)}.$$
Thus, for the case $m=k$ we have
\begin{align}
    t_{kk}^{(1)} = \boldd_k^T\bolda_1^{-1}\boldd_k & = \boldd_k^T\bolda_0^{-1}\boldd_k - \begin{bmatrix}\tn{\boldmu_1} h_{1k}^{(0)}& t_{1k}^{(0)}& h_{1k}^{(0)}\end{bmatrix} \frac{\text{adj}(\mathbf{B})}{\det_0}\begin{bmatrix}
    \tn{\boldmu_1}h_{1k}^{(0)}\\ 
    h_{1k}^{(0)}\\
    t_{1k}^{(0)}
    \end{bmatrix} \notag\\
    & =  t_{kk}^{(0)} - \frac{1}{\det_0}\Big((\tn{\boldmu_1}^2 - t_{11}^{(0)}){h_{1k}^{(0)}}^2 + 2t_{1k}^{(0)}{h_{1k}^{(0)}}h_{11}^{(0)}-s_{11}^{(0)}{t_{1k}^{(0)}}^2 + 2t_{1k}^{(0)}h_{1k}^{(0)}\Big).\label{eq:pfinversetkk}
\end{align}
Next, we have
\begin{align}
    f_{ki}^{(1)} = \boldd_k^T\bolda_1^{-1}\mathbf{e}_i & = \boldd_k^T\bolda_0^{-1}\mathbf{e}_i - \begin{bmatrix}\tn{\boldmu_1} h_{1k}^{(0)}& t_{1k}^{(0)}& h_{1k}^{(0)}\end{bmatrix} \frac{\text{adj}(\mathbf{B})}{\det_0}\begin{bmatrix}
    \tn{\boldmu_1}g_{1i}^{(0)}\\ 
    g_{1i}^{(0)}\\
    f_{1i}^{(0)}
    \end{bmatrix} \notag\\
    & =  f_{ki}^{(0)} - \frac{1}{\det_0}(\star)_f^{(0)},\label{eq:pfinversef}
\end{align}
where we define
$$(\star)_f^{(0)} = (\tn{\boldmu_1}^2 - t_{11}^{(0)})h_{1k}^{(0)}g_{1i}^{(0)} + t_{1k}^{(0)}g_{1i}^{(0)}+t_{1k}^{(0)}h_{11}^{(0)}g_{1i}^{(0)} + h_{1k}^{(0)}f_{1i}^{(0)}+h_{1k}^{(0)}h_{11}^{(0)}f_{1i}^{(0)}-s_{11}^{(0)}t_{1k}^{(0)}f_{1i}^{(0)}.$$
Finally, we have
\begin{align}
   g_{ji}^{(1)} = \vb_j^T\bolda_1^{-1}\mathbf{e}_i & = \vb_j^T\bolda_0^{-1}\mathbf{e}_i - \begin{bmatrix}\tn{\boldmu_1} s_{j1}^{(0)}& h_{j1}^{(0)}& s_{j1}^{(0)}\end{bmatrix} \frac{\text{adj}(\mathbf{B})}{\det_0}\begin{bmatrix}
    \tn{\boldmu_1}g_{1i}^{(0)}\\ 
    g_{1i}^{(0)}\\
    f_{1i}^{(0)}
    \end{bmatrix} \notag\\
    & = g_{ji}^{(0)} - \frac{1}{\det_0}(\star)_{gj}^{(0)},\label{eq:pfinverseg}
\end{align}
where we define
$$(\star)_{gj}^{(0)} = (\Vert \boldmu_1 \Vert_2^2 - t_{11}^{(0)})s_{1j}^{(0)}g_{1i}^{(0)} + g_{1i}^{(0)}h_{11}^{(0)}h_{j1}^{(0)}+g_{1i}^{(0)}h_{j1}^{(0)}+s_{1j}^{(0)}f_{1i}^{(0)}+s_{1j}^{(0)}h_{11}^{(0)}f_{1i}^{(0)} - s_{11}^{(0)}h_{j1}^{(0)}f_{1i}^{(0)}.$$ 
\qed

\section{One-vs-all SVM}\label{sec:OvA_sm}
In this section, we derive conditions under which the OvA solutions $\w_{{\rm OvA},c}$  interpolate, i.e, all data points are support vectors in \Equation~\eqref{eq:ova-svm}.

\subsection{Gaussian mixture model}
As in the case of the multiclass SVM, we assume nearly equal priors and nearly equal energy on the class means (Assumption~\ref{ass:equalmuprob}). 
\begin{theorem}
\label{thm:svmgmmova}
Assume that the training set follows a multiclass GMM with noise covariance $\Sigmab =\mathbf{I}_p$ and Assumption~\ref{ass:equalmuprob} holds. 
Then, there exist constants $c_1,c_2,c_3>1$ and $C_1,C_2>1$ such that the solutions of the OvA-SVM and MNI are identical with probability at least $1-\frac{c_1}{n}-c_2ke^{-\frac{n}{c_3k^2}}$
{provided that
\begin{align}
\label{eq:thmsvmgmmova}
    p > C_1kn\log(kn)+n-1\quad\text{ and }\quad p>C_2n^{1.5}\tn{\boldmu}.
\end{align}
}
\end{theorem}
We can compare \Equation~\eqref{eq:thmsvmgmmova} with the corresponding condition for multiclass SVM in Theorem~\ref{thm:svmgmm} (\Equation~\eqref{eq:thmsvmgmm}).
Observe that the right-hand-side of \Equation~\eqref{eq:thmsvmgmmova} above does not scale with $k$, while the right-hand-side of \Equation~\eqref{eq:thmsvmgmm} scales with $k$ as $k^3$. Otherwise, the scalings with $n$ and energy of class means $\|\mub\|_2$ are identical. 
This discrepancy with respect to $k$-dependence arises because the multiclass SVM is equivalent to the OvA-SVM in \Equation~\eqref{eq:sym-cs-svm} with unequal margins $1/k$ and $(k-1)/k$ (as we showed in \Theorem~\ref{lem:key}).

\begin{proof}[Proof sketch.]
Recall from Section~\ref{sec:pforderoneoutline} that we derived conditions under which the multiclass SVM interpolates the training data by studying the related symmetric OvA-type classifier defined in \Equation~\eqref{eq:sym-cs-svm-sketch}. Thus, this proof is similar to the proof of Theorem~\ref{thm:svmgmm} provided in Section~\ref{sec:pforderoneoutline}.
The only difference is that the margins for the OvA-SVM are not $1/k$ and $(k-1)/k$, but $1$ for all classes.
Owing to the similarity between the arguments, we restrict ourselves to a proof sketch here.

Following Section~\ref{sec:pforderoneoutline} and \Equation~\eqref{eq:numeratorsum}, we consider $y_i = k$.
We will derive conditions under which the condition
\begin{align}
\label{eq:numeratorsumova}
    \big((1+h_{kk}^{(-k)})g_{ki}^{(-k) }-s_{kk}^{(-k)}f_{ki}^{(-k)}\big) + C\sum_{j \ne k}\big((1+h_{jj}^{(-j)})g_{ji}^{(-j) }-s_{jj}^{(-j)}f_{ji}^{(-j)}\big) > 0,
\end{align}
holds with high probability for some $C > 1$. 
We define
\begin{align*}
    \epsilon := \frac{n^{1.5}\tn{\boldmu}}{p} \le \tau,
\end{align*}
where $\tau$ is chosen to be a sufficiently small constant. Applying the same trick as in Lemma~\ref{lem:ineqforanoj} (with the newly defined parameters $\epsilon$ and $\tau$) gives us with probability at least $1-\frac{c_1}{n}-c_2ke^{-\frac{n}{c_3 k^2}}$:
\begin{align}
   ~\eqref{eq:numeratorsumova} &\ge \lp(\lp(1-\frac{C_1\epsilon}{\sqrt{k}\sqrt{n}}\rp)\lp(1-\frac{1}{C_2}\rp)\frac{1}{p} - \frac{C_3\epsilon}{n}\cdot\frac{n}{kp}\rp) - \frac{k}{C_4}\lp(\lp(1+\frac{C_5\epsilon}{\sqrt{k}\sqrt{n}}\rp)\frac{1}{kp}-\frac{C_6\epsilon}{n}\cdot\frac{n}{kp}\rp) \notag\\
    &\ge \lp(1-{\frac{1}{C_9}}-\frac{C_{10}\epsilon}{\sqrt{k}\sqrt{n}} - \frac{C_{11}\epsilon}{{k}} - C_{12}\epsilon\rp)\frac{1}{p} \nn\\
    &\geq \frac{1}{p}\lp(1-\frac{1}{C_9}-C_{0}\tau\rp), \label{eq:gmmsumlowerova}
\end{align}
for some constants $C_i$'s $>1$.
We used the fact that $|g_{ji}^{(0)}| \le (1/C)(1/(kp))$ for $j \ne y_i$ with probability at least $1-\frac{c_1}{n}-c_2ke^{-\frac{n}{c_3 k^2}}$ provided that $p > C_1kn\log(kn)+n-1$, which is the first sufficient condition in the theorem statement.
\end{proof}

\subsection{Multinomial logistic model}
Recall that we defined the data covariance matrix $\Sigmab = \sum_{i=1}^p \lambda_i\vb_i\vb_i^T =\boldsymbol{V} \boldsymbol{\Lambda} \boldsymbol{V}^T$ and its spectrum $\lbdb = \begin{bmatrix}\lambda_1 & \cdots & \lambda_p\end{bmatrix}$.
We also defined the effective dimensions $d_2 :=\frac{\Vert \lbdb \Vert_1^2}{\Vert \lbdb \Vert_2^2}$ and $d_\infty:=\frac{\Vert \lbdb \Vert_1}{\Vert \lbdb \Vert_\infty}$. 

The following result provides sufficient conditions under which the OvA SVM and MNI classifier have the same solution with high probability under the MLM.
\begin{theorem}
\label{thm:svmmlmanisoova}
Assume that the training set follows a multiclass MLM. There exist constants $c$ and $C_1, C_2 >1$ such that, if the following conditions hold:
\begin{align}
d_{\infty} > C_1n\log(kn)~~ \text{and} ~~ d_2 > C_2(\log(kn) + n), \label{eq:thmanisolinkova}
\end{align}
the solutions of the OvA-SVM and MNI are identical with probability at least $(1 - \frac{c}{n})$.
In the special case of isotropic covariance, the same result holds provided that 
\begin{align}
p > 10n\log (\sqrt{k}n) + n - 1, \label{eq:thmisolinkova}
\end{align}
\end{theorem}
Comparing this result to the corresponding results in Theorems~\ref{thm:svmmlmaniso}, we observe that $k$ now only appears in the $\log$ function (as a result of $k$ union bounds). 
Thus, the unequal $1/k$ and $(k-1)/k$ margins that appear in the multiclass-SVM make interpolation harder than with the OvA-SVM, just as in the GMM case.

\begin{proof}[Proof sketch.]
For the OvA SVM classifier, we need to solve $k$ binary max-margin classification problems, hence the proof follows directly from \cite[Theorem 1]{muthukumar2021classification} and \cite[Theorem 1]{hsu2020proliferation} by applying $k$ union bounds. We omit the details for brevity.
\end{proof}

\section*{One-vs-one SVM}\label{sec:ovo_sm}
In this section, we first derive conditions under which the OvO solutions interpolate, i.e, all data points are support vectors. We then provide an upper bound on the classification error of the OvO solution.

In OvO classification, we solve $k(k-1)/2$ binary classification problems, e.g. for classes pair $(c, j)$, we solve
\begin{align}
\label{eq:ovo-svm}
    \w_{{\rm OvO},(c,j)}:=\arg\min_{\w} \tn{\w}\quad\text{sub. to}~~ \w^T\x_i \ge 1,~\text{if}~\y_i =c;~~\w^T\x_i \le -1~\text{if}~\y_i = j, ~\forall i\in[n].
\end{align}
Then we apply these $k(k-1)/2$ classifiers to a fresh sample and the class that got the highest $+1$ voting gets predicted.

We now present conditions under which every data point becomes a support vector over these $k(k-1)/2$ problems. We again assume nearly equal priors and nearly equal energy on the class means (Assumption~\ref{ass:equalmuprob}). 

\begin{theorem}
\label{thm:svmgmmovo}
Assume that the training set follows a multiclass GMM with noise covariance $\Sigmab =\mathbf{I}_p$ and Assumption~\ref{ass:equalmuprob} holds. 
Then, there exist constants $c_1,c_2,c_3>1$ and $C_1,C_2>1$ such that the solutions of the OvA-SVM and MNI are identical with probability at least $1-\frac{c_1}{n}-c_2ke^{-\frac{n}{c_3k^2}}$
{provided that
\begin{align}
\label{eq:thmsvmgmmovo}
    p > C_1n\log(kn)+(2n/k)-1\quad\text{ and }\quad p>C_2n^{1.5}\tn{\boldmu}.
\end{align}
}
\end{theorem}
\begin{proof}[Proof sketch]
Note that the margins of OvO SVM are $1$ and $-1$, hence the proof is similar to the proof of Theorem \ref{thm:svmgmmova}. Recall that in OvO SVM, we solve $k(k-1)/2$ binary problems and each problems has sample size $2n/k$ with high probability. Therefore, compared to OvA SVM which solves $k$ problems each with sample size $n$, OvO SVM needs less overparameterization to achieve interpolation. Thus the first condition in \Equation~\eqref{eq:thmsvmgmmova} reduces to $p > C_1n\log(kn)+(2n/k)-1$.
\end{proof}
We now derive the classification risk for OvO SVM classifiers. Recall that OvO classification solves $k(k-1)/2$ binary subproblems. Specifically, for each pair of classes, say $(i,j)\in [k]\times [k]$, we train a classifier $\w_{ij}\in\mathbb{R}^p$ and the corresponding decision rule for a fresh sample $\x\in\mathbb{R}^p$ is $\hat{y}_{ij} = \text{sign}(\x^T\hat{\w}_{ij})$. Overall, each class $i\in[k]$ gets a voting score $s_i = \sum_{j \ne i}\mathbf{1}_{\hat{y}_{ij} = +1}$. Thus, the final decision is given by majority rule that \emph{decides the class with the highest score}, i.e. $\arg\max_{i \in [k]}s_i$. Having described the classification process, the total classification error $\Pro_{e}$ for balanced classes is given by the conditional error $\Pro_{e|c}$ given the fresh sample belongs to class $c$. Without loss of generality, we assume $c = 1$. Formally, $\Pro_{e} = \Pro_{e|1} = \Pro_{e|1}(s_1 < s_2 ~\text{or} ~s_1 < s_3~\text{or} \cdots \text{or}~s_1 < s_k)$. Under the nearly equal prior and energy assumption, by symmetry and union bound, the conditional classification risk given that true class is $1$ can be upper bounded as:
\begin{align*}
    \Pro_{e|1}(s_1 < s_2 ~\text{or} ~s_1 < s_3~\text{or} \cdots \text{or}~s_1 < s_k) \le \Pro_{e|1}(s_1 < k-1) = \Pro_{e|1}(\exists j~s.t.~ \hat{y}_{1j} \ne 1) \le (k-1)\Pro_{e|1}(\hat{y}_{12} \ne 1).
\end{align*}
Therefore, it suffices to bound $\Pro_{e|1}(y_{12} \ne 1)$. We can directly apply Theorem \ref{thm:classerrorgmm} with changing $k$ to $2$ and $n$ to $2n/k$.
\begin{theorem}
\label{thm:ovoclasserror}
Let Assumptions~\ref{ass:equalmuprob} and \ref{ass:orthomean}, and the condition in \Equation~\eqref{eq:thmsvmgmmovo} hold. Further assume constants $C_1,C_2
,C_3>1$ such that $\big(1-C_1\sqrt{\frac{k}{n}}-\frac{C_2n}{kp}\big)\tn{\boldmu} > C_3.$ Then, there exist additional constants $c_1,c_2,c_3$ and $C_4 > 1$ such that the OvO SVM solutions satisfies:
\begin{align}
\Pro_{e|c} \leq  (k-1)\exp{\left(-\tn{\boldmu}^2\frac{\left(\left(1-C_1\sqrt{\frac{k}{n}}-\frac{C_2n}{kp}\right)\tn{\boldmu}-C_3\right)^2}{C_4\left(\tn{\boldmu}^2+\frac{kp}{n}\right)}\right)}
\end{align}
with probability at least $1-\frac{c_1}{n} - c_2ke^{-\frac{n}{c_3k^2}}$, for every $c \in [k]$.
Moreover, the same bound holds for the total classification error $\Pro_{e}$.
\end{theorem}

\end{document}